\documentclass{article}

\usepackage[preprint]{neurips_2026}

\usepackage{titletoc}
\usepackage[utf8]{inputenc} %
\usepackage[T1]{fontenc}    %
\usepackage{url}            %
\usepackage{booktabs}       %
\usepackage{amsfonts}       %
\usepackage{nicefrac}       %
\usepackage{microtype}      %
\usepackage{xcolor}         %
\usepackage{graphicx}
\usepackage{subcaption}
\usepackage{booktabs} %
\usepackage{enumitem}
\usepackage{amsmath}
\usepackage{amssymb}
\usepackage{mathtools}
\usepackage{amsthm}
\usepackage{style}
\usepackage{natbib}
\usepackage{wrapfig} 
\usepackage{algorithm}
\usepackage{algorithmic}
\usepackage{tikz} 
\usetikzlibrary{positioning, shadows}
\usepackage[colorlinks=true,
    linktoc=all]{hyperref}       %
\hypersetup{
    urlcolor=[rgb]{0.75, 0.5, 0.51
},
    citecolor=[rgb]{0.75, 0.5, 0.51
},
    linkcolor=[rgb]{0.75, 0.5, 0.51
}
}

\title{Collaborative Yet Personalized Policy Training: Single-Timescale Federated Actor-Critic}

\author{%
  Leo Muxing Wang
  \\
  Northeastern University\\
  \texttt{wang.muxin@northeastern.edu} \\
  \And
  Pengkun Yang \\
  Tsinghua University \\
  \texttt{yangpengkun@tsinghua.edu.cn} \\
  \AND
  Lili Su \\
  Northeastern University \\
  \texttt{l.su@northeastern.edu} \\
}

\begin{document}

\maketitle
\begin{abstract}
Despite the popularity of the actor-critic method and the practical needs of collaborative policy training, existing works typically either overlook environmental heterogeneity or give up personalization altogether by training a single shared policy across all agents. We consider a federated actor-critic framework in which agents share a common linear subspace representation while maintaining personalized local policy components, and agents iteratively estimate the common subspace, local critic heads, and local policies (i.e., actors). Under canonical single-timescale updates with Markovian sampling, we establish finite-time convergence via a novel joint linear approximation framework. Specifically, we show that the critic error converges to zero at the rate of $\tilde{\calO}(1/((1-\gamma)^4\sqrt{TK}))$, and the policy gradient norm converges to zero at the rate of $\tilde{\calO}(1/((1-\gamma)^6\sqrt{TK}))$, where $T$ is the number of rounds, $K$ is the number of agents, and $\gamma\in (0,1)$ is the discount factor. These results demonstrate linear speedup with respect to the number of agents $K$, despite heterogeneous Markovian trajectories under distinct transition kernels and coupled learning dynamics. To address these challenges, we develop a new perturbation analysis for the projected subspace updates and QR decomposition steps, together with conditional mixing arguments for heterogeneous Markovian noise. Furthermore, to handle the additional complications induced by policy updates and temporal dependence, we establish fine-grained characterizations of the discrepancies between function evaluations under Markovian sampling and under temporally frozen policies. Experiments instantiate the framework within PPO on federated \texttt{Hopper-v5} action-map heterogeneity, showing gains over Single PPO and FedAvg PPO and downstream transfer from the learned shared trunk.
\end{abstract}

\section{Introduction}
\label{sec: intro}
Collaborative yet personalized policy training is becoming increasingly important in modern reinforcement learning systems, particularly in emerging Embodied AI applications, %
where the state and action spaces are highly complex and large-scale.
Enabling agents to collaborate can substantially improve sample efficiency and reduce training costs, while personalization remains essential for adapting to environment-specific characteristics and avoiding performance degradation caused by misaligned learning signals.
The growing demand for simultaneously achieving collaboration and personalization is also reflected in supervised learning, where personalized federated learning (PFL) has achieved remarkable empirical and theoretical success by effectively exploiting commonality while adapting to heterogeneity in local data distributions \cite{collins2021exploiting, xu2023personalized,mclaughlin2024personalized,t2020personalized,li2021ditto,deng2020adaptive,fallah2020personalized,jiang2019improving}. 

However, training personalized policies in federated reinforcement learning remains largely underexplored.   
Existing works typically either overlook environmental heterogeneity or give up personalization altogether by training a single shared policy across all agents. 
Yang et al.\,\cite{yang2023federated} studied the training of personalized policies in the context where each agent has a private reward function but shares a common transition kernel for the environment. The analysis focused solely on the convergence of the global policy, defined as the average of the local policies.  
Jin et al.~\cite{jin_federated_2022} proposed a heuristic deep FedRL method, where the policy network can be decomposed into a shared subnetwork and environment embedding layer. 
Xiong et al.~\cite{xionglinear} studied personalized TD learning with two timescales and established a linear convergence speedup. 
Environmental heterogeneity was also considered in~\citep{jin_federated_2022,wang2023federated,xie_fedkl_2023,zhang2024finite}, which mostly focused on training a common policy or value-/Q-function. Moreover, their notions of optimality are often characterized with respect to an imaginary averaged environment constructed from weighted combinations of the agents' transition kernels~\cite{jin_federated_2022}. 

Actor–critic methods have emerged as a widely adopted reinforcement learning framework and have been successfully applied across a range of practical domains. By jointly performing gradient-based policy updates to maximize cumulative rewards while simultaneously learning value function estimates under the current policy, actor–critic methods effectively balance exploration and evaluation. This coupled structure enables the critic to provide low-variance feedback to the actor, leading to more stable and sample-efficient learning.   
Despite this empirical success, finite-time analysis of actor–critic algorithms has only recently begun to emerge~\cite{olshevsky2023small,chen2024finite,chen2021closing}, and remains significantly less understood in multi-agent and heterogeneous environments. 

{\bf Contributions.}
In this paper, 
we consider a federated actor-critic framework in which agents share a common linear subspace representation while maintaining personalized local policy components. In this setting, agents collaboratively estimate the shared subspace while simultaneously updating their local critic heads and personalized policies (i.e., actors). In particular, we focus on the canonical single-timescale update scheme under Markovian sampling, which is both practically appealing and analytically more challenging.  
Our contributions can be summarized as follows. 
\begin{itemize}[leftmargin=*]
\item We propose a single-timescale personalized federated AC algorithm ({\bf pFedAC}) that enables the $K$ federated agents to collaboratively train policies for maximizing their own cumulative rewards. 
To the best of our knowledge, we are the first to enable personalized AC in federated reinforcement learning. 
\item We establish finite-time convergence via a novel joint linear approximation framework. 
We show that the critic error converges to zero at the rate of $\tilde{\calO}(1/((1-\gamma)^4\sqrt{TK}))$, and the policy gradient norm converges to zero at the rate of $\tilde{\calO}(1/((1-\gamma)^6\sqrt{TK}))$, where $T$ is the number of rounds, $K$ is the number of agents, and $\gamma\in (0,1)$ is the discount factor. Here $\tilde{\calO}$ hides polylogarithmic factors, absolute constants, and fixed problem-dependent Lipschitz constants.
These results demonstrate linear speedup with respect to the number of agents $K$. 
The dependence on $(1-\gamma)$ may not be optimal, and we would like to leave this as a future direction.  
\item Different from existing finite-time analyses of actor--critic methods, our framework requires controlling three tightly coupled error dynamics. In particular, the shared subspace updates and heterogeneous local transition kernels necessitate a careful exploitation of the underlying geometric structures to establish convergence. 
Moreover, departing from traditional PFL analyses, our setting observes heterogeneous Markovian trajectories under distinct transition kernels. To address these challenges, we develop a new perturbation analysis for the projected subspace updates and QR decomposition steps, together with conditional mixing arguments for heterogeneous Markovian noise. 
To handle the additional complications introduced by policy updates and Markovian dependence, we further develop fine-grained characterizations of the discrepancies between functions evaluated under Markovian sampling and under temporally frozen policies. %
\item We instantiate the proposed shared-representation and personalized-head structure within PPO and evaluate it on a federated \texttt{Hopper-v5} benchmark with action-map heterogeneity. The results show that personalization outperforms Single PPO and FedAvg PPO, and that the learned shared trunk can accelerate downstream adaptation with only newly initialized heads trained.

\end{itemize}

Due to space limitations, a formal related work section is deferred to Appendix \ref{app: additional related work}.

\section{Preliminaries}
\label{sec: setup}
\noindent{\bf Markov decision process.} 
Before introducing the federated setting, we first provide preliminaries on the single-agent policy gradient and actor--critic (AC) algorithms. %
An agent is modelled as a Markov Decision Process (MDP) $\calM = < P, \calS, \calA, R, \gamma> $, where $\calS$ is the (possibly infinite) state space, 
$\calA$ is a finite action space, $\gamma\in (0,1)$ is the discount factor, 
$P$ is the transition kernel that governs the state transition distribution given $(s,a)$, 
and $R$ is the reward function: $\calS\times\calA\rightarrow [-U_r,U_r]$.

Let the policy $\pi$ be parameterized by $\theta$.  
The value function is defined as 
\begin{align}
\label{eq: value}
V^{\pi_{\theta}}(s) = \bbE_{\pi_{\theta}, P} \bigg[\sum_{t=0}^\infty \gamma^t R(s_t, a_t) \mid s_0 = s\bigg], %
\end{align}
and the action-value function is 
\begin{equation}
\label{eqn: Q function def}
\begin{aligned}
 Q^{\pi_{\theta}}(s,a) %
& = R(s,a)+\gamma \bbE_{s^{\prime}\sim P(\cdot \mid (s,a))} [V^{\pi_{\theta}}(s^{\prime})],        
\end{aligned}   
\end{equation}
where $(s_t, a_t)$ is the state-action pair at iteration $t$ for $t=0, 1,\cdots$.  
The expectation in Eq.\,(\ref{eq: value}) and Eq.\eqref{eqn: Q function def} is taken over the state-action trajectories generated under the Markov chain that is jointly determined by the policy $\pi_{\theta}$ and the transition kernels $P$.    

The reinforcement learning tasks typically aim to find a policy $\pi_{\theta}$ that maximizes the cumulative discounted rewards averaged over the given initial state distribution $\eta(s)$:   
\begin{equation}
\label{eqn: J theta def} 
J(\theta) = \bbE_{s \sim \eta}\qth{V^{\pi_{\theta}}(s)},     
\end{equation}

For ease of exposition, we write $V^{\pi_{\theta}}(s)$ and $Q^{\pi_{\theta}}(s,a) $ as $V_\theta(s)$ and $Q_\theta(s,a)$, respectively.  

We assume that the limiting state distribution induced by the transition kernel $P$ under policy $\pi_\theta$ exists. Define the {\em discounted state visitation distribution} under policy $\pi_{\theta}$ as 
\[
\nu_{\theta}(s) = (1-\gamma) \bbE_{s_0\sim \eta}\bigg[\sum_{t=0}^\infty \gamma^t p_{\theta,t}(s|s_0)\bigg],  
\]
where $s_0$ denotes initial state, and $p_{\theta,t}(s|s_0)$ is the probability that $s_t=s$ given the initial state is $s_0$ under policy $\pi_{\theta}$. The cumulative discounted reward $J(\theta)$ can be rewritten as \cite{sutton1999policy}: 
\begin{align*}
J(\theta) = \frac{1}{1-\gamma} \bbE_{s\sim \nu_{\theta}, a\sim \pi_{\theta}}[R(s, a)].    
\end{align*}
For ease of exposition, we write $\bbE_{s\sim \nu_{\theta}, a\sim \pi_{\theta}}[R(s, a)]$ as $\bbE_{\nu_{\theta}\times \pi_{\theta}}[R(s, a)]$ for short.

\noindent {\bf Policy gradient theorem.}  
For parameterized $\pi_{\theta}$, one natural idea is to mimic the traditional optimization framework and run gradient ascent to maximize the cumulative discounted reward as defined in Eq.\eqref{eqn: J theta def}, assuming it is differentiable.
The policy gradient theorem \cite{sutton2018reinforcement} provides an analytic expression for the gradient:   
\begin{align} 
\label{eq: policy gradient thm}
\nabla_{\theta}J(\theta) 
= \frac{1}{1-\gamma}\bbE_{\nu_{\theta}\otimes \pi_{\theta}}[Q_{\theta}(s,a)\nabla_{\theta} \log \pi_{\theta}(a|s)].  
\end{align}

Note that for any function $h: \calS \to \reals$ that is independent of the action, we have 
\begin{align*}
\sum_{a\in \calA} h(s) \nabla_{\theta} \pi_{\theta}(a\mid s) = h(s)   \nabla_{\theta} \sum_{a\in \calA}\pi_{\theta}(a\mid s) = h(s)\nabla_{\theta} 1 = 0, ~~~ \text{for any $s\in \calS$}.     
\end{align*}
Thus, choosing $h(\cdot) = V_{\theta}(\cdot)$, $\nabla_{\theta}J(\theta)$ can be rewritten as 
\begin{align*} 
\nabla_{\theta}J(\theta) 
&=\frac{1}{1-\gamma}\bbE_{\nu_{\theta}\otimes\pi_{\theta}}[(Q_{\theta}(s,a)- V_{\theta}(s))\nabla_{\theta} \log \pi_{\theta}(a|s) ], \end{align*}
where $(Q_{\theta}(s,a)- V_{\theta}(s))$ is often called the advantage function. 
In practice, the $Q_{\theta}$ and $V_{\theta}$ are often unknown and must be approximated in an online manner. Traditional approaches, such as Monte Carlo–based methods, suffer from high variance and thus learn slowly in practice \cite{chen2024finite,olshevsky2023small}.

\noindent {\bf Actor-critic algorithm.}  
By jointly performing gradient-based policy updates while simultaneously learning $Q_{\theta}$, the AC methods effectively balance exploration and evaluation, providing lower-variance feedback to the actor and thereby promoting faster convergence \cite{konda1999actor}. 

Let $\calV$ be a function space that $V_\theta$ belongs to. 
Let $\phi: \calS \to L^2(\naturals)$ be a feature map of $\calV$.  
Denote the coefficients of the value $V_\theta$ with respect to $\phi$ as $z_\theta^{*}$, where $z_\theta^{*} \in L^2(\naturals)$. For simplicity, we assume $\phi$ is known and has finite dimension $d$, i.e., $\phi: \calS\to \reals^d$. Consequently, 
\begin{align}
\label{eq: linear approximation}
V_\theta (s) = \iprod{\phi(s)}{z_\theta^{*}}, ~~~~ \forall ~ s\in \calS.  
\end{align} 
This corresponds to the standard linear function approximation setting commonly adopted in the reinforcement learning literature \cite{jin2020provably,srikant2019finite,qiu2021finite}. More generally, one may allow for approximation error by writing
\[
V_\theta (s) = \iprod{\phi(s)}{z_\theta^{*}} +\epsilon , ~~~~ \forall ~ s\in \calS, 
\]
where $\epsilon$ captures the model misspecification. In this paper, we assume Eq.\eqref{eq: linear approximation} for ease of exposition. Our analysis readily extends to the more general case. 

In practice, $V_{\theta}$ is unknown. 
We consider the linear function approximation $\hat{V} (s; z) = \phi(s)^{\top} z$ to $V_\theta$. 
For any given $(s, a, s^{\prime})$, where $s^{\prime}\sim P(\cdot\mid s, a)$, we use the one-step temporal difference error (i.e., TD(0) error) to estimate the advantage function, i.e,  
$
\delta:= R(s,a) + \pth{\gamma \phi(s^{\prime}) - \phi(s)}^{\top}z,
$
which in turn yields a stochastic gradient–type value function and policy updates:   
\begin{align*} 
\theta^{\prime} = \theta +\alpha  \delta \nabla_{\theta} \log \pi_{\theta}(a|s), 
\qquad \text{and} \qquad 
z^{\prime} = z + \beta \delta \phi(s), 
\end{align*} 
where $\alpha, \beta >0$ are the stepsizes.  
It is well-known that under some standard technical assumptions $\lim_{t\diverge}z_{t} \to z_{\theta}^{*}$ and $\lim_{t\diverge}\|\nabla J(\theta_t)\| =0$ \citep{tsitsiklis1999average,sutton2018reinforcement,konda1999actor}.

\section{Problem Formulation}
\label{sec: problem formulation}

We consider a federated system that consists of one parameter server and $K$ agents. 
The agents are modelled as Markov Decision Processes (MDP) $\{\calM^k \mid \calM^k = <\calS, \calA, \gamma, P^k, R^k > \text{ for } k=1,2,...,K \}$. 
The environments $P^k$ with which the agents interact are heterogeneous across agents.    
Let $J^k(\theta)$, $V_{\theta}^k(s)$, and $Q_{\theta}^k(s,a)$ denote, respectively, the cumulative reward, the value function, and the action-value ($Q$) function under policy $\pi_\theta$, where the expectations are taken with respect to $\pi_\theta$ and the local transition kernel $P^k$. %

The parameter server can orchestrate the learning at the agents through iteratively aggregating the updates at the agents. The $K$ agents aim to collaboratively train {\bf personalized} policies $\pi_{\theta^*}^{k}$'s that maximize their cumulative rewards $J^k(\theta)$, respectively.  
Learning with shared representation is widely recognized as an effective way to separate commonalities from heterogeneity 
\begin{wrapfigure}[28]{r}{0.51\columnwidth}
\begin{minipage}{\linewidth}
\small

\vspace{-0.6em}
\hrule
\hrule 
\vspace{0.3em}
\captionof{algorithm}{pFedAC}
\label{alg: fedac-per}
\vspace{-0.6em}
\hrule

\begin{algorithmic}[1]
\STATE \textbf{Input:} Initial heads $\{\omega_0^k\}_{k=1}^K$;  initial subspace $\bfB_0\in \reals^{d\times r}$, which is orthonormal; stepsizes: $\beta$ for local heads, $\zeta$ for subspace, $\gamma$ for reward estimator, and $\alpha $ for actor; total iteration $T$; projection upper bound $U_{\omega}$ such that $\|\omega^{k,*}\| \le U_{\omega}$ for $k\in [K]$.
\FOR{$k= 1, 2, \dots, K$} 
\STATE Draw $s_{0,0}^k\sim \eta$,  and $\hat{s}_{0,0}^k\sim \eta$.  
\ENDFOR 
\FOR{$t = 0, 1, 2, \dots, T-1$}
    \FOR{$k = 1, 2, \dots, K$}
        \STATE Sample $\{s_{t,0}^k, a_{t,0}^k, \dots, s_{t,L}^k\} \sim \pi_{\theta_t^k}\otimes P^k$. %
        \STATE $\delta_{t, L}^k = \sum_{\ell=0}^{L-1} \gamma^\ell(r_{t,\ell}^k + (\gamma \phi(s_{t,\ell+1}^k) - \phi(s_{t,\ell}^k))^\top \mathbf{B}_t \omega_t^k)$.
        \STATE Sample $\{\hat{s}_{t,0}^k, \hat{a}_{t,0}^k, \dots, \hat{s}_{t,L}^k\}\sim \pi_{\theta_t^k} \otimes \hat{P}^k$. 
     \STATE $\delta_{t,\ell}^{k,\mathrm{act}} = \hat{r}_{t,\ell}^k+(\gamma \phi(\hat{s}_{t,\ell+1}^k)-\phi(\hat{s}_{t,\ell}^k))^\top\mathbf{B}_t\omega_t^k$ for $\ell=0,\ldots,L-1$.  
        \STATE $\omega^k_{t+1} = \Pi_{U_{\omega}} \left( \omega^k_t + \beta\frac{1}{L} \delta_{t, L}^k \mathbf{B}_t^\top \phi(s_{t,0}^k) \right)$.
        \STATE $\mathbf{B}_{t+1}^k = \mathbf{B}_t + \zeta \mathbf{B}_{t,\perp}\mathbf{B}_{t,\perp}^\top \frac{1}{L}\delta_{t, L}^k \phi(s_{t,0}^k)(\omega^k_t)^\top$.
        \STATE $\theta_{t+1}^k = \theta_t^k+\alpha \frac{1}{L}\sum_{\ell=0}^{L-1}\delta_{t,\ell}^{k,\mathrm{act}}\nabla_\theta \log \pi_{\theta_t^k}(\hat{a}_{t,\ell}^k|\hat{s}_{t,\ell}^k)$.
        \STATE $s_{t+1,0}^k = s_{t,L}^k$, and $\hat{s}_{t+1,0}^k = \hat{s}_{t,L}^k$. 
    \ENDFOR
    
    \STATE $\bar{\mathbf{B}}_{t+1} = \frac{1}{K} \sum_{k=1}^K \mathbf{B}_{t+1}^k$.
    \STATE $\mathbf{B}_{t+1}, \mathbf{R}_{t+1} = \texttt{QR}(\bar{\mathbf{B}}_{t+1})$.
    \STATE Server broadcasts $\mathbf{B}_{t+1}$ to all agents.
\ENDFOR
\end{algorithmic}

\hrule
\hrule 
\end{minipage}
\end{wrapfigure}
 across various heterogeneous data \cite{caruana1997multitask,ando2005framework,bengio2013representation,lecun2015deep,niu2024collaborative}.   
Inspired by the literature of personalized FL and multi-task learning, 
following  \cite{tripuraneni2021provable,collins2021exploiting,thekumparampil2021statistically,duchi2022subspace,du2021fewshot,tian2023learning,niu2024collaborative},  
we assume that $\{z^{k,*}: ~ k=1, \cdots, K\}$ -- the best linear approximation coefficients under $\phi(s)$ -- belong to and fully span an $r$-dimensional subspace. %
For ease of exposition, and without significant loss of generality, we assume $\epsilon = 0$. 
Formally, assume that there exists an orthonormal matrix $\mathbf{B}^*\in \reals^{d\times r}$ and $\omega_\theta^{k,*}\in \reals^r$ such that 
\begin{align}
\label{eq: low dimensional rewrite}
z_\theta^{k,*} = \mathbf{B}^* \omega_\theta^{k,*}, ~~~ \forall ~ \theta. 
\end{align}

\section{Algorithm: Personalized Single-Timescale Federated AC}
\label{sec: alg}

We propose a single-timescale personalized federated AC algorithm ({\bf pFedAC}) that enables the $K$  agents to collaboratively train policies for maximizing their own cumulative rewards $J^k(\theta)$. 
Inspired by the seminal work \cite{collins2021exploiting} in federated learning, our proposed approach 
exploits the shared underlying structure across agents while mitigating negative knowledge transfer through a joint linear  
approximation framework, in which agents iteratively estimate the common latent subspace $\mathbf{B}^*$ together with their respective local ``head'' parameters $\omega_{\theta}^{k,*}$. 

To mitigate the mismatch between the limiting state distribution involved in the critic updates and the discounted state visitation distribution involved in the policy gradient, following \cite{chen2025convergence}, we employ a simulator to generate auxiliary state-action pairs for the policy updates.    
A formal description can be found in Algorithm \ref{alg: fedac-per}, which we explain next.

\noindent{\bf Critic.}
At a high level, in each iteration, each agent locally performs the TD($L$) update, and uses the obtained TD($L$) error $\delta_{t, L}^k$ -- formally defined in Eq.\,(\ref{eq: L step TD error}) -- and rewards to update estimates of its local head $\omega^k$ and common subspace $\bfB$. A similar TD($L$) scheme was also considered in \cite[Algorithm 2]{xionglinear}, which studied personalization through shared feature representations.  

Intuitively, the updates of $\omega^k$ and $\bfB$ can be interpreted as performing a perturbed stochastic gradient on the following joint linear approximation problem: 
\begin{equation*}
\min_{\bfB, \omega^1, \cdots, \omega^K} ~ \sum_{k=1}^K \mathbb{E}_{\mu_\theta^k}[(\hat{V}^k(s;\omega^k,\mathbf{B}) - V_{\theta}^{k}(s))^2],     
\end{equation*}
where $\hat{V}^k(s;\omega^k,\mathbf{B}) =  \phi(s)^\top\mathbf{B}{\omega^k}.$ Since $V^k$ is unknown, a true stochastic gradient cannot be computed directly. 
Following the semi-gradient method TD($L$), we approximate the stochastic gradients w.r.t. $\omega$ and $\mathbf{B}$ using the temporal-difference (TD) error, defined as
\begin{align}
\label{eq: L step TD error}  
   \delta_{t,L}^k = \sum_{\ell=0}^{L-1} \gamma^\ell r_{t,\ell}^k+(\gamma^L \phi(s_{t,L}^k)  - \phi(s_{t,0}^k))^\top {\mathbf{B}_t}{\omega_t^k},
\end{align}
where we write $r_{t,\ell}^k = R(s_{t,\ell}^k, a_{t,\ell}^k)$ for ease of exposition.   
Thus, 
\begin{align}
\omega_{t+1}^k = \prod\nolimits_{U_{\omega}}(\omega_{t}^k+ \frac{\beta}{L} \delta_{t,L}^k{\mathbf{B}_t}^\top \phi(s_{t,0}^k)), ~~
\text{and} ~~ \mathbf{B}_{t+1}^k =\mathbf{B}_{t}+ \frac{\zeta}{L}  \bfB_{t,\perp}\bfB_{t,\perp}^\top\delta_{t,L}^k\phi(s_{t,0}^k) (\omega_t^k)^\top,
\label{eqn:update rule for B L step}
\end{align}
where $\prod\nolimits_w$ denotes the projection operator onto the $\ell_2$ ball of radius at most $U_{\omega}$.  

\begin{remark}[On the Benefits of Projection]
\label{rmk: projection}    
In Eq.\eqref{eqn:update rule for B L step}, the updates of both $\omega$ and $\bfB$ involve (different) projection operations. 
The projection operation on $\omega$ gives us a natural coarse bound $U_{\omega}$, while the innovation projection in the update of $\bfB$ substantially suppresses perturbation amplification from cross terms; in particular, it ensures that the QR decomposition perturbation remains second order.

Although projection operations are not necessary in several existing analyses for single-agent or federated RL with a common policy~\cite{khodadadian2022federated,srikant2019finite,mitra2024simple},  
our setting involves tightly coupled dynamics among the shared subspace $\bfB$, local critic heads $\omega^k$, and personalized policies $\theta^k$, together with heterogeneous transition kernels $P^k$.  Consequently, establishing a tractable finite-time analysis requires combining fine-grained perturbation arguments with coarse uniform bounds induced by the projection operators. 
Without relying on any absolute bound on $\omega^k$, characterizing the error dynamics quickly becomes intractable. 
\end{remark}

\noindent{\bf Actor.} 
Following \cite{chen2025convergence,hu2025finite}, for the actor updates at each agent, we sample from a companion MDP with transition kernel $\hat P^k = \gamma P^k +(1-\gamma)\eta$ -- recalling that $\eta$ is the given initial state distribution. That is, with probability $\gamma$, the next state is generated according to $P^k$; otherwise, it is drawn from the initial state distribution $\eta$. 
Formally, in each round $t$, the trajectory $\{\hat{s}_{t,0}^k, \hat{a}_{t,0}^k, \hat{r}_{t,0}^k, \dots, \hat{s}_{t,L}^k\}$ for the actor update is generated as follows:  
\begin{equation}
\label{eq: actor sample trajectory generation}
\begin{aligned}
\hat{a}_{t, \ell} \sim ~ \pi_{\theta_t^k}(\cdot |~\hat{s}_{t, \ell}),
~ \text{and} ~~  
\hat{s}_{t, \ell+1} \sim ~ \hat{P}^k(\cdot |~ \hat{s}_{t, \ell}, \hat{a}_{t, \ell}), ~~ \text{where} ~ \hat{P}^k = \gamma P^k + (1-\gamma)\eta  
\end{aligned}   
\end{equation}
for $\ell=0, \cdots, L-1$.  
For the actor update, we also define the one-step TD error
\begin{align}
\label{eq: actor one step TD error}
    \delta_{t,\ell}^{k,\mathrm{act}}
    =
    \hat{r}_{t,\ell}^k+
    \pth{\gamma\phi(\hat{s}_{t,\ell+1}^k)-\phi(\hat{s}_{t,\ell}^k)}^\top \mathbf{B}_t\omega_t^k,
    \qquad \ell=0,\ldots,L-1.
\end{align}
where $\hat{r}_{t,\ell}^k = R^k(\hat{s}_{t, \ell}, \hat{a}_{t, \ell})$, $\beta$ and $\zeta$ are stepsizes of $\omega^k$ and $\bfB$, respectively.
We use the length-$L$ trajectory as a minibatch of one-step TD policy-gradient estimators, leading to the update rule
\begin{align*}
    \theta_{t+1}^k
    =
    \theta_t^k
    +
    \alpha \frac{1}{L}
    \sum_{\ell=0}^{L-1}
    \delta_{t,\ell}^{k,\mathrm{act}}
    \nabla_\theta \log \pi_{\theta_t^k}(\hat{a}_{t,\ell}^k|\hat{s}_{t,\ell}^k).
\end{align*}

Note that the above Markovian sampling procedure generally requires access to a simulator whose state can be reset arbitrarily. This approach also enjoys several favorable properties that facilitate accurate estimation of the policy gradient. 
\begin{proposition}\cite{xionglinear}
\label{prop: mixing: companion}
For the actor Markov chain in Eq.\eqref{eq: actor sample trajectory generation}, we have 
$d_{TV}(\mathbb{P}(\hat s_t\mid s_0), \nu_\theta) \le \gamma^t$, for all $t\ge 0$ and $s_0\in \calS$.  
\end{proposition}

\section{Convergence Analysis} 
\label{sec: main convergence}

We present some technical assumptions adopted in our convergence analysis. 
Assumptions \ref{ass: per-agent full exploration}, \ref{assp: uniform ergodicity}, and \ref{ass: bounded features} are widely adopted in the RL literature. 
Assumption \ref{assp: eigval ZZ^T} quantifies how well spread the underlying truth $z_{\theta}^{k,*}$ is in covering the $r$-dimensional subspace of $\bfB^*$.

\begin{assumption}[Exploration]
\label{ass: per-agent full exploration} 
The matrix $
A_{L,\theta}^k = \mathbb{E}_{\theta}[\phi(s^{(0)})(\gamma^L\phi(s^{(L)})-\phi(s^{(0)}))^\top],
$
where the expectation is taken with respect to $s^{(0)}\sim\mu_{\theta}^k, (a^{(\ell)}, s^{(\ell+1)}) \sim \pi_{\theta} \otimes P^k ~ \forall \ell\in [0, L-1]$, is negative definite and its largest eigenvalue is upper bounded by $-L\lambda$. 
\end{assumption} 

This assumption is widely adopted in analyzing TD learning with linear function approximation \cite{chen2024finite,bhandari2018finite,chen2021closing,olshevsky2023small,qiu2021finite,wu2020finite,zou2019finite}.  
Intuitively, $A_{L,\theta_t^k}^k$ captures the exploration of the policy $\pi$ under the transition kernel $P^k$ within $L$ steps. For the tabular setting, Assumption \ref{ass: per-agent full exploration} holds when the policy $\pi$ explores all state-action pairs with $L$ steps. %

\begin{assumption}[Uniform ergodicity and one-step contraction]
\label{assp: uniform ergodicity}
For any policy parameter $\theta$, let
$
\calK_\theta^k(s,ds'):=\sum_{a\in\calA}\pi_\theta(a|s)P^k(ds'|s,a)
$
be the induced state-transition kernel, and let $\bbP_{0:\tau}^k(\cdot|s_0^k=s)$ denote the state distribution after $\tau$ steps given $s_0^k=s$ for agent $k$.  There exists $\rho\in(0,1)$, uniformly over $k\in[K]$ and $\theta$, such that
    \begin{equation}
        d_{TV}(\mathbb{P}_{0:\tau}^k( \cdot|s_0^k=s),\mu_\theta^k(\cdot))\leq \rho^\tau,~ \forall \tau\geq0,  s\in \calS,
    \end{equation}
where $d_{TV}$ is the total variation distance. Moreover, %
\[
d_{TV}(\nu_1\calK_\theta^k,\nu_2\calK_\theta^k)\leq \rho d_{TV}(\nu_1,\nu_2) 
\]
for any two probability measures $\nu_1,\nu_2$ on $\calS$, i.e., satisfying one-step Dobrushin contraction. 
\end{assumption} 
In our analysis, with a little abuse of notation, we also use $\bbP_{0:\tau}^k(\cdot, \cdot|s_0^k=s)$, which denotes the state-action pair at time $\tau$ given $s_0^k = s$.  
Recall that $\mu_\theta^k$ is the limiting distribution induced by policy $\pi_\theta$ and the transition kernel $P^k, \forall k$. Uniform mixing conditions are often assumed to characterize the noise induced by Markovian sampling \citep{chen2024finite,olshevsky2023small,zou2019finite,wu2020finite}. The Dobrushin contraction part is the stronger one-step condition used below to compare the original and auxiliary chains.

\begin{assumption}
\label{assmp: Lipschitz of policy}
For any state-action pair $(s,a)\in \mathcal{S}\times \mathcal{A}$,  it holds that:
\begin{align*}
&\|\nabla\log\pi_{\theta}(a|s)\| \leq B, ~~\forall \theta;\\
&\|\nabla\log\pi_{\theta_1}(a|s) - \nabla\log\pi_{\theta_2}(a|s)\| \leq L_g\|\theta_1 - \theta_2\|, ~~ ~~ \forall ~ \theta_1, \theta_2;\\
&|\pi_{\theta_1}(a|s) - \pi_{\theta_2}(a|s)| \leq L_\pi\|\theta_1 - \theta_2\|, ~~ ~~ \forall ~ \theta_1, \theta_2.  
\end{align*}
\end{assumption}
Assumption \ref{assmp: Lipschitz of policy} is standard in the literature and holds for policy classes, such as Gaussian and softmax.

\begin{assumption}
\label{ass: bounded features}
$\|\phi(s)\| \leq 1$ for each $s\in \calS$.     
\end{assumption} 
This assumption is rather standard in RL literature \cite{sutton1999policy,sutton2018reinforcement,chen2024finite,olshevsky2023small,zou2019finite,wu2020finite}.  
The existence of a shared common subspace helps to reduce the local learning problem from a $d$-dimensional problem to an $r$-dimensional problem. In PFL, when $r\ll d$, these benefits are well characterized and are indeed substantial in derived bounds \cite{collins2021exploiting,niu2024collaborative}. Under Assumption \ref{ass: bounded features}, however, the dependence on $d$ is masked, potentially obscuring these gains in the resulting bounds. Adopting more fine-grained assumptions could therefore lead to significantly sharper results. Developing such refinements, however, would require a fundamental rethinking of the problem and is left for future work.
Let 
\begin{align}
\label{eq: collective weights}
\bfZ_t^* \in \reals^{d\times K}, ~~~ \text{with }z_t^{k,*} ~~ \text{as the $k$-th column},
\end{align}
where $z_t^{k,*}$ is the shorthand for $z^{k,*}(\theta_t^k)$.
By Eq.\,(\ref{eq: low dimensional rewrite}), we know that the rank of $\bfZ_t^*\bfZ_t^{*\top}$ is $r$. We require that each of the $r$ dimensions of $\bfB^*$ is well-covered by $\{z_t^{k,*}\}_{k=1}^K$, formally stated in Assumption \ref{assp: eigval ZZ^T}.  
\begin{assumption}
\label{assp: eigval ZZ^T}
    $\frac{1}{K}\lambda^+_{\min}(\bfZ_t^*\bfZ_t^{*\top})  \geq \nu $ for some $\nu >0$.
\end{assumption}
We also use $\lambda_{\min}^+  = \lambda_{\min}^+(\bfZ_t^*\bfZ_t^{*\top})$ to denote the smallest nonzero eigenvalue of $\bfZ_t^*\bfZ_t^{*\top}$.

\subsection{Main Convergence Results}
\label{subsec: error iterates}
In Algorithm \ref{alg: fedac-per}, three groups of variables are updated: the common subspace estimate $\bfB_t$, the local heads $\omega_t^k$, and the actor parameters $\theta_t^k$. In this subsection, we characterize their evolutions over time. For ease of exposition, we introduce a set of notation:
\begin{align}
 x_t^k &= \bfB_t\omega_t^k-z_t^{k,*}, \qquad  m_t = \bfB^{*\top}_\perp \bfB_t, \label{eq: main iterates}\\
& \qquad \qquad \tilde\omega_{t+1}^k = \omega^k_t + \frac{\beta}{L} \delta_{t,L}^k \mathbf{B}_t^\top \phi(s_{t,0}^k), \label{eq: intermediate variables 0}\\
 &\tilde x_{t+1}^k = \bar\bfB_{t+1} \tilde\omega_{t+1}^k-z_{t+1}^{k,*}, \quad\quad  \bar m_{t+1} = \bfB^{*\top}_\perp \bar\bfB_{t+1}.
\label{eq: intermediate variables}
\end{align}

Fix $\tau$. 
For any $t\ge \tau$, define
$
M_t = \bbE\|m_{t}\|_F^2,\quad
\bar X_t = \frac{1}{K}\sum_{k=1}^K \bbE\|x_t^k\|^2,\quad
\bar G_t = \frac{1}{K}\sum_{k=1}^K \bbE\|\nabla J^k(\theta_t^k)\|^2.
$
For the time-averaged quantities, for any $T>\tau$, define
\begin{align*}
M_T
&=\frac{1}{T-\tau}\sum_{t=\tau}^{T-1}M_t, 
\qquad 
X_T^k
=\frac{1}{T-\tau}\sum_{t=\tau}^{T-1}
\bbE\|x_t^k\|^2,\qquad
G_T^k
=\frac{1}{T-\tau}\sum_{t=\tau}^{T-1}
\bbE\|\nabla J^k(\theta_t^k)\|^2,\\
\bar X_T
&=\frac{1}{K}\sum_{k=1}^K X_T^k
=\frac{1}{T-\tau}\sum_{t=\tau}^{T-1}\bar X_t,\qquad
\bar G_T
=\frac{1}{K}\sum_{k=1}^K G_T^k
=\frac{1}{T-\tau}\sum_{t=\tau}^{T-1}\bar G_t .
\end{align*}
\begin{lemma}[Upper bound of local head errors (informal)]
\label{lm: local head: upper bound} 
Let $\bfP^* = \bfB^* (\bfB^*)^{\top}$, and $\bfP_t = \bfB_t \bfB_t^{\top}$ for each $t$.
Choose the mixing window so that $\tau L=\lceil 2\log_\rho(\zeta)\rceil$, $U_{\delta} U_{\omega} \zeta\le \frac{1}{2}$, where $U_{\delta} = U_r + 2U_{\omega}$, $\beta = c\zeta$, and $\alpha =c_\theta\zeta$ for arbitrary $c,c_\theta>0$ that can be tuned. For any positive constants $c_1,c_2,c_3,c_4$, there exist constants $\calC_{X,1}(\tau^4;c,c_\theta)$, $\calC_{X,2}(\tau^4;c,c_\theta)$, $\calC_{X,3}(\tau^4;c,c_\theta)$, and $\calC_{X,4}(\tau^6;c,c_\theta)$ such that, for all $t\geq \tau$,
\begin{align*}
    &\bar X_{t+1}  \nonumber
   \leq \Bigg(1-2\lambda c\zeta
   +\pth{
   c_2^2\frac{c}{L}
   +c_3^2\frac{1}{L}
   +\frac{4U_\omega^4}{c_3^2L}
   +4B L_{*,1}c_\theta
   +c_4^2c_\theta
   }\zeta\\
   &\qquad\qquad\qquad+\frac{864U_\delta U_\omega^3c}{c_1^2}\frac{\zeta^2}{L^3(1-\gamma)}
   +\frac{\calC_{X,1}(\tau^4;c,c_\theta)}{(1-\gamma)^4}\zeta^2
\Bigg)\bar X_t\nonumber\\
&\qquad+ \pth{
\frac{36U_\omega^2}{c_2^2}\frac{c\zeta}{L}
+6U_\omega U_\delta c\frac{\zeta\pth{1+c_1^2\zeta}}{L(1-\gamma)}
} M_t
+
\pth{
\frac{(1-\gamma)^2L_{*,1}^2c_\theta}{c_4^2}\zeta
+\frac{\calC_{X,2}(\tau^4;c,c_\theta)}{(1-\gamma)^2}\zeta^2
}\bar G_t\nonumber\\
&\qquad+
\frac{\calC_{X,3}(\tau^4;c,c_\theta)}{(1-\gamma)^4}
\pth{\frac{\zeta^2}{\sqrt L}+\frac{\zeta^2}{L}+\frac{\zeta^2}{K}}
+\frac{\calC_{X,4}(\tau^6;c,c_\theta)}{(1-\gamma)^4}\zeta^3
+8c_\theta U_\omega L_{*,1}U_\delta B\,\zeta\,\gamma^{\tau L}.
\end{align*}
The constants are defined in \eqref{eqn:C_X constants}
\end{lemma}

Intuitively, when $M_t \lesssim \bar X_t$, the above upper bound is similar to the single-agent setting \cite{chen2024finite}.

Deriving Lemma \ref{lm: local head: upper bound} is highly nontrivial and requires a fundamental detour from the existing analysis \cite{chen2024finite}. Furthermore, existing analysis in PFL \cite{collins2021exploiting} is not applicable to our problem due to the Markovian sampling, the TD updates, and the lack of responses.

We further show that it is impossible for $M_t \gg \bar X_t$. 
\begin{lemma}\cite{wang2026personalized}[Lower bound of local head errors]
\label{lm: local head: lower bound}
Suppose that $d\ge 2r$. It holds that 
\begin{align*}
    &\frac{1}{K}\sum_{k=1}^K \|x_t^k\|^2\geq \frac{\|m_t\|^2}{K}\lambda^+_{\min}(\bfZ_t^*\bfZ_t^{*\top})
    \geq  \frac{\|m_t\|_F^2}{rK}\lambda^+_{\min}(\bfZ_t^*\bfZ_t^{*\top}).
\end{align*}
\end{lemma}  
Importantly, Lemma \ref{lm: local head: lower bound} should not be interpreted as a problem-dependent lower bound characterizing the fundamental limits of a class of problems. Rather, it is an algorithm-specific intermediate result that plays a central role in our analysis. In particular, the lemma reveals that under Algorithm \ref{alg: fedac-per}, the quantities $M_t$ and $X_t$ are necessarily learned simultaneously throughout the training process.

\begin{lemma}[Principal angle distance]  
\label{lm: PAD analysis}
Choosing $\beta=c\zeta$ and $\alpha =c_\theta\zeta$ for some fixed $c,c_\theta>0$, and a mixing window satisfying $\tau L=\lceil 2\log_\rho(\zeta)\rceil$, suppose that $d\ge 2r$, $L,K\ge 1$, and
$
\zeta\leq \min\left\{
1,\frac{1}{\sqrt{6}U_\delta U_\omega},
\frac{\lambda\nu}{3rU_\delta^2U_\omega^2},
\frac{L(1-\gamma)}{2U_\delta U_\omega},
\frac{1}{2U_\omega}
\right\}.
$
Then, with the notation introduced above, for any $c_1>0$, there exist constants $\calC_{M,1}(\tau^4;c,c_\theta)$, $\calC_{M,2}(\tau^2;c,c_\theta)$, and $\calC_{M,3}(\tau^6;c,c_\theta)$ such that, for all $t\ge \tau$,
\begin{align*}
M_{t+1}
&\leq
\pth{1-\frac{\lambda\nu\zeta}{r}+c_1^2\zeta}M_t
+\pth{
\frac{144U_\omega^2}{c_1^2}\frac{\zeta}{L^2}
+\frac{\calC_{M,1}(\tau^4;c,c_\theta)}{(1-\gamma)^2}
\frac{\zeta^2}{L^2}
}\bar X_t\nonumber\\
&\quad +\frac{\calC_{M,2}(\tau^2;c,c_\theta)}{(1-\gamma)^3}
\pth{
\frac{\zeta^2}{L}
+\frac{\zeta^2}{K}
}
+\frac{\calC_{M,3}(\tau^6;c,c_\theta)}{(1-\gamma)^4}\zeta^3.
\end{align*}
The constants are defined in \eqref{eqn: C_M constants}.
\end{lemma}

\begin{lemma}[Policy gradient analysis]
\label{lm: policy gradient analysis}
Choosing $\beta=c\zeta$ and $\alpha=c_\theta\zeta$ for some fixed
$c,c_\theta>0$, and the mixing window satisfies $\tau L=\lceil2\log_\rho(\zeta)\rceil$, suppose
$\alpha\leq 1$ and $16L_{J'}\alpha\leq1$. Then, with the notation
introduced above, for any $T>\tau$ there exist positive constants
$\calC_{G,0},\ldots,\calC_{G,11}$ such that
\begin{align*}
\bar G_T
\leq&
\frac{\calC_{G,0}}{(1-\gamma)^2c_\theta\zeta (T-\tau)}
+\frac{\calC_{G,1}}{(1-\gamma)^2}\bar X_T
+\frac{\calC_{G,2}\tau^2c_\theta^2\zeta^2}{(1-\gamma)^3}
+\frac{\calC_{G,3}\tau^2cc_\theta\zeta^2}{L(1-\gamma)^2}
+\frac{\calC_{G,4}\tau^2c_\theta\zeta^2}{L(1-\gamma)^2}\\
&+
\frac{\calC_{G,5}\tau c_\theta\zeta }{(1-\gamma)\sqrt{L(1-\gamma)}}
+\frac{\calC_{G,6}\tau^2c^2c_\theta\zeta^3}{L^2(1-\gamma)^3}
+\frac{\calC_{G,7}\tau^2c_\theta\zeta^3}{L^2(1-\gamma)^3}\\
&+
\frac{\calC_{G,8}\tau^2c_\theta^3\zeta^3}{1-\gamma}
+\frac{\calC_{G,9}\tau c_\theta^2\zeta^2}{(1-\gamma)^3}
+\frac{\calC_{G,10}c_\theta\zeta}{L(1-\gamma)^2}
+\frac{\calC_{G,11}\gamma^{\tau L}}{(1-\gamma)^2}.
\end{align*}
The constants are defined in \eqref{eqn: C_G constants}.
\end{lemma}
{
\begin{theorem}[Convergence (informal)]
\label{thm: 1}
Consider Algorithm \ref{alg: fedac-per} with $\beta=c\zeta$ and
$\alpha=c_\theta\zeta$, where $c,c_\theta>0$ are fixed constants satisfying the
feasibility conditions in the proof. Choose $\zeta=\frac{L^{1/4}}{\sqrt T}.$
Let $
\tau L=\max\{\lceil2\log_\rho(\zeta)\rceil,\lceil2\log_\gamma(\zeta)\rceil\}.
$
If $T$ is sufficiently large to satisfy the small-stepsize conditions in the
proof, then
\begin{align*}
\bar X_T
&\leq
\tilde{\calO}\left(
\frac{1}{(1-\gamma)^4L^{1/4}\sqrt T}
+\frac{L^{1/4}}{(1-\gamma)^4K\sqrt T}
+\frac{\sqrt L}{(1-\gamma)^5T}
\right),\\
\bar G_T
&\leq
\tilde{\calO}\left(
\frac{1}{(1-\gamma)^6L^{1/4}\sqrt T}
+\frac{L^{1/4}}{(1-\gamma)^6K\sqrt T}
+\frac{\sqrt L}{(1-\gamma)^7T}
\right).
\end{align*}
Taking $L=K^2$ and assuming that the higher-order term is
dominated gives
\begin{align*}
\bar X_T
&\leq
\tilde{\calO}\left(\frac{1}{(1-\gamma)^4\sqrt{TK}}\right),\qquad
\bar G_T
\leq
\tilde{\calO}\left(\frac{1}{(1-\gamma)^6\sqrt{TK}}\right).
\end{align*}
Here $\tilde{\calO}$ hides logarithmic factors and the Lipschitz constants
$L_{*,1}$, $L_{s,1}$, and $L_{J'}$.
\end{theorem}
}
This gives the standard linear speedup for nonconvex stochastic optimization: the dominant policy-gradient term scales as $\tilde{\calO}(1/\sqrt{TK})$.

\section{Experiments}
\label{sec:experiments}

We further instantiate our shared-representation and personalized-head structure within PPO \cite{schulman2017proximal}, yielding FedPer PPO, and evaluate it against Single PPO and FedAvg PPO on a federated \texttt{Hopper-v5} benchmark. In this benchmark, clients share the same underlying dynamics but receive distinct \emph{action maps}.
Appendix~\ref{app:setup} gives the full construction. All results are env-step-fair: we report the rolling-100 episode return at $0.4\,\text{M}$ per-client environment steps, with means $\pm$ standard deviations over $3$ seeds.

\subsection{Personalization benefits most under fully-distinct heterogeneity}
\label{sec:exp-hetero}

The grouped setup allows FedAvg's averaged parameters to specialize to two shared action maps, whereas the fully distinct setup forces FedAvg to reconcile six conflicting permutations, leading to a substantial performance drop. In contrast, FedPer's shared trunk can no longer specialize to a small number of maps and is therefore forced to learn an action-mapping-invariant representation that supports all six interfaces. At the fair budget, the ratios are $\text{FedPer}/\text{Single}=2.63\times$ and $\text{FedPer}/\text{FedAvg}=3.78\times$ in \textsc{6-unique}, compared with $2.04\times$ and $1.40\times$ in the grouped setup.
\begin{wrapfigure}[14]{r}{0.51\columnwidth}
    \centering
    \includegraphics[width=0.9\linewidth]{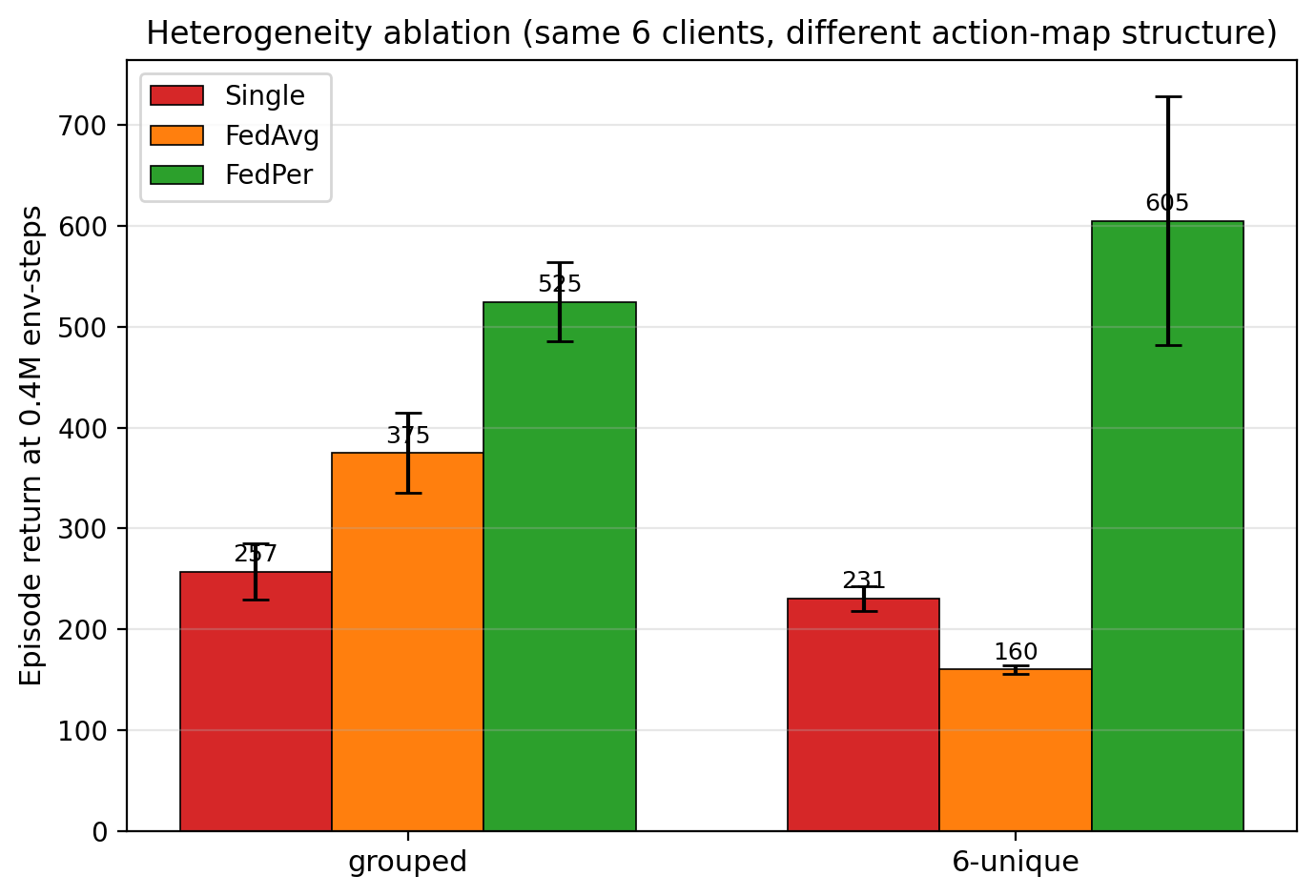}
    \caption{\small
    Same six clients under two action-map heterogeneity structures: grouped, where three clients share each action map, and \textsc{6-unique}, where every client has a distinct action map.}
    \label{fig:hetero-bar}
\end{wrapfigure}

Thus, personalization with a shared trunk yields its largest gain precisely in the regime where direct parameter averaging becomes unreliable.

\subsection{The FedPer trunk is a transferable foundation model}
\label{sec:exp-foundation}

\begin{figure}
    \centering
    \includegraphics[width=0.8\linewidth]{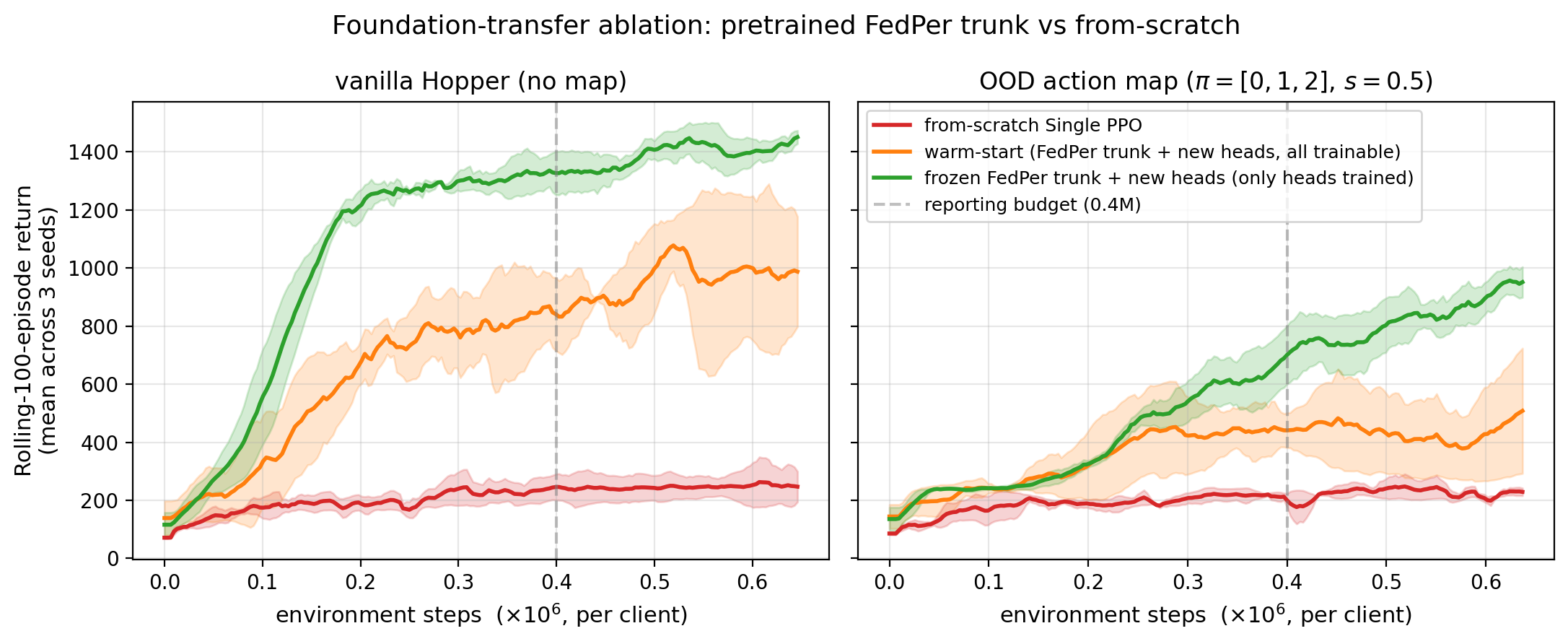}
    \caption{\small
    Foundation transfer. The trunk is loaded from a single FedPer \textsc{6-unique} checkpoint; the actor and critic output heads are randomly re-initialized. \textsc{frozen} trains only the new output heads ($1{,}031/69{,}895$ parameters); \textsc{trainable} fine-tunes the entire warm-started network; \textsc{scratch} trains a single PPO model from scratch. The dashed line marks the $0.4\,\text{M}$ reporting budget.}
    \label{fig:foundation}
\end{figure}

The \textsc{6-unique} result in Section~\ref{sec:exp-hetero} suggests that the shared trunk learns a representation that is largely invariant to the action interface, which is the property that allows each client's personalized head to specialize successfully. If this representation is reusable beyond the federated training tasks, it should accelerate learning on downstream tasks that were never optimized during federation, even when the pretrained trunk is frozen and only a tiny fraction of the network, namely the randomly initialized heads, is trained. We test this on two single-agent settings: (i)~\textbf{vanilla} Hopper with no action remapping, and (ii)~an \textbf{out-of-distribution (OOD)} variant using an action scale below the entire pretraining range, so that the trunk has not seen this interface during federation. At $0.4\,\text{M}$ env-steps, \textsc{frozen} attains $5.42\times$ \textsc{scratch} on vanilla ($1328\pm74$ vs.\ $245\pm45$) and $3.50\times$ on OOD ($701\pm101$ vs.\ $200\pm23$). The OOD result indicates out-of-distribution transfer, and \textsc{frozen}\ exceeds \textsc{trainable} on both tasks ($5.42\!>\!3.44\times$ vanilla, $3.50\!>\!2.21\times$ OOD), suggesting that on-policy fine-tuning can disrupt the pretrained representation before the randomly initialized heads recover. Additional results are provided in Appendix~\ref{app:additional-results}.

\section{Conclusion and Discussions}
\label{sec: dis & conclu}
In this paper, we studied collaborative yet personalized policy training in heterogeneous environments through a federated single-timescale actor--critic framework. By leveraging a shared linear subspace representation while maintaining personalized local policy components, our approach enables agents to benefit from collaboration without sacrificing adaptability to local environments. We established finite-time convergence guarantees under Markovian sampling and heterogeneous transition dynamics. 

Several important directions remain open for future work. In particular, improving the dependence of the convergence rates on the discount factor $(1-\gamma)^{-1}$ remains a key challenge, especially in the near-undiscounted regime where $\gamma\to 1$. We believe sharper analyses that better exploit the underlying Markovian structure, variance reduction techniques, or refined stability properties of actor--critic dynamics may lead to significantly improved bounds.

\newpage

\bibliography{main}
\bibliographystyle{abbrv}

\newpage 

\appendix
\onecolumn
\begin{center}
\LARGE 
\bf Appendices 
\end{center}

\startcontents[sections]
\printcontents[sections]{l}{1}{\setcounter{tocdepth}{3}}

\newpage 

\section{Limitations}
\label{app: limitations}
See Section~\ref{sec: dis & conclu}.
\section{Broader Impacts}
\label{app: broad impact}
This paper presents work whose goal is to advance the field of Machine Learning. There are many potential societal consequences of our work, none of which we feel must be specifically highlighted here.

\section{Related Work}
\label{app: additional related work}

\noindent{\bf MARL under homogeneous environments.}
When the environments of all agents are homogeneous, it has been shown that the federated version of classic reinforcement learning algorithms can significantly alleviate the data collection burden on individual agents \citep{woo_blessing_2023,khodadadian2022federated}. Specifically, the per-agent sample complexity decreases linearly in terms of the number of agents -- an effect commonly referred to in existing literature as linear speedup.  
Recent work has made significant progress in understanding federated RL under homogeneous environments. 
\cite{woo_blessing_2023} studied federated Q-learning.  
In addition to a linear speedup, \cite{woo_blessing_2023} uncovered the blessing of heterogeneity in terms of state-action exploration  -- a completely different notion of heterogeneity from our focus. 
\cite{salgia2025sample} studied the sample-communication complexity trade-off, and developed a new algorithm that achieves order-optimal sample and communication complexities. 

\noindent {\bf Model/Policy personalization.} 
In FL, model personalization (often referred to as PFL) has garnered significant attention in recent years. 
The success of personalization techniques depends on balancing the bias introduced by using global knowledge that may not generalize to individual clients, and the variance inherent in learning from limited local datasets. Popular PFL techniques include regularized local objectives \citep{t2020personalized,li2021ditto}, local-global parameter interpolation \citep{deng2020adaptive}, meta-learning \citep{fallah2020personalized, jiang2019improving}, and representation learning \citep{collins2021exploiting, xu2023personalized}.
Extending these ideas of PFL to RL introduces an even greater level of difficulty that arises from the inherent sample correlation induced by sequential decision-making, and the non-stationarity of sample quality as the policy evolves.

In the RL literature, training personalized (or localized) policies has recently garnered considerable attention in the context of multi-agent Markov games or competitive MARL \cite{filar2012competitive, zhang2018fully, qu2022scalable, zhang2023global, daskalakis2023complexity}, which depart fundamentally from our focus.  
Specifically, in their setting, agents interact with one another within a shared environment, and the rewards received by individual agents are determined by the joint actions of all agents. Each agent aims to learn a localized policy that contributes to a joint policy across all agents, with the goal of maximizing its own welfare.
A commonly adopted objective among the participating agents is to learn some form of equilibrium, such as a Nash Equilibrium.
Departing from Markov games, we focus on cooperative MARL, where we try to avoid the curse of heterogeneity in the joint learning.

\noindent{\bf Federated RL and actor-critic under heterogeneity.} 
Federated RL under heterogeneous environments or tasks has recently attracted increasing attention. 
Existing theoretical studies include value-based control methods such as federated Q-learning and SARSA, as well as critic-only TD learning with linear function approximation \citep{jin_federated_2022,wang2025on,wang2023federated,zhang2024finite}. 
These value-based or critic-only methods are distinct from the policy-gradient family: Q-learning and SARSA learn action-value functions for control, while TD learning focuses on policy evaluation or critic estimation. 
Policy-gradient methods have also been studied in heterogeneous federated RL; for example, \cite{wang2024momentum} developed momentum-based federated policy-gradient methods, and \cite{labbiglobal} studied global convergence rates of federated softmax policy gradient, both aiming to learn a common policy that optimizes an average objective across heterogeneous environments.
Federated actor-critic methods form the most directly related line of work. \cite{NEURIPS2024_dbdea785} studied federated natural policy-gradient and actor-critic methods for multi-task reinforcement learning, where agents collaborate to learn a shared policy under task heterogeneity. 
More recently, \cite{zhu2025single} analyzed a single-loop federated actor-critic method for learning a shared policy across heterogeneous environments. Their ``single-loop'' terminology refers to preserving the critic across policy updates, while the actor is still updated after multiple critic communication rounds.
\cite{xie2025actor} studied that, in the setting of FRL, Proximal Policy Optimization (PPO) suffers from data heterogeneity when local actors are updated based on local critic estimates, and they proposed FedRAC to mitigate this problem.

Despite these advances, existing federated reinforcement learning analyses mainly focus on learning a common policy, a common value function, or a shared objective across agents. 
In contrast, our goal is to learn personalized policies under heterogeneous transition kernels, where agents share a low-dimensional representation while maintaining agent-specific local heads. 
This personalization structure introduces additional coupled dynamics among policy updates, local heads, shared representation, and Markovian sampling errors, which are not captured by existing federated RL or federated actor-critic analyses.

\section{Notations and Preliminary Results}

\subsection{Auxiliary Markov chains for critic and actor}
\label{sec: Notations}

Following \cite{chen2024finite,chen2025convergence}, our finite-time analysis leverages two auxiliary Markov chains that are tightly coupled with the critic and actor Markov chains, respectively. The key distinction from the original chains is that, in both auxiliary chains, the policy parameter $\theta^k$ is frozen over the most recent $\tau$ rounds. Note that both auxiliary chains are purely hypothetical constructs introduced for analysis purposes and do not incur any additional sampling complexity. 

Let $v_{0:t}^k$ denote the Markovian trajectories up to time $t$ generated under initial state distribution $\eta$, the local transition kernel $P^k$, the critic and actor updates in Algorithm \ref{alg: fedac-per}: 
\begin{align*}
v_{0:t}^k=(s_0^k, \hat{s}_0^k, ~\{s_{0,0}^k, a_{0,0}^k, \cdots, s_{0,L}^k\}, \{\hat{s}_{0,0}^k, \hat{a}_{0,0}^k, &\cdots, \hat{s}_{0,L}^k\}, \cdots,\\
& \qquad ~\{s_{t,0}^k, a_{t,0}^k, \cdots, s_{t,L}^k\}, \{\hat{s}_{t,0}^k, \hat{a}_{t,0}^k, \cdots, \hat{s}_{t,L}^k\})      
\end{align*}
with the understanding that $s_{0,0}^k = s_0^k$, $\hat{s}_{0,0}^k = \hat{s}_{0}^k$, $s_{r-1,L}^k = s_{r,0}^k$, and $\hat{s}_{r-1,L}^k = \hat{s}_{r,0}^k$ for $1\le r\le t$. 
Let $v_{0:t}$ denote the collection of Markovian trajectories generated by all agents:
\begin{align*}
v_{0:t} = \{v_{0:t}^k\}_{k=1}^K.
\end{align*}
Similarly, we define $v_{0:(t,\ell)}^k$ and $v_{0:(t,\ell)}$ as the Markovian trajectory generated up to the $\ell$-th sample in time $t$. 
For ease of exposition, let $tr_{t}^k = \{s_{t,0}^k, a_{t,0}^k, \cdots, s_{t,L}^k\}$ and $\hat{tr}_{t}^k = \{\hat{s}_{t,0}^k, \hat{a}_{t,0}^k, \cdots, \hat{s}_{t,L}^k\}$.

Furthermore, we note that for the original critic and actor Markov chains, the policy parameter $\theta^k$ changes each round but remains fixed within a round. In contrast, for the auxiliary Markov chains associated with round $t$, the policy is last updated at round $t-\tau$ and is thereafter held fixed through round $t$.   
That is, with a little abuse of notation, 
\begin{align*}
&s_0^k \xrightarrow{\theta_0^k\times P^k, \theta_0^k\times \hat{P}^k} tr_{0}^k, \hat{tr}_{0}^k  \xrightarrow{\theta_1^k\times P^k, \theta_1^k\times \hat{P}^k} \cdots  
\xrightarrow{\theta_{t-\tau}^k\times P^k, \theta_{t-\tau}^k\times \hat{P}^k} tr_{t-\tau}^k, \hat{tr}_{t-\tau}^k  \\
& \qquad \qquad \qquad \qquad \qquad \xrightarrow{\theta_{t-\tau+1}^k\times P^k, \theta_{t-\tau+1}^k\times \hat{P}^k}  tr_{t-\tau+1}^k, \hat{tr}_{t-\tau+1}^k \cdots  
\xrightarrow{\theta_{t}^k\times P^k, \theta_{t}^k\times \hat{P}^k} tr_{t}^k, \hat{tr}_{t}^k. 
\end{align*}
We define the auxiliary Markov chains:
\begin{align*}
&s_0^k \xrightarrow{\theta_0^k\times P^k, \theta_0^k\times \hat{P}^k} tr_{0}^k, \hat{tr}_{0}^k  \xrightarrow{\theta_1^k\times P^k, \theta_1^k\times \hat{P}^k} \cdots  
\xrightarrow{\theta_{t-\tau}^k\times P^k, \theta_{t-\tau}^k\times \hat{P}^k} tr_{t-\tau}^k, \hat{tr}_{t-\tau}^k  \\
& \qquad \qquad \qquad \qquad \qquad \xrightarrow{\theta_{t-\tau}^k\times P^k, \theta_{t-\tau}^k\times \hat{P}^k}  \tilde tr_{t-\tau+1}^k, \bar{tr}_{t-\tau+1}^k \cdots  
\xrightarrow{\theta_{t-\tau}^k\times P^k, \theta_{t-\tau}^k\times \hat{P}^k} \tilde{tr}_{t}^k, \bar{tr}_{t}^k, 
\end{align*}
where for $t-\tau <r\le t$, 
$\tilde{tr}_{t}^k = \{\tilde s_{t,0}^k, \tilde a_{t,0}^k, \cdots, \tilde s_{t,L}^k\}$ and $\bar{tr}_{t}^k = \{\bar{s}_{t,0}^k, \bar{a}_{t,0}^k, \cdots, \bar{s}_{t,L}^k\}$.

To simplify notation, we use dots to indicate different levels of marginalization: 

{\bf For critic updates:}   
\begin{itemize}
    \item \( \bbP_{t-\tau:t}^k(\cdot) \)/\( \bbU_{t-\tau:t}^k(\cdot) \) denotes the marginal probability distribution over states \( s \in \mathcal{S} \) at time $t$ given the past trajectory $v_{0:t-\tau}$ for the original/auxiliary Markov chain of critic.
    \item \( \bbP_{t-\tau:t}^k(\cdot,\cdot) \)/\( \bbU_{t-\tau:t}^k(\cdot,\cdot) \) denotes the marginal probability distribution over state-action pairs \( (s, a) \in \mathcal{S} \times \mathcal{A} \) at time $t$ given the past trajectory $v_{0:t-\tau}$ for the original/auxiliary Markov chain of critic.
    \item \( \bbP_{t-\tau:t}^k(\cdot,\cdot,\cdot) \)/\( \bbU_{t-\tau:t}^k(\cdot,\cdot,\cdot) \) denotes the full joint probability distribution over tuple \( (s, a, s') \in \mathcal{S} \times \mathcal{A} \times \mathcal{S} \) at time $t$ given the past trajectory $v_{0:t-\tau}$ for the original/auxiliary Markov chain of critic.
\end{itemize}
\( \bbP_{t-\tau:(t,\ell)}^k \) and \( \bbU_{t-\tau:(t,\ell)}^k \) are defined analogously for the original/auxiliary Markov chain of critic. 

{\bf For actor updates:} 
\begin{itemize}
    \item \( \hat{\bbP}_{t-\tau:t}^k(\cdot) \)/\( \bar \bbU_{t-\tau:t}^k(\cdot) \) denotes the marginal probability distribution over states \( s \in \mathcal{S} \) at time $t$ given the past trajectory $v_{0:t-\tau}$ for the original/auxiliary Markov chain of actor.
    \item \( \hat{\bbP}_{t-\tau:t}^k(\cdot,\cdot) \)/\( \bar\bbU_{t-\tau:t}^k(\cdot,\cdot) \) denotes the marginal probability distribution over state-action pairs \( (s, a) \in \mathcal{S} \times \mathcal{A} \) at time $t$ given the past trajectory $v_{0:t-\tau}$ for the original/auxiliary Markov chain of actor.
    \item \( \hat{\bbP}_{t-\tau:t}^k(\cdot,\cdot,\cdot) \)/\( \bar\bbU_{t-\tau:t}^k(\cdot,\cdot,\cdot) \) denotes the full joint probability distribution over tuple \( (s, a, s') \in \mathcal{S} \times \mathcal{A} \times \mathcal{S} \) at time $t$ given the past trajectory $v_{0:t-\tau}$ for the original/auxiliary Markov chain of actor.
\end{itemize}
\( \hat{\bbP}_{t-\tau:(t,\ell)}^k \) and \(\bar\bbU_{t-\tau:(t,\ell)}^k \) are defined analogously for the original/auxiliary Markov chain of the actor.

For any $j\le t$, let $\bbE_{j}[X_t]$ denote the conditional expectation of any random variable $X_t$ given the collection of Markovian trajectories up to time $j$:
\begin{align*}
    \bbE_{j}[X_t] = \bbE\qth{X_t\mid v_{0:j}} = \bbE\qth{X_t\mid \{v_{0:j}^k\}_{k=1}^K}.
\end{align*}
When $j=t$, $\bbE_{j}[X_t] = \bbE[X_t]$ is the full expectation.  

For ease of exposition, for any given $\theta$, we occasionally abbreviate $\mathbb{E}_{s^{(0)}\sim\mu_{\theta}^k, (a^{(\ell)}, s^{(\ell+1)}) \sim \pi_{\theta} \otimes P^k ~ \forall \ell\in [0, L-1]} $ as  $\mathbb{E}_{\mu_{\theta}^k,  \pi_{\theta}, P^k }$.

\subsection{Boundedness and TD Features Decomposition}
\label{app: TD decomposition}
Let 
\begin{equation}
\label{eq: Markovian noise}   
\begin{aligned}
\xi_{t,L}^k  :=& \sum_{\ell=0}^{L-1} \gamma^\ell r_{t,\ell}^k\phi(s_{t,0}^k)
+ (\gamma^L \phi(s_{t,L}^k) - \phi(s_{t,0}^k))^\top \mathbf{B}_t \omega_t^k\phi(s_{t,0}^k)\\
    &-\sum_{\ell=0}^{L-1}\bbE_{\mu_{\theta_t^k}^k,  \pi_{\theta_t^k}, P^k} \qth{\gamma^\ell r(s^{(\ell)},a^{(\ell)})\phi(s^{(0)})}- \bbE_{\mu_{\theta_t^k}^k,  \pi_{\theta_t^k}, P^k }\qth{(\gamma^L \phi(s^{(L)}) - \phi(s^{(0)}))^\top \mathbf{B}_t \omega_t^k\phi(s^{(0)}) },
\end{aligned} 
\end{equation}
denote the Markovian noise of the TD($L$) error of agent $k$ at time $t$.

Denote 
\begin{align}
\tilde{b}_{t,L}^k &= \sum_{\ell=0}^{L-1}\gamma^\ell r_{t,\ell}^k\phi(s_{t,0}^k),  ~~~\label{eq: local reward Markovian noise}\\
\bar{b}_{L,\theta}^k &=  
\bbE_{\mu_{\theta}^k,  \pi_{\theta}, P^k }\qth{\sum_{\ell=0}^{L-1}\gamma^\ell r(s^{(\ell)},a^{(\ell)})\phi(s^{(0)})}, ~~~\label{eq: local reward expectation} \\
\tilde A_{t,L}^k &= \phi(s_{t,0}^k)(\gamma^L\phi(s_{t,L}^k)-\phi(s_{t,0}^k))^\top, ~~~ \label{def: Markov drift}\\
\bfb_{t,L}^k &= \sum_{\ell=0}^{L-1}\gamma^\ell\big(  r_{t,\ell}^k + (\gamma\phi(s_{t,\ell+1}^k) - \phi(s_{t,\ell}^k))^\top \mathbf{B}^*\omega_t^{k,*}\big)\phi(s_{t,0}^k). ~~~ \label{def: Markov noise}  
\end{align}
The expression of $\xi_{t,L}^k$ can be simplified as 
\begin{align}
\label{xi noise decomp}
    \xi_{t,L}^k = \tilde b_{t,L}^k-\bar{b}_{L,\theta_t^k}^k + \pth{\tilde A_{t,L}^k - A_{L,\theta_t^k}^k}\bfB_t\omega_t^k.
\end{align}

The following rewrites of the intermediate quantitiy $\delta_{t,L}^k\phi(s_t^k)$, which we refer to as TD feature, appear repeatedly and alternatively in our analysis.  
They correspond to two different ways of separating the Markovian noise from the value function estimation error $x_t^k$ as per defined in Eq.\eqref{eq: main iterates}.
\begin{proposition}[TD feature decomposition]
\label{prop: key intermediate}
For any $t$ and agent $k$, $\delta_{t,L}^k\phi(s_{t,0}^k)$ can be rewritten as 
\begin{align}
\label{eq: delta-phi: rewritting 1}  
\delta_{t,L}^k\phi(s_t^k) = \xi_{t,L}^k+A_{L,\theta_t^k}^kx_t^k,
\end{align}     
and 
\begin{align}
\label{eq: delta-phi: rewritting 2} 
\delta_{t,L}^k\phi(s_t^k) = \tilde A_{t,L}^kx_t^k+\bfb_{t,L}^k.
\end{align} 
\end{proposition}

The proof of Proposition \ref{prop: key intermediate} can be found in Appendix \ref{app: key intermediate proof}.

\begin{lemma}
\label{lmm: U delta}
Suppose that Assumption \ref{ass: bounded features} holds. 
Define $U_{\delta}:=U_r+ 2U_\omega.$
For any $t\geq 0, L\geq0, k \in [K]$, it holds that 
\begin{align}
&|\delta^k_{t,L}| \leq  \frac{U_\delta}{1-\gamma}, 
\quad \quad \qquad \qquad \qquad \qquad   |\delta_{i,\ell}^{k,\mathrm{act}}| \le U_{\delta} ~~ \forall ~ \ell=0, \cdots, L-1 \label{eq: bounded of delta act}\\
&\|\bfb_{t,L}^k\| \leq  \frac{U_\delta}{1-\gamma}, \label{eq: Markovian noise 1} \\
 &\bbE\|\tilde b_{t,L}^k - \bar b_{L,\theta_t^k}^k\|\leq\frac{2U_r}{1-\gamma},%
\quad \quad \quad \quad \quad ~ ~ 
\bbE\|\tilde b_{t,L}^k - \bar b_{L,\theta_t^k}^k\|^2\leq \frac{4U_r^2}{(1-\gamma)^2},\label{eqn: btilde-bbar ^2 local noise}\\
&\bbE\lnorm{\tilde A_{t,L}^k-A_{L,\theta_t^k}^k}{}\leq 4 %
\quad \quad \quad \quad \quad \quad \quad
\bbE\lnorm{\tilde A_{t,L}^k-A_{L,\theta_t^k}^k}{}^2\leq 16 \label{eqn: Atilde-Abar ^2 local noise},
\end{align}
where the expectation is taken over the Markovian sampling trajectories.

\end{lemma}
\begin{proof}
From Eq.\,(\ref{eq: L step TD error}), we have 
    \begin{align*}
        |\delta^k_{t,L}| &= \abth{\sum_{\ell=0}^{L-1} \gamma^\ell \pth{r_{t,\ell}^k + (\gamma \phi(s_{t,\ell+1}^k) - \phi(s_{t,\ell}^k))^\top \mathbf{B}_t \omega_t^k}} \\
        &\leq \frac{U_r+ 2U_\omega}{1-\gamma} 
        = \frac{U_\delta}{1-\gamma}.
    \end{align*}
    where the last inequality holds due to $\|\gamma^L \phi(s^k_{t+1}) - \phi(s^k_{t})\|\le 2$ under Assumption \ref{ass: bounded features}. 
Similarly, 
\begin{align*}
|\delta_{t,\ell}^{k,\mathrm{act}}| 
&\le  |\hat{r}_{t,\ell}^k|+|(\gamma \phi(\hat{s}_{t,\ell+1}^k)-\phi(\hat{s}_{t,\ell}^k))^\top\mathbf{B}_t\omega_t^k| \\
& \le U_{r} + 2U_{\omega} = U_{\delta}, ~~~ \forall ~ \ell=0,\ldots,L-1.
\end{align*}

From Eq.\eqref{def: Markov noise}, we have 
\begin{align*}
\| \bfb_{t,L}^k\| 
& = \|\sum_{\ell=0}^{L-1} \gamma^\ell r_{t,\ell}^k \phi(s_{t,0}^k)  + (\gamma^L \phi(s_{t,L}^k)  - \phi(s_{t,0}^k))^\top z_t^{k,*}\phi(s_{t,0}^k)\| \\
& \le \|\sum_{\ell=0}^{L-1} \gamma^\ell r_{t,\ell}^k + (\gamma^L \phi(s_{t,L}^k)  - \phi(s_{t,0}^k))^\top z_t^{k,*}\| ~~~~ \text{by Assumption \ref{ass: bounded features}} \\
& \le \frac{U_r + 2U_{\omega}}{1-\gamma}
=  \frac{U_\delta}{1-\gamma}.
\end{align*}

Similarly,
\begin{align*}
    \|\tilde b_{t,L}^k\|
    \leq \sum_{\ell=0}^{L-1}\gamma^\ell |r_{t,\ell}^k|\|\phi(s_{t,0}^k)\|
    \leq \frac{U_r}{1-\gamma},
    \qquad \text{and} ~ \qquad 
    \|\bar b_{L,\theta_t^k}^k\|
    \leq \frac{U_r}{1-\gamma}, 
\end{align*}
resulting in 
\begin{align*}
    \|\tilde b_{t,L}^k-\bar b_{L,\theta_t^k}^k\|
    \leq \frac{2U_r}{1-\gamma},
    \qquad
    \|\tilde b_{t,L}^k-\bar b_{L,\theta_t^k}^k\|^2
    \leq \frac{4U_r^2}{(1-\gamma)^2}. 
\end{align*}

Furthermore, by Assumption \ref{ass: bounded features}, we have 
\begin{align*}
    \|\tilde A_{t,L}^k\|
    &\le \|\phi(s_{t,0}^k)\| \|(\gamma^L\phi(s_{t,L}^k)-\phi(s_{t,0}^k))^\top\|
    \leq 1+\gamma^L
    \leq 2,\\
    \|A_{L,\theta_t^k}^k\|
    &\leq \bbE\|\phi(s^{(0)})(\gamma^L\phi(s^{(L)})-\phi(s^{(0)}))^\top\|
    \leq 2.
\end{align*}
Thus,
\begin{align*}
    \|\tilde A_{t,L}^k-A_{L,\theta_t^k}^k\|
    \leq 4,
    \qquad
    \|\tilde A_{t,L}^k-A_{L,\theta_t^k}^k\|^2
    \leq 16.
\end{align*}
\end{proof}

\subsection{Bound on QR Decomposition and Perturbation}
\label{sec: QR}

Recall from Algorithm \ref{alg: fedac-per} the update of $\bar \bfB$: 
\begin{align*}
    &\bar{\mathbf{B}}_{t+1} = \mathbf{B}_t
    +\frac{\zeta}{KL}\sum_{k=1}^K \bfB_{t,\perp}\bfB_{t,\perp}^\top\delta_{t,L}^k\phi(s_{t,0}^k)(\omega_{t}^k)^\top
\end{align*}
For ease of exposition, let 
\begin{align}
\label{eq: def Q}
\bfQ_t : = \frac{\zeta}{KL}\sum_{k=1}^K\bfB_{t,\perp}\bfB_{t,\perp}^\top\delta_{t,L}^k\phi(s_{t,0}^k)(\omega_{t}^k)^\top. 
\end{align}

The following couple of lemmas enable us to bound the ``distortion'' caused by the QR decomposition. 
Specifically, Lemma \ref{lm: perturbation QR} bounds the perturbation in terms of $\|\bf Q_t\|$, which is further bounded in Lemma \ref{lmm: perturbation QR: Q}. 
It is worth noting that, in the traditional PFL \cite{collins2021exploiting}, such distortion can be well controlled by the number of $\iid$ local samples that are freshly drawn in each iteration. However, their analysis does not extend to our setting due to the temporal dependence induced by Markovian sampling. In contrast to \cite{collins2021exploiting}, a recent work derived upper bounds on these distortions through a novel fixed-point iteration argument. Notably, the results in \cite{wang2026personalized} were established for average reward. Their analysis can be easily extended to the discounted setting. 
\begin{lemma}[Perturbation of QR]\cite{wang2026personalized}
\label{lm: perturbation QR}
For each $t\ge 1$, when $\frac{U_{\delta} U_{\omega}}{L(1-\gamma)}\zeta \le 1/2$, it holds that
\begin{align}
&\Norm{\bfR_{t+1}-\bfI}\leq 2\|\bfQ_t\|_F^2,\label{eqn: R related bounds: R-I}\\
&\|\bfR_{t+1}^{-1}\|\leq \frac{1}{1-\|\bfR_{t+1}-\bfI\|_2}\leq \frac{1}{1-2\|\bfQ_t\|_F^2}, \label{eqn: R related bounds: R inv}\\
&\Norm{\bfR_{t+1}^{-1}-\bfI} \leq 4 \|\bfQ_t\|_F^2. \label{eqn: R related bounds: R inv-I}
\end{align}     
\end{lemma}

\begin{lemma}
\label{lmm: perturbation QR: Q}
For any $t\ge \tau$, when $\frac{U_{\delta} U_{\omega}}{L(1-\gamma)}\zeta \le 1/2$, 
\begin{align}
\|\mathbf{Q}_t\|_F^2 \leq & \frac{U_\delta^2 U_\omega^2}{L^2(1-\gamma)^2}\zeta^2.  ~ \label{eqn: rough Q bound} 
\end{align} 
\end{lemma}

\begin{proof}

For any $t\ge \tau$, we have      
\begin{align*}
\|\bfQ_t\|_F^2
&=\|\frac{\zeta}{KL}\sum_{k=1}^K \bfB_{t,\perp}\bfB_{t,\perp}^\top \delta_{t,L}^k\phi(s_{t,0}^k)(\omega_{t}^k)^\top\|_F^2
\leq \|\frac{\zeta}{KL}\sum_{k=1}^K \delta_{t,L}^k\phi(s_{t,0}^k)(\omega_{t}^k)^\top\|_F^2\\
& \le \frac{\zeta^2}{K^2L^2} K \sum_{k=1}^K\|\delta_{t,L}^k\phi(s_{t,0}^k)(\omega_{t}^k)^\top\|_F^2
\le \frac{\zeta^2}{KL^2} \sum_{k=1}^K\|\delta_{t,L}^k\|^2 \|\phi(s_{t,0}^k)\|^2 \|\omega_{t}^k\|^2  \\
& \le \frac{U_\delta^2 U_\omega^2}{L^2(1-\gamma)^2}\zeta^2, 
\end{align*}
where the last inequality follows from Assumption \ref{ass: bounded features} and Lemma \ref{lmm: U delta}.

\end{proof}

\subsection{Lipschitz Bounds for Policy-Parameter Perturbations}

\begin{lemma}\cite{wu2020finite}
\label{lmm: L_J lip}
    For any $\theta_1,\theta_2$, we have 
    \[|J^k(\theta_1)-J^k(\theta_2)|\leq L_J\|\theta_1-\theta_2\|, \text{ for } k \in [K]\]
    where $L_J = 2U_r|\calA|L_\pi(1+\frac{1}{1-\rho}).$
\end{lemma}

\begin{lemma}\cite{zhang2020global}
\label{lmm: LJ' lip}
    For the performance function $J^k(\theta)$, there exists a constant $L_{J'}>0$ such that for all $\theta_1,\theta_2$, it holds that 
    \begin{equation*}
    \label{eqn:J smooth 1}
        \|\nabla J^k(\theta_1)-\nabla J^k(\theta_2) \|\leq L_{J'}\|\theta_1-\theta_2\|, \text{ for } k \in [K],
    \end{equation*}
    which further implies 
    \begin{equation*}
    \label{eqn:J smooth 2}
        J^k(\theta_2)\geq J^k(\theta_1)+ \langle\nabla J^k(\theta_1),\theta_2-\theta_1\rangle-\frac{L_{J'}}{2}\|\theta_1-\theta_2\|^2,
    \end{equation*}
    \begin{equation*}
    \label{eqn:J smooth 3}
        J^k(\theta_2)\leq J^k(\theta_1)+ \langle\nabla J^k(\theta_1),\theta_2-\theta_1\rangle+\frac{L_{J'}}{2}\|\theta_1-\theta_2\|^2.
    \end{equation*}
\end{lemma}

\begin{lemma}
\label{lmm: L* lip}
    There exists a constant $L_{*,1}>0$ such that
    \begin{equation*}
        \|\omega^{k,*}(\theta_1,\mathbf{B}^*)-\omega^{k,*}({\theta_2},\mathbf{B}^*)\| \leq L_{*,1} \|\theta_1-\theta_2\|, \forall \theta_1,\theta_2\in \bbR^d,
    \end{equation*}
    where one admissible choice is 
\(
L_{*,1}
=
\frac{1}{\lambda L}
\left(
\frac{2U_r C_\mu}{1-\gamma}
+
\frac{U_r L_\pi|\calA|}{(1-\gamma)^2}
\right)
+
\frac{U_r}{1-\gamma}(\lambda L)^{-2}
\left(
\frac{4C_\mu}{1-\gamma}
+
\frac{2L_\pi|\calA|}{(1-\gamma)^2}
\right).
\)
\end{lemma}
\begin{lemma}
\label{lmm: Ls lip}
    Assume that $\nabla_\theta\mu_\theta^k$ and $\nabla_\theta\nu_\theta^k$ exist and are uniformly bounded and Lipschitz in $\theta$ for every $k\in[K]$.
    For any $\theta_1,\theta_2,$ we have
    \[\|\nabla \omega^{k,*}(\theta_1,\mathbf{B}^*)-\nabla \omega^{k,*}(\theta_2,\mathbf{B}^*)\|\leq L_{s,1} \|\theta_1-\theta_2\|,\]
    where $L_{s,1}$ is a positive constant.
\end{lemma}

\begin{lemma}
\label{lmm: st-st perturb}
    For any $\theta_1,\theta_2$, it holds that 
    \[d_{TV}(\mu^k_{\theta_1},\mu^k_{\theta_2})\leq \underbrace{|\calA|L_\pi (\frac{1}{1-\rho})}_{:=C_\mu}\|\theta_1-\theta_2\|,\]
    \[d_{TV}(\mu^k_{\theta_1}\otimes \pi_{\theta_1},\mu^k_{\theta_2}\otimes \pi_{\theta_2})\leq \underbrace{|\calA|L_\pi (1+\frac{1}{1-\rho})}_{:=C_{\mu,\pi}}\|\theta_1-\theta_2\|,\]
    \[d_{TV}(\mu^k_{\theta_1}\otimes \pi_{\theta_1}\otimes P^k,\mu^k_{\theta_2}\otimes \pi_{\theta_2}\otimes P^k)\leq |\calA|L_\pi (1+\frac{1}{1-\rho})\|\theta_1-\theta_2\|.\]
\end{lemma}

Lemmas \ref{lmm: L* lip}, \ref{lmm: Ls lip}, and \ref{lmm: st-st perturb} are proved in Appendix \ref{app: proofs of preliminary lemmas}. 
Following the same line of analysis as in the proof of Lemma \ref{lmm: st-st perturb}, the following holds. 
\begin{corollary}
\label{lmm: st-st perturb: actor}    
For any $\theta_1,\theta_2$, it holds that 
    \[d_{TV}(\nu^k_{\theta_1},\nu^k_{\theta_2})\leq \underbrace{|\calA|L_\pi (\frac{1}{1-\gamma})}_{:=C_\nu}\|\theta_1-\theta_2\|,\]
    \[d_{TV}(\nu^k_{\theta_1}\otimes \pi_{\theta_1},\nu^k_{\theta_2}\otimes \pi_{\theta_2})\leq \underbrace{|\calA|L_\pi (1+\frac{1}{1-\gamma})}_{:=C_{\nu,\pi}}\|\theta_1-\theta_2\|,\]
    \[d_{TV}(\nu^k_{\theta_1}\otimes \pi_{\theta_1}\otimes P^k,\nu^k_{\theta_2}\otimes \pi_{\theta_2}\otimes P^k)\leq |\calA|L_\pi (1+\frac{1}{1-\gamma})\|\theta_1-\theta_2\|.\]
\end{corollary}

\newpage

\section{Proof of Lemma \ref{lm: local head: upper bound} (Bounds on Local Head Errors)}
\label{sec: proof of local head upper bound}

We first provide a formal statement of Lemma \ref{lm: local head: upper bound}, which involves a set of constants whose formal definitions are deferred to the lemma proof.
\begin{lemma}[Upper bound of local head errors]
\label{lm: local head: upper bound: formal}    
Choose $\beta$ such that $\frac{U_{\delta}}{L(1-\gamma)}\beta\leq U_\omega$.
Let $\bfP^* = \bfB^* (\bfB^*)^{\top}$, and $\bfP_t = \bfB_t \bfB_t^{\top}$ for each $t$.
Choose the mixing window so that $\tau L=\lceil 2\log_\rho(\zeta)\rceil$, $U_{\delta} U_{\omega} \zeta\le \frac{1}{2}$, where $U_{\delta} = U_r + 2U_{\omega}$, $\beta = c\zeta$, and $\alpha =c_\theta\zeta$ for arbitrary $c,c_\theta>0$ that can be tuned. For any positive constants $c_1,c_2,c_3,c_4$, there exist constants $\calC_{X,1}(\tau^4;c,c_\theta)$, $\calC_{X,2}(\tau^4;c,c_\theta)$, $\calC_{X,3}(\tau^4;c,c_\theta)$, and $\calC_{X,4}(\tau^6;c,c_\theta)$ such that, for all $t\geq \tau$,
\begin{align*}
    &\bar X_{t+1}  \nonumber
   \leq \Bigg(1-2\lambda c\zeta
   +\pth{
   c_2^2\frac{c}{L}
   +c_3^2\frac{1}{L}
   +\frac{4U_\omega^4}{c_3^2L}
   +4B L_{*,1}c_\theta
   +c_4^2c_\theta
   }\zeta\nonumber\\
&\qquad\qquad\qquad\qquad\qquad
   +\frac{864U_\delta U_\omega^3c}{c_1^2}\frac{\zeta^2}{L^3(1-\gamma)}
   +\frac{\calC_{X,1}(\tau^4;c,c_\theta)}{(1-\gamma)^4}\zeta^2
\Bigg)\bar X_t\nonumber\\
&\qquad+ \pth{
\frac{36U_\omega^2}{c_2^2}\frac{c\zeta}{L}
+6U_\omega U_\delta c\frac{\zeta}{L(1-\gamma)}
\pth{1+c_1^2\zeta}
}M_t\nonumber\\
&\qquad+
\pth{
\frac{(1-\gamma)^2L_{*,1}^2c_\theta}{c_4^2}\zeta
+\frac{\calC_{X,2}(\tau^4;c,c_\theta)}{(1-\gamma)^2}\zeta^2
}\bar G_t\nonumber\\
&\qquad+
\frac{\calC_{X,3}(\tau^4;c,c_\theta)}{(1-\gamma)^4}
\pth{\frac{\zeta^2}{\sqrt L}+\frac{\zeta^2}{L}+\frac{\zeta^2}{K}}
+\frac{\calC_{X,4}(\tau^6;c,c_\theta)}{(1-\gamma)^4}\zeta^3
+8c_\theta U_\omega L_{*,1}U_\delta B\,\zeta\,\gamma^{\tau L}.
\end{align*}
\end{lemma}

\begin{proof}
Recall that $\omega_{t+1}^k = \Pi_{U_\omega}(\omega_{t}^k +\frac{\beta}{L}\delta_{t,L}^k\bfB_t^\top\phi(s_{t,0}^k))$, $\bfB_{t+1} = \bar \bfB_{t+1}\bfR_{t+1}^{-1}$, and $\tilde\omega_{t+1}^k = \omega_{t}^k +\frac{\beta}{L}\delta_{t,L}^k\bfB_t^\top\phi(s_{t,0}^k)$. 
We have
\begin{align*}
   \|x_{t+1}^k \| = &\|\bfB_{t+1}\omega_{t+1}^k- z_{t+1}^{k,*}\| \\
    = & \|\bfB_{t+1}\Pi_{U_\omega}(\tilde\omega_{t+1}^k) - z_{t+1}^{k,*}\|\\
    =&\|\bfB_{t+1}\bfR_{t+1}\tilde\omega_{t+1}^k-z_{t+1}^{k,*} + \bfB_{t+1}\Pi_{U_\omega}(\tilde\omega_{t+1}^k)-\bfB_{t+1}\bfR_{t+1}\tilde\omega_{t+1}^k \|\\
    =&\|\bar\bfB_{t+1}\tilde\omega_{t+1}^k-z_{t+1}^{k,*} + \bfB_{t+1}\Pi_{U_\omega}(\tilde\omega_{t+1}^k)-\bfB_{t+1}\bfR_{t+1}\tilde\omega_{t+1}^k\| \\
    =& \| \tilde{x}_{t+1}^k + \bfB_{t+1}\Pi_{U_\omega}(\tilde\omega_{t+1}^k)-\bfB_{t+1}\bfR_{t+1}\tilde\omega_{t+1}^k\|. 
\end{align*}
Taking the square on both sides, we get   
\begin{equation}
\label{eq: upper bound: key intermediate}
\begin{aligned}
    \|x_{t+1}^k\|^2 = & \|\tilde x_{t+1}^k\|^2  + \|\bfB_{t+1}\Pi_{U_\omega}(\tilde\omega_{t+1}^k)-\bfB_{t+1}\bfR_{t+1}\tilde\omega_{t+1}^k\|^2\\
    & + 2 \langle\bfB_{t+1}\Pi_{U_\omega}(\tilde\omega_{t+1}^k)-\bfB_{t+1}\bfR_{t+1}\tilde\omega_{t+1}^k, \tilde x_{t+1}^k \rangle. 
\end{aligned}    
\end{equation}

\underline{We bound $\|\bfB_{t+1}\Pi_{U_\omega}(\tilde\omega_{t+1}^k)-\bfB_{t+1}\bfR_{t+1}\tilde\omega_{t+1}^k\|^2$ as follows:} \\ 
We have \begin{align*}
    &\|\bfB_{t+1}\Pi_{U_\omega}(\tilde\omega_{t+1}^k)-\bfB_{t+1}\bfR_{t+1}\tilde\omega_{t+1}^k\| \\
    =& \|\Pi_{U_\omega}(\tilde\omega_{t+1}^k)-\bfR_{t+1}\tilde\omega_{t+1}^k\|\\
    =& \|\Pi_{U_\omega}(\tilde\omega_{t+1}^k)-\tilde\omega_{t+1}^k+\tilde\omega_{t+1}^k-\bfR_{t+1}\tilde\omega_{t+1}^k\|\\
    \leq&\|\Pi_{U_\omega}(\tilde\omega_{t+1}^k)-\tilde\omega_{t+1}^k\|+\|(\bfI-\bfR_{t+1})\tilde\omega_{t+1}^k\|\\
    \leq&\|\frac{\beta}{L}\delta_{t,L}^k\bfB_t^\top\phi(s_{t,0}^k)\|+\|(\bfI-\bfR_{t+1})\tilde\omega_{t+1}^k\| \\
    \leq&\|\frac{\beta}{L}\delta_{t,L}^k\bfB_t^\top\phi(s_{t,0}^k)\|+4 U_\omega\|\bfQ\|_F^2,
\end{align*}
where the last inequality follows from  Eq.\,(\ref{eqn: R related bounds: R-I}) and we use $\frac{U_{\delta}}{L(1-\gamma)}\beta\leq U_\omega$ such that $\|\tilde\omega_{t+1}^k\|\leq 2 U_\omega$.
Then, 
\begin{align*}
    &\|\bfB_{t+1}\Pi_{U_\omega}(\tilde\omega_{t+1}^k)-\bfB_{t+1}\bfR_{t+1}\tilde\omega_{t+1}^k\| ^2\\
    \leq& 2\frac{\beta^2}{L^2}\|\delta_{t,L}^k\bfB_t^\top\phi(s_{t,0}^k)\|^2+32U_\omega^2 \|\bfQ_t\|_F^4.
\end{align*}

For the first term, we further decompose it as:
\begin{align*}
    &\|\delta_{t,L}^k\bfB_t^\top\phi(s_{t,0}^k)\|^2\leq \|\delta_{t,L}^k\phi(s_{t,0}^k)\|^2\\
    =& \|\tilde A_{t,L}^kx_t^k+\bfb_{t,L}^k\|^2  ~~~~(\text{by Proposition \ref{prop: key intermediate}})\\
    \leq& 2\|\tilde A_{t,L}^kx_t^k\|^2+2\|\bfb_{t,L}^k\|^2\\
    \leq& 8\|x_t^k\|^2+2\|\bfb_{t,L}^k\|^2. 
\end{align*}
Taking the conditional expectation over the Markovian trajectory up to $t-\tau$, we get 
\begin{align}
\label{eqn:part of projection error}
   &\bbE_{t-\tau} \|\delta_{t,L}^k\bfB_t^\top\phi(s_{t,0}^k)\|^2\nonumber\\
   \leq&8\bbE_{t-\tau}\|x_t^k\|^2+2\bbE_{t-\tau}\|\bfb_{t,L}^k\|^2\nonumber\\
   \leq&8\bbE_{t-\tau}\|x_t^k\|^2+\frac{2U_\delta^2}{(1-\gamma)^2}. ~~ (\text{by Lemma \ref{lmm: U delta}} )
\end{align}

Finally, the second term in Eq.\,(\ref{eq: upper bound: key intermediate}) can be bounded as 
\begin{align}
\label{eq: upper bound of qr-proj squared} 
    &\bbE_{t-\tau}\|\bfB_{t+1}\Pi_{U_\omega}(\tilde\omega_{t+1}^k)-\bfB_{t+1}\bfR_{t+1}\tilde\omega_{t+1}^k\|^2 \nonumber\\
    \leq& 16\frac{\beta^2}{L^2}\bbE_{t-\tau}\|x_t^k\|^2+4\frac{\beta^2}{L^2}\frac{U_\delta^2}{(1-\gamma)^2}+32U_\omega^2 \bbE_{t-\tau}\|\bfQ_t\|_F^4, 
\end{align}
where $\bbE_{t-\tau}\|\bfQ_t\|_F^4$ can be bounded by invoking Lemma \ref{lmm: perturbation QR: Q}.

{ 
\underline{As for $2 \langle\bfB_{t+1}\Pi_{U_\omega}(\tilde\omega_{t+1}^k)-\bfB_{t+1}\bfR_{t+1}\tilde\omega_{t+1}^k, \tilde x_{t+1}^k \rangle$, }

we first write $\tilde x_{t+1}^k$ as 
\begin{align*}
\tilde x_{t+1}^k 
= \bar \bfB_{t+1}\tilde \omega_{t+1}^k - \bfB^* \omega_{t+1}^{*,k}  
=  \bfB_{t+1}\omega_{t+1}^k - \bfB^* \omega_{t+1}^{k,*}  + \bar \bfB_{t+1}\tilde \omega_{t+1}^k - \bfB_{t+1}\omega_{t+1}^k.  
\end{align*}
Thus, 
\begin{align*}
&\langle\bfB_{t+1}\Pi_{U_\omega}(\tilde\omega_{t+1}^k)-\bfB_{t+1}\bfR_{t+1}\tilde\omega_{t+1}^k, \tilde x_{t+1}^k \rangle  \\
& = \underbrace{\langle\bfB_{t+1}\Pi_{U_\omega}(\tilde\omega_{t+1}^k)-\bfB_{t+1}\bfR_{t+1}\tilde\omega_{t+1}^k, \bfB_{t+1}\omega_{t+1}^k - \bfB^* \omega_{t+1}^{k,*} \rangle}_{I_1} \\
&\quad +  \underbrace{\langle\bfB_{t+1}\Pi_{U_\omega}(\tilde\omega_{t+1}^k)-\bfB_{t+1}\bfR_{t+1}\tilde\omega_{t+1}^k, \bar \bfB_{t+1}\tilde \omega_{t+1}^k - \bfB_{t+1}\omega_{t+1}^k \rangle}_{I_2}.  
\end{align*}
Both $I_1$ and $I_2$ can be well controlled by the projection error of $\tilde{w}$ and the perturbation error that arises from the QR decomposition of $\bar{\bfB}$. 
We bound $I_1$ and $I_2$ separately.

On $I_2$, we have 
\begin{align*}
\bar \bfB_{t+1}\tilde \omega_{t+1}^k - \bfB_{t+1}\omega_{t+1}^k  
&= \bfB_{t+1}\bfR_{t+1} \tilde\omega_{t+1}^k - \bfB_{t+1}\omega_{t+1}^k \\
& = \bfB_{t+1}\bfR_{t+1} \tilde\omega_{t+1}^k - \bfB_{t+1}\bfR_{t+1}\omega_{t+1}^k + \bfB_{t+1}\bfR_{t+1}\omega_{t+1}^k - \bfB_{t+1}\omega_{t+1}^k  \\
& = \bfB_{t+1}\bfR_{t+1} \pth{\tilde\omega_{t+1}^k - \omega_{t+1}^k} 
+ \bfB_{t+1}\pth{\bfR_{t+1}-\bfI}\omega_{t+1}^k. 
\end{align*}
So, $I_2$ can be rewritten as  
\begin{align*}
&\langle\bfB_{t+1}\Pi_{U_\omega}(\tilde\omega_{t+1}^k)-\bfB_{t+1}\bfR_{t+1}\tilde\omega_{t+1}^k, \bar \bfB_{t+1}\tilde \omega_{t+1}^k - \bfB_{t+1}\omega_{t+1}^k \rangle    \\
& = \underbrace{\langle\bfB_{t+1}\Pi_{U_\omega}(\tilde\omega_{t+1}^k)-\bfB_{t+1}\bfR_{t+1}\tilde\omega_{t+1}^k, \bfB_{t+1}\bfR_{t+1} \pth{\tilde\omega_{t+1}^k - \omega_{t+1}^k}  \rangle}_{I_{21}} \\
& \quad + \underbrace{\langle\bfB_{t+1}\Pi_{U_\omega}(\tilde\omega_{t+1}^k)-\bfB_{t+1}\bfR_{t+1}\tilde\omega_{t+1}^k, \bfB_{t+1}\pth{\bfR_{t+1}-\bfI}\omega_{t+1}^k \rangle}_{I_{22}}.
\end{align*} 
For $I_{21}$, by Cauchy Schwartz, we have 
\begin{align*}
    &I_{21} \leq \|\bfB_{t+1}\Pi_{U_\omega}(\tilde\omega_{t+1}^k)-\bfB_{t+1}\bfR_{t+1}\tilde\omega_{t+1}^k\| \| \bfB_{t+1}\bfR_{t+1} \pth{\tilde\omega_{t+1}^k - \omega_{t+1}^k}\|\\
    \leq&\frac{1}{2}\|\bfB_{t+1}\Pi_{U_\omega}(\tilde\omega_{t+1}^k)-\bfB_{t+1}\bfR_{t+1}\tilde\omega_{t+1}^k\| ^2+ \frac{1}{2}\| \bfB_{t+1}\bfR_{t+1} \pth{\tilde\omega_{t+1}^k - \omega_{t+1}^k}\|^2\\
    \leq&\frac{1}{2}\|\bfB_{t+1}\Pi_{U_\omega}(\tilde\omega_{t+1}^k)-\bfB_{t+1}\bfR_{t+1}\tilde\omega_{t+1}^k\| ^2+ \|\tilde\omega_{t+1}^k - \omega_{t+1}^k\|^2, 
\end{align*}
where the last inequality holds because that $\|\bfR_{t+1}\|\le 2$, which we establish next.  
From Lemma \ref{lm: perturbation QR} and Lemma \ref{lmm: perturbation QR: Q}, we have 
\begin{align*}
 \| \bfR_{t+1}\| \le \|\bfI \| + \|\bfI - \bfR_{t+1}\| \le 1 + 2\|\bfQ_{t}\|_F^2 
  \le 1+2 \frac{U_{\delta}^2 U_{\omega}^2 \zeta^2}{L^2(1-\gamma)^2}.  
\end{align*}
When $2\frac{U_\delta^2 U_\omega^2}{L^2(1-\gamma)^2}\zeta^2\leq 1/2$, $\|\bfR_{t+1}\|\leq 2$.  
As for $\|\tilde\omega_{t+1}^k - \omega_{t+1}^k\|^2$, by \eqref{eqn:part of projection error}, we have
\begin{align*}
    &\bbE_{t-\tau}\|\omega_{t+1}^{k}-\tilde\omega_{t+1}^k\|^2
    \leq
    \bbE_{t-\tau}\|\frac{\beta}{L}\delta_{t,L}^k\bfB_t^\top \phi(s_{t,0}^k)\|^2\\
    \leq&8\frac{\beta^2}{L^2}\bbE_{t-\tau}\|x_t^k\|^2+\frac{\beta^2}{L^2}\frac{2U_\delta^2}{(1-\gamma)^2}.
\end{align*}

For $I_{22}$, we have 
\begin{align*}
    &I_{22} \leq \|\bfB_{t+1}\Pi_{U_\omega}(\tilde\omega_{t+1}^k)-\bfB_{t+1}\bfR_{t+1}\tilde\omega_{t+1}^k\| \| \bfB_{t+1}\pth{\bfR_{t+1}-\bfI}\omega_{t+1}^k\|\\
    \leq&\frac{1}{2}\|\bfB_{t+1}\Pi_{U_\omega}(\tilde\omega_{t+1}^k)-\bfB_{t+1}\bfR_{t+1}\tilde\omega_{t+1}^k\|^2 + \frac{1}{2}\| \bfB_{t+1}\pth{\bfR_{t+1}-\bfI}\omega_{t+1}^k\|^2\\
    \leq&\frac{1}{2}\|\bfB_{t+1}\Pi_{U_\omega}(\tilde\omega_{t+1}^k)-\bfB_{t+1}\bfR_{t+1}\tilde\omega_{t+1}^k\|^2 + \frac{1}{2}U_\omega^2\|\bfR_{t+1}-\bfI\|^2\\
    \leq&\frac{1}{2}\|\bfB_{t+1}\Pi_{U_\omega}(\tilde\omega_{t+1}^k)-\bfB_{t+1}\bfR_{t+1}\tilde\omega_{t+1}^k\|^2 + 2U_\omega^2\| \bfQ_t\|_F^4.  ~~~~\text{(by \eqref{eqn: R related bounds: R-I})}
\end{align*}

Combining $I_{21}$ and $I_{22}$, we get
\begin{align*}
    &\bbE_{t-\tau}[I_2]\leq \bbE_{t-\tau}\|\bfB_{t+1}\Pi_{U_\omega}(\tilde\omega_{t+1}^k)-\bfB_{t+1}\bfR_{t+1}\tilde\omega_{t+1}^k\|^2 \\
    &+8\frac{\beta^2}{L^2}\bbE_{t-\tau}\|x_t^k\|^2+\frac{\beta^2}{L^2}\frac{2U_\delta^2}{(1-\gamma)^2}+ 2U_\omega^2 \bbE_{t-\tau}\| \bfQ_t\|_F^4\\
    \leq& 24\frac{\beta^2}{L^2}\bbE_{t-\tau}\|x_t^k\|^2+\frac{\beta^2}{L^2}\frac{6U_\delta^2}{(1-\gamma)^2}+34U_\omega^2 \bbE_{t-\tau}\|\bfQ_t\|_F^4.
     ~~~~~~\text{(by \eqref{eq: upper bound of qr-proj squared})}\\
\end{align*}

Bounding $I_1$ is is considerably more delicate and fundamentally differs from the analysis of non-personalized actor--critic methods. In particular, the errors induced by the projection operators associated with the local heads and the QR decomposition are tightly coupled. Without carefully controlling their interaction, the desired linear-speedup characterization cannot be established in the finite-time analysis. Towards this, we leverage geometric properties of projection operators and principal-angle distances. 
We have 
\begin{align*}
&I_1 = \langle\bfB_{t+1}\Pi_{U_\omega}(\tilde\omega_{t+1}^k)-\bfB_{t+1}\bfR_{t+1}\tilde\omega_{t+1}^k, \bfB_{t+1}\omega_{t+1}^k - \bfB^* \omega_{t+1}^{k,*} \rangle  \\
& = \underbrace{\langle\bfB_{t+1}\Pi_{U_\omega}(\tilde\omega_{t+1}^k)-\bfB_{t+1}\tilde\omega_{t+1}^k, \bfB_{t+1}\omega_{t+1}^k - \bfB^* \omega_{t+1}^{k,*} \rangle}_{I_{11}} \\
& \quad +\underbrace{\langle\bfB_{t+1}\tilde\omega_{t+1}^k-\bfB_{t+1}\bfR_{t+1}\tilde\omega_{t+1}^k, \bfB_{t+1}\omega_{t+1}^k - \bfB^* \omega_{t+1}^{k,*} \rangle}_{I_{12}}
\end{align*}
For $I_{11}$, let $\calR = \bfU\bfV^\top,$ where $\bfU\Sigma\bfV^\top=\bfB_{t+1}^\top\bfB^*$ is the SVD. Note that $\calR$ is the optimal rotation matrix that aligns $\bfB_{t+1}$ and $\bfB^*$ in the sense that $\calR=\text{argmin}_{\calR^\top\calR=\bfI}\|\bfB_{t+1}\calR-\bfB^*\|_F$. Then, we have
\begin{align*}
&I_{11}=\langle\bfB_{t+1}(\omega_{t+1}^k - \tilde\omega_{t+1}^k), \bfB_{t+1}\calR\calR^{-1}\omega_{t+1}^k - \bfB^* \omega_{t+1}^{k,*} \rangle ~~~ (\text{by rotation})\\
    =& \langle\bfB_{t+1}\calR\calR^{-1}(\omega_{t+1}^k - \tilde\omega_{t+1}^k),(\bfB_{t+1}\calR -\bfB^*)\calR^{-1}\omega_{t+1}^k\rangle\\
    &+\langle\bfB_{t+1}(\omega_{t+1}^k - \tilde\omega_{t+1}^k),\bfB^*(\calR^{-1}\omega_{t+1}^k- \omega_{t+1}^{k,*}) \rangle\\
    =& (\calR^{-1}(\omega_{t+1}^k - \tilde\omega_{t+1}^k))^{\top} (\bfI -(\bfB_{t+1}\calR)^\top\bfB^*)\calR^{-1}\omega_{t+1}^k\\
    &+\langle\bfB_{t+1}\calR\calR^{-1}(\omega_{t+1}^k - \tilde\omega_{t+1}^k),(\bfB^*-\bfB_{t+1}\calR)(\calR^{-1}\omega_{t+1}^k- \omega_{t+1}^{k,*}) \rangle\\
    &+ \langle\bfB_{t+1}\calR\calR^{-1}(\omega_{t+1}^k - \tilde\omega_{t+1}^k),\bfB_{t+1}\calR(\calR^{-1}\omega_{t+1}^k- \omega_{t+1}^{k,*}) \rangle\\
    =&(\calR^{-1}(\omega_{t+1}^k - \tilde\omega_{t+1}^k))^{\top} (\bfI -(\bfB_{t+1}\calR)^\top\bfB^*)\calR^{-1}\omega_{t+1}^k\\
    &+\langle\calR^{-1}(\omega_{t+1}^k - \tilde\omega_{t+1}^k),((\bfB_{t+1}\calR)^\top\bfB^*-\bfI)(\calR^{-1}\omega_{t+1}^k- \omega_{t+1}^{k,*}) \rangle\\
    &+\langle \omega_{t+1}^k - \tilde\omega_{t+1}^k, \omega_{t+1}^k- \calR\omega_{t+1}^{k,*} \rangle.\\
\end{align*}

We have the following facts: 
\begin{itemize}
    \item Fact 1: Since the ball with radius $U_{\omega}$ is convex, and $\omega_{t+1}^{k,*}$ belongs to the ball, it holds that 
    
    $$\langle \omega_{t+1}^k - \tilde\omega_{t+1}^k, \omega_{t+1}^k- \calR\omega_{t+1}^{k,*} \rangle \le 0. $$
    \item Fact 2: 
   \begin{align}
    \label{eq: fact 2}
    \|\bfI -(\bfB_{t+1}\calR)^\top\bfB^*\|\leq \|\bfB_\perp^{*\top}\bfB_{t+1}\|^2.
   \end{align} 
   This follows from the standard principal-angle/SVD relation.
\end{itemize}
Therefore,
\begin{align*}
    I_{11} \leq& \norm{\omega_{t+1}^k - \tilde\omega_{t+1}^k}\norm{\bfI -(\bfB_{t+1}\calR)^\top\bfB^*} \norm{\omega_{t+1}^k} \\
&\quad+ \norm{\omega_{t+1}^k - \tilde\omega_{t+1}^k}\norm{\bfI -(\bfB_{t+1}\calR)^\top\bfB^*}\norm{\calR^{-1}\omega_{t+1}^k - \omega_{t+1}^{k,*}} \\
\leq&3U_\omega\norm{\omega_{t+1}^k - \tilde\omega_{t+1}^k}\|m_{t+1}\|^2.
\end{align*}
Similarly, for $I_{12}$, we have
\begin{align*}
    &I_{12} = \langle\bfB_{t+1}\tilde\omega_{t+1}^k-\bfB_{t+1}\bfR_{t+1}\tilde\omega_{t+1}^k, \bfB_{t+1}\omega_{t+1}^k - \bfB^* \omega_{t+1}^{k,*} \rangle\\
    &= \langle\bfB_{t+1}(\bfI-\bfR_{t+1})\tilde\omega_{t+1}^k, \bfB_{t+1}\omega_{t+1}^k - \bfB^* \omega_{t+1}^{k,*} \rangle\\
    & \leq 4 U_\omega^2 \|\bfI-\bfR_{t+1}\|\\
    &\leq 8 U_\omega^2 \|\bfQ_t\|_F^2.
\end{align*}
Summing up $I_1$ and $I_2$ and multiplying by 2, we have
\begin{align}
\label{eq: x_t+1 intermediate bound cross term}
   &  2\bbE_{t-\tau}\langle\bfB_{t+1}\Pi_{U_\omega}(\tilde\omega_{t+1}^k)-\bfB_{t+1}\bfR_{t+1}\tilde\omega_{t+1}^k, \tilde x_{t+1}^k \rangle \nonumber \\
   \leq& 48\frac{\beta^2}{L^2}\bbE_{t-\tau}\|x_t^k\|^2+\frac{\beta^2}{L^2}\frac{36U_\delta^2}{(1-\gamma)^2}+68U_\omega^2 \bbE_{t-\tau}\|\bfQ_t\|_F^4\nonumber\\
     &+6U_\omega\bbE_{t-\tau}\norm{\omega_{t+1}^k - \tilde\omega_{t+1}^k}\|m_{t+1}\|^2+16U_\omega^2 \bbE_{t-\tau}\|\bfQ_t\|_F^2.
\end{align}
}

\underline{We bound $\|\tilde x_{t+1}^k\|$ as follows:}
\begin{align*}
    \tilde x_{t+1}^k
    & = \bar\bfB_{t+1}\tilde\omega_{t+1}^k- z_{t+1}^{k,*} \\ 
    & = \pth{\bfB_{t} + \frac{\zeta}{KL}\sum_{i=1}^K \bfP_{t,\perp}\delta_{t,L}^i\phi(s_{t,0}^i)(\omega_t^i)^\top} \pth{\omega_{t}^k + \frac{\beta}{L} \delta_{t,L}^k\bfB_t^{\top}\phi(s_{t,0}^k)} - z_{t+1}^{k,*}  \\
    &= \bfB_t\omega_{t}^k -z_{t}^{k,*}+z_{t}^{k,*}- z_{t+1}^{k,*} + \frac{\beta}{L}\bfB_t(\delta_{t,L}^k\bfB_t^\top\phi(s_{t,0}^k)) + \frac{\zeta}{KL}\sum_{i=1}^K  \bfP_{t,\perp}\delta_{t,L}^i\phi(s_{t,0}^i)(\omega_t^i)^\top\omega_t^k \\
    &\qquad + \zeta \beta(\frac{1}{KL^2}\sum_{i=1}^K  \bfP_{t,\perp}\delta_{t,L}^i\phi(s_{t,0}^i)(\omega_t^i)^\top) \pth{\delta_{t,L}^k\bfB_t^{\top}\phi(s_{t,0}^k)}. 
\end{align*}
Squaring both sides, we have 
\begin{align}
\label{eqn:xtilde squared decomp}
    \|\tilde x_{t+1}^k\|^2=&\|\bar\bfB_{t+1}\tilde\omega_{t+1}^k- z_{t+1}^{k,*}\|^2 \nonumber\\
    \leq& \|\bfB_t\omega_{t}^k- z_t^{k,*}\|^2+\underbrace{2\frac{\beta}{L}\langle\bfB_t\omega_{t}^k- z_t^{k,*}, \bfB_t\bfB_t^\top\delta_{t,L}^k\phi(s_{t,0}^k)\rangle}_{I_1} \nonumber\\ 
    &\quad +\underbrace{2\zeta\langle\bfB_t\omega_{t}^k- z_t^{k,*}, \frac{1}{KL}\sum_{i=1}^K  \bfP_{t,\perp}\phi(s_{t,0}^i)\delta_{t,L}^i(\omega_t^i)^\top\omega_t^k\rangle}_{I_2} \nonumber\\
    & \quad + \underbrace{2\frac{\zeta \beta}{L^2} \iprod{\bfB_t\omega_{t}^k- z_t^{k,*}}{(\frac{1}{K}\sum_{i=1}^K  \bfP_{t,\perp}\delta_{t,L}^i\phi(s_{t,0}^i)(\omega_t^i)^\top) \pth{\delta_{t,L}^k\bfB_t^{\top}\phi(s_{t,0}^k)}}}_{I_3}
    \nonumber\\
    &\quad + \underbrace{4\frac{\beta^2}{L^2} \|\bfB_t(\delta_{t,L}^k\bfB_t^\top\phi(s_{t,0}^k))\|^2}_{I_4}
    + \underbrace{4\frac{\zeta^2}{L^2} \|\frac{1}{K}\sum_{i=1}^K  \bfP_{t,\perp}\delta_{t,L}^i\phi(s_{t,0}^i)(\omega_t^i)^\top\omega_t^k\|^2 }_{I_5}\nonumber\\
    & \quad + \underbrace{4\frac{\beta^2\zeta^2}{L^4} \|(\frac{1}{K}\sum_{i=1}^K  \bfP_{t,\perp}\delta_{t,L}^i\phi(s_{t,0}^i)(\omega_t^i)^\top) \pth{\delta_{t,L}^k\bfB_t^{\top}\phi(s_{t,0}^k)}\|^2}_{I_6}\nonumber\\
    &\quad+ \underbrace{2\langle\bfB_t\omega_{t}^k- z_t^{k,*},z_{t}^{k,*}- z_{t+1}^{k,*}\rangle}_{I_7}
    + \underbrace{4 \lnorm{z_{t}^{k,*}- z_{t+1}^{k,*}}{}^2}_{I_8}
\end{align}

We bound these terms separately.

\underline{Bounding $I_1$.} At a high level, the contraction behavior of $\|x_{t}^k\|$ is primarily driven by the term $2\beta\langle x_{t}^k, \bfB_t\bfB_t^\top\phi(s_t^k)\delta_{t,L}^k\rangle$ (i.e., $I_1$ in Eq.\,(\ref{eqn:xtilde squared decomp})).
Nevertheless, establishing a sharp characterization of this term becomes substantially more challenging under environmental heterogeneity. In particular, Proposition \ref{prop: key intermediate} shows that $I_1$ admits the following decomposition:  
\begin{align}
\nonumber
&2\frac{\beta}{L}\langle x_{t}^k, \bfB_t\bfB_t^\top\phi(s_{t,0}^k)\delta_{t,L}^k\rangle\\
\nonumber 
=&2\frac{\beta}{L}\langle x_{t}^k, \bfB_t\bfB_t^\top\xi_{t,L}^k \rangle+2\frac{\beta}{L}\langle x_{t}^k, \bfB_t\bfB_t^\top A_{L,\theta_t^k}^kx_t^k \rangle \\
=& \underbrace{2\frac{\beta}{L}\langle x_{t}^k, \bfP_t  A_{L,\theta_t^k}^kx_t^k \rangle}_{I_{11}} + \underbrace{2\frac{\beta}{L}\langle x_{t}^k, \bfP_t \xi_{t,L}^k \rangle}_{I_{12}},
\label{eq: roughly drift of upper bound}
\end{align}     
recalling that $\bfP_t = \bfB_t\bfB_t^\top$. Inside $I_1$, the main negative drift arises from the term $\langle x_{t}^k, \bfP_t  A_{L,\theta_t^k}^kx_t^k \rangle$. Intuitively, in traditional single-agent or homogeneous environment settings, $\langle x_{t}^k, \bfP_t  A_{L,\theta_t^k}^kx_t^k \rangle \approx \langle x_{t}^k, Ax_t^k \rangle$, which is mainly controlled by the spectrum of $A$, with the desired property directly assumed in Assumption \ref{ass: per-agent full exploration}. However, in the presence of environmental heterogeneity, Assumption \ref{ass: per-agent full exploration} does not directly guarantee a negative drift of the $x_t^k$ due to the existence of $\bfP_t$ and its intricate interplay with the heterogeneous $A_{L,\theta_t^k}^k$. 
Specifically, (1) $\bfP_t$ is not of full rank, (2) the local head error will be distorted in a different manner due to the product $\bfP_t A_{L,\theta_t^k}^k$, and (3) $\bfP_t$ varies over time.  
To address this, we further decompose $I_{11}$ as 
\begin{align}
\label{eq: local head: upper bound: drift: intermediate} 
     &2\frac{\beta}{L}\langle x_t^k, \bfP_t A_{L,\theta_t^k}^k x_t^k \rangle \nonumber \\
    =&  2\frac{\beta}{L}\langle x_t^k, (\bfP_t-\bfP^*) A_{L,\theta_t^k}^k x_t^k \rangle+2\frac{\beta}{L}\langle x_t^k, \bfP^* A_{L,\theta_t^k}^k x_t^k \rangle \nonumber\\
    =&  \underbrace{2\frac{\beta}{L}\langle x_t^k,  A_{L,\theta_t^k}^k x_t^k \rangle}_{I_{111}}  + \underbrace{2\frac{\beta}{L}\langle x_t^k, (\bfP_t-\bfP^*) A_{L,\theta_t^k}^k x_t^k \rangle}_{I_{112}} - \underbrace{2\frac{\beta}{L}\langle x_t^k, \bfP_\perp^* A_{L,\theta_t^k}^k x_t^k \rangle}_{I_{113}}.
\end{align}
The first term of Eq.\,(\ref{eq: local head: upper bound: drift: intermediate}), $I_{111}$, is a contraction by Assumption \ref{ass: per-agent full exploration}. Specifically,
\begin{align*}
    2\frac{\beta}{L}\langle x_t^k,  A_{L,\theta_t^k}^k x_t^k \rangle\leq -2\lambda\beta\|x_t^k\|^2.
\end{align*}

$I_{112}$ can be controlled via the principal angle distance between the subspace estimate $\bfB_t$ and the underlying truth $\bfB^*$. Since
\begin{align}
\label{eqn: expected A crude bound}
    \|A_{L,\theta_t^k}^k\| \le  \mathbb{E}_{s^{0}\sim\mu_{\theta_t^k}^k,a^0\sim \pi_{\theta_t^k},...,s^{(L)}\sim P^k}\lnorm{\phi(s^{(0)})(\phi(s^{(L)})-\phi(s^{(0)}))^\top}{} \leq  2,
\end{align} 
it holds that 
\begin{align*}
    2\frac{\beta}{L}\langle x_t^k, (\bfP_t-\bfP^*) A_{L,\theta_t^k}^k x_t^k \rangle\leq 8U_\omega\frac{\beta}{L}\|\bfP_t-\bfP^*\|\|x_t^k\|.
\end{align*}
The third term, $I_{113}$, can be bounded as  
\begin{align}
\label{eq: aaa}
-2\frac{\beta}{L}\langle x_t^k, \bfP_\perp^* A_{L,\theta_t^k}^k x_t^k \rangle 
&= -2\frac{\beta}{L}(\bfP_\perp^*x_t^k)^\top A_{L,\theta_t^k}^k x_t^k = - 2\frac{\beta}{L}(\bfP_\perp^* \pth{\bfB_t\omega_t^k - \bfB^*\omega_t^{k,*}})^\top A_{L,\theta_t^k}^k x_t^k\nonumber\\
& = -2\frac{\beta}{L}\pth{\bfB_\perp^*\bfB_\perp^{*\top} \pth{\bfB_t\omega_t^k - \bfB^*\omega_t^{k,*}}}^{\top}A_{L,\theta_t^k}^k x_t^k \nonumber\\
& = - 2\frac{\beta}{L}(\bfB_\perp^*\bfB_\perp^{*\top} \bfB_t\omega_t^k)^\top A_{L,\theta_t^k}^k x_t^k\nonumber\\
&\leq 4\frac{\beta}{L} U_\omega\|\bfB_\perp^{*\top} \bfB_t\|\| x_t^k\|.
\end{align}
wherein the local head error $x_t^k$ and the principal angle distance $\|\bfB_\perp^{*\top} \bfB_t\|$ are also coupled.   
Then, 
\begin{align}
\label{eqn: first term in first ip final}
    &I_{11}\leq -2\lambda\beta\|x_t^k\|^2+12U_\omega\frac{\beta}{L}  \|m_t\|\|x_t^k\|.
\end{align}

It is worth noting that when $\|\bfB_\perp^{*\top} \bfB_t\| =0$, i.e., when $\bfB_t$ and $\bfB^*$ span the same $r$-dimensional subspace, the second and the third terms in Eq.\,(\ref{eq: local head: upper bound: drift: intermediate}) become zero, and Eq.\,(\ref{eq: local head: upper bound: drift: intermediate}) reduces to the standard negative drift term. We postpone the characterization of the convergence of $\|\bfB_\perp^{*\top} \bfB_t\|$ to {Lemma \ref{lm: PAD analysis}}. 

Recall \[\tilde{b}_{t,L}^k = \sum_{\ell=0}^{L-1}\gamma^\ell r_{t,\ell}^k\phi(s_{t,0}^k),\]
\[\quad \bar{b}_{L,\theta_t^k}^k=  \bbE_{s^{(0)}\sim\mu_{\theta_t^k}^k, (a^{(\ell)}, s^{(\ell)}) \sim \pi_{\theta_t^k} \otimes P^k ~ \forall \ell\in [0, L-1]}\qth{\sum_{\ell=0}^{L-1}\gamma^\ell r(s^{(\ell)},a^{(\ell)})\phi(s^{(0)})}.\]

For $I_{12}$, recall from \eqref{eq: Markovian noise} that, 
\begin{align}
\label{eqn: xi decomp}
\xi_{t,L}^k  
= \tilde b_{t,L}^k - \bar b_{L,\theta_t^k}^k+ (\tilde A_{t,L}^k-A_{L,\theta_t^k}^k)\bfB_t\omega_t^k.
\end{align}
Then,
\begin{align*}
    &2\frac{\beta}{L}\iprod{x_t^k}{\bfP_t\xi_{t,L}^k}\\
    =&2\frac{\beta}{L}\iprod{x_t^k}{\bfP_t\big[\tilde b_{t,L}^k - \bar b_{L,\theta_t^k}^k+ (\tilde A_{t,L}^k-A_{L,\theta_t^k}^k)\bfB_t\omega_t^k\big]}\\
     =&2\frac{\beta}{L}\iprod{x_t^k}{\bfP_t(\tilde b_{t,L}^k - \bar b_{L,\theta_t^k}^k)} + 2\frac{\beta}{L}\iprod{x_t^k}{\bfP_t (\tilde A_{t,L}^k-A_{L,\theta_t^k}^k)\bfB_t\omega_t^k}\\
     =&\underbrace{2\frac{\beta}{L}\iprod{x_t^k-x_{t-\tau}^k}{\bfP_t(\tilde b_{t,L}^k - \bar b_{L,\theta_t^k}^k)}}_{I_{121}} 
     +\underbrace{2\frac{\beta}{L}\iprod{x_{t-\tau}^k}{(\bfP_t-\bfP_{t-\tau})(\tilde b_{t,L}^k - \bar b_{L,\theta_t^k}^k)}}_{I_{122}}\\ 
     &+\underbrace{2\frac{\beta}{L}\iprod{x_{t-\tau}^k}{\bfP_{t-\tau}(\tilde b_{t,L}^k - \bar b_{L,\theta_t^k}^k)}}_{I_{123}}
      + \underbrace{2\frac{\beta}{L}\iprod{x_t^k-x_{t-\tau}^k}{\bfP_t (\tilde A_{t,L}^k-A_{L,\theta_t^k}^k)\bfB_t\omega_t^k}}_{I_{124}}\\
     &+\underbrace{2\frac{\beta}{L}\iprod{x_{t-\tau}^k}{(\bfP_t-\bfP_{t-\tau}) (\tilde A_{t,L}^k-A_{L,\theta_t^k}^k)\bfB_t\omega_t^k}}_{I_{125}}\\
     &+\underbrace{2\frac{\beta}{L}\iprod{x_{t-\tau}^k}{\bfP_{t-\tau} (\tilde A_{t,L}^k-A_{L,\theta_t^k}^k)(\bfB_t\omega_t^k-\bfB_{t-\tau}\omega_{t-\tau}^k)}}_{I_{126}}\\
     &+\underbrace{2\frac{\beta}{L}\iprod{x_{t-\tau}^k}{\bfP_{t-\tau} (\tilde A_{t,L}^k-A_{L,\theta_t^k}^k)\bfB_{t-\tau}\omega_{t-\tau}^k}}_{I_{127}}.
\end{align*}
For $I_{121}$, we have under conditional expectation,
\begin{align}
\label{eqn:121}
    &2\frac{\beta}{L}\bbE_{t-\tau}\iprod{x_t^k-x_{t-\tau}^k}{\bfP_t\big(\tilde b_{t,L}^k - \bar b_{L,\theta_t^k}^k\big)}\nonumber\\
    \leq&2\frac{\beta}{L}\bbE_{t-\tau}\|x_t^k-x_{t-\tau}^k\| \|\bfP_t\|\|\tilde b_{t,L}^k - \bar b_{L,\theta_t^k}^k\|\nonumber\\
    \leq&2\frac{\beta}{L}\bbE_{t-\tau}\Big(\|\bfB_t\omega_t^k-\bfB_{t-\tau}\omega_{t-\tau}^k\|
    +L_{*,1}\|\theta_t^k-\theta_{t-\tau}^k\|\Big)\|\tilde b_{t,L}^k - \bar b_{L,\theta_t^k}^k\|\nonumber\\
    \leq& \frac{1}{2}\bbE_{t-\tau}\|\bfB_t\omega_t^k-\bfB_{t-\tau}\omega_{t-\tau}^k\|^2
    +\frac{L_{*,1}^2}{2}\bbE_{t-\tau}\|\theta_t^k-\theta_{t-\tau}^k\|^2
    +  4\frac{\beta^2}{L^2}\bbE_{t-\tau}\|\tilde b_{t,L}^k - \bar b_{L,\theta_t^k}^k\|^2\nonumber\\
    \leq& \bbE_{t-\tau}\|\bfB_t\omega_t^k-\bfB_t\omega_{t-\tau}^k\|^2+\bbE_{t-\tau}\|\bfB_t\omega_{t-\tau}^k- \bfB_{t-\tau}\omega_{t-\tau}^k\|^2 +  4\frac{\beta^2}{L^2}\bbE_{t-\tau}\|\tilde b_{t,L}^k - \bar b_{L,\theta_t^k}^k\|^2\nonumber\\
    &\qquad+\frac{L_{*,1}^2}{2}\bbE_{t-\tau}\|\theta_t^k-\theta_{t-\tau}^k\|^2\nonumber\\
    \leq& \bbE_{t-\tau}\|\omega_t^k-\omega_{t-\tau}^k\|^2+U_\omega^2\bbE_{t-\tau}\|\bfB_t- \bfB_{t-\tau}\|^2
    +\frac{L_{*,1}^2}{2}\bbE_{t-\tau}\|\theta_t^k-\theta_{t-\tau}^k\|^2
    +  4\frac{\beta^2}{L^2}\bbE_{t-\tau}\|\tilde b_{t,L}^k - \bar b_{L,\theta_t^k}^k\|^2.
\end{align}

Similarly, for $I_{122}$, we have
\begin{align}
\label{eqn:122}
    &2\frac{\beta}{L} \bbE_{t-\tau}\iprod{x_{t-\tau}^k}{(\bfP_t-\bfP_{t-\tau})\big[\tilde b_{t,L}^k - \bar b_{L,\theta_t^k}^k \big]}\nonumber\\
    \leq&2\frac{\beta}{L} \bbE_{t-\tau}\|x_{t-\tau}^k\| \|\bfP_t-\bfP_{t-\tau}\|\|\tilde b_{t,L}^k - \bar b_{L,\theta_t^k}^k \| ~~~~ (\text{by Cauchy Schwartz})\nonumber\\
    \leq&8U_\omega \frac{\beta}{L}  \bbE_{t-\tau}\|\bfB_t-\bfB_{t-\tau}\|\|\tilde b_{t,L}^k - \bar b_{L,\theta_t^k}^k \|\nonumber\\
    \leq&  \frac{1}{2}\bbE_{t-\tau}\|\bfB_t-\bfB_{t-\tau}\|^2 +32U_\omega^2 \frac{\beta^2}{L^2}\bbE_{t-\tau}\|\tilde b_{t,L}^k - \bar b_{L,\theta_t^k}^k \|^2.
\end{align}
For $I_{124}$,
\begin{align}
\label{eqn:124}
    &2\frac{\beta}{L}\bbE_{t-\tau}\iprod{x_t^k-x_{t-\tau}^k}{\bfP_t (\tilde A_{t,L}^k-A_{L,\theta_t^k}^k)\bfB_t\omega_t^k}\nonumber\\
    \leq& 2U_\omega\frac{\beta}{L} \bbE_{t-\tau}\Big(\|\bfB_t\omega_t^k -\bfB_{t-\tau}\omega_{t-\tau}^k\|
    +L_{*,1}\|\theta_t^k-\theta_{t-\tau}^k\|\Big) \|\tilde A_{t,L}^k-A_{L,\theta_t^k}^k\|\nonumber\\
    \leq& \frac{1}{2}\bbE_{t-\tau}\|\bfB_t\omega_t^k -\bfB_{t-\tau}\omega_{t-\tau}^k\|^2
    +\frac{L_{*,1}^2}{2}\bbE_{t-\tau}\|\theta_t^k-\theta_{t-\tau}^k\|^2
    + 4U_\omega^2\frac{\beta^2}{L^2} \bbE_{t-\tau}\|\tilde A_{t,L}^k-A_{L,\theta_t^k}^k\|^2\nonumber\\
    \leq&  \bbE_{t-\tau}\|\omega_t^k-\omega_{t-\tau}^k\|^2+U_\omega^2\bbE_{t-\tau}\|\bfB_t- \bfB_{t-\tau}\|^2
    +\frac{L_{*,1}^2}{2}\bbE_{t-\tau}\|\theta_t^k-\theta_{t-\tau}^k\|^2
    + 4U_\omega^2\frac{\beta^2}{L^2} \bbE_{t-\tau}\|\tilde A_{t,L}^k-A_{L,\theta_t^k}^k\|^2 .
\end{align}
For $I_{125}$,
\begin{align}
\label{eqn:125}
    &2\frac{\beta}{L}\bbE_{t-\tau}\iprod{x_{t-\tau}^k}{(\bfP_t-\bfP_{t-\tau}) (\tilde A_{t,L}^k-A_{L,\theta_t^k}^k)\bfB_t\omega_t^k}\nonumber\\
    \leq& 8\frac{\beta}{L} U_\omega^2\bbE_{t-\tau}\|\bfB_t-\bfB_{t-\tau}\| \|\tilde A_{t,L}^k-A_{L,\theta_t^k}^k\|\nonumber \\
    \leq& \frac{1}{2}\bbE_{t-\tau}\|\bfB_t-\bfB_{t-\tau}\|^2+ 32U_\omega^4\frac{\beta^2}{L^2} \bbE_{t-\tau}\|\tilde A_{t,L}^k-A_{L,\theta_t^k}^k\|^2.
\end{align}
For $I_{126}$, we have,
\begin{align}
\label{eqn:126}
    &2\frac{\beta}{L}\bbE_{t-\tau}\iprod{x_{t-\tau}^k}{\bfP_{t-\tau} (\tilde A_{t,L}^k-A_{L,\theta_t^k}^k)(\bfB_t\omega_t^k-\bfB_{t-\tau}\omega_{t-\tau}^k)}\nonumber\\
    \leq & 4\frac{\beta}{L} U_\omega \bbE_{t-\tau}\|\tilde A_{t,L}^k-A_{L,\theta_t^k}^k\| \|\bfB_t\omega_t^k-\bfB_{t-\tau}\omega_{t-\tau}^k\|\nonumber\\
    \leq& 8\frac{\beta^2}{L^2} U_\omega^2 \bbE_{t-\tau}\|\tilde A_{t,L}^k-A_{L,\theta_t^k}^k\|^2 +\bbE_{t-\tau}\|\omega_t^k-\omega_{t-\tau}^k\|^2+U_\omega^2\bbE_{t-\tau}\|\bfB_t- \bfB_{t-\tau}\|^2
\end{align}

Summing up the upper bounds of $I_{121}$(\eqref{eqn:121})$ ,I_{122} $(\eqref{eqn:122})$, I_{124}$ (\eqref{eqn:124}), $I_{125}$ (\eqref{eqn:125}) and $I_{126}$ (\eqref{eqn:126}), conditioning on $v_{0:t-\tau}$, we get,
\begin{align}
\label{eqn: sum of I121 to I126}
    &\bbE_{t-\tau}\qth{I_{121}+I_{122}+I_{124}+I_{125}+I_{126}}\nonumber\\
    \leq& 3\bbE_{t-\tau}\|\omega_t^k-\omega_{t-\tau}^k\|^2+ \pth{1+3U_\omega^2}\bbE_{t-\tau}\|\bfB_t- \bfB_{t-\tau}\|^2\nonumber\\
    &+L_{*,1}^2\bbE_{t-\tau}\|\theta_t^k-\theta_{t-\tau}^k\|^2
    +\pth{12U_\omega^2+32U_\omega^4}\frac{\beta^2}{L^2} \bbE_{t-\tau} \|\tilde A_{t,L}^k-A_{L,\theta_t^k}^k\|^2\nonumber\\
    &+\pth{4+32U_\omega^2}\frac{\beta^2}{L^2}\bbE_{t-\tau}\|\tilde b_{t,L}^k - \bar b_{L,\theta_t^k}^k \|^2\nonumber\\
    \leq&  3\tau \sum_{j=t-\tau}^{t-1}\bbE_{t-\tau}\|\omega_{j+1}^k-\omega_{j}^k\|^2+ \pth{1+3U_\omega^2}\tau \sum_{j=t-\tau}^{t-1}\bbE_{t-\tau}\|\bfB_{j+1}- \bfB_{j}\|^2\nonumber\\
    &+L_{*,1}^2\tau \sum_{j=t-\tau}^{t-1}\bbE_{t-\tau}\|\theta_{j+1}^k-\theta_j^k\|^2
    +\qth{48\pth{12U_\omega^2+32U_\omega^4}+12U_r^2\pth{4+32U_\omega^2}}
    \frac{\beta^2}{L^2(1-\gamma)^2}.
\end{align}
where the last inequality uses Lemma \ref{lmm: U delta} and the tower property.

For $I_{123}$, taking conditional expectation and applying Cauchy--Schwarz, we have
\begin{align*}
&2\frac{\beta}{L} \bbE_{t-\tau}\iprod{x_{t-\tau}^k}{\bfP_{t-\tau}\big[\tilde b_{t,L}^k - \bar b_{L,\theta_t^k}^k \big]}
 = 2\frac{\beta}{L} \iprod{x_{t-\tau}^k}{\bfP_{t-\tau}\bbE_{t-\tau}\big[\tilde b_{t,L}^k - \bar b_{L,\theta_t^k}^k \big]} \\
& \le 2\frac{\beta}{L} \|x_{t-\tau}^k\| \|\bfP_{t-\tau}\| \|\bbE_{t-\tau}\big[\tilde b_{t,L}^k - \bar b_{L,\theta_t^k}^k \big]\| \\
& \le 4U_{\omega}\frac{\beta}{L}
\|\bbE_{t-\tau}\big[\tilde b_{t,L}^k - \bar b_{L,\theta_t^k}^k \big]\|.
\end{align*}

For $I_{127}$, similarly,
\begin{align*}
    &2\frac{\beta}{L}\bbE_{t-\tau}\iprod{x_{t-\tau}^k}{\bfP_{t-\tau} (\tilde A_{t,L}^k-A_{L,\theta_t^k}^k)\bfB_{t-\tau}\omega_{t-\tau}^k}\nonumber\\
    &=2\frac{\beta}{L}\iprod{x_{t-\tau}^k}{\bfP_{t-\tau}\bbE_{t-\tau}[\tilde A_{t,L}^k-A_{L,\theta_t^k}^k]\bfB_{t-\tau}\omega_{t-\tau}^k}\nonumber\\
    &\leq 4U_\omega^2\frac{\beta}{L}
    \|\bbE_{t-\tau}[\tilde A_{t,L}^k-A_{L,\theta_t^k}^k]\|.
\end{align*} 

Then, combining the two displays, we have
\begin{align}
\label{eqn: I123+I127}
    \bbE_{t-\tau}[I_{123}+I_{127}]
    \leq&4U_\omega\frac{\beta}{L}
    \|\bbE_{t-\tau}[\tilde b_{t,L}^k-\bar b_{L,\theta_t^k}^k]\|
    +4U_\omega^2\frac{\beta}{L}
    \|\bbE_{t-\tau}[\tilde A_{t,L}^k-A_{L,\theta_t^k}^k]\|.
\end{align}

Putting \eqref{eqn: I123+I127} and \eqref{eqn: sum of I121 to I126} together, we get
\begin{align}
\label{eqn: second term in first ip final}
    &\bbE_{t-\tau}[I_{12}] \leq 3\tau \sum_{j=t-\tau}^{t-1}\bbE_{t-\tau}\|\omega_{j+1}^k-\omega_{j}^k\|^2+ \pth{1+3U_\omega^2}\tau \sum_{j=t-\tau}^{t-1}\bbE_{t-\tau}\|\bfB_{j+1}- \bfB_{j}\|^2\nonumber\\
    &+L_{*,1}^2\tau \sum_{j=t-\tau}^{t-1}\bbE_{t-\tau}\|\theta_{j+1}^k-\theta_j^k\|^2\nonumber\\
    &+\qth{48\pth{12U_\omega^2+32U_\omega^4}+12U_r^2\pth{4+32U_\omega^2}}
    \frac{\beta^2}{L^2(1-\gamma)^2}\nonumber\\
    &+4U_\omega\frac{\beta}{L}
    \|\bbE_{t-\tau}[\tilde b_{t,L}^k-\bar b_{L,\theta_t^k}^k]\|
    +4U_\omega^2\frac{\beta}{L}
    \|\bbE_{t-\tau}[\tilde A_{t,L}^k-A_{L,\theta_t^k}^k]\|.
\end{align}
Invoking Assumption \ref{ass: per-agent full exploration}, putting \eqref{eqn: first term in first ip final}, and \eqref{eqn: second term in first ip final},  back to \eqref{eq: roughly drift of upper bound}, and taking conditional expectation on $v_{0: t-\tau}$, we get 
\begin{align}
&\bbE_{t-\tau}[I_1] = 2\frac{\beta}{L}\bbE_{t-\tau}\langle x_{t}^k, \bfB_t\bfB_t^\top\phi(s_{t,0}^k)\delta_{t,L}^k\rangle \nonumber\\
\leq&  -2\lambda\beta\bbE_{t-\tau}\|x_t^k\|^2+12U_\omega\frac{\beta}{L}  \bbE_{t-\tau}\|m_t\|\|x_t^k\|\nonumber \\
&  +3\tau \sum_{j=t-\tau}^{t-1}\bbE_{t-\tau}\|\omega_{j+1}^k-\omega_{j}^k\|^2+ \pth{1+3U_\omega^2}\tau \sum_{j=t-\tau}^{t-1}\bbE_{t-\tau}\|\bfB_{j+1}- \bfB_{j}\|^2\nonumber\\
    &+L_{*,1}^2\tau \sum_{j=t-\tau}^{t-1}\bbE_{t-\tau}\|\theta_{j+1}^k-\theta_j^k\|^2\nonumber\\
    &+\qth{48\pth{12U_\omega^2+32U_\omega^4}+12U_r^2\pth{4+32U_\omega^2}}
    \frac{\beta^2}{L^2(1-\gamma)^2}\nonumber\\
    &+4U_\omega\frac{\beta}{L}
    \|\bbE_{t-\tau}[\tilde b_{t,L}^k-\bar b_{L,\theta_t^k}^k]\|
    +4U_\omega^2\frac{\beta}{L}
    \|\bbE_{t-\tau}[\tilde A_{t,L}^k-A_{L,\theta_t^k}^k]\|.
\end{align}
Apply Young's inequality, for arbitrary $c_2>0$, we have
\begin{align*}
    &12U_\omega\frac{\beta}{L}  \bbE_{t-\tau}\|m_t\|\|x_t^k\|\leq c_2^2\frac{\beta}{L} \bbE_{t-\tau}\|x_t^k\|^2 + \frac{36U_\omega^2}{c_2^2}\frac{\beta}{L} \bbE_{t-\tau}\|m_t\|^2,
\end{align*}
we get
\begin{align}
\label{eq: bound of 1st inner product: upper bound: local head}
    &\bbE_{t-\tau}[I_1]\leq-2\lambda\beta\bbE_{t-\tau}\|x_t^k\|^2+c_2^2\frac{\beta}{L} \bbE_{t-\tau}\|x_t^k\|^2 + \frac{36U_\omega^2}{c_2^2}\frac{\beta}{L} \bbE_{t-\tau}\|m_t\|^2\nonumber \\
    &  + 3 \tau \sum_{j=t-\tau}^{t-1}\bbE_{t-\tau}\|\omega_{j+1}^k-\omega_{j}^k\|^2+ \pth{1+3U_\omega^2}\tau \sum_{j=t-\tau}^{t-1}\bbE_{t-\tau}\|\bfB_{j+1}- \bfB_{j}\|^2\nonumber\\
    &+L_{*,1}^2\tau \sum_{j=t-\tau}^{t-1}\bbE_{t-\tau}\|\theta_{j+1}^k-\theta_j^k\|^2\nonumber\\
    &+\qth{48\pth{12U_\omega^2+32U_\omega^4}+12U_r^2\pth{4+32U_\omega^2}}
    \frac{\beta^2}{L^2(1-\gamma)^2}\nonumber\\
    &+ 4U_\omega\frac{\beta}{L}
    \|\bbE_{t-\tau}[\tilde b_{t,L}^k-\bar b_{L,\theta_t^k}^k]\|
    +4U_\omega^2\frac{\beta}{L}
    \|\bbE_{t-\tau}[\tilde A_{t,L}^k-A_{L,\theta_t^k}^k]\|, 
\end{align}
where the last two terms (i.e., $\|\bbE_{t-\tau}[\tilde b_{t,L}^k-\bar b_{L,\theta_t^k}^k]\|$ and $\|\bbE_{t-\tau}[\tilde A_{t,L}^k-A_{L,\theta_t^k}^k]\|$) can be bounded by the fine-grained bounds given by Lemma \ref{lmm: MC observable mixing}. 

\underline{Bounding $I_2$}. By Proposition \ref{prop: key intermediate}, $I_2$ can be written as,
\begin{align}
\label{eq: local head: upper bound: x tilde: 2nd inner produc}
    &\frac{2\zeta}{KL}\langle x_t^k, \sum_{i=1}^K \bfP_{t,\perp}\phi(s_{t,0}^i)\delta_{t,L}^i(\omega_t^i)^\top\omega_t^k\rangle \nonumber\\
    = &\frac{2\zeta}{KL}\langle x_t^k, \sum_{i=1}^K  \bfP_{t,\perp}\tilde A_{t,L}^ix_t^i(\omega_t^i)^\top\omega_t^k\rangle+\frac{2\zeta}{KL}\langle x_t^k, \sum_{i=1}^K  \bfP_{t,\perp}\bfb_{t,L}^i(\omega_t^i)^\top\omega_t^k\rangle\nonumber\\
    \leq &\frac{4 U_\omega^2\zeta}{KL}\sum_{i=1}^K \|x_t^k\| \|x_t^i\|+\underbrace{\frac{2\zeta}{KL}\langle x_t^k, \sum_{i=1}^K  \bfP_{t,\perp}\bfb_{t,L}^i(\omega_t^i)^\top\omega_t^k\rangle}_{I_{21}}.
\end{align}
For $I_{21}$, similar to the derivations of the upper bounds of terms from $I_{121}$ to $I_{126}$, we have 
\begin{align}
\label{eqn: I_{21}}
    &\frac{2\zeta}{KL}\sum_{i=1}^K \langle x_t^k,  \bfP_{t,\perp}\bfb_{t,L}^i(\omega_t^i)^\top\omega_t^k\rangle\nonumber\\
    =&\frac{2\zeta}{KL}\sum_{i=1}^K\langle x_t^k-x_{t-\tau}^k,   \bfP_{t,\perp}\bfb_{t,L}^i(\omega_t^i)^\top\omega_t^k\rangle
    +\frac{2\zeta}{KL} \sum_{i=1}^K\langle x_{t-\tau}^k,  (\bfP_{t,\perp}-\bfP_{t-\tau,\perp})\bfb_{t,L}^i(\omega_t^i)^\top\omega_t^k\rangle\nonumber\\
    &+\frac{2\zeta}{KL} \sum_{i=1}^K\langle x_{t-\tau}^k,  \bfP_{t-\tau,\perp}\bfb_{t,L}^i(\omega_t^i-\omega_{t-\tau}^i)^\top\omega_t^k\rangle 
    +\frac{2\zeta}{KL} \sum_{i=1}^K\langle x_{t-\tau}^k,  \bfP_{t-\tau,\perp}\bfb_{t,L}^i(\omega_{t-\tau}^i)^\top(\omega_t^k-\omega_{t-\tau}^k)\rangle\nonumber\\
    &+\frac{2\zeta}{KL} \sum_{i=1}^K\langle x_{t-\tau}^k,  \bfP_{t-\tau,\perp}\bfb_{t,L}^i(\omega_{t-\tau}^i)^\top\omega_{t-\tau}^k\rangle\nonumber\\
    \leq  
    &\frac{3}{2}\|\omega_t^k-\omega_{t-\tau}^k\|^2  +\frac{3}{2}\|\bfB_{t}-\bfB_{t-\tau}\|^2 + L_{*,1}^2\|\theta_t^k-\theta_{t-\tau}^k\|^2
    + \pth{18U_\omega^4+32U_\omega^6}\frac{\zeta^2}{KL^2}\sum_{i=1}^K\|\bfb_{t,L}^i\|^2\nonumber\\
    &
    + \frac{1}{2}\frac{1}{K} \sum_{i=1}^K \|\omega_t^i-\omega_{t-\tau}^i\|^2  +\frac{2\zeta}{KL} \sum_{i=1}^K\langle x_{t-\tau}^k,  \bfP_{t-\tau,\perp}\bfb_{t,L}^i(\omega_{t-\tau}^i)^\top\omega_{t-\tau}^k\rangle.
\end{align}
For the last term in Eq.\,(\ref{eqn: I_{21}}), by Cauchy--Schwarz inequality, we have
\begin{align}
\label{eqn: I21 Markovian}
&\frac{2\zeta}{KL} \sum_{i=1}^K\bbE_{t-\tau}\langle x_{t-\tau}^k,  \bfP_{t-\tau,\perp}\bfb_{t,L}^i(\omega_{t-\tau}^i)^\top\omega_{t-\tau}^k\rangle
 \le \frac{2\zeta}{KL} \sum_{i=1}^K 2U_{\omega}^3 \|\bbE_{t-\tau}[\bfb_{t,L}^i]\|\nonumber\\ 
&= \frac{4U_{\omega}^3\zeta}{KL}\sum_{i=1}^K
\|\bbE_{t-\tau}[\bfb_{t,L}^i]\|.
\end{align}
Then, putting \eqref{eqn: I21 Markovian} back to \eqref{eqn: I_{21}} and then putting the result back to \eqref{eq: local head: upper bound: x tilde: 2nd inner produc}, and using Lemma \ref{lmm: U delta} with the tower property, we have
\begin{align}
    &\bbE_{t-\tau}[I_2] \leq 4 U_\omega^2\frac{\zeta}{KL}\sum_{i=1}^K \bbE_{t-\tau}\|x_t^k\| \|x_t^i\|+ \frac{3}{2}\bbE_{t-\tau}\|\omega_t^k-\omega_{t-\tau}^k\|^2+\frac{3}{2}\bbE_{t-\tau}\|\bfB_{t}-\bfB_{t-\tau}\|^2 \nonumber\\
    &  + L_{*,1}^2\bbE_{t-\tau}\|\theta_t^k-\theta_{t-\tau}^k\|^2
    +3\pth{18U_\omega^4+32U_\omega^6}\frac{U_\delta^2\zeta^2}{L^2(1-\gamma)^2}
    + \frac{1}{2}\frac{1}{K} \sum_{i=1}^K \bbE_{t-\tau}\|\omega_t^i-\omega_{t-\tau}^i\|^2  \nonumber\\
    &+\frac{4U_{\omega}^3\zeta}{KL}\sum_{i=1}^K
    \|\bbE_{t-\tau}[\bfb_{t,L}^i]\|.
\end{align}
Applying Young's inequality, for arbitrary $c_3>0$, we have
\begin{align*}
    &4 U_\omega^2\frac{\zeta}{KL}\sum_{i=1}^K \bbE_{t-\tau}\|x_t^k\| \|x_t^i\|
    \leq c_3^2\frac{\zeta}{L}\bbE_{t-\tau}\|x_t^k\|^2 + \frac{4U_\omega^4}{c_3^2}\frac{\zeta}{KL}\sum_{i=1}^K  \bbE_{t-\tau}\|x_t^i\|^2.
\end{align*}
We get 
\begin{align}
    \label{eqn: second term in x tilde upper bound}
    &\bbE_{t-\tau}[I_2] \leq c_3^2\frac{\zeta}{L}\bbE_{t-\tau}\|x_t^k\|^2 + \frac{4U_\omega^4}{c_3^2}\frac{\zeta}{KL}\sum_{i=1}^K  \bbE_{t-\tau}\|x_t^i\|^2
    + \frac{3}{2}\bbE_{t-\tau}\|\omega_t^k-\omega_{t-\tau}^k\|^2 \nonumber\\
    & +\frac{3}{2}\bbE_{t-\tau}\|\bfB_{t}-\bfB_{t-\tau}\|^2 + L_{*,1}^2\bbE_{t-\tau}\|\theta_t^k-\theta_{t-\tau}^k\|^2
    +3\pth{18U_\omega^4+32U_\omega^6}\frac{U_\delta^2\zeta^2}{L^2(1-\gamma)^2} \nonumber\\
    &
    + \frac{1}{2}\frac{1}{K} \sum_{i=1}^K \bbE_{t-\tau}\|\omega_t^i-\omega_{t-\tau}^i\|^2
    \nonumber\\
    &+\frac{4U_{\omega}^3\zeta}{KL}\sum_{i=1}^K
    \|\bbE_{t-\tau}[\bfb_{t,L}^i]\|, 
\end{align}
where $ \|\bbE_{t-\tau}[\bfb_{t,L}^i]\|$ can be controlled using the fine-grained mixing bound established in Lemma \ref{lmm: MC observable mixing}.

\underline{For $I_3$}, by Proposition \ref{prop: key intermediate} and Lemma \ref{lmm: U delta}, we have 
\begin{align*}
    I_3
    =&2\frac{\zeta \beta}{L^2}
    \Bigg(
    \iprod{x_t^k}{\pth{\frac{1}{K}\sum_{i=1}^K \bfP_{t,\perp}\delta_{t,L}^i\phi(s_{t,0}^i)(\omega_t^i)^\top} \pth{\bfB_t^{\top}\tilde A_{t,L}^k x_t^k}}\\
    &\qquad
    +\iprod{x_t^k}{\pth{\frac{1}{K}\sum_{i=1}^K \bfP_{t,\perp}\delta_{t,L}^i\phi(s_{t,0}^i)(\omega_t^i)^\top} \pth{\bfB_t^{\top}\bfb_{t,L}^k}}
    \Bigg)\\
    \le & 2\frac{\zeta \beta}{L^2}\Bigg( 2\frac{U_{\delta}U_\omega}{1-\gamma} \|x_t^k\|^2 
    + 2 U_{\omega}^2 \frac{U_{\delta}^2}{(1-\gamma)^2} \Bigg).   
\end{align*}

Thus, we have  
\begin{align}
\label{eqn: third term in x tilde upper bound} 
\bbE_{t-\tau}[I_3] \le \frac{2\zeta \beta}{L^2} \pth{ \frac{2U_{\omega}^2 U_{\delta}^2}{(1-\gamma)^2} +  \frac{2 U_{\omega}U_{\delta}}{1-\gamma} \|x_t^k\|^2}.   
\end{align}

\underline{For $I_4$}, we have, 
\begin{align}
\label{eqn: I_4 before exp}
\|\bfB_t(\delta_{t,L}^k\bfB_t^\top\phi(s_{t,0}^k))\|^2  
=&\|\bfB_t\bfB_t^\top(\tilde A_{t,L}^kx_t^k+\bfb_{t,L}^k)\|^2
~~~~\text{(by Proposition \ref{prop: key intermediate})}\nonumber\\
\leq&2\|\bfB_t\bfB_t^\top(\tilde A_{t,L}^kx_t^k)\|^2
+2\|\bfB_t\bfB_t^\top\bfb_{t,L}^k\|^2\nonumber\\
\leq&8\|x_t^k\|^2+2\|\bfb_{t,L}^k\|^2,
\end{align}
where the last inequality uses $\|\bfB_t\bfB_t^\top\|\le 1$ and
$\|\tilde A_{t,L}^k\|\le 2$ by Assumption \ref{ass: bounded features}.
From Lemma \ref{lmm: U delta} and the tower property, we know
\begin{align}
\label{eqn: tower property for bfb}
    \bbE_{t-\tau}\|\bfb_{t,L}^k\|^2
    \leq \frac{ U_\delta^2}{(1-\gamma)^2}.
\end{align}
Then, taking conditional expectation on \eqref{eqn: I_4 before exp}, we get
\begin{align}
\label{I4 final}
    &\bbE_{t-\tau}\|\bfB_t(\delta_{t,L}^k\bfB_t^\top\phi(s_{t,0}^k))\|^2\nonumber\\
    &\leq 8\bbE_{t-\tau}\|x_t^k\|^2+\frac{2 U_\delta^2}{(1-\gamma)^2}.
\end{align}

Therefore, we have
\begin{align}
\label{eqn: fourth term in x tilde upper bound}
    &\bbE_{t-\tau}[I_{4}] \leq 32\frac{\beta^2}{L^2} \bbE_{t-\tau}\|x_t^k\|^2+\frac{8 U_\delta^2}{(1-\gamma)^2}\frac{\beta^2}{L^2}.
\end{align}

\underline{For $I_5$}, by Proposition \ref{prop: key intermediate}, under conditional expectation, we have
\begin{align*}
&\bbE_{t-\tau}\|\frac{1}{K}\sum_{i=1}^K  \bfP_{t,\perp}\delta_{t,L}^i\phi(s_{t,0}^i)(\omega_t^i)^\top\omega_t^k\|^2\\
=&\bbE_{t-\tau}\|\frac{1}{K}\sum_{i=1}^K
\bfP_{t,\perp}(\tilde A_{t,L}^ix_t^i+\bfb_{t,L}^i)
(\omega_t^i)^\top\omega_t^k\|^2\nonumber\\
\leq&2\bbE_{t-\tau}\|\frac{1}{K}\sum_{i=1}^K
\bfP_{t,\perp}\tilde A_{t,L}^ix_t^i(\omega_t^i)^\top\omega_t^k\|^2
+2\bbE_{t-\tau}\|\frac{1}{K}\sum_{i=1}^K
\bfP_{t,\perp}\bfb_{t,L}^i(\omega_t^i)^\top\omega_t^k\|^2\nonumber\\
\leq& 8 U_\omega^4 \frac{1}{K}\sum_{i=1}^K\bbE_{t-\tau}\|x_t^i\|^2
+2U_\omega^4\frac{1}{K}\sum_{i=1}^K \bbE_{t-\tau}\| \bfb_{t,L}^i\|^2 \nonumber\\
\leq& 8 U_\omega^4 \frac{1}{K}\sum_{i=1}^K\bbE_{t-\tau}\|x_t^i\|^2
+ \frac{6 U_\omega^4 U_\delta^2}{(1-\gamma)^2}.
\end{align*}
The second inequality uses Jensen's inequality, $\|\bfP_{t,\perp}\|\le 1$,
$\|\tilde A_{t,L}^i\|\le 2$, and $\|\omega_t^i\|,\|\omega_t^k\|\le U_\omega$.
The last inequality uses \eqref{eqn: tower property for bfb}.
Then, we have
\begin{align}
\label{eqn: fifth term in x tilde upper bound}
    &\bbE_{t-\tau}[I_5] \leq \frac{32\zeta^2 U_\omega^4}{L^2}\frac{1}{K}\sum_{i=1}^K\mathbb{E}_{t-\tau}\|x_t^i\|^2 + \frac{24\zeta^2 U_\omega^4 U_\delta^2}{L^2(1-\gamma)^2}.
\end{align}

\underline{For $I_6$}, we have,
\begin{align}
\label{eqn: sixth term in x tilde upper bound}
I_6
&=4\frac{\beta^2\zeta^2}{L^4}
\left\|
\left(\frac{1}{K}\sum_{i=1}^K
\bfP_{t,\perp}\delta_{t,L}^i\phi(s_{t,0}^i)(\omega_t^i)^\top\right)
\left(\delta_{t,L}^k\bfB_t^{\top}\phi(s_{t,0}^k)\right)
\right\|^2\nonumber\\
&\leq 4\frac{\beta^2\zeta^2}{L^4}
\left(\frac{U_\delta U_\omega}{1-\gamma}\right)^2
\left(\frac{U_\delta}{1-\gamma}\right)^2\nonumber\\
&=4\frac{U_{\delta}^4U_{\omega}^2}{L^4(1-\gamma)^4}\beta^2\zeta^2.
\end{align}

{ 
\underline{For $I_7$}, write
\begin{align}
\label{eqn: actor gradient stochastic}
g_{t,L}^k&:=\frac{1}{L}\sum_{\ell=0}^{L-1}\delta_{t,\ell}^{k,\mathrm{act}}\nabla_\theta\log\pi_{\theta_t^k}(\hat a_{t,\ell}^k|\hat s_{t,\ell}^k)\nonumber\\
&=\frac{1}{L}\sum_{\ell=0}^{L-1}\pth{\hat r_{t,\ell}^k+(\gamma\phi(\hat s_{t,\ell+1}^k)-\phi(\hat s_{t,\ell}^k))^\top\mathbf{B}_{t}\omega_{t}^k}\nabla\log\pi_{\theta_{t}^k}(\hat a_{t,\ell}^k|\hat s_{t,\ell}^k).
\end{align}

For ease of exposition, we let 
\begin{align}
\label{eqn: sample O notation: actor}
g(\hat O_{t,\ell}, \bfB, \omega, \theta) := \pth{\hat r_{t,\ell}+(\gamma\phi(\hat s_{t,\ell+1})-\phi(\hat s_{t,\ell}))^\top\mathbf{B}\omega} \nabla\log\pi_{\theta}(\hat a_{t,\ell}|\hat s_{t,\ell}),
\end{align}
where $\hat{O}_{t,\ell} = \pth{\hat s_{t,\ell}, \hat a_{t,\ell}, \hat s_{t,\ell+1}}$. 
Then,
\begin{align}
\label{eqn: def of g}
    g_{t,L}^k = \frac{1}{L}\sum_{\ell=0}^{L-1}g(\hat O_{t,\ell}^k, \bfB_t, \omega_t^k, \theta_t^k).
\end{align}

We have,

\begin{align}
\label{eqn: I_7 decomp to first and second order}
    &2\langle\bfB_t\omega_{t}^k- z_t^{k,*},z_{t}^{k,*}- z_{t+1}^{k,*}\rangle\nonumber\\
    =&2\langle\bfB_t\omega_{t}^k- z_t^{k,*},\bfB^* (\omega_{t}^{k,*}- \omega_{t+1}^{k,*})\rangle\nonumber\\
    =&\underbrace{2\iprod{\bfB_t\omega_{t}^k- z_t^{k,*}}{\bfB^* \pth{\omega_{t}^{k,*}- \omega_{t+1}^{k,*} + (\nabla_{\theta}\omega_{t}^{k,*})^\top(\theta_{t+1}^k-\theta_{t}^k)}}}_{\text{second order remainder}}\nonumber \\
    & 
    - 2\langle\bfB_t\omega_{t}^k- z_t^{k,*},\bfB^* (\nabla_{\theta}\omega_{t}^{k,*})^\top(\theta_{t+1}^k-\theta_{t}^k)\rangle\nonumber\\
    \overset{(a)}{\leq}& 2U_\omega L_{s,1}\|\theta_{t+1}^k-\theta_{t}^k\|^2 - 2\alpha \langle x_t^k,\bfB^* (\nabla_{\theta}\omega_{t}^{k,*})^\top g_{t,L}^k\rangle,
\end{align}
where (a) holds by invoking Lemma \ref{lmm: Ls lip} (i.e., $\omega^{k,*}(\theta)$ is smooth). Taking conditional expectation $\bbE_{t-\tau}$ on both sides and then applying $-\bbE[x]\leq\abth{\bbE[x]}$ to the signed second term yields
\begin{align}
\label{eqn: I_7 decomp after expectation}
    \bbE_{t-\tau}[I_7]\leq\; 2U_\omega L_{s,1}\bbE_{t-\tau}\|\theta_{t+1}^k-\theta_{t}^k\|^2 + 2\alpha \abth{\bbE_{t-\tau}\iprod{x_t^k}{\bfB^* (\nabla_{\theta}\omega_{t}^{k,*})^\top g_{t,L}^k}}.
\end{align}

We now bound the second term in \eqref{eqn: I_7 decomp after expectation} by connecting the minibatch actor direction to the true policy gradient $\nabla J^k(\theta_t^k)$:
\begin{align}
\label{eqn: first order perturb of the policy}
    &\abth{\bbE_{t-\tau}\langle x_t^k,\bfB^* (\nabla_{\theta}\omega_{t}^{k,*})^\top g_{t,L}^k\rangle}.
\end{align}

On the other hand, the true policy gradient can be written by the one-step policy gradient identity as
\begin{align}
\label{eqn: one step policy gradient formula in I7}
    \nabla J^k(\theta_t^k)
    =
    \frac{1}{1-\gamma}\bbE_{\nu_{\theta_t^k}^k,\pi_{\theta_t^k},P^k}
    \qth{
    \pth{
    R(\hat s, \hat a)+(\gamma\phi(\hat s')-\phi(\hat s))^\top\mathbf{B}^*\omega_t^{k,*}
    }
    \nabla_\theta\log\pi_{\theta_t^k}(\hat a|\hat s)
    }.
\end{align}
Adding and subtracting $(1-\gamma)\nabla J^k(\theta_t^k)$ inside the bracket of \eqref{eqn: first order perturb of the policy}, by Lemma \ref{lmm: L* lip}, the triangle inequality, and Cauchy--Schwarz, we obtain
\begin{align*}
    \eqref{eqn: first order perturb of the policy}
    \leq& \abth{\bbE_{t-\tau}\iprod{x_t^k}{\bfB^* (\nabla_{\theta}\omega_{t}^{k,*})^\top \pth{g_{t,L}^k-(1-\gamma)\nabla J^k(\theta_t^k)}}}\\
    &\quad +(1-\gamma) L_{*,1}\bbE_{t-\tau}\lnorm{x_t^k}{} \lnorm{ \nabla J^k(\theta_t^k)}{} \quad\text{(Cauchy--Schwarz)}.
\end{align*}

To bound the first term, we telescope the inner product inside the absolute value into four differences $D_1, D_2, D_3, D_4$ by inserting the true TD fixed point, a reference copy at time $t-\tau$ on the original Markov chain, and a reference copy at time $t-\tau$ on the auxiliary chain:
\begin{align}
\label{eqn: I_7 three-diff decomp}
    &\bbE_{t-\tau}\iprod{x_t^k}{\bfB^* (\nabla_{\theta}\omega_{t}^{k,*})^\top \pth{g_{t,L}^k-(1-\gamma)\nabla J^k(\theta_t^k)}}\nonumber\\
    =& 
    \bbE_{t-\tau}\iprod{x_t^k}{\bfB^* (\nabla_{\theta}\omega_{t}^{k,*})^\top \pth{g_{t,L}^k-(1-\gamma)\nabla J^k(\theta_t^k)}}\nonumber\\
    &-\bbE_{t-\tau}\iprod{x_t^k}{\bfB^* (\nabla_{\theta}\omega_{t}^{k,*})^\top \pth{\frac{1}{L}\sum_{\ell=0}^{L-1}g(\hat O_{t,\ell}^k, \bfB^*, \omega_t^{k,*}, \theta_t^k)-(1-\gamma)\nabla J^k(\theta_t^k)}}\nonumber\\
    &+\bbE_{t-\tau}\iprod{x_t^k}{\bfB^* (\nabla_{\theta}\omega_{t}^{k,*})^\top \pth{\frac{1}{L}\sum_{\ell=0}^{L-1}g(\hat O_{t,\ell}^k, \bfB^*, \omega_t^{k,*}, \theta_t^k)-(1-\gamma)\nabla J^k(\theta_t^k)}}\nonumber\\
    &-\bbE_{t-\tau}\iprod{x_{t-\tau}^k}{\bfB^* (\nabla_{\theta}\omega_{t-\tau}^{k,*})^\top \qth{\frac{1}{L}\sum_{\ell=0}^{L-1}g(\hat O_{t,\ell}^k, \bfB^*, \omega_{t-\tau}^{k,*}, \theta_{t-\tau}^k)-(1-\gamma)\nabla J^k(\theta_{t-\tau}^k)}}\nonumber\\
    &+\bbE_{t-\tau}\iprod{x_{t-\tau}^k}{\bfB^* (\nabla_{\theta}\omega_{t-\tau}^{k,*})^\top \qth{\frac{1}{L}\sum_{\ell=0}^{L-1}g(\hat O_{t,\ell}^k, \bfB^*, \omega_{t-\tau}^{k,*}, \theta_{t-\tau}^k)-(1-\gamma)\nabla J^k(\theta_{t-\tau}^k)}}\nonumber\\
    &-\bbE_{t-\tau}\iprod{x_{t-\tau}^k}{\bfB^* (\nabla_{\theta}\omega_{t-\tau}^{k,*})^\top \qth{\frac{1}{L}\sum_{\ell=0}^{L-1}g(\bar O_{t,\ell}^k, \bfB^*, \omega_{t-\tau}^{k,*}, \theta_{t-\tau}^k)-(1-\gamma)\nabla J^k(\theta_{t-\tau}^k)}}\nonumber\\
    &+\bbE_{t-\tau}\iprod{x_{t-\tau}^k}{\bfB^* (\nabla_{\theta}\omega_{t-\tau}^{k,*})^\top \qth{\frac{1}{L}\sum_{\ell=0}^{L-1}g(\bar O_{t,\ell}^k, \bfB^*, \omega_{t-\tau}^{k,*}, \theta_{t-\tau}^k)-(1-\gamma)\nabla J^k(\theta_{t-\tau}^k)}}, 
\end{align}
where $\bar O_{t, \ell}^k = \pth{\bar s_{t, \ell}, \bar a_{t, \ell}, \bar s_{t, \ell+1}}$. 
The four differences are:
\begin{itemize}
    \item $D_1 := $ (row 1) $-$ (row 2): captures the \emph{critic approximation error} relative to the true TD fixed point at time $t$ under the actor sampling $\hat O$;
    \item $D_2 := $ (row 3) $-$ (row 4): captures the \emph{parameter drift} between $t-\tau$ and $t$ on the original actor Markov chain;
    \item $D_3 := $ (row 5) $-$ (row 6): captures the \emph{original vs.\ auxiliary actor chain mismatch} at the true TD fixed point and parameters frozen at $t-\tau$;
    \item $D_4 := $ (row 7): captures the policy-gradient Markovian sampling error under the \emph{auxiliary actor chain} at the true TD fixed point.
\end{itemize}

We now bound $D_1, D_2, D_3, D_4$ in turn.

\noindent\underline{Bounding $\abth{D_1}$.} By Cauchy--Schwarz and the definition $x_t^k=\bfB_t\omega_t^k-\bfB^*\omega_t^{k,*}$, we have
\begin{align*}
    \abth{D_1}
    \leq& L_{*,1}\bbE_{t-\tau}\lnorm{x_t^k}{}\lnorm{\frac{1}{L}\sum_{\ell=0}^{L-1}\pth{g(\hat O_{t,\ell}^k,\bfB_t,\omega_t^k,\theta_t^k)-g(\hat O_{t,\ell}^k,\bfB^*,\omega_t^{k,*},\theta_t^k)}}{}\\
    \leq& 2B L_{*,1}\bbE_{t-\tau}\lnorm{x_t^k}{}^2.
\end{align*}
That is,
\begin{equation}
\label{eqn: I_7 D_1 bound}
\abth{D_1} \leq 2B L_{*,1}\bbE_{t-\tau}\lnorm{x_t^k}{}^2.
\end{equation}

\noindent\underline{Bounding $\abth{D_2}$.} Applying Cauchy--Schwarz together with Lemma \ref{lmm: LJ' lip}, Lemma \ref{lmm: L* lip}, and Assumption \ref{assmp: Lipschitz of policy} (i.e., the Lipschitz continuity of $\nabla J^k(\theta)$, $\nabla \log \pi_{\theta}(a|s)$, and $\omega^{k,*}(\theta)$), and the perturbations $\|\mathbf{B}_{t}-\mathbf{B}_{t-\tau}\|$, $\|x_t^k-x_{t-\tau}^k\|$, $\|\omega_t^k-\omega_{t-\tau}^k\|$, and $\|\omega_t^{k,*}-\omega_{t-\tau}^{k,*}\|$ that are established in Lemma \ref{lmm: crude parameter bound} and Lemma \ref{lmm: L* lip}, we have
\begin{align*}
    \abth{D_2}
    \leq&
    2BU_\delta L_{*,1}\,
    \bbE_{t-\tau}\|x_t^k-x_{t-\tau}^k\|
    +4BU_\delta U_\omega\,
    \bbE_{t-\tau}\|\nabla_{\theta}\omega_{t}^{k,*}
    -\nabla_{\theta}\omega_{t-\tau}^{k,*}\|\\
    &+4 U_\omega^2 L_{*,1}B\,
    \bbE_{t-\tau}\|\bfB_t-\bfB_{t-\tau}\|
    +4 U_\omega L_{*,1}B\,
    \bbE_{t-\tau}\|\omega_t^k-\omega_{t-\tau}^k\|\\
    &+4U_\omega L_{*,1}^2B\,
    \bbE_{t-\tau}\|\theta_t^k-\theta_{t-\tau}^k\|\\
    &+2U_\omega L_{*,1}U_\delta\,
    \bbE_{t-\tau}\|
    \nabla \log \pi_{\theta_t^k}(\hat a_{t,\ell}^k|\hat s_{t,\ell}^k)
    -\nabla \log \pi_{\theta_{t-\tau}^k}(\hat a_{t,\ell}^k|\hat s_{t,\ell}^k)
    \|\\
    &+2(1-\gamma)U_\omega L_{*,1}\,
    \bbE_{t-\tau}\|
    \nabla J^k(\theta_{t}^k)-\nabla J^k(\theta_{t-\tau}^k)
    \|\\
    \leq&
    (2U_\delta+4U_\omega)L_{*,1}B\,
    \bbE_{t-\tau}\|\omega_t^k-\omega_{t-\tau}^k\|
    +(2U_\delta U_\omega+4U_\omega^2)L_{*,1}B\,
    \bbE_{t-\tau}\|\bfB_t-\bfB_{t-\tau}\|\\
    &+\left(
    2U_\delta L_{*,1}^2B
    +4U_\omega L_{*,1}^2B
    +4U_\delta U_\omega B L_{s,1}
    +2U_\omega L_{*,1}U_\delta L_g
    +2U_\omega L_{*,1}L_{J'}
    \right)
    \bbE_{t-\tau}\|\theta_t^k-\theta_{t-\tau}^k\|.
\end{align*}

For the last inequality, we used
\begin{align*}
    \|x_t^k-x_{t-\tau}^k\|
    \leq \|\omega_t^k-\omega_{t-\tau}^k\|
    +U_\omega\|\bfB_t-\bfB_{t-\tau}\|
    +L_{*,1}\|\theta_t^k-\theta_{t-\tau}^k\|.
\end{align*}
We also used $\|\omega_t^{k,*}-\omega_{t-\tau}^{k,*}\|\leq L_{*,1}\|\theta_t^k-\theta_{t-\tau}^k\|$ and
$\|\nabla_\theta\omega_t^{k,*}-\nabla_\theta\omega_{t-\tau}^{k,*}\|\leq L_{s,1}\|\theta_t^k-\theta_{t-\tau}^k\|$.
That is,
\begin{align}
\label{eqn: I_7 D_2 bound}
&\abth{D_2}
\leq
 (2U_\delta+4U_\omega)L_{*,1}B\,
    \bbE_{t-\tau}\|\omega_t^k-\omega_{t-\tau}^k\|
    +(2U_\delta U_\omega+4U_\omega^2)L_{*,1}B\,
    \bbE_{t-\tau}\|\bfB_t-\bfB_{t-\tau}\|\nonumber\\
    &+\left(
    2U_\delta L_{*,1}^2B
    +4U_\omega L_{*,1}^2B
    +4U_\delta U_\omega B L_{s,1}
    +2U_\omega L_{*,1}U_\delta L_g
    +2U_\omega L_{*,1}L_{J'}
    \right)
    \bbE_{t-\tau}\|\theta_t^k-\theta_{t-\tau}^k\|.
\end{align}

\noindent\underline{Bounding $\abth{D_3}$.} Applying Cauchy--Schwarz (pulling $\|x_{t-\tau}^k\|\cdot\|\bfB^*(\nabla_\theta\omega_{t-\tau}^{k,*})^\top\|$ out of the inner product) and Corollary \ref{cor: AC, OC dist on f with a0: actor} (to bound the difference of expectations under the original and auxiliary actor chains, with parameters held fixed at $\theta_{t-\tau}^k$), we have
\begin{align}
\label{eqn: I_7 D_3 bound}
\abth{D_3}
\leq&
2U_\omega L_{*,1}BU_\delta
\left(
\frac{1}{(1-\gamma^L)(1-\gamma)}
+
\frac{3-2\gamma}{1-\gamma}
\right)
L_\pi |\calA|
\sum_{i=t-\tau}^{t-1}
\bbE_{t-\tau}\|\theta_{i+1}^k-\theta_i^k\| \nonumber\\
\leq&
\frac{8U_\omega L_{*,1}BU_\delta}{(1-\gamma)^2}
L_\pi |\calA|
\sum_{i=t-\tau}^{t-1}
\bbE_{t-\tau}\|\theta_{i+1}^k-\theta_i^k\|.
\end{align}

\noindent\underline{Bounding $\abth{D_4}$.} Applying Cauchy--Schwarz and then the geometric mixing of the actor reset chain, we have
\begin{align*}
    \abth{D_4}
    \leq&
    2\cdot 2U_\omega L_{*,1}BU_\delta \cdot
    d_{TV}\pth{
    \bar\bbU_{t-\tau:(t,L)}^k,
    \nu_{\theta_{t-\tau}^k}^k\otimes(\pi_{\theta_{t-\tau}^k}\otimes \hat P^k)^{L-1}
    }\\
    \leq&
    4U_\omega L_{*,1}BU_\delta \gamma^{\tau L}.
\end{align*}

That is,
\begin{equation}
\label{eqn: I_7 D_4 bound}
\abth{D_4}
\leq
4U_\omega L_{*,1}BU_\delta \gamma^{\tau L}.
\end{equation}

\smallskip
Putting \eqref{eqn: I_7 D_1 bound}, \eqref{eqn: I_7 D_2 bound}, \eqref{eqn: I_7 D_3 bound}, and \eqref{eqn: I_7 D_4 bound} back to \eqref{eqn: I_7 three-diff decomp}, we get
\begin{align}
\label{eqn: noise term of I_7 final bound}
    &\abth{\bbE_{t-\tau}\iprod{x_t^k}{\bfB^* (\nabla_{\theta}\omega_{t}^{k,*})^\top \pth{g_{t,L}^k-(1-\gamma)\nabla J^k(\theta_t^k)}}}\nonumber\\
    \leq& 2B L_{*,1}\bbE_{t-\tau}\lnorm{x_t^k}{}^2
    +(2U_\delta+4U_\omega)L_{*,1}B
    \bbE_{t-\tau}\|\omega_t^k-\omega_{t-\tau}^k\|\nonumber\\
    &+(2U_\delta U_\omega+4U_\omega^2)L_{*,1}B
    \bbE_{t-\tau}\|\bfB_t-\bfB_{t-\tau}\|\nonumber\\
    &+\left(
    2U_\delta L_{*,1}^2B
    +4U_\omega L_{*,1}^2B
    +4U_\delta U_\omega B L_{s,1}
    +2U_\omega L_{*,1}U_\delta L_g
    +2U_\omega L_{*,1}L_{J'}
    \right)
    \bbE_{t-\tau}\|\theta_t^k-\theta_{t-\tau}^k\|\nonumber\\
    &+\frac{8U_\omega L_{*,1} U_\delta B}{(1-\gamma)^2}L_\pi |\calA|
    \sum_{i=t-\tau}^{t-1}\bbE_{t-\tau}[\|\theta_{i+1}^k-\theta_i^k\|]
    + 4U_\omega L_{*,1}U_\delta B\,\gamma^{\tau L}.
\end{align}
Combining \eqref{eqn: noise term of I_7 final bound} with \eqref{eqn: first order perturb of the policy}, using Lemma \ref{lmm: crude parameter bound} for the $\omega$ and $\bfB$ perturbations and the triangle inequality for $\|\theta_t^k-\theta_{t-\tau}^k\|$, and then putting the result back to \eqref{eqn: I_7 decomp after expectation}, we get
\begin{align}
\label{eqn: I_7 final bound}
    &\bbE_{t-\tau}[I_7]\nonumber\\
    \leq& 2U_\omega L_{s,1}\bbE_{t-\tau}\|\theta_{t+1}^k-\theta_t^k\|^2
    +2(1-\gamma)\alpha L_{*,1}\bbE_{t-\tau}\lnorm{x_t^k}{}\lnorm{\nabla J^k(\theta_t^k)}{}\nonumber\\
    &+4\alpha B L_{*,1}\bbE_{t-\tau}\|x_t^k\|^2\nonumber\\
    &+\frac{2\alpha (2U_\delta+4U_\omega)L_{*,1}BU_\delta}{1-\gamma}\frac{\tau\beta}{L}\nonumber\\
    &+\frac{4\alpha (2U_\delta U_\omega+4U_\omega^2)L_{*,1}BU_\delta U_\omega}{1-\gamma}\frac{\tau\zeta}{L}
    +\frac{8\alpha (2U_\delta U_\omega+4U_\omega^2)L_{*,1}BU_\delta^2U_\omega^2}{(1-\gamma)^2}\frac{\tau\zeta^2}{L^2}\nonumber\\
    &+2\alpha\Big(
    2U_\delta L_{*,1}^2B
    +4U_\omega L_{*,1}^2B
    +4U_\delta U_\omega B L_{s,1}
    +2U_\omega L_{*,1}U_\delta L_g
    +2U_\omega L_{*,1}L_{J'}\nonumber\\
    &\qquad\qquad
    +\frac{8U_\omega L_{*,1}U_\delta B}{(1-\gamma)^2}L_\pi|\calA|
    \Big)
    \sum_{i=t-\tau}^{t-1}\bbE_{t-\tau}[\|\theta_{i+1}^k-\theta_i^k\|]
    +8\alpha U_\omega L_{*,1}U_\delta B\,\gamma^{\tau L}.
\end{align}

\underline{For $I_8$}, using the orthonormality of $\bfB^*$ and the Lipschitz condition of the fixed point $\omega^{k,*}(\cdot,\bfB^*)$ (Lemma \ref{lmm: L* lip}), we have
\begin{align*}
    4 \lnorm{z_{t}^{k,*}- z_{t+1}^{k,*}}{}^2
    =& 4 \lnorm{\bfB^* (\omega_{t}^{k,*}- \omega_{t+1}^{k,*})}{}^2\\
    =& 4 \lnorm{\omega_{t}^{k,*}- \omega_{t+1}^{k,*}}{}^2 \quad(\text{$\bfB^{*\top}\bfB^*=\bfI_r$})\\
    \leq& 4 L_{*,1}^2 \lnorm{\theta_t^k - \theta_{t+1}^k}{}^2 \quad(\text{Lemma \ref{lmm: L* lip}}).
\end{align*}
Taking conditional expectation, we get
\begin{equation}
\label{eqn: I_8 bound}
\bbE_{t-\tau}[I_8]\leq 4 L_{*,1}^2
\bbE_{t-\tau}\|\theta_{t+1}^k-\theta_t^k\|^2.
\end{equation}
}

Plugging the upper bounds of $I_1, I_2, I_3, I_4, I_5, I_6, I_7$, and $I_8$ (\eqref{eq: bound of 1st inner product: upper bound: local head}, \eqref{eqn: second term in x tilde upper bound}, \eqref{eqn: third term in x tilde upper bound}, \eqref{eqn: fourth term in x tilde upper bound}, \eqref{eqn: fifth term in x tilde upper bound}, \eqref{eqn: sixth term in x tilde upper bound}, \eqref{eqn: I_7 final bound}, and \eqref{eqn: I_8 bound}) back to \eqref{eqn:xtilde squared decomp}, and taking expectation, we have

\begin{align*}
    &\bbE_{t-\tau}\|\tilde x_{t+1}^k\|^2
    \leq
    (1-2\lambda\beta)\bbE_{t-\tau}\|x_t^k\|^2
    +c_2^2\frac{\beta}{L} \bbE_{t-\tau}\|x_t^k\|^2
    + \frac{36U_\omega^2}{c_2^2}\frac{\beta}{L} \bbE_{t-\tau}\|m_t\|^2\nonumber \\
&  +3\tau \sum_{j=t-\tau}^{t-1}\bbE_{t-\tau}\|\omega_{j+1}^k-\omega_{j}^k\|^2
+ \pth{1+3U_\omega^2}\tau \sum_{j=t-\tau}^{t-1}\bbE_{t-\tau}\|\bfB_{j+1}- \bfB_{j}\|^2\nonumber\\
    &+L_{*,1}^2\tau \sum_{j=t-\tau}^{t-1}\bbE_{t-\tau}\|\theta_{j+1}^k-\theta_j^k\|^2\nonumber\\
    &+\qth{48\pth{12U_\omega^2+32U_\omega^4}+12U_r^2\pth{4+32U_\omega^2}}
    \frac{\beta^2}{L^2(1-\gamma)^2}\\
    &+4U_\omega\frac{\beta}{L}
    \|\bbE_{t-\tau}[\tilde b_{t,L}^k-\bar b_{L,\theta_t^k}^k]\|
    +4U_\omega^2\frac{\beta}{L}
    \|\bbE_{t-\tau}[\tilde A_{t,L}^k-A_{L,\theta_t^k}^k]\|\\
    &+ c_3^2\frac{\zeta}{L}\bbE_{t-\tau}\|x_t^k\|^2
    + \frac{4U_\omega^4}{c_3^2}\frac{\zeta}{KL}\sum_{i=1}^K\bbE_{t-\tau}\|x_t^i\|^2\nonumber\\
    &+ \frac{3}{2}\tau\sum_{j=t-\tau}^{t-1}\bbE_{t-\tau}\|\omega_{j+1}^k-\omega_j^k\|^2
    +\frac{3}{2}\tau\sum_{j=t-\tau}^{t-1}\bbE_{t-\tau}\|\bfB_{j+1}-\bfB_j\|^2\\
    &
    +L_{*,1}^2\tau\sum_{j=t-\tau}^{t-1}\bbE_{t-\tau}\|\theta_{j+1}^k-\theta_j^k\|^2+3\pth{18U_\omega^4+32U_\omega^6}\frac{U_\delta^2\zeta^2}{L^2(1-\gamma)^2}\nonumber\\
    &+ \frac{1}{2K}\tau\sum_{j=t-\tau}^{t-1}\sum_{i=1}^K
    \bbE_{t-\tau}\|\omega_{j+1}^i-\omega_j^i\|^2
    +\frac{4U_\omega^3\zeta}{KL}\sum_{i=1}^K
    \|\bbE_{t-\tau}[\bfb_{t,L}^i]\|\\
    &+ \frac{4\zeta\beta U_\delta U_\omega}{L^2(1-\gamma)}\bbE_{t-\tau}\|x_t^k\|^2
    + \frac{4\zeta\beta U_\omega^2 U_\delta^2}{L^2(1-\gamma)^2}\\
    &+ 32\frac{\beta^2}{L^2}\bbE_{t-\tau}\|x_t^k\|^2 + \frac{8U_\delta^2}{(1-\gamma)^2}\frac{\beta^2}{L^2}\\
    &+ \frac{32\zeta^2 U_\omega^4}{L^2}\frac{1}{K}\sum_{i=1}^K\bbE_{t-\tau}\|x_t^i\|^2 + \frac{24\zeta^2 U_\omega^4 U_\delta^2}{L^2(1-\gamma)^2}\\
    &+ 4\frac{U_\delta^4 U_\omega^2}{L^4(1-\gamma)^4}\beta^2\zeta^2\\
    &+2U_\omega L_{s,1}\bbE_{t-\tau}\|\theta_{t+1}^k-\theta_t^k\|^2
    +2(1-\gamma)\alpha L_{*,1}\bbE_{t-\tau}\|x_t^k\|\|\nabla J^k(\theta_t^k)\|
    +4\alpha B L_{*,1}\bbE_{t-\tau}\|x_t^k\|^2\nonumber\\
    &+\frac{2\alpha (2U_\delta+4U_\omega)L_{*,1}BU_\delta}{1-\gamma}\frac{\tau\beta}{L}
    +\frac{4\alpha (2U_\delta U_\omega+4U_\omega^2)L_{*,1}BU_\delta U_\omega}{1-\gamma}\frac{\tau\zeta}{L}\nonumber\\
    &+\frac{8\alpha (2U_\delta U_\omega+4U_\omega^2)L_{*,1}BU_\delta^2U_\omega^2}{(1-\gamma)^2}\frac{\tau\zeta^2}{L^2}\nonumber\\
    &+2\alpha \Big(
    2U_\delta L_{*,1}^2B
    +4U_\omega L_{*,1}^2B
    +4U_\delta U_\omega B L_{s,1}
    +2U_\omega L_{*,1}U_\delta L_g
    +2U_\omega L_{*,1}L_{J'}\\
    &
    +\frac{8U_\omega L_{*,1}U_\delta B}{(1-\gamma)^2}L_\pi|\calA|
    \Big)
    \sum_{i=t-\tau}^{t-1}\bbE[\|\theta_{i+1}^k-\theta_i^k\||v_{0:t-\tau}]\nonumber\\
    &+ 8\alpha U_\omega L_{*,1}U_\delta B\,\gamma^{\tau L}
    + 4 L_{*,1}^2\bbE_{t-\tau}\|\theta_{t+1}^k-\theta_t^k\|^2.
\end{align*}

Combining like terms and applying Young's inequality to the policy-gradient cross term, for any $c_4>0$,
\[
2(1-\gamma)\alpha L_{*,1}\bbE_{t-\tau}\|x_t^k\|\|\nabla J^k(\theta_t^k)\|
\leq
c_4^2\alpha \bbE_{t-\tau}\|x_t^k\|^2
+\frac{(1-\gamma)^2L_{*,1}^2}{c_4^2}\alpha \bbE_{t-\tau}\|\nabla J^k(\theta_t^k)\|^2,
\]
we get
\begin{align*}
&\bbE_{t-\tau}\|\tilde x_{t+1}^k\|^2
    \leq
    \pth{1-2\lambda\beta +4\alpha B L_{*,1}+ c_4^2\alpha + c_2^2\frac{\beta}{L} + c_3^2\frac{\zeta}{L} + \frac{4\zeta\beta U_\delta U_\omega}{L^2(1-\gamma)} + 32\frac{\beta^2}{L^2}}\bbE_{t-\tau}\|x_t^k\|^2\\
&+ \frac{36 U_\omega^2}{c_2^2}\frac{\beta}{L}\bbE_{t-\tau}\|m_t\|^2
+ \pth{\frac{4U_\omega^4}{c_3^2}\frac{\zeta}{L} + \frac{32\zeta^2 U_\omega^4}{L^2}}\frac{1}{K}\sum_{i=1}^K \bbE_{t-\tau}\|x_t^i\|^2\\
&+ \frac{(1-\gamma)^2L_{*,1}^2}{c_4^2}\alpha \bbE_{t-\tau}\|\nabla J^k(\theta_t^k)\|^2\nonumber\\
&+4U_\omega^2\frac{\beta}{L}
\|\bbE_{t-\tau}[\tilde A_{t,L}^k-A_{L,\theta_t^k}^k]\|
+4U_\omega\frac{\beta}{L}
\|\bbE_{t-\tau}[\tilde b_{t,L}^k-\bar b_{L,\theta_t^k}^k]\|\nonumber\\
&+\frac{4U_\omega^3\zeta}{KL}\sum_{i=1}^K
\|\bbE_{t-\tau}[\bfb_{t,L}^i]\|\\
&+ \frac{9}{2}\tau\sum_{j=t-\tau}^{t-1}\bbE_{t-\tau}\|\omega_{j+1}^k-\omega_j^k\|^2
+ \pth{\frac{5}{2}+3U_\omega^2}\tau\sum_{j=t-\tau}^{t-1}\bbE_{t-\tau}\|\bfB_{j+1}-\bfB_j\|^2\\
&+ \frac{1}{2K}\tau\sum_{j=t-\tau}^{t-1}\sum_{i=1}^K \bbE_{t-\tau}\|\omega_{j+1}^i-\omega_j^i\|^2
+2L_{*,1}^2\tau\sum_{j=t-\tau}^{t-1}
\bbE_{t-\tau}\|\theta_{j+1}^k-\theta_j^k\|^2\\
&
+\pth{2U_\omega L_{s,1}+4L_{*,1}^2}
\bbE_{t-\tau}\|\theta_{t+1}^k-\theta_t^k\|^2
+ \frac{4\zeta\beta U_\omega^2 U_\delta^2}{L^2(1-\gamma)^2}\\
&+ \qth{48\pth{12U_\omega^2+32U_\omega^4}
+12U_r^2\pth{4+32U_\omega^2}+8U_\delta^2}
\frac{\beta^2}{L^2(1-\gamma)^2}\\
&+ \pth{78U_\omega^4+96U_\omega^6}
\frac{U_\delta^2\zeta^2}{L^2(1-\gamma)^2}
+4\frac{U_\delta^4 U_\omega^2}{L^4(1-\gamma)^4}\beta^2\zeta^2\\
&+\frac{2\alpha (2U_\delta+4U_\omega)L_{*,1}BU_\delta}{1-\gamma}\frac{\tau\beta}{L}
+\frac{4\alpha (2U_\delta U_\omega+4U_\omega^2)L_{*,1}BU_\delta U_\omega}{1-\gamma}\frac{\tau\zeta}{L}\nonumber\\
&+\frac{8\alpha (2U_\delta U_\omega+4U_\omega^2)L_{*,1}BU_\delta^2U_\omega^2}{(1-\gamma)^2}\frac{\tau\zeta^2}{L^2}\nonumber\\
&+2\alpha \Big(
2U_\delta L_{*,1}^2B
+4U_\omega L_{*,1}^2B
+4U_\delta U_\omega B L_{s,1}
+2U_\omega L_{*,1}U_\delta L_g
\\&\qquad\qquad+2U_\omega L_{*,1}L_{J'}
+\frac{8U_\omega L_{*,1}U_\delta B}{(1-\gamma)^2}L_\pi|\calA|
\Big)
\sum_{i=t-\tau}^{t-1}\bbE[\|\theta_{i+1}^k-\theta_i^k\||v_{0:t-\tau}]\nonumber\\
&+8\alpha U_\omega L_{*,1}U_\delta B\,\gamma^{\tau L}.
\end{align*}
Taking $\bbE_{t-\tau}$ on both sides of the last display, applying Lemma \ref{lmm: one step temporal difference bound of omega}, Lemma \ref{lmm: one-step perturbation of B}, Lemma \ref{lmm: theta onestep distance}, using $\beta = c\zeta$, applying Lemma \ref{lmm: MC observable mixing} to the remaining linear conditional means, and using $\|\nabla J^k(\theta_t^k)\|\le BU_\delta$, we further get,
\begin{align}
\label{eqn: xtilde final}
&\bbE_{t-\tau}\|\tilde x_{t+1}^k\|^2
\leq
\Bigg(
1-2\lambda c\zeta
+4\alpha B L_{*,1}
+c_4^2\alpha 
+c_2^2\frac{c\zeta}{L}
+c_3^2\frac{\zeta}{L}
+\frac{4c\zeta^2 U_\delta U_\omega}{L^2(1-\gamma)}
+32\frac{c^2\zeta^2}{L^2}\nonumber\\
&\qquad
+72\tau^2\frac{c^2\zeta^2}{L^2}
+\frac{14c\zeta^2U_\omega^2C_{\mathrm{mix},3}^2\tau^2\alpha^2}{L^2(1-\gamma)^2}
+C_{7}^{\mathrm{act}}\alpha^2  \pth{2L_{*,1}^2\tau^2+2U_\omega L_{s,1}+4L_{*,1}^2}
\Bigg)\bbE_{t-\tau}\|x_t^k\|^2\nonumber\\
&+\frac{36U_\omega^2}{c_2^2}\frac{c\zeta}{L}\bbE_{t-\tau}\|m_t\|^2
+\Bigg(
\frac{4U_\omega^4}{c_3^2}\frac{\zeta}{L}
+\frac{32\zeta^2 U_\omega^4}{L^2}
+8\tau^2\frac{c^2\zeta^2}{L^2}
+\frac{14c\zeta^2U_\omega^2C_{\mathrm{mix},3}^2\tau^2\alpha^2  }{L^2(1-\gamma)^2}\nonumber\\
&\qquad\qquad\qquad
+\pth{\frac{5}{2}+3U_\omega^2}\tau^2
\frac{192U_\omega^2\zeta^2}{L^2}
\Bigg)\frac{1}{K}\sum_{i=1}^K\bbE_{t-\tau}\|x_t^i\|^2\nonumber\\
&+\left[
\frac{4c\zeta(U_\omega C_{\mathrm{mix},1}+U_\omega^2 C_{\mathrm{mix},2})}{L(1-\gamma)^2}
+\frac{4U_\omega^3C_{\mathrm{mix},3}\zeta}{L(1-\gamma)}
\right]\tau\alpha \pth{(1-\gamma)BU_\delta+2U_\omega}\nonumber\\
&+2\alpha^2  \tau\left(
2U_\delta L_{*,1}^2B
+4U_\omega L_{*,1}^2B
+4U_\delta U_\omega B L_{s,1}
+2U_\omega L_{*,1}U_\delta L_g
+2U_\omega L_{*,1}L_{J'}
+\frac{8U_\omega L_{*,1}U_\delta B}{(1-\gamma)^2}L_\pi|\calA|
\right)\nonumber\\
&\qquad\qquad\times
\left(
(1-\gamma)\bbE_{t-\tau}\|\nabla J^k(\theta_t^k)\|
+C_{1}^{\mathrm{act}}\bbE_{t-\tau}\|x_t^k\|
\right)\nonumber\\
&+\left[
\frac{(1-\gamma)^2L_{*,1}^2}{c_4^2}\alpha
{}+\frac{14c\zeta^2U_\omega^2C_{\mathrm{mix},3}^2\tau^2\alpha^2  }{L^2}
+8(1-\gamma)^2\alpha^2  \pth{2L_{*,1}^2\tau^2+2U_\omega L_{s,1}+4L_{*,1}^2}
\right]\bbE_{t-\tau}\|\nabla J^k(\theta_t^k)\|^2\nonumber\\
&+\frac{14c\tau^2\alpha^2\zeta^2U_\omega^2C_{\mathrm{mix},3}^2}{L^2}
\frac{1}{K}\sum_{i=1}^K
\bbE_{t-\tau}\|\nabla J^i(\theta_t^i)\|^2\nonumber\\
&+\left[
\frac{4c\zeta(U_\omega C_{\mathrm{mix},1}+U_\omega^2 C_{\mathrm{mix},2})}{L(1-\gamma)^2}
+\frac{4U_\omega^3C_{\mathrm{mix},3}\zeta}{L(1-\gamma)}
\right]
\left(
\frac{\tau^2\alpha c\zeta}{L(1-\gamma)}
+\frac{\tau^2\alpha \zeta}{L(1-\gamma)}
+\tau^2\alpha^2  
+\frac{\tau^2\alpha^2}{(1-\gamma)^2}
+\frac{\tau\alpha }{\sqrt{L(1-\gamma)}}
\right)\nonumber\\
&+\frac{28c\zeta^2U_\omega^2C_{\mathrm{mix},3}^2}{L^2(1-\gamma)^2}
\left(
\frac{\tau^4\alpha^2  c^2\zeta^2}{L^2(1-\gamma)^2}
+\frac{\tau^4\alpha^2  \zeta^2}{L^2(1-\gamma)^2}
+\tau^4\alpha ^4
+\frac{\tau^4\alpha^4}{(1-\gamma)^4}
+\frac{\tau^2\alpha^2  }{L(1-\gamma)}
\right)\nonumber\\
&+5\tau^2
\left(
64\tau^2\frac{U_\delta^2c^4\zeta^4}{L^4(1-\gamma)^2}
+1024\tau^2U_\omega^4\frac{U_\delta^2c^2\zeta^4}{L^4(1-\gamma)^2}
+32\tau^2L_{*,1}^2B^2U_\delta^2\frac{c^2\zeta^2\alpha^2  }{L^2}
+\frac{6c^2\zeta^2U_\delta^2}{L^2(1-\gamma)^2}
\right)\nonumber\\
&+\pth{\frac{5}{2}+3U_\omega^2}\tau^2
\Bigg(
768\tau^2\frac{U_\delta^2U_\omega^2c^2\zeta^4}{L^4(1-\gamma)^2}
+12288\tau^2\frac{U_\delta^2U_\omega^5\zeta^4}{L^4(1-\gamma)^2}\nonumber\\
&\qquad\qquad
+384\tau^2U_\omega^2L_{*,1}^2B^2U_\delta^2\frac{\alpha^2  \zeta^2}{L^2}
+\frac{24U_\omega^2U_\delta^2}{L^2(1-\gamma)^2}\zeta^2
+32\frac{U_\delta^4U_\omega^4\zeta^4}{L^4(1-\gamma)^4}
\Bigg)\nonumber\\
&+\pth{2L_{*,1}^2\tau^2+2U_\omega L_{s,1}+4L_{*,1}^2}
\Bigg(
C_{8}^{\mathrm{act}}\frac{\tau^2\alpha^2  c^2\zeta^2}{L^2(1-\gamma)^2}
+C_{9}^{\mathrm{act}}\frac{\tau^2\alpha^2  \zeta^2}{L^2(1-\gamma)^2}\nonumber\\
&\qquad\qquad\qquad\qquad\qquad\qquad\qquad\qquad\qquad
+C_{10}^{\mathrm{act}}\tau^2\alpha^4
+C_{11}^{\mathrm{act}}\frac{\tau\alpha^3}{(1-\gamma)^2}
+C_{12}^{\mathrm{act}}\frac{\alpha^2  }{L(1-\gamma)}
\Bigg)\nonumber\\
&+\qth{48\pth{12U_\omega^2+32U_\omega^4}
+12U_r^2\pth{4+32U_\omega^2}+8U_\delta^2}
\frac{c^2\zeta^2}{L^2(1-\gamma)^2}\nonumber\\
&+\pth{78U_\omega^4+96U_\omega^6}
\frac{U_\delta^2\zeta^2}{L^2(1-\gamma)^2}
+4\frac{U_\delta^4U_\omega^2c^2\zeta^4}{L^4(1-\gamma)^4}
\displaybreak[4]
+\frac{4c\zeta^2U_\omega^2U_\delta^2}{L^2(1-\gamma)^2}\nonumber\\
&+\frac{2\alpha (2U_\delta+4U_\omega)L_{*,1}BU_\delta}{1-\gamma}\frac{\tau c\zeta}{L}
+\frac{4\alpha (2U_\delta U_\omega+4U_\omega^2)L_{*,1}BU_\delta U_\omega}{1-\gamma}\frac{\tau\zeta}{L}\nonumber\\
&+\frac{8\alpha (2U_\delta U_\omega+4U_\omega^2)L_{*,1}BU_\delta^2U_\omega^2}{(1-\gamma)^2}\frac{\tau\zeta^2}{L^2}\nonumber\\
&+2\alpha \tau\left(
2U_\delta L_{*,1}^2B
+4U_\omega L_{*,1}^2B
+4U_\delta U_\omega B L_{s,1}
+2U_\omega L_{*,1}U_\delta L_g
+2U_\omega L_{*,1}L_{J'}
+\frac{8U_\omega L_{*,1}U_\delta B}{(1-\gamma)^2}L_\pi|\calA|
\right)\nonumber\\
&\qquad\qquad\times
\Bigg(
C_{2}^{\mathrm{act}}\frac{\tau\alpha c\zeta}{L(1-\gamma)}
+C_{3}^{\mathrm{act}}\frac{\tau\alpha \zeta}{L(1-\gamma)}
+C_{4}^{\mathrm{act}}\tau\alpha^2  
+C_{5}^{\mathrm{act}}\frac{\alpha }{\sqrt{L(1-\gamma)}}
+C_{6}^{\mathrm{act}}\frac{\tau\alpha^2}{(1-\gamma)^2}
\Bigg)\nonumber\\
&+\left[
\frac{8c\zeta}{L}\left(\frac{U_\omega U_r}{1-\gamma}+U_\omega^2\right)
+\frac{8U_\omega^3U_\delta\zeta}{L(1-\gamma)}
\right]m\rho^{\tau L}
+8\alpha U_\omega L_{*,1}U_\delta B\,\gamma^{\tau L}
+\frac{112c\zeta^2U_\omega^2U_\delta^2}{L^2(1-\gamma)^2}m^2\rho^{2\tau L}.
\end{align}

Putting \eqref{eq: x_t+1 intermediate bound cross term} and \eqref{eq: upper bound of qr-proj squared}, back to \eqref{eq: upper bound: key intermediate}, and by tower property, we get,
\begin{align}
\label{eqn: x t+1 before sub m t+1 and q}
    &\bbE_{t-\tau}\|x_{t+1}^k\|^2 = \bbE_{t-\tau}\|\tilde x_{t+1}^k\|^2 + \bbE_{t-\tau}\|\bfB_{t+1}\Pi_{U_\omega}(\tilde\omega_{t+1}^k)-\bfB_{t+1}\bfR_{t+1}\tilde\omega_{t+1}^k\|^2 \nonumber\\& +2\bbE_{t-\tau}\langle\bfB_{t+1}\Pi_{U_\omega}(\tilde\omega_{t+1}^k)-\bfB_{t+1}\bfR_{t+1}\tilde\omega_{t+1}^k, \tilde x_{t+1}^k \rangle \nonumber\\
   \leq& \bbE_{t-\tau}\|\tilde x_{t+1}^k\|^2 \nonumber
    +64\frac{\beta^2}{L^2}\bbE_{t-\tau}\|x_t^k\|^2
    +48\frac{\beta^2}{L^2}\frac{U_\delta^2}{(1-\gamma)^2}
    +100U_\omega^2 \bbE_{t-\tau}\|\bfQ_t\|_F^4\nonumber\\
     &+6U_\omega\bbE_{t-\tau}\norm{\omega_{t+1}^k - \tilde\omega_{t+1}^k}\|m_{t+1}\|^2+16U_\omega^2 \bbE_{t-\tau}\|\bfQ_t\|_F^2.
\end{align}
For $\norm{\omega_{t+1}^k - \tilde\omega_{t+1}^k}$, we have 
\begin{align*}
\norm{\omega_{t+1}^k - \tilde\omega_{t+1}^k}\leq \left\|\frac{\beta}{L} \delta_{t,L}^k\bfB_t^\top \phi(s_{t,0}^k)\right\|\leq \frac{\beta U_\delta}{L(1-\gamma)},
\end{align*}
where the last inequality uses $|\delta_{t,L}^k|\leq U_\delta/(1-\gamma)$ and $\|\bfB_t^\top \phi(s_{t,0}^k)\|\leq 1$.
For both $\bbE_{t-\tau}\|\bfQ_t\|_F^2$ and $\bbE_{t-\tau}\|\bfQ_t\|_F^4$, we invoke the crude bound \eqref{eqn: rough Q bound}. Then, \eqref{eqn: x t+1 before sub m t+1 and q} can be written as
\begin{align}
    \label{eqn: x t+1 before sub m t+1}
    &\bbE_{t-\tau}\|x_{t+1}^k\|^2  \nonumber\\
   \leq& \bbE_{t-\tau}\|\tilde x_{t+1}^k\|^2 \nonumber
   +64\frac{\beta^2}{L^2}\bbE_{t-\tau}\|x_t^k\|^2
   +48\frac{\beta^2}{L^2}\frac{U_\delta^2}{(1-\gamma)^2}
   +100\frac{U_\omega^6 U_\delta^4}{L^4(1-\gamma)^4}\zeta^4\nonumber\\
      &+6U_\omega U_\delta \frac{\beta}{L(1-\gamma)}\bbE_{t-\tau}\|m_{t+1}\|^2
      +\frac{16U_\delta^2U_\omega^4}{L^2(1-\gamma)^2}\zeta^2 .
\end{align}
We next compress the terms obtained from \eqref{eqn: xtilde final} and the QR-projection residuals in \eqref{eqn: x t+1 before sub m t+1} into preliminary coefficient envelopes. These envelopes do not yet include the later $M_{t+1}$ plug-in. Taking full expectation in \eqref{eqn: x t+1 before sub m t+1}, invoking \eqref{eqn: xtilde final}, averaging over agents, using the crude boundedness of the mixing first-moment terms, and applying Young's inequality to the remaining first-moment terms gives
\begin{align}
\label{eqn: Xt+1 before sub Mt+1}
     &\bar X_{t+1}  \nonumber\\
   \leq& \Bigg(1-2\lambda c\zeta
   +\pth{
   c_2^2\frac{c}{L}
   +c_3^2\frac{1}{L}
   +\frac{4U_\omega^4}{c_3^2L}
   +4B L_{*,1}c_\theta
   +c_4^2c_\theta
   }\zeta
   +\frac{\calC_{X,1}'(\tau^4;c,c_\theta)}{(1-\gamma)^4}\zeta^2
   \Bigg)\bar X_t\nonumber\\
&+\frac{36U_\omega^2}{c_2^2}\frac{c\zeta}{L}M_t
+6U_\omega U_\delta c\frac{\zeta}{L(1-\gamma)} M_{t+1}\nonumber\\
&+
\pth{
\frac{(1-\gamma)^2L_{*,1}^2c_\theta}{c_4^2}\zeta
+\frac{\calC_{X,2}'(\tau^4;c,c_\theta)}{(1-\gamma)^2}\zeta^2
}\bar G_t\nonumber\\
&+
\frac{\calC_{X,3}'(\tau^4;c,c_\theta)}{(1-\gamma)^4}
\pth{\frac{\zeta^2}{\sqrt L}+\frac{\zeta^2}{L}+\frac{\zeta^2}{K}}
+\frac{\calC_{X,4}'(\tau^6;c,c_\theta)}{(1-\gamma)^4}\zeta^3
+8c_\theta U_\omega L_{*,1}U_\delta B\,\zeta\,\gamma^{\tau L}.
\end{align}
Under the parametrization $\beta=c\zeta$ and $\alpha =c_\theta\zeta$, and using $\zeta\leq 1$, $L,K\geq 1$, and $(1-\gamma)^{-1}\geq 1$, one concrete admissible choice of the preliminary coefficient envelopes used in \eqref{eqn: Xt+1 before sub Mt+1} is
\begin{align*}
&\calC_{X,1}'(\tau^4;c,c_\theta)
:=\;
4cU_\delta U_\omega
{}+96c^2+80\tau^2c^2
{}+32U_\omega^4
{}+192\pth{\frac{5}{2}+3U_\omega^2}\tau^2U_\omega^2\\
&{}+c_\theta^2
\Bigg[
C_{7}^{\mathrm{act}}\pth{2L_{*,1}^2\tau^2+2U_\omega L_{s,1}+4L_{*,1}^2}
{}+28c\tau^2U_\omega^2C_{\mathrm{mix},3}^2
{}\\
&+96\tau^2U_\omega^4C_{\mathrm{mix},3}^2
+2\tau C_{1}^{\mathrm{act}}
\Big(
2U_\delta L_{*,1}^2B
{}+4U_\omega L_{*,1}^2B
{}+4U_\delta U_\omega B L_{s,1}
{}+2U_\omega L_{*,1}U_\delta L_g\\
&\qquad\qquad\qquad
{}+2U_\omega L_{*,1}L_{J'}
{}+8U_\omega L_{*,1}U_\delta B L_\pi|\calA|
\Big)
\Bigg],
\end{align*}
\begin{align*}
&\calC_{X,2}'(\tau^4;c,c_\theta)
:=\;
c_\theta^2
\Bigg[
8\pth{2L_{*,1}^2\tau^2+2U_\omega L_{s,1}+4L_{*,1}^2}
{}+28c\tau^2U_\omega^2C_{\mathrm{mix},3}^2
{}+96\tau^2U_\omega^4C_{\mathrm{mix},3}^2\\
&\qquad\qquad
+2\tau
\Big(
2U_\delta L_{*,1}^2B
{}+4U_\omega L_{*,1}^2B
{}+4U_\delta U_\omega B L_{s,1}
{}+2U_\omega L_{*,1}U_\delta L_g\\
&\qquad\qquad
{}+2U_\omega L_{*,1}L_{J'}
{}+8U_\omega L_{*,1}U_\delta B L_\pi|\calA|
\Big)
\Bigg],
\end{align*}
and
\begin{align*}
&\calC_{X,3}'(\tau^4;c,c_\theta)
:=\;
4\tau c_\theta(1+(1-\gamma)BU_\delta+2U_\omega)
\pth{c\pth{U_\omega C_{\mathrm{mix},1}+U_\omega^2C_{\mathrm{mix},2}}
{}+U_\omega^3C_{\mathrm{mix},3}}\\
&{}+2\tau c_\theta^2C_{5}^{\mathrm{act}}
\bigg(
2U_\delta L_{*,1}^2B
{}+4U_\omega L_{*,1}^2B
{}+4U_\delta U_\omega B L_{s,1}
{}+2U_\omega L_{*,1}U_\delta L_g\\&\qquad\qquad
{}+2U_\omega L_{*,1}L_{J'}
{}+8U_\omega L_{*,1}U_\delta B L_\pi|\calA|
\bigg)\\
&{}+c^2\left[
48\pth{12U_\omega^2+32U_\omega^4}
{}+12U_r^2\pth{4+32U_\omega^2}
{}+56U_\delta^2
{}+30\tau^2U_\delta^2
\right]\\
&{}+\pth{78U_\omega^4+96U_\omega^6}U_\delta^2
{}+24\pth{\frac{5}{2}+3U_\omega^2}\tau^2U_\omega^2U_\delta^2\\
&{}+16U_\delta^2U_\omega^4
{}+4cU_\omega^2U_\delta^2
{}+C_{12}^{\mathrm{act}}c_\theta^2
\pth{2L_{*,1}^2\tau^2+2U_\omega L_{s,1}+4L_{*,1}^2}\\
&{}+2\tau cc_\theta(2U_\delta+4U_\omega)L_{*,1}BU_\delta
{}+4\tau c_\theta(2U_\delta U_\omega+4U_\omega^2)L_{*,1}BU_\delta U_\omega .
\end{align*}
Let $\calC_{X,4}'(\tau^6;c,c_\theta)$ be a coefficient envelope for the remaining scalar residuals not collected in $\calC_{X,1}'$, $\calC_{X,2}'$, or $\calC_{X,3}'$ after using $\zeta\leq1$, $L,K\geq1$, $\beta=c\zeta$, $\alpha =c_\theta\zeta$, and $\tau L=\lceil2\log_\rho(\zeta)\rceil$, which gives $m\rho^{\tau L}\leq m\zeta^2$.
These are the remaining $\zeta^3$-and-higher terms, including the $m\rho^{\tau L}$ terms and the $\|\bfQ_t\|_F^4$ residual; the actor reset-chain residual $\gamma^{\tau L}$ is kept explicit.
The preliminary envelope $\calC_{X,1}'$ collects the second-order $\bar X_t$ terms from the one-step perturbation bounds, the conditional-mean mixing bounds, and the projection/QR correction. The envelope $\calC_{X,2}'$ collects the second-order $\bar G_t$ terms. The envelope $\calC_{X,3}'$ collects the second-order scalar residuals with $\sqrt L$-, $L$-, or $K$-speedup, and $\calC_{X,4}'$ collects the remaining third-and-higher order residuals.
For the $M_{t+1}$ term in \eqref{eqn: Xt+1 before sub Mt+1}, we invoke the PAD recursion \eqref{eqn: Mt+1 final}. Multiplying \eqref{eqn: Mt+1 final} by $6U_\omega U_\delta c\zeta/(L(1-\gamma))$ and using $1-\frac{\lambda\nu\zeta}{r}+c_1^2\zeta\leq 1+c_1^2\zeta$, this contribution is bounded by
\begin{align}
\label{eqn: local head crude Mt+1 plug}
&6U_\omega U_\delta c\frac{\zeta}{L(1-\gamma)} M_{t+1}\nonumber\\
\leq&
6U_\omega U_\delta c\frac{\zeta}{L(1-\gamma)}
\pth{1+c_1^2\zeta}M_t
+\frac{864U_\delta U_\omega^3c}{c_1^2}
\frac{\zeta^2}{L^3(1-\gamma)}\bar X_t\nonumber\\
&+
\frac{6cU_\delta U_\omega\calC_{M,1}(\tau^4;c,c_\theta)}{(1-\gamma)^3}
\frac{\zeta^3}{L^3}\bar X_t\nonumber\\
&+
\frac{6cU_\delta U_\omega\calC_{M,2}(\tau^2;c,c_\theta)}{(1-\gamma)^4}
\pth{\frac{\zeta^3}{L^2}+\frac{\zeta^3}{LK}}\nonumber\\
&+
\frac{6cU_\delta U_\omega\calC_{M,3}(\tau^6;c,c_\theta)}{(1-\gamma)^5}
\frac{\zeta^4}{L}.
\end{align}
After the PAD plug-in, define the final coefficient envelopes
\begin{align}
\label{eqn:C_X constants}
\calC_{X,1}(\tau^4;c,c_\theta)
&:=\calC_{X,1}'(\tau^4;c,c_\theta)
+6cU_\delta U_\omega\calC_{M,1}(\tau^4;c,c_\theta),\nonumber\\
\calC_{X,2}(\tau^4;c,c_\theta)
&:=\calC_{X,2}'(\tau^4;c,c_\theta),\nonumber\\
\calC_{X,3}(\tau^4;c,c_\theta)
&:=\calC_{X,3}'(\tau^4;c,c_\theta)
+6cU_\delta U_\omega\calC_{M,2}(\tau^2;c,c_\theta),\nonumber\\
\calC_{X,4}(\tau^6;c,c_\theta)
&:=\calC_{X,4}'(\tau^6;c,c_\theta)
+6U_\omega^2\calC_{M,3}(\tau^6;c,c_\theta).
\end{align}
For the $\calC_{M,3}$ contribution we use $\beta U_\delta/(L(1-\gamma))\leq U_\omega$ to absorb the extra factor $\zeta/(L(1-\gamma))$. With these final envelopes in place, choose $\tau L=\lceil 2\log_\rho(\zeta)\rceil$ and use $m\rho^{\tau L}\leq m\zeta^2$. Substituting \eqref{eqn: local head crude Mt+1 plug} into \eqref{eqn: Xt+1 before sub Mt+1} gives
\begin{align}
\label{eqn: Xt+1 after sub Mt+1}
&\bar X_{t+1}  \nonumber\\
\leq&
\Bigg(1-2\lambda c\zeta
+\pth{
c_2^2\frac{c}{L}
+c_3^2\frac{1}{L}
+\frac{4U_\omega^4}{c_3^2L}
+4B L_{*,1}c_\theta
+c_4^2c_\theta
}\zeta\nonumber\\
&\qquad\qquad
+\frac{864U_\delta U_\omega^3c}{c_1^2}\frac{\zeta^2}{L^3(1-\gamma)}
+\frac{\calC_{X,1}(\tau^4;c,c_\theta)}{(1-\gamma)^4}\zeta^2
\Bigg)\bar X_t\nonumber\\
&+
\pth{
\frac{36U_\omega^2}{c_2^2}\frac{c\zeta}{L}
+6U_\omega U_\delta c\frac{\zeta}{L(1-\gamma)}
\pth{1+c_1^2\zeta}
}M_t\nonumber\\
&+
\pth{
\frac{(1-\gamma)^2L_{*,1}^2c_\theta}{c_4^2}\zeta
+\frac{\calC_{X,2}(\tau^4;c,c_\theta)}{(1-\gamma)^2}\zeta^2
}\bar G_t\nonumber\\
&+
\frac{\calC_{X,3}(\tau^4;c,c_\theta)}{(1-\gamma)^4}
\pth{\frac{\zeta^2}{\sqrt L}+\frac{\zeta^2}{L}+\frac{\zeta^2}{K}}
+\frac{\calC_{X,4}(\tau^6;c,c_\theta)}{(1-\gamma)^4}\zeta^3
+8c_\theta U_\omega L_{*,1}U_\delta B\,\zeta\,\gamma^{\tau L}.
\end{align}
The final envelopes $\calC_{X,i}$ add exactly the higher-order contributions generated by substituting the PAD recursion for $M_{t+1}$ and are the coefficient constants used in the lemma statement.

\end{proof}

\section{Proof of Lemma \ref{lm: PAD analysis} (Principal Angle Distance Analysis)}
\label{sec: proof of principal angle distance}
Our analysis follows the general roadmap of \cite{wang2026personalized}, while introducing several key distinctions and technical innovations to address the unique challenges arising from actor updates under Markovian sampling. For example, \cite{wang2026personalized} controls the Markov noise by directly exploiting the mixing properties of the underlying Markov chains. Such an approach is no longer applicable in our setting, since the policies underlying the conditional expectation $\bbE_{t-\tau}$ evolve over time and are themselves coupled with past trajectories. Departing fundamentally from \cite{wang2026personalized}, our analysis instead relies on the perturbed mixing bounds that we develop in Lemma \ref{lmm: MC observable mixing}.

Recall from Algorithm \ref{alg: fedac-per} of the update of $\bar \bfB$ that 
$\bar{\mathbf{B}}_{t+1} = \mathbf{B}_t
    +\frac{\zeta}{KL}\sum_{k=1}^K \bfB_{t,\perp}\bfB_{t,\perp}^\top\delta_{t,L}^k\phi(s_{t,0}^k)(\omega_{t}^k)^\top. $
Since $\mathbf{B}_{t+1} = \bar{\mathbf{B}}_{t+1} \mathbf{R}_{t+1}^{-1}$, we have 
\begin{align*}
    \|\bfB_\perp^{* \top}\mathbf{B}_{t+1}\|^2_F = \|\bfB_\perp^{* \top}\bar{\mathbf{B}}_{t+1}\bfR^{-1}_{t+1}\|^2_F\leq \|\bfB_\perp^{* \top}\bar{\mathbf{B}}_{t+1}\|^2_F \|\bfR^{-1}_{t+1}\|^2.
\end{align*}
That is,
\begin{align*}
    \|m_{t+1}\|_F^2 \leq \|\bar m_{t+1}\|^2_F \|\bfR^{-1}_{t+1}\|^2.
\end{align*}
For $\|\bar m_{t+1}\|^2_F$, we have
\begin{align}
\label{eq: inner product: principal angle: upper bound}
    &\|\bar m_{t+1}\|^2_F = \|\bfB_\perp^{* \top}\bar{\mathbf{B}}_{t+1}\|^2_F=\|\bfB_\perp^{* \top}\mathbf{B}_t+\frac{\zeta}{KL}\sum_{k=1}^K \bfB_\perp^{* \top}\bfB_{t,\perp}\bfB_{t,\perp}^\top\delta_{t,L}^k\phi(s_{t,0}^k)(\omega_{t}^k)^\top\|^2_F\nonumber\\
    =&\|\bfB_\perp^{* \top}\mathbf{B}_t\|_F^2
    +2\langle\bfB_\perp^{* \top}\mathbf{B}_t,\frac{\zeta}{KL}\sum_{k=1}^K \bfB_\perp^{* \top}\bfB_{t,\perp}\bfB_{t,\perp}^\top\delta_{t,L}^k\phi(s_{t,0}^k)(\omega_{t}^k)^\top\rangle\nonumber\\
    &+\|\frac{\zeta}{KL}\sum_{k=1}^K \bfB_\perp^{* \top}\bfB_{t,\perp}\bfB_{t,\perp}^\top\delta_{t,L}^k\phi(s_{t,0}^k)(\omega_{t}^k)^\top\|_F^2 \nonumber\\
    = &\|m_t\|_F^2+ \underbrace{\|\frac{\zeta}{KL}\sum_{k=1}^K \bfB_\perp^{* \top}\bfB_{t,\perp}\bfB_{t,\perp}^\top\delta_{t,L}^k\phi(s_{t,0}^k)(\omega_{t}^k)^\top\|_F^2}_{I.1}
    +2\langle\bfB_\perp^{* \top}\mathbf{B}_t,\frac{\zeta}{KL}\sum_{k=1}^K\bfB_\perp^{* \top} \bfP_{t,\perp}\xi_{t,L}^k(\omega_{t}^k)^\top\rangle\nonumber\\
    &
    +\underbrace{2\langle\bfB_\perp^{* \top}\mathbf{B}_t,\frac{\zeta}{KL}\sum_{k=1}^K\bfB_\perp^{* \top} \bfP_{t,\perp}A_{L,\theta_t^k}^kx_t^k(\omega_{t}^k)^\top\rangle}_{I.2},
    \end{align}
where the last equality follows from Proposition \ref{prop: key intermediate}.

For $I.1$ in Eq.\,(\ref{eq: inner product: principal angle: upper bound}), recalling that $\|\bfQ_t\|_F^2=\|\frac{\zeta}{KL}\sum_{k=1}^K \bfP_{t,\perp}\delta_{t,L}^k\phi(s_{t,0}^k)(\omega_{t}^k)^\top\|_F^2$, we have 
\begin{align*}
 I.1 = &\|\frac{\zeta}{KL}\sum_{k=1}^K\bfB_\perp^{* \top} \bfP_{t,\perp}\delta_{t,L}^k\phi(s_{t,0}^k)(\omega_{t}^k)^\top\|_F^2\\
    =&\|\bfB_\perp^{* \top}\frac{\zeta}{KL}\sum_{k=1}^K \bfP_{t,\perp}\delta_{t,L}^k\phi(s_{t,0}^k)(\omega_{t}^k)^\top\|_F^2\\
    \leq& \|\bfB_\perp^{* \top}\|^2\|\bfQ_t\|_F^2
    \leq \|\bfQ_t\|_F^2.
\end{align*}

$I.2$ in Eq.\,(\ref{eq: inner product: principal angle: upper bound}) is related to the negative drift of the $m_t$. Specifically,  
\begin{align*}
 \frac{1}{2}I.2 =  &\langle\bfB_\perp^{* \top}\mathbf{B}_t,\frac{\zeta}{KL}\sum_{k=1}^K\bfB_\perp^{* \top} \bfP_{t,\perp} A_{L,\theta_t^k}^kx_t^k(\omega_{t}^k)^\top\rangle\\
 & = \frac{\zeta}{KL}\sum_{k=1}^K 
\pth{ \bfP_{t,\perp} \bfP_{\perp}^* \bfB_t \omega_t^k }^\top A_{L,\theta_t^k}^k \pth{ \bfP_{t} + \bfP_{t,\perp} } x_t^k,
\end{align*}
which is true because 
\[
\iprod{AB}{vw^{\top}} = \text{trace}\pth{B^{\top}A^{\top} vw^{\top}} = w^{\top} B^{\top}A^{\top} v
\]
for any $r\times r$ matrix $A$ and $r$-dimensional vectors $v$ and $w$. 
Additionally,
\begin{align*}
\pth{ \bfP_{t,\perp} \bfP_{\perp}^* \bfB_t \omega_t^k }^\top A_{L,\theta_t^k}^k \pth{ \bfP_{t} + \bfP_{t,\perp} } x_t^k 
& \le \pth{ \bfP_{t,\perp} \bfP_{\perp}^* x_t^k }^\top A_{L,\theta_t^k}^k \bfP_{t,\perp} x_t^k 
+ 2 U_\omega \|m_t\|\|x_t^k\| \\
& \le \pth{ \bfP_{t,\perp} x_t^k }^\top A_{L,\theta_t^k}^k \bfP_{t,\perp} x_t^k 
+ 6 U_\omega \|m_t\|\|x_t^k\| \\
& \le - \lambda L \| \bfP_{t,\perp} x_t^k \|^2 + 6 U_\omega \|m_t\|\|x_t^k\|.
\end{align*}

For each $k\in [K]$, it holds that 
\begin{align*}
\|\bfP_{t,\perp} x_t^k\|^2 = \|\bfP_{t,\perp}(\mathbf{B}_t\omega_{t}^k-\bfB^*\omega_t^{k,*}) \|^2 
= \|\bfP_{t,\perp} z_t^{k,*}\|^2
= \|(\bfP_{t} - \bfP^*)z_t^{k,*}\|^2.
\end{align*}
From the analysis in the proof of Lemma \ref{lm: local head: lower bound}, we know that 
\[
\frac{1}{K}\sum_{k=1}^K \|(\bfP_{t} - \bfP^*)z_t^{k,*}\|^2 \ge \|m_t\|_F^2 \frac{\lambda_{\min}^+(\bfZ_t^*\bfZ_t^{*\top})}{rK}.
\]
So, 
\begin{align*}
\frac{1}{K} \sum_{k=1}^K\|\bfP_{t,\perp} x_t^k\|^2
&\ge \|m_t\|_F^2 \frac{\lambda_{\min}^+(\bfZ_t^*\bfZ_t^{*\top})}{rK}.
\end{align*}
Thus, 
\begin{align}
\label{eq: mt: I3: overall}
I.2 \le - 2\lambda\zeta \frac{\lambda_{\min}^+(\bfZ_t^*\bfZ_t^{*\top})}{rK} \|m_t\|_F^2
+ 12U_\omega \zeta\frac{1}{KL}\sum_{k=1}^K\|m_t\|\|x_t^k\|.
\end{align}

Combining the bounds of I.1, and I.2 of Eq.\,(\ref{eq: inner product: principal angle: upper bound}), and applying Assumption \ref{assp: eigval ZZ^T}, we have  
\begin{align}
\label{eq: mt: first-stage bound}
    &\|\bar m_{t+1}\|_F^2 \leq (1-\frac{2\lambda\nu \zeta}{r})\|m_{t}\|_F^2+\|\bfQ_t\|_F^2+2\langle m_t,\frac{\zeta}{KL}\sum_{k=1}^K \bfB_\perp^{*\top} \bfP_{t,\perp}\xi_{t,L}^k(\omega_t^k)^\top \rangle \nonumber \\
    &\qquad\qquad\qquad +12U_\omega  \zeta\frac{1}{KL}\sum_{k=1}^K\|m_t\|\|x_t^k\|. 
\end{align}

Recall that $\|m_{t+1}\|_F^2 \leq \|\bar m_{t+1}\|^2_F \|\bfR^{-1}_{t+1}\|^2$,
we have \begin{align*}
    &\|m_{t+1}\|_F^2 \leq (1-\frac{2\lambda\nu \zeta}{r})\|m_{t}\|_F^2\|\bfR^{-1}_{t+1}\|^2 +\|\bfQ_t\|_F^2\|\bfR^{-1}_{t+1}\|^2\\&+2\langle m_t,\frac{\zeta}{KL}\sum_{k=1}^K \bfB_\perp^{*\top} \bfP_{t,\perp}\xi_{t,L}^k(\omega_t^k)^\top \rangle\|\bfR^{-1}_{t+1}\|^2+12U_\omega \zeta\frac{1}{KL}\sum_{k=1}^K\|m_t\|\|x_t^k\|\|\bfR^{-1}_{t+1}\|^2. 
\end{align*}

For the first term, we use Eq.\,(\ref{eqn: R related bounds: R inv}) that ${ \|\bfR_{t+1}^{-1}\|^2\leq (\frac{1}{1- U_\delta^2 U_\omega^2\zeta^2})^2}$, and for the remaining terms, we use a crude bound that $\|\bfR_{t+1}^{-1}\|^2\leq  (\frac{1}{1- U_\delta^2 U_\omega^2\zeta^2})^2 \leq 2$ when $ U_\delta^2 U_\omega^2\zeta^2\leq \frac{1}{6}$, then,
\begin{align*}
    &\|m_{t+1}\|_F^2 \leq (1-\frac{2\lambda\nu \zeta}{r}) (\frac{1}{1- U_\delta^2 U_\omega^2\zeta^2})^2\|m_{t}\|_F^2+2\|\bfQ_t\|_F^2 +4\langle m_t,\frac{\zeta}{KL}\sum_{k=1}^K \bfB_\perp^{*\top} \bfP_{t,\perp}\xi_{t,L}^k(\omega_t^k)^\top \rangle \\
    &~~~~~~~~+24U_\omega\zeta\frac{1}{KL}\sum_{k=1}^K\|m_t\|\|x_t^k\|. 
\end{align*}
Moreover, since $U_\delta^2 U_\omega^2\zeta^2\leq \frac{1}{6}$, we have $(\frac{1}{1-U_\delta^2 U_\omega^2\zeta^2})^2 \leq 1+3U_\delta^2 U_\omega^2\zeta^2$.
Providing $\frac{\lambda\nu }{r} \geq 3U_\delta^2 U_\omega^2\zeta$, it holds that 
\begin{align*}
    &(1-\frac{2\lambda\nu \zeta}{r}) (\frac{1}{1- U_\delta^2 U_\omega^2\zeta^2})^2 \leq (1-\frac{2\lambda\nu\zeta}{r})(1+3U_\delta^2 U_\omega^2\zeta^2) \leq (1-\frac{\lambda\nu \zeta}{r}).
\end{align*}
Then, 
\begin{align}
\label{eqn: m squared intermediate bound}
    &\|m_{t+1}\|_F^2 \leq (1-\frac{\lambda\nu \zeta}{r}) \|m_{t}\|_F^2+2\|\bfQ_t\|_F^2 +4\langle m_t,\frac{\zeta}{KL}\sum_{k=1}^K \bfB_\perp^{*\top} \bfP_{t,\perp}\xi_{t,L}^k(\omega_t^k)^\top \rangle\nonumber\\
    &~~~~~~~~+24U_\omega \zeta\frac{1}{KL}\sum_{k=1}^K\|m_t\|\|x_t^k\|. %
\end{align}

By Young's inequality, we have for any $c_1>0$,
\begin{align*}
    &24U_\omega \zeta\frac{1}{KL}\sum_{k=1}^K\|m_t\|\|x_t^k\|\leq c_1^2\zeta\|m_t\|_F^2 + \frac{144U_\omega^2}{c_1^2}\frac{\zeta}{L^2}\frac{1}{K}\sum_{k=1}^K\|x_t^k\|^2.
\end{align*}
Putting this back to \eqref{eqn: m squared intermediate bound}, we get,
\begin{align}
    \label{eqn: m squared intermediate bound after young}
    &\|m_{t+1}\|_F^2 \leq (1-\frac{\lambda\nu \zeta}{r}+ c_1^2\zeta) \|m_{t}\|_F^2+2\|\bfQ_t\|_F^2 +4\langle m_t,\frac{\zeta}{KL}\sum_{k=1}^K \bfB_\perp^{*\top} \bfP_{t,\perp}\xi_{t,L}^k(\omega_t^k)^\top \rangle\nonumber\\
    &~~~~~~~~
    + \frac{144U_\omega^2}{c_1^2}\frac{\zeta}{L^2}\frac{1}{K}\sum_{k=1}^K\|x_t^k\|^2.
\end{align}

Next, we bound the third term of Eq.\,(\ref{eqn: m squared intermediate bound after young}) by quantities involving $\|\bfB_t-\bfB_{t-\tau}\|^2$, $\|\omega_t^k-\omega_{t-\tau}^k\|^2$, $\lnorm{\tilde b_{t,L}^k - \bar b_{L,\theta_t^k}^k}{}^2$, and $\lnorm{\tilde A_{t,L}^k-A_{L,\theta_t^k}^k}{}^2$. 
These terms can subsequently be controlled using the newly established bounds in Lemma \ref{lmm: U delta}, Lemma \ref{lmm: MC observable mixing}, Lemma \ref{lmm: one step temporal difference bound of omega}, and Lemma \ref{lmm: one-step perturbation of B}. 
Specially, 
\begin{align}
\label{eqn: markovian noise in M decomp}
    &\langle m_t,\frac{\zeta}{KL}\sum_{k=1}^K \bfB_\perp^{*\top}{ \bfP_{t,\perp}}\xi_{t,L}^k(\omega_t^k)^\top \rangle\nonumber\\
    =&\langle m_t,\frac{\zeta}{KL}\sum_{k=1}^K \bfB_\perp^{*\top}{ \bfP_{t,\perp}}\big[\tilde b_{t,L}^k - \bar b_{L,\theta_t^k}^k+ (\tilde A_{t,L}^k-A_{L,\theta_t^k}^k)\bfB_t\omega_t^k\big](\omega_t^k)^\top \rangle ~~~~~(\text{by \eqref{eqn: xi decomp}})\nonumber\\
    =&\langle m_t,\frac{\zeta}{KL}\sum_{k=1}^K \bfB_\perp^{*\top}{ \bfP_{t,\perp}}\big[\tilde b_{t,L}^k - \bar b_{L,\theta_t^k}^k\big](\omega_t^k)^\top \rangle
    + \langle m_t,\frac{\zeta}{KL}\sum_{k=1}^K \bfB_\perp^{*\top}{ \bfP_{t,\perp}} (\tilde A_{t,L}^k-A_{L,\theta_t^k}^k)\bfB_t\omega_t^k(\omega_t^k)^\top \rangle\nonumber\\
    =&\langle m_t, \bfB_\perp^{*\top}{ \bfP_{t,\perp}}\frac{\zeta}{KL}\sum_{k=1}^K\big[\tilde b_{t,L}^k - \bar b_{L,\theta_t^k}^k\big](\omega_t^k- \omega_{t-\tau}^k)^\top \rangle
    \nonumber\\
    &+\langle m_t-m_{t-\tau}, \bfB_\perp^{*\top}{ \bfP_{t,\perp}}\frac{\zeta}{KL}\sum_{k=1}^K\big[\tilde b_{t,L}^k - \bar b_{L,\theta_t^k}^k\big](\omega_{t-\tau}^k)^\top \rangle \nonumber\\
    &+\langle m_{t-\tau}, \bfB_\perp^{*\top} (\bfP_{t,\perp}-\bfP_{t-\tau,\perp})\frac{\zeta}{KL}\sum_{k=1}^K\big[\tilde b_{t,L}^k - \bar b_{L,\theta_t^k}^k\big](\omega_{t-\tau}^k)^\top \rangle \nonumber\\
    &+ \underbrace{\langle m_{t-\tau}, \bfB_\perp^{*\top} \bfP_{t-\tau,\perp}\frac{\zeta}{KL}\sum_{k=1}^K\big[\tilde b_{t,L}^k - \bar b_{L,\theta_t^k}^k\big](\omega_{t-\tau}^k)^\top \rangle}_{I_1}\nonumber\\
    &+ \langle m_t, \bfB_\perp^{*\top}{ \bfP_{t,\perp}} \frac{\zeta}{KL}\sum_{k=1}^K(\tilde A_{t,L}^k-A_{L,\theta_t^k}^k)\pth{\bfB_t-\bfB_{t-\tau}}\omega_t^k(\omega_t^k)^\top \rangle \nonumber\\
    &+ \langle m_t, \bfB_\perp^{*\top}{ \bfP_{t,\perp}} \frac{\zeta}{KL}\sum_{k=1}^K(\tilde A_{t,L}^k-A_{L,\theta_t^k}^k)\bfB_{t-\tau}(\omega_t^k-\omega_{t-\tau}^k)(\omega_t^k)^\top \rangle           \nonumber\\
    &+ \langle m_t, \bfB_\perp^{*\top}{ \bfP_{t,\perp}} \frac{\zeta}{KL}\sum_{k=1}^K(\tilde A_{t,L}^k-A_{L,\theta_t^k}^k)\bfB_{t-\tau}\omega_{t-\tau}^k(\omega_t^k-\omega_{t-\tau}^k)^\top \rangle           \nonumber\\
    &+ \langle m_t-m_{t-\tau}, \bfB_\perp^{*\top}{ \bfP_{t,\perp}} \frac{\zeta}{KL}\sum_{k=1}^K(\tilde A_{t,L}^k-A_{L,\theta_t^k}^k)\bfB_{t-\tau}\omega_{t-\tau}^k(\omega_{t-\tau}^k)^\top \rangle           \nonumber\\
    &+ \langle m_{t-\tau}, \bfB_\perp^{*\top}\pth{ \bfP_{t,\perp}-\bfP_{t-\tau,\perp}} \frac{\zeta}{KL}\sum_{k=1}^K(\tilde A_{t,L}^k-A_{L,\theta_t^k}^k)\bfB_{t-\tau}\omega_{t-\tau}^k(\omega_{t-\tau}^k)^\top \rangle           \nonumber\\
    &+ \underbrace{\langle m_{t-\tau}, \bfB_\perp^{*\top}{ \bfP_{t-\tau,\perp}} \frac{\zeta}{KL}\sum_{k=1}^K(\tilde A_{t,L}^k-A_{L,\theta_t^k}^k)\bfB_{t-\tau}\omega_{t-\tau}^k(\omega_{t-\tau}^k)^\top \rangle}_{I_2}       .
\end{align}
For the first term in Eq.\,\eqref{eqn: markovian noise in M decomp}, we have  
\begin{align*}
    &\langle m_t, \bfB_\perp^{*\top}{ \bfP_{t,\perp}}\frac{\zeta}{KL}\sum_{k=1}^K\big[\tilde b_{t,L}^k - \bar b_{L,\theta_t^k}^k\big](\omega_t^k- \omega_{t-\tau}^k)^\top \rangle\\
    \leq& \frac{\zeta}{KL}\sum_{k=1}^K\lnorm{\tilde b_{t,L}^k - \bar b_{L,\theta_t^k}^k}{} \lnorm{\omega_t^k- \omega_{t-\tau}^k}{} ~~~~~~~~~~  \text{(since $\|m_t\|\le 1$)}\\
    \leq& \frac{1}{2}\frac{\zeta^2}{K L^2}\sum_{k=1}^K \lnorm{\tilde b_{t,L}^k - \bar b_{L,\theta_t^k}^k}{}^2 + \frac{1}{2K}\sum_{k=1}^K \lnorm{\omega_t^k- \omega_{t-\tau}^k}{}^2.
\end{align*}

Note that 
\begin{align*}
m_t - m_{t-\tau} = \bfB^*_{\perp}(\bfB_t - \bfB_{t-\tau}), ~~~~\text{and} ~~~ 
\bfP_{t,\perp}-\bfP_{t-\tau,\perp} = \bfI - \bfB_t(\bfB_t)^{\top} - \pth{\bfI - \bfB_{t-\tau}(\bfB_{t-\tau})^{\top}}.  
\end{align*}
We have 
\begin{align*}
 \|m_t - m_{t-\tau}\| \le \|\bfB_t - \bfB_{t-\tau}\|, ~~~\text{and} ~~~ \|\bfP_{t,\perp}-\bfP_{t-\tau,\perp} \|\le 2 \|\bfB_t - \bfB_{t-\tau}\|.    
\end{align*}

Applying similar analysis of the first term of Eq.\,(\ref{eqn: markovian noise in M decomp}) to all other terms in Eq.\,(\ref{eqn: markovian noise in M decomp}), we obtain 
\begin{align*}
    &\langle m_t-m_{t-\tau}, \bfB_\perp^{*\top}{ \bfP_{t,\perp}}\frac{\zeta}{KL}\sum_{k=1}^K\big[\tilde b_{t,L}^k - \bar b_{L,\theta_t^k}^k\big](\omega_{t-\tau}^k)^\top \rangle \\
    &\qquad \qquad \qquad \qquad\leq\frac{1}{2}\|\bfB_t-\bfB_{t-\tau}\|^2 + \frac{U_\omega^2}{2}\frac{\zeta^2}{K L^2} \sum_{k=1}^K\lnorm{\tilde b_{t,L}^k - \bar b_{L,\theta_t^k}^k }{}^2 ,   \\
    &\langle m_{t-\tau}, \bfB_\perp^{*\top} (\bfP_{t,\perp}-\bfP_{t-\tau,\perp})\frac{\zeta}{KL}\sum_{k=1}^K\big[\tilde b_{t,L}^k - \bar b_{L,\theta_t^k}^k\big](\omega_{t-\tau}^k)^\top \rangle\\
    &\qquad \qquad \qquad \qquad\leq2\|\bfB_{t}-\bfB_{t-\tau}\|^2+\frac{1}{2}U_\omega^2\frac{\zeta^2}{KL^2} \sum_{k=1}^K\lnorm{\tilde b_{t,L}^k - \bar b_{L,\theta_t^k}^k }{}^2,\\
       & \langle m_t, \bfB_\perp^{*\top}{ \bfP_{t,\perp}} \frac{\zeta}{KL}\sum_{k=1}^K(\tilde A_{t,L}^k-A_{L,\theta_t^k}^k)\pth{\bfB_t-\bfB_{t-\tau}}\omega_t^k(\omega_t^k)^\top \rangle \\
&\qquad \qquad \qquad \qquad\leq  \frac{1}{2} U_\omega^4\frac{\zeta^2}{KL^2}\sum_{k=1}^K\|\tilde A_{t,L}^k-A_{L,\theta_t^k}^k\|^2 + \frac{1}{2} \|\bfB_t-\bfB_{t-\tau} \|^2,  \\
    & \langle m_t, \bfB_\perp^{*\top}{ \bfP_{t,\perp}} \frac{\zeta}{KL}\sum_{k=1}^K(\tilde A_{t,L}^k-A_{L,\theta_t^k}^k)\bfB_{t-\tau}(\omega_t^k-\omega_{t-\tau}^k)(\omega_t^k)^\top \rangle \\          
    &\qquad \qquad \qquad \qquad
    \leq\frac{U_\omega}{2} \frac{\zeta^2}{KL^2}\sum_{k=1}^K\|\tilde A_{t,L}^k-A_{L,\theta_t^k}^k \|^2+\frac{U_\omega}{2}\! \frac{1}{K}\sum_{k=1}^K\|\omega_t^k-\omega_{t-\tau}^k \|^2,           \\
    & \langle m_t, \bfB_\perp^{*\top}{ \bfP_{t,\perp}} \frac{\zeta}{KL}\sum_{k=1}^K(\tilde A_{t,L}^k-A_{L,\theta_t^k}^k)\bfB_{t-\tau}\omega_{t-\tau}^k(\omega_t^k-\omega_{t-\tau}^k)^\top \rangle           \\
    &\qquad \qquad \qquad \qquad
    \leq \frac{U_\omega}{2}\frac{\zeta^2}{KL^2}\sum_{k=1}^K\|\tilde A_{t,L}^k-A_{L,\theta_t^k}^k\|^2+ \frac{U_\omega}{2} \frac{1}{K}\sum_{k=1}^K\|\omega_t^k-\omega_{t-\tau}^k\|^2 ,         \\
    & \langle m_t-m_{t-\tau}, \bfB_\perp^{*\top}{ \bfP_{t,\perp}} \frac{\zeta}{KL}\sum_{k=1}^K(\tilde A_{t,L}^k-A_{L,\theta_t^k}^k)\bfB_{t-\tau}\omega_{t-\tau}^k(\omega_{t-\tau}^k)^\top \rangle           \\
    &\qquad \qquad \qquad \qquad
    \leq \frac{1}{2}\| \bfB_t-\bfB_{t-\tau}\|^2+\frac{1}{2} U_\omega^4\frac{\zeta^2}{KL^2}\sum_{k=1}^K\|\tilde A_{t,L}^k-A_{L,\theta_t^k}^k \|^2,\\
    \displaybreak[4]
    & \langle m_{t-\tau}, \bfB_\perp^{*\top}\pth{ \bfP_{t,\perp}-\bfP_{t-\tau,\perp}} \frac{\zeta}{KL}\sum_{k=1}^K(\tilde A_{t,L}^k-A_{L,\theta_t^k}^k)\bfB_{t-\tau}\omega_{t-\tau}^k(\omega_{t-\tau}^k)^\top \rangle           \\
    &\qquad \qquad \qquad \qquad
    \leq 2\| \bfB_{t}-\bfB_{t-\tau} \|^2+\frac{U_\omega^4}{2} \frac{\zeta^2}{KL^2}\sum_{k=1}^K\|\tilde A_{t,L}^k-A_{L,\theta_t^k}^k\|^2
\end{align*}

Combining the above bounds and plugging them back to \eqref{eqn: markovian noise in M decomp}, and taking expectation conditioned on $v_{0:t-\tau}$, we get
\begin{align}
\label{eqn: mt noise combined param perturb}
    &\bbE_{t-\tau}\langle m_t,\frac{\zeta}{KL}\sum_{k=1}^K \bfB_\perp^{*\top}{ \bfP_{t,\perp}}\xi_{t,L}^k(\omega_t^k)^\top \rangle\nonumber\\
    \leq&
    \pth{\frac{1}{2}+U_\omega^2}\frac{\zeta^2}{K L^2}\sum_{k=1}^K \bbE_{t-\tau}\lnorm{\tilde b_{t,L}^k - \bar b_{L,\theta_t^k}^k}{}^2 + \pth{\frac{1}{2}+U_\omega}\frac{1}{K}\sum_{k=1}^K\bbE_{t-\tau}\lnorm{\omega_t^k- \omega_{t-\tau}^k}{}^2\nonumber\\
    &+\pth{U_\omega+\frac{3}{2}U_\omega^4} \frac{\zeta^2}{KL^2}\sum_{k=1}^K\bbE_{t-\tau}\|\tilde A_{t,L}^k-A_{L,\theta_t^k}^k\|^2+  8\bbE_{t-\tau}\| \bfB_{t}-\bfB_{t-\tau} \|^2\nonumber\\
    &+ \underbrace{\langle m_{t-\tau},\frac{\zeta}{KL}\sum_{k=1}^K \bfB_\perp^{*\top}\bfP_{t-\tau,\perp}\bbE_{t-\tau}\big[\tilde b_{t,L}^k - \bar b_{L,\theta_t^k}^k\big] (\omega_{t-\tau}^k)^\top\rangle}_{I_1}\nonumber\\
    &+\underbrace{\langle m_{t-\tau},\frac{\zeta}{KL}\sum_{k=1}^K \bfB_\perp^{*\top}\bfP_{t-\tau,\perp} \bbE_{t-\tau}[\tilde A_{t,L}^k-A_{L,\theta_t^k}^k]\bfB_{t-\tau}\omega_{t-\tau}^k(\omega_{t-\tau}^k)^\top \rangle}_{I_2}.
\end{align}
Applying Lemma \ref{lmm: U delta} to the two local-noise second moments in \eqref{eqn: mt noise combined param perturb}, and using
\[
\|\omega_t^k-\omega_{t-\tau}^k\|^2
\leq \tau\sum_{j=t-\tau}^{t-1}\|\omega_{j+1}^k-\omega_j^k\|^2,
\qquad
\|\bfB_t-\bfB_{t-\tau}\|^2
\leq \tau\sum_{j=t-\tau}^{t-1}\|\bfB_{j+1}-\bfB_j\|^2,
\]
we get
\begin{align}
\label{eqn: mt noise final}
    &\bbE_{t-\tau}\langle m_t,\frac{\zeta}{KL}\sum_{k=1}^K
    \bfB_\perp^{*\top}\bfP_{t,\perp}\xi_{t,L}^k(\omega_t^k)^\top\rangle\nonumber\\
    \leq&
\frac{4\pth{\frac{1}{2}+U_\omega^2}U_r^2
+16\pth{U_\omega+\frac{3}{2}U_\omega^4}}{(1-\gamma)^2}
\frac{\zeta^2}{L^2}+\pth{\frac{1}{2}+U_\omega}\frac{\tau}{K}
\sum_{j=t-\tau}^{t-1}\sum_{k=1}^K
\bbE_{t-\tau}\|\omega_{j+1}^k-\omega_j^k\|^2\nonumber\\
&
+8\tau\sum_{j=t-\tau}^{t-1}
\bbE_{t-\tau}\|\bfB_{j+1}-\bfB_j\|^2
+I_1+I_2.
\end{align}

Putting \eqref{eqn: mt noise final} back to \eqref{eqn: m squared intermediate bound after young} and taking expectation, we get
\begin{align}
\label{eqn: m squared before fine perturb and mixing}
&\bbE_{t-\tau}\|m_{t+1}\|_F^2 \nonumber\\
\leq& \pth{1-\frac{\lambda\nu\zeta}{r}+c_1^2\zeta}\bbE_{t-\tau}\|m_{t}\|_F^2
+2\bbE_{t-\tau}\|\bfQ_t\|_F^2
+\frac{144U_\omega^2}{c_1^2}\frac{\zeta}{L^2}
\frac{1}{K}\sum_{k=1}^K\bbE_{t-\tau}\|x_t^k\|^2\nonumber\\
&+\frac{4\zeta^2}{L^2}
\frac{4\pth{\frac{1}{2}+U_\omega^2}U_r^2
+16\pth{U_\omega+\frac{3}{2}U_\omega^4}}{(1-\gamma)^2}
+4\pth{\frac{1}{2}+U_\omega}\frac{\tau}{K}
\sum_{j=t-\tau}^{t-1}\sum_{k=1}^K
\bbE_{t-\tau}\|\omega_{j+1}^k-\omega_j^k\|^2
\nonumber\\
&+32\tau\sum_{j=t-\tau}^{t-1}
\bbE_{t-\tau}\|\bfB_{j+1}-\bfB_j\|^2
+4I_1+4I_2.
\end{align}

For the remaining conditional-mean terms, using $\|\bfB_\perp^{*\top}\bfB_{t-\tau}\|\le 1$, $\|\bfP_{t-\tau,\perp}\|\le 1$, and $\|\omega_{t-\tau}^k\|\le U_\omega$, Cauchy--Schwarz and Lemma \ref{lmm: MC observable mixing} give
\begin{align}
\label{eqn: mt mixing terms}
&\bbE_{t-\tau}\qth{I_1+I_2}\nonumber\\
\leq&
\frac{\zeta}{L}
\pth{
\frac{U_\omega C_{\mathrm{mix},1}+U_\omega^2C_{\mathrm{mix},2}}{(1-\gamma)^2}
}
\Bigg(
(1-\gamma)\tau\alpha \frac{1}{K}\sum_{k=1}^K
\bbE_{t-\tau}\|\nabla J^k(\theta_t^k)\|
+\tau\alpha \frac{1}{K}\sum_{k=1}^K
\bbE_{t-\tau}\|x_t^k\|
+\frac{\tau^2\alpha c\zeta}{L(1-\gamma)}\nonumber\\
&\qquad \qquad 
+\frac{\tau^2\alpha \zeta}{L(1-\gamma)}
+\tau^2\alpha^2  
+ \frac{\tau^2\alpha^2}{(1-\gamma)^2}
+\frac{\tau\alpha}{\sqrt{L(1-\gamma)}}
\Bigg)+
\frac{\zeta}{L}
\pth{
\frac{2U_rU_\omega}{1-\gamma}+2U_\omega^2
}
m\rho^{\tau L}\nonumber\\
\leq&
\frac{\zeta}{L}
\pth{
\frac{U_\omega C_{\mathrm{mix},1}+U_\omega^2C_{\mathrm{mix},2}}{(1-\gamma)^2}
}
\Bigg(
\tau\alpha \pth{(1-\gamma)BU_\delta+2U_\omega}
+\frac{\tau^2\alpha c\zeta}{L(1-\gamma)}
+\frac{\tau^2\alpha \zeta}{L(1-\gamma)}
+\tau^2\alpha^2  
+ \frac{\tau^2\alpha^2}{(1-\gamma)^2} \nonumber\\  
& \qquad \qquad  
+ \frac{\tau\alpha }{\sqrt{L(1-\gamma)} }\Bigg) +
\frac{\zeta}{L}
\pth{
\frac{2U_rU_\omega}{1-\gamma}+2U_\omega^2
}
m\rho^{\tau L}.
\end{align}
where the second inequality uses $\|\nabla J^k(\theta_t^k)\|\leq BU_\delta$ and $\|x_t^k\|\leq 2U_\omega$.

In addition, applying Eq.\,\eqref{eqn: rough Q bound}, we have 
\begin{align*}
&2\bbE\|\bfQ_t\|_F^2
\leq\frac{2U_{\delta}^2U_\omega^2}{L^2(1-\gamma)^2}\zeta^2. 
\end{align*}

Taking full expectation on both sides of \eqref{eqn: m squared before fine perturb and mixing}, using the tower property, and using $\beta=c\zeta$, 
Recall that 
\[
M_t=\bbE\|m_t\|_F^2,\quad
\bar X_t=\frac{1}{K}\sum_{k=1}^K\bbE\|x_t^k\|^2,  \quad \bar G_t =\frac{1}{K}\sum_{k=1}^K
\bbE\|\nabla J^k(\theta_t^k)\|^2.
\]

Taking full expectation on both sides of \eqref{eqn: m squared before fine perturb and mixing}, using the tower property, using $\beta=c\zeta$, and Lemma \ref{lmm: one step temporal difference bound of omega}, we have 
\begin{align*}
&4\pth{\frac{1}{2}+U_\omega}\frac{\tau}{K}
\sum_{j=t-\tau}^{t-1}\sum_{k=1}^K
\bbE\|\omega_{j+1}^k-\omega_j^k\|^2\\
\leq&
64\pth{\frac{1}{2}+U_\omega}\frac{\tau^2c^2\zeta^2}{L^2}\bar X_t
+4\pth{\frac{1}{2}+U_\omega}\tau^2
\Bigg(
64\tau^2\frac{U_\delta^2c^4\zeta^4}{L^4(1-\gamma)^2}
+1024\tau^2U_\omega^3\frac{U_\delta^2c^2\zeta^4}{L^4(1-\gamma)^2}\nonumber\\
&\qquad\qquad\qquad\qquad
+32\tau^2L_{*,1}^2B^2U_\delta^2\frac{c^2\zeta^2\alpha^2  }{L^2}
+2\frac{c^2\zeta^2}{L^2}\frac{U_\delta^2}{(1-\gamma)^2}
\Bigg).
\end{align*}
Lemma \ref{lmm: one-step perturbation of B} gives
\begin{align*}
&32\tau\sum_{j=t-\tau}^{t-1}
\bbE\|\bfB_{j+1}-\bfB_j\|^2\\
\leq&
6144\tau^2U_\omega^2\frac{\zeta^2}{L^2}\bar X_t+32\tau^2
\Bigg(
768\tau^2\frac{U_\delta^2U_\omega^2c^2\zeta^4}{L^4(1-\gamma)^2}
+12288\frac{U_\delta^2\tau^2U_\omega^5\zeta^4}{L^4(1-\gamma)^2}
+384\tau^2U_\omega^2L_{*,1}^2B^2U_\delta^2
\frac{\alpha^2  \zeta^2}{L^2}\nonumber\\
&\qquad
+24U_\omega^2\frac{U_\delta^2}{L^2(1-\gamma)^2}\zeta^2
+32\frac{U_\delta^4U_\omega^4\zeta^4}{L^4(1-\gamma)^4}
\Bigg).
\end{align*}
Substituting these displays and the full-expectation version of \eqref{eqn: mt mixing terms} into the full-expectation version of \eqref{eqn: m squared before fine perturb and mixing} yields
\begin{align}
\label{eqn: Mt+1 expanded}
&M_{t+1}
\leq
\pth{1-\frac{\lambda\nu\zeta}{r}+c_1^2\zeta}M_t\nonumber\\
&+
\Bigg(
\frac{144U_\omega^2}{c_1^2}\frac{\zeta}{L^2}
+64\pth{\frac{1}{2}+U_\omega}\frac{\tau^2c^2\zeta^2}{L^2}
+6144\tau^2U_\omega^2\frac{\zeta^2}{L^2}
\Bigg)\bar X_t\nonumber\\
&+
\frac{4\zeta}{L}
\pth{
\frac{U_\omega C_{\mathrm{mix},1}+U_\omega^2C_{\mathrm{mix},2}}{(1-\gamma)^2}
}
\Bigg(
\tau\alpha \pth{(1-\gamma)BU_\delta+2U_\omega}
+\frac{\tau^2\alpha c\zeta}{L(1-\gamma)}
+\frac{\tau^2\alpha \zeta}{L(1-\gamma)}\nonumber\\
&\qquad\qquad\qquad\qquad\qquad\qquad\qquad\qquad\qquad
+\tau^2\alpha^2  
+\frac{\tau^2\alpha^2}{(1-\gamma)^2}
+\frac{\tau\alpha }{\sqrt{L(1-\gamma)}}
\Bigg)\nonumber\\
&+
\frac{4\zeta}{L}
\pth{
\frac{2U_rU_\omega}{1-\gamma}+2U_\omega^2
}
m\rho^{\tau L}
+
\frac{4\zeta^2}{L^2}
\frac{12\pth{\frac{1}{2}+U_\omega^2}U_r^2
+48\pth{U_\omega+\frac{3}{2}U_\omega^4}}{(1-\gamma)^2}\nonumber\\
&+
4\pth{\frac{1}{2}+U_\omega}\tau^2
\Bigg(
64\tau^2\frac{U_\delta^2c^4\zeta^4}{L^4(1-\gamma)^2}
+1024\tau^2U_\omega^3\frac{U_\delta^2c^2\zeta^4}{L^4(1-\gamma)^2}
\nonumber\\
&\qquad\qquad\qquad\qquad\qquad\qquad\qquad\qquad
+32\tau^2L_{*,1}^2B^2U_\delta^2\frac{c^2\zeta^2\alpha^2  }{L^2}
+6\frac{c^2\zeta^2}{L^2}\frac{U_\delta^2}{(1-\gamma)^2}
\Bigg)\nonumber\\
&+
32\tau^2
\Bigg(
768\tau^2\frac{U_\delta^2U_\omega^2c^2\zeta^4}{L^4(1-\gamma)^2}
+12288\frac{U_\delta^2\tau^2U_\omega^5\zeta^4}{L^4(1-\gamma)^2}
+384\tau^2U_\omega^2L_{*,1}^2B^2U_\delta^2
\frac{\alpha^2  \zeta^2}{L^2}\nonumber\\
&\qquad
+24U_\omega^2\frac{U_\delta^2}{L^2(1-\gamma)^2}\zeta^2
+32\frac{U_\delta^4U_\omega^4\zeta^4}{L^4(1-\gamma)^4}
\Bigg)+\frac{2U_\delta^2U_\omega^2}{L^2(1-\gamma)^2}\zeta^2
\nonumber\\
&\qquad
+168U_\omega^2\frac{U_\delta^2}{(1-\gamma)^2}
\frac{\zeta^2}{L^2}m^2\rho^{2\tau L}.
\end{align}
Under the single-timescale parametrization $\beta=c\zeta$ and $\alpha =c_\theta\zeta$ with fixed $c,c_\theta>0$, define
\begin{align}
\label{eqn: C_M constants}
\calC_{M,1}(\tau^4;c,c_\theta)
:=&\;
64\pth{\frac{1}{2}+U_\omega}\tau^2c^2
+6144\tau^2U_\omega^2,\nonumber\\
\calC_{M,2}(\tau^2;c,c_\theta)
:=&\;
4\tau c_\theta
\pth{U_\omega C_{\mathrm{mix},1}+U_\omega^2C_{\mathrm{mix},2}}
\pth{(1-\gamma)BU_\delta+2U_\omega+1}\nonumber\\
&+4\qth{
12\pth{\frac{1}{2}+U_\omega^2}U_r^2
+48\pth{U_\omega+\frac{3}{2}U_\omega^4}
}\nonumber\\
&+24\pth{\frac{1}{2}+U_\omega}\tau^2c^2U_\delta^2
+\pth{2+768\tau^2}U_\omega^2U_\delta^2.
\end{align}
Let $\calC_{M,3}(\tau^6;c,c_\theta)$ be a coefficient envelope for the terms in \eqref{eqn: Mt+1 expanded} not collected in $\calC_{M,1}$ or $\calC_{M,2}$ after using $\beta=c\zeta$, $\alpha =c_\theta\zeta$, $\zeta\le 1$, $L,K\ge 1$, and $\tau L=\lceil 2\log_\rho(\zeta)\rceil$, which gives $m\rho^{\tau L}\le m\zeta^2$.
The terms absorbed into $\calC_{M,3}$ are the remaining $\zeta^3$-and-higher terms, including the mixing residuals after $m\rho^{\tau L}\le m\zeta^2$.
Then
\begin{align}
\label{eqn: Mt+1 final}
&M_{t+1}
\leq
\pth{1-\frac{\lambda\nu\zeta}{r}+c_1^2\zeta}M_t+\pth{
\frac{144U_\omega^2}{c_1^2}\frac{\zeta}{L^2}
+\frac{\calC_{M,1}(\tau^4;c,c_\theta)}{(1-\gamma)^2}
\frac{\zeta^2}{L^2}
}\bar X_t\nonumber\\
&+\frac{\calC_{M,2}(\tau^2;c,c_\theta)}{(1-\gamma)^3}
\pth{
\frac{\zeta^2}{L}
+\frac{\zeta^2}{K}
}
+\frac{\calC_{M,3}(\tau^6;c,c_\theta)}{(1-\gamma)^4}\zeta^3.
\end{align}

\newpage 

\section{Proof of Lemma \ref{lm: policy gradient analysis} (Policy Gradient Analysis)}

Recall from Eq.\,(\ref{eq: actor one step TD error}) that the actor uses the minibatch estimator
\begin{align}
\label{eqn: actor minibatch estimator}
    g_{t,L}^k
    =
    \frac{1}{L}\sum_{\ell=0}^{L-1}
    \delta_{t,\ell}^{k,\mathrm{act}}
    \nabla_\theta \log \pi_{\theta_t^k}(\hat a_{t,\ell}^k|\hat s_{t,\ell}^k).
\end{align}
By smoothness of $J^k(\theta)$ and the actor update
$\theta_{t+1}^k=\theta_t^k+\alpha g_{t,L}^k$, we have
\begin{align*}
    J^k(\theta_{t+1}^k)
    \geq&
    J^k(\theta_t^k)
    +\alpha \iprod{\nabla J^k(\theta_t^k)}{g_{t,L}^k}
    -\frac{L_{J'}}{2}\alpha^2  \|g_{t,L}^k\|^2\\
    =&
    J^k(\theta_t^k)
    +(1-\gamma)\alpha \|\nabla J^k(\theta_t^k)\|^2
    +\alpha \iprod{\nabla J^k(\theta_t^k)}{g_{t,L}^k-(1-\gamma)\nabla J^k(\theta_t^k)}
    -\frac{L_{J'}}{2}\alpha^2  \|g_{t,L}^k\|^2.
\end{align*}
Rearranging and taking full expectation gives
\begin{align}
\label{eqn: expectation of policy gradient}
    &\bbE\|\nabla J^k(\theta_t^k)\|^2
    \leq
    \frac{1}{\alpha (1-\gamma)}\bbE\qth{J^k(\theta_{t+1}^k)-J^k(\theta_t^k)}
    +\frac{1}{1-\gamma} \calE_{t,L}^k
    +\frac{L_{J'}}{2(1-\gamma)}\alpha \bbE\|g_{t,L}^k\|^2,
\end{align}
where
\begin{align*}
    \calE_{t,L}^k
    :=
    \left|
    \bbE
    \iprod{\nabla J^k(\theta_t^k)}
    {g_{t,L}^k-(1-\gamma)\nabla J^k(\theta_t^k)}
    \right|.
\end{align*}
We first derive a fine-grained bound for the second-order term. By the
tower property,
\begin{align}
\label{eqn: actor second order tower}
\bbE\|g_{t,L}^k\|^2
=
\bbE\qth{
\bbE\qth{\|g_{t,L}^k\|^2\mid v_{0:t-\tau}}}
= 
\bbE\qth{
\bbE_{t-\tau}\qth{\|g_{t,L}^k\|^2}
}.
\end{align}
Since $\theta_{t+1}^k-\theta_t^k=\alpha g_{t,L}^k$, the second-order
bound in Lemma \ref{lmm: theta onestep distance} with $j=t$ gives
\begin{align}
\label{eqn: actor second order fine bound}
&\bbE_{t-\tau}\qth{\|g_{t,L}^k\|^2}\nonumber\\
\leq&\;
8(1-\gamma)^2\bbE_{t-\tau}\qth{\|\nabla J^k(\theta_t^k)\|^2}
+C_{7}^{\mathrm{act}}\bbE_{t-\tau}\qth{\|x_t^k\|^2}
+C_{8}^{\mathrm{act}}\frac{\tau^2\beta^2}{L^2(1-\gamma)^2}\nonumber\\
&+
C_{9}^{\mathrm{act}}\frac{\tau^2\zeta^2}{L^2(1-\gamma)^2}
+C_{10}^{\mathrm{act}}\tau^2\alpha^2
+C_{11}^{\mathrm{act}}\frac{\tau\alpha}{(1-\gamma)^2}
+C_{12}^{\mathrm{act}}\frac{1}{L(1-\gamma)}.
\end{align}
Combining \eqref{eqn: actor second order tower} and
\eqref{eqn: actor second order fine bound}, we obtain
\begin{align}
\label{eqn: actor second order fine scaled}
\frac{L_{J'}}{2(1-\gamma)}\alpha\bbE\|g_{t,L}^k\|^2
\leq&\;
4(1-\gamma)L_{J'}\alpha\bbE\|\nabla J^k(\theta_t^k)\|^2
+\frac{C_{7}^{\mathrm{act}}L_{J'}}{2(1-\gamma)}\alpha\bbE\|x_t^k\|^2\nonumber\\
&+
\frac{C_{8}^{\mathrm{act}}L_{J'}}{2}
\frac{\tau^2\alpha\beta^2}{L^2(1-\gamma)^3}
+
\frac{C_{9}^{\mathrm{act}}L_{J'}}{2}
\frac{\tau^2\alpha\zeta^2}{L^2(1-\gamma)^3}\nonumber\\
&+
\frac{C_{10}^{\mathrm{act}}L_{J'}}{2}\frac{\tau^2\alpha^3}{1-\gamma}
+\frac{C_{11}^{\mathrm{act}}L_{J'}}{2}\frac{\tau\alpha^2}{(1-\gamma)^3}
+\frac{C_{12}^{\mathrm{act}}L_{J'}}{2}\frac{\alpha}{L(1-\gamma)^2}.
\end{align}

It remains to bound $\calE_{t,L}^k$. The one-step policy gradient identity gives
\begin{align}
\label{eqn: one step policy gradient formula}
    (1-\gamma)\nabla J^k(\theta_t^k)
    =
    \bbE_{\nu_{\theta_t^k}^k,\pi_{\theta_t^k},P^k}
    \qth{
    \pth{
    R(\hat s,\hat a)+(\gamma\phi(\hat s')-\phi(\hat s))^\top\mathbf{B}^*\omega_t^{k,*}
    }
    \nabla_\theta\log\pi_{\theta_t^k}(\hat a|\hat s)
    }.
\end{align}
Write $\nabla J_t^k:=\nabla J^k(\theta_t^k)$ and $\nabla J_{t-\tau}^k:=\nabla J^k(\theta_{t-\tau}^k)$. Using the notation of \eqref{eqn: def of g}, we insert and subtract the true TD fixed-point copy, the frozen-parameter original-chain copy, and the frozen-parameter auxiliary-chain copy:
\begin{align*}
D_1
:=&\;
\bbE\iprod{\nabla J_t^k}
{g_{t,L}^k
-\frac{1}{L}\sum_{\ell=0}^{L-1}
g(\hat O_{t,\ell}^k,\bfB^*,\omega_t^{k,*},\theta_t^k)},\\
D_2
:=&\;
\bbE\iprod{\nabla J_t^k}
{\frac{1}{L}\sum_{\ell=0}^{L-1}
g(\hat O_{t,\ell}^k,\bfB^*,\omega_t^{k,*},\theta_t^k)-(1-\gamma)\nabla J_t^k}\\
&-
\bbE\iprod{\nabla J_{t-\tau}^k}
{\frac{1}{L}\sum_{\ell=0}^{L-1}
g(\hat O_{t,\ell}^k,\bfB^*,\omega_{t-\tau}^{k,*},\theta_{t-\tau}^k)-(1-\gamma)\nabla J_{t-\tau}^k},\\
D_3
:=&\;
\bbE\iprod{\nabla J_{t-\tau}^k}
{\frac{1}{L}\sum_{\ell=0}^{L-1}
g(\hat O_{t,\ell}^k,\bfB^*,\omega_{t-\tau}^{k,*},\theta_{t-\tau}^k)-(1-\gamma)\nabla J_{t-\tau}^k}\\
&-
\bbE\iprod{\nabla J_{t-\tau}^k}
{\frac{1}{L}\sum_{\ell=0}^{L-1}
g(\bar O_{t,\ell}^k,\bfB^*,\omega_{t-\tau}^{k,*},\theta_{t-\tau}^k)-(1-\gamma)\nabla J_{t-\tau}^k},\\
D_4
:=&\;
\bbE\iprod{\nabla J_{t-\tau}^k}
{\frac{1}{L}\sum_{\ell=0}^{L-1}
g(\bar O_{t,\ell}^k,\bfB^*,\omega_{t-\tau}^{k,*},\theta_{t-\tau}^k)-(1-\gamma)\nabla J_{t-\tau}^k}.
\end{align*}
Then
\begin{align}
\label{eqn: actor bias four diff start}
\calE_{t,L}^k
\leq |D_1|+|D_2|+|D_3|+|D_4|.
\end{align}
We now bound the four terms explicitly before collecting coefficients. The first difference is the critic approximation error. Since
$x_t^k=\bfB_t\omega_t^k-\bfB^*\omega_t^{k,*}$,
\begin{align}
\label{eqn: actor bias D1 bound}
|D_1|
&\leq
2B\,\bbE\qth{\|\nabla J_t^k\|\,\|x_t^k\|}
\leq
\frac{1-\gamma}{8}\bbE\|\nabla J_t^k\|^2
+\frac{8B^2}{1-\gamma}\bbE\|x_t^k\|^2.
\end{align}
For the parameter-drift term, Lemma \ref{lmm: LJ' lip}, Lemma \ref{lmm: L* lip}, Assumption \ref{assmp: Lipschitz of policy}, and
$\|\theta_t^k-\theta_{t-\tau}^k\|\leq\sum_{i=t-\tau}^{t-1}\|\theta_{i+1}^k-\theta_i^k\|$ give
\begin{align}
\label{eqn: actor bias D2 bound}
|D_2|
\leq&
\left[
2BU_\delta L_{J'}
+\frac{BU_\delta}{1-\gamma}
\pth{2BL_{*,1}+U_\delta L_g+(1-\gamma)L_{J'}}
\right]
\sum_{i=t-\tau}^{t-1}
\bbE\|\theta_{i+1}^k-\theta_i^k\|.
\end{align}
For the original-chain versus auxiliary-chain mismatch, Corollary \ref{cor: AC, OC dist on f with a0: actor} and the variational characterization of total variation give
\begin{align}
\label{eqn: actor bias D3 bound}
|D_3|
\leq&
\frac{B^2U_\delta^2}{1-\gamma}
\left(
\frac{1}{(1-\gamma^L)(1-\gamma)}
+
\frac{3-2\gamma}{1-\gamma}
\right)
L_\pi|\calA|
\sum_{i=t-\tau}^{t-1}
\bbE\|\theta_{i+1}^k-\theta_i^k\| \nonumber\\
\leq&
\frac{4B^2U_\delta^2}{(1-\gamma)^3}L_\pi|\calA|
\sum_{i=t-\tau}^{t-1}
\bbE\|\theta_{i+1}^k-\theta_i^k\|.
\end{align}
Finally, by the geometric mixing of the actor reset chain applied with the parameters frozen at $t-\tau$,
\begin{align}
\label{eqn: actor bias D4 bound}
|D_4|
\leq
\frac{2B^2U_\delta^2}{1-\gamma}\gamma^{\tau L}.
\end{align}
Define the explicit drift coefficient
\begin{align}
\label{eqn: actor bias Lambda G}
\Lambda_G
:=&\;
2BU_\delta L_{J'}
+\frac{BU_\delta}{1-\gamma}
\pth{2BL_{*,1}+U_\delta L_g+(1-\gamma)L_{J'}}
+\frac{4B^2U_\delta^2}{(1-\gamma)^3}L_\pi|\calA|.
\end{align}
Combining \eqref{eqn: actor bias four diff start}--\eqref{eqn: actor bias D4 bound}, we obtain
\begin{align}
\label{eqn: actor minibatch perturbation bound}
\calE_{t,L}^k
\leq&
\frac{1-\gamma}{8}\bbE\|\nabla J_t^k\|^2
+\frac{8B^2}{1-\gamma}\bbE\|x_t^k\|^2
+\Lambda_G
\sum_{i=t-\tau}^{t-1}
\bbE\|\theta_{i+1}^k-\theta_i^k\|
+\frac{2B^2U_\delta^2}{1-\gamma}\gamma^{\tau L}.
\end{align}
Lemma \ref{lmm: theta onestep distance} implies, after taking full expectation and summing over $i=t-\tau,\ldots,t-1$,
\begin{align}
\label{eqn: actor movement sum bound for PG}
\sum_{i=t-\tau}^{t-1}\bbE\|\theta_{i+1}^k-\theta_i^k\|
\leq&
(1-\gamma)\tau\alpha \,\bbE\|\nabla J_t^k\|
+C_{1}^{\mathrm{act}}\tau\alpha \,\bbE\|x_t^k\|
+C_{2}^{\mathrm{act}}\frac{\tau^2\alpha \beta}{L(1-\gamma)}
+C_{3}^{\mathrm{act}}\frac{\tau^2\alpha \zeta}{L(1-\gamma)}
\nonumber\\
&\quad
+C_{4}^{\mathrm{act}}\tau^2\alpha^2  
+C_{5}^{\mathrm{act}}\frac{\tau\alpha }{\sqrt{L(1-\gamma)}}
+C_{6}^{\mathrm{act}}\frac{\tau^2\alpha^2}{(1-\gamma)^2}.
\end{align}
Applying Young's inequality to the first two terms on the right-hand side of
\eqref{eqn: actor movement sum bound for PG} gives
\begin{align}
\label{eqn: actor movement absorbed for PG}
\Lambda_G
\sum_{i=t-\tau}^{t-1}
\bbE\|\theta_{i+1}^k-\theta_i^k\|
\leq&
\frac{1-\gamma}{8}\bbE\|\nabla J_t^k\|^2
+\bbE\|x_t^k\|^2
\nonumber\\
&+
2\Lambda_G^2(1-\gamma)\tau^2\alpha^2
+\pth{\frac{\Lambda_G^2(C_{1}^{\mathrm{act}})^2}{4}+\Lambda_GC_{4}^{\mathrm{act}}}
\tau^2\alpha^2
+\Lambda_GC_{2}^{\mathrm{act}}\frac{\tau^2\alpha \beta}{L(1-\gamma)}\nonumber\\
&+
\Lambda_GC_{3}^{\mathrm{act}}\frac{\tau^2\alpha \zeta}{L(1-\gamma)}
+\Lambda_GC_{5}^{\mathrm{act}}\frac{\tau\alpha }{\sqrt{L(1-\gamma)}}
+\Lambda_GC_{6}^{\mathrm{act}}\frac{\tau^2\alpha^2}{(1-\gamma)^2}.
\end{align}
Substituting \eqref{eqn: actor movement absorbed for PG} into
\eqref{eqn: actor minibatch perturbation bound} yields
\begin{align}
\label{eqn: actor minibatch perturbation bound final}
\calE_{t,L}^k
\leq&
\frac{1-\gamma}{4}\bbE\|\nabla J_t^k\|^2
+\pth{\frac{8B^2}{1-\gamma}+1}\bbE\|x_t^k\|^2
+2\Lambda_G^2(1-\gamma)\tau^2\alpha^2\nonumber\\
&
+\pth{\frac{\Lambda_G^2(C_{1}^{\mathrm{act}})^2}{4}+\Lambda_GC_{4}^{\mathrm{act}}}
\tau^2\alpha^2
+\Lambda_GC_{2}^{\mathrm{act}}\frac{\tau^2\alpha \beta}{L(1-\gamma)}\nonumber\\
&+
\Lambda_GC_{3}^{\mathrm{act}}\frac{\tau^2\alpha \zeta}{L(1-\gamma)}
+\Lambda_GC_{5}^{\mathrm{act}}\frac{\tau\alpha }{\sqrt{L(1-\gamma)}}
+\Lambda_GC_{6}^{\mathrm{act}}\frac{\tau^2\alpha^2}{(1-\gamma)^2}
+\frac{2B^2U_\delta^2}{1-\gamma}\gamma^{\tau L}.
\end{align}
Define the final coefficient envelopes
\begin{align}
\label{eqn: C_G constants}
\calC_{G,0}&:=4U_r,\nonumber\\
\calC_{G,1}&:=16B^2+\pth{2+C_{7}^{\mathrm{act}}L_{J'}}(1-\gamma),\nonumber\\
\calC_{G,2}&:=4\Lambda_G^2(1-\gamma)^3
+2\pth{\frac{\Lambda_G^2(C_{1}^{\mathrm{act}})^2}{4}+\Lambda_GC_{4}^{\mathrm{act}}}(1-\gamma)^2
+2\Lambda_GC_{6}^{\mathrm{act}},\nonumber\\
\calC_{G,3}&:=2\Lambda_GC_{2}^{\mathrm{act}},\qquad
\calC_{G,4}:=2\Lambda_GC_{3}^{\mathrm{act}},\nonumber\\
\calC_{G,5}&:=2\Lambda_GC_{5}^{\mathrm{act}},\qquad
\calC_{G,6}:=C_{8}^{\mathrm{act}}L_{J'},\nonumber\\
\calC_{G,7}&:=C_{9}^{\mathrm{act}}L_{J'},\qquad
\calC_{G,8}:=C_{10}^{\mathrm{act}}L_{J'},\nonumber\\
\calC_{G,9}&:=C_{11}^{\mathrm{act}}L_{J'},\qquad
\calC_{G,10}:=C_{12}^{\mathrm{act}}L_{J'},\nonumber\\
\calC_{G,11}&:=4B^2U_\delta^2.
\end{align}
Combining \eqref{eqn: expectation of policy gradient},
\eqref{eqn: actor second order fine scaled}, and
\eqref{eqn: actor minibatch perturbation bound final}, using
$\alpha\leq1$ and $16L_{J'}\alpha\leq1$, and absorbing the gradient-square
terms into the left-hand side, gives
\begin{align}
\label{eqn: expectation of policy gradient final}
\bbE\|\nabla J^k(\theta_t^k)\|^2
\leq&
\frac{2}{\alpha(1-\gamma)}
\bbE\qth{J^k(\theta_{t+1}^k)-J^k(\theta_t^k)}
+\frac{\calC_{G,1}}{(1-\gamma)^2}\bbE\|x_t^k\|^2
\nonumber\\
&+
\frac{\calC_{G,2}\tau^2\alpha^2}{(1-\gamma)^3}
+\frac{\calC_{G,3}\tau^2\alpha\beta}{L(1-\gamma)^2}
+\frac{\calC_{G,4}\tau^2\alpha\zeta}{L(1-\gamma)^2}
+\frac{\calC_{G,5}\tau\alpha}{(1-\gamma)\sqrt{L(1-\gamma)}}
\nonumber\\
&+
\frac{\calC_{G,6}\tau^2\alpha\beta^2}{L^2(1-\gamma)^3}
+\frac{\calC_{G,7}\tau^2\alpha\zeta^2}{L^2(1-\gamma)^3}
+\frac{\calC_{G,8}\tau^2\alpha^3}{1-\gamma}
+\frac{\calC_{G,9}\tau\alpha^2}{(1-\gamma)^3}
+\frac{\calC_{G,10}\alpha}{L(1-\gamma)^2}
+\frac{\calC_{G,11}\gamma^{\tau L}}{(1-\gamma)^2}.
\end{align}
Averaging \eqref{eqn: expectation of policy gradient final} over the agents
first gives the single-time averaged policy-gradient bound
\begin{align}
\label{eqn: average policy gradient one step}
\bar G_t
\leq&
\frac{2}{K\alpha(1-\gamma)}
\sum_{k=1}^K
\bbE\qth{J^k(\theta_{t+1}^k)-J^k(\theta_t^k)}
+\frac{\calC_{G,1}}{(1-\gamma)^2}\bar X_t
\nonumber\\
&+
\frac{\calC_{G,2}\tau^2\alpha^2}{(1-\gamma)^3}
+\frac{\calC_{G,3}\tau^2\alpha\beta}{L(1-\gamma)^2}
+\frac{\calC_{G,4}\tau^2\alpha\zeta}{L(1-\gamma)^2}
+\frac{\calC_{G,5}\tau\alpha}{(1-\gamma)\sqrt{L(1-\gamma)}}
\nonumber\\
&+
\frac{\calC_{G,6}\tau^2\alpha\beta^2}{L^2(1-\gamma)^3}
+\frac{\calC_{G,7}\tau^2\alpha\zeta^2}{L^2(1-\gamma)^3}
+\frac{\calC_{G,8}\tau^2\alpha^3}{1-\gamma}
+\frac{\calC_{G,9}\tau\alpha^2}{(1-\gamma)^3}
+\frac{\calC_{G,10}\alpha}{L(1-\gamma)^2}
+\frac{\calC_{G,11}\gamma^{\tau L}}{(1-\gamma)^2}.
\end{align}
Taking the time average of \eqref{eqn: average policy gradient one step}
from $\tau$ to $T-1$ and using the telescoping sum in the objective values,
\begin{align*}
\bar G_T
\leq&
\frac{2}{K\alpha(1-\gamma)(T-\tau)}
\sum_{k=1}^K
\bbE\qth{J^k(\theta_T^k)-J^k(\theta_{\tau}^k)}
+\frac{\calC_{G,1}}{(1-\gamma)^2}\bar X_T
\nonumber\\
&+
\frac{\calC_{G,2}\tau^2\alpha^2}{(1-\gamma)^3}
+\frac{\calC_{G,3}\tau^2\alpha\beta}{L(1-\gamma)^2}
+\frac{\calC_{G,4}\tau^2\alpha\zeta}{L(1-\gamma)^2}
+\frac{\calC_{G,5}\tau\alpha}{(1-\gamma)\sqrt{L(1-\gamma)}}
\nonumber\\
&+
\frac{\calC_{G,6}\tau^2\alpha\beta^2}{L^2(1-\gamma)^3}
+\frac{\calC_{G,7}\tau^2\alpha\zeta^2}{L^2(1-\gamma)^3}
+\frac{\calC_{G,8}\tau^2\alpha^3}{1-\gamma}
+\frac{\calC_{G,9}\tau\alpha^2}{(1-\gamma)^3}
+\frac{\calC_{G,10}\alpha}{L(1-\gamma)^2}
+\frac{\calC_{G,11}\gamma^{\tau L}}{(1-\gamma)^2}\\
\leq&
\frac{\calC_{G,0}}{(1-\gamma)^2c_\theta\zeta(T-\tau)}
+\frac{\calC_{G,1}}{(1-\gamma)^2}\bar X_T
\nonumber\\
&+
\frac{\calC_{G,2}\tau^2c_\theta^2\zeta^2}{(1-\gamma)^3}
+\frac{\calC_{G,3}\tau^2cc_\theta\zeta^2}{L(1-\gamma)^2}
+\frac{\calC_{G,4}\tau^2c_\theta\zeta^2}{L(1-\gamma)^2}
+\frac{\calC_{G,5}\tau c_\theta\zeta}{(1-\gamma)\sqrt{L(1-\gamma)}}
\nonumber\\
&+
\frac{\calC_{G,6}\tau^2c^2c_\theta\zeta^3}{L^2(1-\gamma)^3}
+\frac{\calC_{G,7}\tau^2c_\theta\zeta^3}{L^2(1-\gamma)^3}
+\frac{\calC_{G,8}\tau^2c_\theta^3\zeta^3}{1-\gamma}
+\frac{\calC_{G,9}\tau c_\theta^2\zeta^2}{(1-\gamma)^3}
+\frac{\calC_{G,10}c_\theta\zeta}{L(1-\gamma)^2}
+\frac{\calC_{G,11}\gamma^{\tau L}}{(1-\gamma)^2},
\end{align*}
where the last inequality uses $|J^k(\theta)|\leq U_r/(1-\gamma)$ and
$\alpha=c_\theta\zeta$, $\beta=c\zeta$, and $\calC_{G,0}=4U_r$. The fixed-agent bound follows from applying the same
time-averaging argument directly to \eqref{eqn: expectation of policy gradient final}.
\newpage

\section{Proof of Theorem \ref{thm: 1} (Main Convergence)}
Throughout the proof, choose
\[
\tau L=\max\{\lceil2\log_\rho(\zeta)\rceil,\lceil2\log_\gamma(\zeta)\rceil\},
\]
so that $m\rho^{\tau L}\leq m\zeta^2$ and $\gamma^{\tau L}\leq \zeta^2$.
We invoke the three lemma bounds with this value of $\tau$.
We first collect the three estimates that drive the coupled system. Under the
stepsize and mixing-window choices in Lemmas \ref{lm: local head: upper bound: formal},
\ref{lm: PAD analysis}, and \ref{lm: policy gradient analysis}, the local-head
recursion \eqref{eqn: Xt+1 after sub Mt+1} gives, for all $t\geq \tau$,
\begin{align}
\label{eqn: theorem local head input}
   &\bar X_{t+1}  \nonumber\\
   \leq& \Bigg(1-2\lambda c\zeta
   +\pth{
   c_2^2\frac{c}{L}
   +c_3^2\frac{1}{L}
   +\frac{4U_\omega^4}{c_3^2L}
   +4B L_{*,1}c_\theta
   +c_4^2c_\theta
   }\zeta\nonumber\\
&\qquad\qquad\qquad\qquad\qquad
   +\frac{864U_\delta U_\omega^3c}{c_1^2}\frac{\zeta^2}{L^3(1-\gamma)}
   +\frac{\calC_{X,1}(\tau^4;c,c_\theta)}{(1-\gamma)^4}\zeta^2
\Bigg)\bar X_t\nonumber\\
&\quad+ \pth{
\frac{36U_\omega^2}{c_2^2}\frac{c\zeta}{L}
+6U_\omega U_\delta c\frac{\zeta}{L(1-\gamma)}
\pth{1+c_1^2\zeta}
}M_t\nonumber\\
&\quad+
\pth{
\frac{(1-\gamma)^2L_{*,1}^2c_\theta}{c_4^2}\zeta
+\frac{\calC_{X,2}(\tau^4;c,c_\theta)}{(1-\gamma)^2}\zeta^2
}\bar G_t\nonumber\\
&\quad+
\frac{\calC_{X,3}(\tau^4;c,c_\theta)}{(1-\gamma)^4}
\pth{\frac{\zeta^2}{\sqrt L}+\frac{\zeta^2}{L}+\frac{\zeta^2}{K}}
+\frac{\calC_{X,4}(\tau^6;c,c_\theta)}{(1-\gamma)^4}\zeta^3
+8c_\theta U_\omega L_{*,1}U_\delta B\,\zeta\,\gamma^{\tau L}.
\end{align}
The principal-angle recursion \eqref{eqn: Mt+1 final} gives, for all
$t\geq\tau$,
\begin{align}
\label{eqn: theorem PAD input}
&M_{t+1}
\leq
\pth{1-\frac{\lambda\nu\zeta}{r}+c_1^2\zeta}M_t
+\pth{
\frac{144U_\omega^2}{c_1^2}\frac{\zeta}{L^2}
+\frac{\calC_{M,1}(\tau^4;c,c_\theta)}{(1-\gamma)^2}
\frac{\zeta^2}{L^2}
}\bar X_t\nonumber\\
&\quad
+\frac{\calC_{M,2}(\tau^2;c,c_\theta)}{(1-\gamma)^3}
\pth{
\frac{\zeta^2}{L}
+\frac{\zeta^2}{K}
}
+\frac{\calC_{M,3}(\tau^6;c,c_\theta)}{(1-\gamma)^4}\zeta^3.
\end{align}
Finally, Lemma \ref{lm: policy gradient analysis} gives the time-averaged
actor bound
\begin{align}
\label{eqn: theorem policy gradient input}
\bar G_T
\leq&
\frac{\calC_{G,0}}{(1-\gamma)^2c_\theta\zeta (T-\tau)}
+\frac{\calC_{G,1}}{(1-\gamma)^2}\bar X_T
+\frac{\calC_{G,2}\tau^2c_\theta^2\zeta^2}{(1-\gamma)^3}
+\frac{\calC_{G,3}\tau^2cc_\theta\zeta^2}{L(1-\gamma)^2}\nonumber\\
&+
\frac{\calC_{G,4}\tau^2c_\theta\zeta^2}{L(1-\gamma)^2}
+\frac{\calC_{G,5}\tau c_\theta\zeta }{(1-\gamma)\sqrt{L(1-\gamma)}}
+\frac{\calC_{G,6}\tau^2c^2c_\theta\zeta^3}{L^2(1-\gamma)^3}
+\frac{\calC_{G,7}\tau^2c_\theta\zeta^3}{L^2(1-\gamma)^3}\nonumber\\
&+
\frac{\calC_{G,8}\tau^2c_\theta^3\zeta^3}{1-\gamma}
+\frac{\calC_{G,9}\tau c_\theta^2\zeta^2}{(1-\gamma)^3}
+\frac{\calC_{G,10}c_\theta\zeta}{L(1-\gamma)^2}
+\frac{\calC_{G,11}\gamma^{\tau L}}{(1-\gamma)^2}.
\end{align}
We next reduce the two recursive inequalities to their leading-order forms.
Choose $c_1,c_2,c_3,c_4,c_\theta$ and $\zeta$ such that
\begin{align}
\label{eqn: theorem contraction simplification conditions}
&c_1^2\leq \frac{\lambda\nu}{2r},\qquad
c_2^2\frac{c}{L}
+c_3^2\frac{1}{L}
+\frac{4U_\omega^4}{c_3^2L}
+4B L_{*,1}c_\theta
+c_4^2c_\theta
\leq \lambda c,\nonumber\\
&\zeta\leq \min\left\{
\frac{1}{c_1^2},
\frac{\lambda c}{2\left(
\frac{864U_\delta U_\omega^3c}{c_1^2L^3(1-\gamma)}
+\frac{\calC_{X,1}(\tau^4;c,c_\theta)}{(1-\gamma)^4}
\right)},
\frac{(1-\gamma)^4L_{*,1}^2c_\theta}
{c_4^2\calC_{X,2}(\tau^4;c,c_\theta)},
\frac{144U_\omega^2(1-\gamma)^2}
{c_1^2\calC_{M,1}(\tau^4;c,c_\theta)}
\right\}.
\end{align}
Then \eqref{eqn: theorem local head input} reduces to
\begin{align}
\label{eqn: theorem local head reduced}
   \bar X_{t+1}
   \leq&
   \pth{1-\frac{\lambda c}{2}\zeta}\bar X_t
   +\left(
   \frac{36U_\omega^2c}{c_2^2L}
   +\frac{12U_\omega U_\delta c}{L(1-\gamma)}
   \right)\zeta M_t\nonumber\\
&\quad
   +\frac{2(1-\gamma)^2L_{*,1}^2c_\theta}{c_4^2}\zeta\,\bar G_t
   +\frac{\calC_{X,3}(\tau^4;c,c_\theta)}{(1-\gamma)^4}
   \pth{\frac{\zeta^2}{\sqrt L}+\frac{\zeta^2}{L}+\frac{\zeta^2}{K}}\nonumber\\
&\quad
   +\frac{\calC_{X,4}(\tau^6;c,c_\theta)}{(1-\gamma)^4}\zeta^3
   +8c_\theta U_\omega L_{*,1}U_\delta B\,\zeta\,\gamma^{\tau L}.
\end{align}
Similarly, \eqref{eqn: theorem PAD input} reduces to
\begin{align}
\label{eqn: theorem PAD reduced}
M_{t+1}
\leq&
\pth{1-\frac{\lambda\nu}{2r}\zeta}M_t
+\frac{288U_\omega^2}{c_1^2}\frac{\zeta}{L^2}\bar X_t\nonumber\\
&\quad
+\frac{\calC_{M,2}(\tau^2;c,c_\theta)}{(1-\gamma)^3}
\pth{
\frac{\zeta^2}{L}
+\frac{\zeta^2}{K}
}
+\frac{\calC_{M,3}(\tau^6;c,c_\theta)}{(1-\gamma)^4}\zeta^3.
\end{align}
Rearranging \eqref{eqn: theorem local head reduced}, summing from
$t=\tau$ to $T-1$, and using the telescoping sum together with
$\bar X_\tau\leq 4U_\omega^2$ gives
\begin{align}
\label{eqn: theorem local head time average}
\bar X_T
\leq&
\frac{8U_\omega^2}{\lambda c\zeta(T-\tau)}
+\left(
\frac{72U_\omega^2}{\lambda c_2^2L}
+\frac{24U_\omega U_\delta}{\lambda L(1-\gamma)}
\right)M_T\nonumber\\
&\quad
+\frac{4(1-\gamma)^2L_{*,1}^2c_\theta}{\lambda c c_4^2}\bar G_T
+\frac{2\calC_{X,3}(\tau^4;c,c_\theta)}{\lambda c(1-\gamma)^4}
\pth{\frac{\zeta}{\sqrt L}+\frac{\zeta}{L}+\frac{\zeta}{K}}\nonumber\\
&\quad
+\frac{2\calC_{X,4}(\tau^6;c,c_\theta)}{\lambda c(1-\gamma)^4}\zeta^2
+\frac{16c_\theta U_\omega L_{*,1}U_\delta B}{\lambda c}\gamma^{\tau L}.
\end{align}
Substituting \eqref{eqn: theorem policy gradient input} into
\eqref{eqn: theorem local head time average} and using
$\gamma^{\tau L}\leq \zeta^2$ gives
\begin{align}
\label{eqn: theorem local head after sub G}
\bar X_T
\leq&
\frac{8U_\omega^2}{\lambda c\zeta(T-\tau)}
+\frac{4L_{*,1}^2\calC_{G,0}}{\lambda c c_4^2\zeta(T-\tau)}
+\frac{4L_{*,1}^2\calC_{G,1}c_\theta}{\lambda c c_4^2}\bar X_T\nonumber\\
&\quad
+\left(
\frac{72U_\omega^2}{\lambda c_2^2L}
+\frac{24U_\omega U_\delta}{\lambda L(1-\gamma)}
\right)M_T
+\frac{2\calC_{X,3}(\tau^4;c,c_\theta)}{\lambda c(1-\gamma)^4}
\pth{\frac{\zeta}{\sqrt L}+\frac{\zeta}{L}+\frac{\zeta}{K}}\nonumber\\
&\quad
+\frac{4(1-\gamma)L_{*,1}^2\calC_{G,5}\tau c_\theta^2\zeta}
{\lambda c c_4^2\sqrt{L(1-\gamma)}}
+\frac{4L_{*,1}^2\calC_{G,10}c_\theta^2\zeta}
{\lambda c c_4^2L}
+\tilde{\calO}(\zeta^2).
\end{align}
Rearranging \eqref{eqn: theorem PAD reduced}, summing from $t=\tau$ to
$T-1$, and using the telescoping sum together with $M_\tau\leq r$ gives
\begin{align}
\label{eqn: theorem PAD time average}
M_T
\leq&
\frac{2r^2}{\lambda\nu\zeta(T-\tau)}
+\frac{576rU_\omega^2}{\lambda\nu c_1^2L^2}\bar X_T
+\frac{2r\calC_{M,2}(\tau^2;c,c_\theta)}{\lambda\nu(1-\gamma)^3}
\pth{\frac{\zeta}{L}+\frac{\zeta}{K}}
+\tilde{\calO}(\zeta^2).
\end{align}
Substituting \eqref{eqn: theorem PAD time average} into
\eqref{eqn: theorem local head after sub G} gives
\begin{align}
\label{eqn: theorem local head after sub G and M}
\bar X_T
\leq&
\frac{8U_\omega^2}{\lambda c\zeta(T-\tau)}
+\frac{4L_{*,1}^2\calC_{G,0}}{\lambda c c_4^2\zeta(T-\tau)}
+\frac{2r^2}{\lambda\nu\zeta(T-\tau)}
\left(
\frac{72U_\omega^2}{\lambda c_2^2L}
+\frac{24U_\omega U_\delta}{\lambda L(1-\gamma)}
\right)\nonumber\\
&\quad
+\left[
\frac{4L_{*,1}^2\calC_{G,1}c_\theta}{\lambda c c_4^2}
+\frac{576rU_\omega^2}{\lambda\nu c_1^2L^2}
\left(
\frac{72U_\omega^2}{\lambda c_2^2L}
+\frac{24U_\omega U_\delta}{\lambda L(1-\gamma)}
\right)
\right]\bar X_T\nonumber\\
&\quad
+\frac{2\calC_{X,3}(\tau^4;c,c_\theta)}{\lambda c(1-\gamma)^4}
\pth{\frac{\zeta}{\sqrt L}+\frac{\zeta}{L}+\frac{\zeta}{K}}\nonumber\\
&\quad
+\frac{4(1-\gamma)L_{*,1}^2\calC_{G,5}\tau c_\theta^2\zeta}
{\lambda c c_4^2\sqrt{L(1-\gamma)}}
+\frac{4L_{*,1}^2\calC_{G,10}c_\theta^2\zeta}
{\lambda c c_4^2L}\nonumber\\
&\quad
+\frac{2r\calC_{M,2}(\tau^2;c,c_\theta)}{\lambda\nu(1-\gamma)^3}
\left(
\frac{72U_\omega^2}{\lambda c_2^2L}
+\frac{24U_\omega U_\delta}{\lambda L(1-\gamma)}
\right)
\pth{\frac{\zeta}{L}+\frac{\zeta}{K}}
+\tilde{\calO}(\zeta^2).
\end{align}
It is sufficient to choose $c_\theta/c$, $c_1$, and $L$ such that
\begin{align}
\label{eqn: theorem absorption conditions}
\frac{4L_{*,1}^2\calC_{G,1}c_\theta}{\lambda c c_4^2}
\leq \frac{1}{4},\qquad
\frac{576rU_\omega^2}{\lambda\nu c_1^2L^2}
\left(
\frac{72U_\omega^2}{\lambda c_2^2L}
+\frac{24U_\omega U_\delta}{\lambda L(1-\gamma)}
\right)
\leq \frac{1}{4}.
\end{align}
Then the $\bar X_T$ terms on the right-hand side of
\eqref{eqn: theorem local head after sub G and M} can be absorbed into the
left-hand side, yielding
\begin{align}
\label{eqn: theorem local head after absorption}
\bar X_T
\leq&
\frac{16U_\omega^2}{\lambda c\zeta(T-\tau)}
+\frac{8L_{*,1}^2\calC_{G,0}}{\lambda c c_4^2\zeta(T-\tau)}
+\frac{4r^2}{\lambda\nu\zeta(T-\tau)}
\left(
\frac{72U_\omega^2}{\lambda c_2^2L}
+\frac{24U_\omega U_\delta}{\lambda L(1-\gamma)}
\right)\nonumber\\
&\quad
+\frac{4\calC_{X,3}(\tau^4;c,c_\theta)}{\lambda c(1-\gamma)^4}
\pth{\frac{\zeta}{\sqrt L}+\frac{\zeta}{L}+\frac{\zeta}{K}}\nonumber\\
&\quad
+\frac{8(1-\gamma)L_{*,1}^2\calC_{G,5}\tau c_\theta^2\zeta}
{\lambda c c_4^2\sqrt{L(1-\gamma)}}
+\frac{8L_{*,1}^2\calC_{G,10}c_\theta^2\zeta}
{\lambda c c_4^2L}\nonumber\\
&\quad
+\frac{4r\calC_{M,2}(\tau^2;c,c_\theta)}{\lambda\nu(1-\gamma)^3}
\left(
\frac{72U_\omega^2}{\lambda c_2^2L}
+\frac{24U_\omega U_\delta}{\lambda L(1-\gamma)}
\right)
\pth{\frac{\zeta}{L}+\frac{\zeta}{K}}
+\tilde{\calO}(\zeta^2).
\end{align}

We now specify a feasible parameter choice. Fix $c_2,c_3,c_4>0$ with
$c_2^2\leq \lambda/4$. Choose constants $c,c_\theta>0$ such that
\begin{align}
\label{eqn: theorem c ctheta feasible}
c
&\geq
\frac{4}{\lambda}
\left(c_3^2+\frac{4U_\omega^4}{c_3^2}\right),\nonumber\\
c_\theta
&\leq
\min\left\{
\frac{\lambda c}{4(4B L_{*,1}+c_4^2)},
\frac{\lambda c c_4^2}{16L_{*,1}^2\calC_{G,1}}
\right\}.
\end{align}
Set
\begin{align}
\label{eqn: theorem stepsize choice}
\zeta=\frac{L^{1/4}}{\sqrt T},\qquad
c_1^2=\frac{\lambda\nu}{4r}.
\end{align}
With this choice, the first-order contraction and absorption requirements in
\eqref{eqn: theorem contraction simplification conditions} and
\eqref{eqn: theorem absorption conditions} are satisfied provided
\begin{align}
\label{eqn: theorem L feasible}
\frac{2304r^2U_\omega^2}{\lambda^2\nu^2L^2}
\left(
\frac{72U_\omega^2}{\lambda c_2^2L}
+\frac{24U_\omega U_\delta}{\lambda L(1-\gamma)}
\right)
\leq \frac{1}{4}.
\end{align}
The remaining small-stepsize requirements are collected as
\begin{align}
\label{eqn: theorem zeta feasible}
\frac{L^{1/4}}{\sqrt T}
\leq
\min\Bigg\{&
\frac{1}{2U_\delta U_\omega},
\frac{1}{\sqrt 6 U_\delta U_\omega},
\frac{\lambda\nu}{3rU_\delta^2U_\omega^2},
\frac{L(1-\gamma)}{2U_\delta U_\omega},
\frac{1}{2U_\omega},
\frac{1}{c_\theta},
\frac{1}{16L_{J'}c_\theta},
\frac{L U_\omega(1-\gamma)}{c U_\delta},
\frac{4r}{\lambda\nu},
\nonumber\\
&\frac{\lambda c}
{2\left(
\frac{864U_\delta U_\omega^3c}{c_1^2L^3(1-\gamma)}
+\frac{\calC_{X,1}(\tau^4;c,c_\theta)}{(1-\gamma)^4}
\right)},
\frac{(1-\gamma)^4L_{*,1}^2c_\theta}
{c_4^2\calC_{X,2}(\tau^4;c,c_\theta)},
\frac{144U_\omega^2(1-\gamma)^2}
{c_1^2\calC_{M,1}(\tau^4;c,c_\theta)}
\Bigg\}.
\end{align}
These conditions are feasible. Indeed, \eqref{eqn: theorem c ctheta feasible}
first makes the local-head contraction coefficient and the $\bar X_T$ absorption
coefficient small enough. The choice $c_1^2=\lambda\nu/(4r)$ leaves a positive
contraction margin for the PAD recursion, and \eqref{eqn: theorem L feasible}
then makes the PAD-to-local-head feedback absorbable for all sufficiently large
$L$. Finally, \eqref{eqn: theorem zeta feasible} holds for all
sufficiently large $T$ up to logarithmic factors from $\tau$ and the displayed
discount/coefficient envelopes. Below we also take $T$ large enough so that the
higher-order $\zeta^2$ residuals are dominated by the leading terms.

It remains to read off the rates. In \eqref{eqn: theorem local head after absorption},
we keep the scalar residuals with their original $L$-speedups from
\eqref{eqn: xtilde final}, \eqref{eqn: x t+1 before sub m t+1},
\eqref{eqn: local head crude Mt+1 plug}, and \eqref{eqn: Mt+1 final},
before they were merged into the single envelope $\calC_{X,3}$. The constant
$L_{*,1}$ appears explicitly in the policy-feedback terms in
\eqref{eqn: theorem local head after absorption}; $L_{s,1}$ enters only through
the scalar residual envelopes generated by \eqref{eqn: xtilde final}. Under
\eqref{eqn: theorem stepsize choice}, this gives
\begin{align}
\label{eqn: theorem local head rate intermediate}
\bar X_T
\leq
\tilde{\calO}\left(
\frac{1}{c\zeta T}
+\frac{1}{(1-\gamma)\zeta TL}
+\frac{c_\theta^2\zeta}{c(1-\gamma)^4\sqrt L}
+\frac{c\zeta}{(1-\gamma)^4L}
+\frac{c\zeta}{(1-\gamma)^4K}
+\frac{\zeta^2}{(1-\gamma)^5}
\right),
\end{align}
where the last high-order term keeps the worst $(1-\gamma)^{-1}$ power from
the PAD residual after it is multiplied by the local-head $M_T$ coefficient.
The displayed powers are the external powers in the recursive bounds; the
remaining dependence through $L_{*,1}$, $L_{s,1}$, and $L_{J'}$ is not expanded
in this rate display.
Substituting
\eqref{eqn: theorem stepsize choice} into \eqref{eqn: theorem local head rate intermediate}
yields, using $L\geq1$ and $(1-\gamma)^{-1}\geq1$,
\begin{align}
\label{eqn: theorem local head rate}
\bar X_T
\leq
\tilde{\calO}\left(
\frac{1}{(1-\gamma)^4L^{1/4}\sqrt T}
+\frac{L^{1/4}}{(1-\gamma)^4K\sqrt T}
+\frac{\sqrt L}{(1-\gamma)^5T}
\right).
\end{align}
Substituting \eqref{eqn: theorem local head rate} into
\eqref{eqn: theorem policy gradient input}, and using
\eqref{eqn: theorem stepsize choice}, gives
\begin{align}
\label{eqn: theorem policy gradient final rate}
\bar G_T
\leq
\tilde{\calO}\left(
\frac{1}{(1-\gamma)^6L^{1/4}\sqrt T}
+\frac{L^{1/4}}{(1-\gamma)^6K\sqrt T}
+\frac{\sqrt L}{(1-\gamma)^7T}
\right).
\end{align}
Thus, before relating $L$ and $K$, the actor bound keeps both the local-rollout
speedup and the agent-averaging speedup through the two leading terms
$1/(L^{1/4}\sqrt T)$ and $L^{1/4}/(K\sqrt T)$. If $L\lesssim K^2$, then the
agent-averaging term is no larger than the local-rollout term. Taking $L=K^2$
and assuming that the higher-order term is dominated gives
\begin{align}
\label{eqn: theorem local head final rate}
\bar X_T
\leq
\tilde{\calO}\left(\frac{1}{(1-\gamma)^4\sqrt{TK}}\right),
\end{align}
and
\begin{align}
\label{eqn: theorem policy gradient simplified rate}
\bar G_T
\leq
\tilde{\calO}\left(
\frac{1}{(1-\gamma)^6\sqrt{TK}}
\right).
\end{align}
For fixed $\gamma$, this gives the actor stationarity rate
$\bar G_T=\tilde{\calO}(1/\sqrt{TK})$ up to logarithmic and
coefficient-envelope factors. The $L^{1/4}$ speedup before choosing $L=K^2$
comes from the unavoidable $\alpha/\sqrt L$ actor perturbation term in
\eqref{eqn: theorem policy gradient input}.

\newpage

\section{Additional Supporting Propositions and Lemmas}
\label{sec: supporting}

\subsection{Proof of Proposition \ref{prop: key intermediate} (TD Feature Decomposition)}
\label{app: key intermediate proof}

\begin{proof}[Proof of Proposition \ref{prop: key intermediate}]

It follows from Eq.\,(\ref{eq: L step TD error}) that for any $t$ and $L$: 
\begin{align*}
&\delta_{t,L}^k\phi(s_{t,0}^k) = \pth{\sum_{\ell=0}^{L-1} \gamma^\ell r_{t,\ell}^k+(\gamma^L \phi(s_{t,L}^k)  - \phi(s_{t,0}^k))^\top {\mathbf{B}_t}{\omega_t^k}}\phi(s_{t,0}^k)\\
=& \sum_{\ell=0}^{L-1} \gamma^\ell r_{t,\ell}^k \phi(s_{t,0}^k)  + (\gamma^L \phi(s_{t,L}^k)  - \phi(s_{t,0}^k))^\top {\mathbf{B}_t}{\omega_t^k}\phi(s_{t,0}^k)\\
&- \sum_{\ell=0}^{L-1} (\gamma^L \phi(s_{t,L}^k)  - \phi(s_{t,0}^k))^\top z_t^{k,*}\phi(s_{t,0}^k) 
+ \sum_{\ell=0}^{L-1} (\gamma^L \phi(s_{t,L}^k)  - \phi(s_{t,0}^k))^\top z_t^{k,*}\phi(s_{t,0}^k)\\
=& \underbrace{\phi(s_{t,0}^k)(\gamma^L \phi(s_{t,L}^k)  - \phi(s_{t,0}^k))^\top}_{\tilde A_{t,L}^k} x_t^k + \underbrace{\sum_{\ell=0}^{L-1} \gamma^\ell r_{t,\ell}^k \phi(s_{t,0}^k)  + (\gamma^L \phi(s_{t,L}^k)  - \phi(s_{t,0}^k))^\top z_t^{k,*}\phi(s_{t,0}^k)}_{\bfb_{t,L}^k}, 
\end{align*}
where the last equality follows because both $(\gamma^L \phi(s_{t,L}^k)  - \phi(s_{t,0}^k))^\top {\mathbf{B}_t}{\omega_t^k}$ and $(\gamma^L \phi(s_{t,L}^k)  - \phi(s_{t,0}^k))^\top z_t^{k,*}$ are scalars, thereby establishing Eq.\,(\ref{eq: delta-phi: rewritting 2}).

In the remaining proof, we focus on proving the second decomposition in Eq.\,(\ref{eq: delta-phi: rewritting 1}). 
We have,
\begin{align*}
    &\delta_{t,L}^k\phi(s_{t,0}^k) = \pth{\sum_{\ell=0}^{L-1} \gamma^\ell r_{t,\ell}^k+(\gamma^L \phi(s_{t,L}^k)  - \phi(s_{t,0}^k))^\top {\mathbf{B}_t}{\omega_t^k}}\phi(s_{t,0}^k)\\
    =& \pth{\sum_{\ell=0}^{L-1} \gamma^\ell r_{t,\ell}^k+(\gamma^L \phi(s_{t,L}^k)  - \phi(s_{t,0}^k))^\top {\mathbf{B}_t}{\omega_t^k}}\phi(s_{t,0}^k)\\
    &-\sum_{\ell=0}^{L-1}\bbE_{s^{(0)}\sim\mu_{\theta_t^k}^k, (a^{(\ell)}, s^{(\ell+1)}) \sim \pi_{\theta_t^k} \otimes P^k ~ \forall \ell\in [0, L-1]} \left ( \gamma^\ell R(s^{(\ell)},a^{(\ell)}) \right. \\
    & \left.\qquad \qquad \qquad \qquad \qquad \qquad \qquad \qquad + (\gamma^{\ell+1}\phi(s^{(\ell+1)}) - \gamma^\ell\phi(s^{(\ell)}))^\top \mathbf{B}_t \omega_t^k \right ) \phi(s^{(0)})\\
    &+\sum_{\ell=0}^{L-1}\bbE_{s^{(0)}\sim\mu_{\theta_t^k}^k, (a^{(\ell)}, s^{(\ell+1)}) \sim \pi_{\theta_t^k} \otimes P^k ~ \forall \ell\in [0, L-1]} \left (\gamma^\ell R(s^{(\ell)},a^{(\ell)}) \right.\\
    & \left.\qquad \qquad \qquad \qquad \qquad \qquad \qquad \qquad + (\gamma^{\ell+1}\phi(s^{(\ell+1)}) - \gamma^\ell\phi(s^{(\ell)}))^\top \mathbf{B}_t \omega_t^k \right ) \phi(s^{(0)})\\
    & = \xi_{t,L}^k \phi(s_{t,0}^k) + \sum_{\ell=0}^{L-1}\bbE_{s^{(0)}\sim\mu_{\theta_t^k}^k, (a^{(\ell)}, s^{(\ell+1)}) \sim \pi_{\theta_t^k} \otimes P^k ~ \forall \ell\in [0, L-1]} \left (\gamma^\ell R(s^{(\ell)},a^{(\ell)}) \right.\\
    & \left.\qquad \qquad \qquad \qquad \qquad \qquad \qquad \qquad + (\gamma^{\ell+1}\phi(s^{(\ell+1)}) - \gamma^\ell\phi(s^{(\ell)}))^\top \mathbf{B}_t \omega_t^k \right ) \phi(s^{(0)}). 
\end{align*}

Under steady-state distribution, the $L$-step TD fixed point $z_t^{k,*}$ satisfies,
\begin{equation}
\label{eqn: L-step fixed point}
\begin{aligned}
&\sum_{\ell=0}^{L-1}\bbE_{s^{(0)}\sim\mu_{\theta_t^k}^k, (a^{(\ell)}, s^{(\ell+1)}) \sim \pi_{\theta_t^k} \otimes P^k ~ \forall \ell\in [0, L-1]} \left (\gamma^\ell R(s^{(\ell)},a^{(\ell)}) \right.\\
& \left. \qquad \qquad \qquad \qquad \qquad \qquad \qquad \qquad +\gamma^{\ell+1} \phi(s^{(\ell+1)}) - \gamma^\ell\phi(s^{(\ell)}))^\top z_t^{k,*} \right)\phi(s^{(0)})=0.
\end{aligned} 
\end{equation}
Then,
\begin{align*}
&\sum_{\ell=0}^{L-1}\bbE_{s^{(0)}\sim\mu_{\theta_t^k}^k, (a^{(\ell)}, s^{(\ell+1)}) \sim \pi_{\theta_t^k} \otimes P^k ~ \forall \ell\in [0, L-1]} \left (\gamma^\ell R(s^{(\ell)},a^{(\ell)})  \right. \\
& \left.\qquad \qquad \qquad \qquad \qquad \qquad \qquad \qquad + (\gamma^{\ell+1}\phi(s^{(\ell+1)}) - \gamma^\ell\phi(s^{(\ell)}))^\top \mathbf{B}_t \omega_t^k \right)\phi(s^{(0)})\\
    =&\sum_{\ell=0}^{L-1}\bbE_{s^{(0)}\sim\mu_{\theta_t^k}^k, (a^{(\ell)}, s^{(\ell+1)}) \sim \pi_{\theta_t^k} \otimes P^k ~ \forall \ell\in [0, L-1]} (\gamma^{\ell+1}\phi(s^{(\ell+1)}) - \gamma^{\ell}\phi(s^{(\ell)}))^\top x_t^k\phi(s^{(0)})\\
    =& \bbE_{s^{(0)}\sim\mu_{\theta_t^k}^k, (a^{(\ell)}, s^{(\ell+1)}) \sim \pi_{\theta_t^k} \otimes P^k ~ \forall \ell\in [0, L-1]}\qth{  (\gamma^{L}\phi(s^{(L)}) - \phi(s^{(0)}))^\top x_t^k\phi(s^{(0)})}\\
    =&\underbrace{\bbE_{s^{(0)}\sim\mu_{\theta_t^k}^k, (a^{(\ell)}, s^{(\ell+1)}) \sim \pi_{\theta_t^k} \otimes P^k ~ \forall \ell\in [0, L-1]}\qth{\phi(s^{(0)})(\gamma^{L}\phi(s^{(L)}) - \phi(s^{(0)}))^\top}}_{A_{L,\theta_t^k}^k}x_t^k\\
    =& A_{L,\theta_t^k}^k x_t^k, 
\end{align*}
proving Eq.\eqref{eq: delta-phi: rewritting 1}. 

\end{proof}

\subsection{Stepback Perturbation}
Lemma~\ref{lmm: jth to tth error bound} aligns the error terms indexed by $j$ to a common reference
time $t$, which is required for a consistent Lyapunov analysis. 
\begin{lemma}
\label{lmm: jth to tth error bound}
Choose $\zeta\le \frac{L(1-\gamma)}{2U_{\delta} U_{\omega}}$. 
Given $\tau^{\prime}$ and $t$, for any $k\in[K]$ and for any time indices $j$ satisfying $0\leq t-\tau^{\prime}\leq j\leq t$, we have 
    \begin{align*}
        &\bbE_{t-\tau^{\prime}}\qth{\|x_j^k\|} %
        \leq \tau^{\prime} \frac{U_{\delta}}{L(1-\gamma)}\beta
        +4\tau^{\prime} \frac{U_{\delta}}{L(1-\gamma)}U_\omega^2 \zeta
        +\tau^{\prime} L_{*,1}BU_\delta\alpha 
        +\bbE_{t-\tau^{\prime}}\qth{\|x_t^k\|}, \\% |v_{0:t-\tau},\\
         &\bbE_{t-\tau^{\prime}}\qth{\|x_j^k\|^2} %
         \leq 8(\tau^{\prime})^2 \frac{U_{\delta}^2}{L^2(1-\gamma)^2}\beta^2
         +128 (\tau^{\prime})^2 U_\omega^4  \frac{U_{\delta}^2}{L^2(1-\gamma)^2}\zeta^2\\
         &\qquad \qquad \qquad \quad +4(\tau^{\prime})^2 L_{*,1}^2B^2U_\delta^2\alpha^2  
         +2\bbE_{t-\tau^{\prime}}\qth{\|x_t^k\|^2},\\
        &\bbE_{t-\tau^{\prime}}[\|\nabla J^k(\theta_j^k)\|]
\le
\bbE_{t-\tau^{\prime}}[\|\nabla J^k(\theta_t^k)\|] 
+
\tau^{\prime} \alpha L_{J'} B U_\delta,\\
        &\bbE_{t-\tau^{\prime}}[\|\nabla J^k(\theta_j^k)\|^2]
\le
2\bbE_{t-\tau^{\prime}}[\|\nabla J^k(\theta_t^k)\|^2]
+2(\tau^{\prime})^2 \alpha^2  L_{J'}^2B^2U_\delta^2, 
    \end{align*}
    where $L_{*,1}$ and $L_{J'}$ are defined in Lemmas \ref{lmm: LJ' lip} and \ref{lmm: L* lip}, respectively. 
\end{lemma}

\begin{proof}
For any $t\ge j\ge t-\tau^{\prime}$, we have
\begin{align*}
    &\|x_j^k\|^2 \leq 2\|x_j^k-x_t^k\|^2+2\|x_t^k\|^2\\
    =& 2\|\bfB_j\omega_j^k-\bfB_t\omega_t^k-(z_j^{k,*}-z_t^{k,*})\|^2+2\|x_t^k\|^2\\
    \leq&4\|\bfB_t\omega_t^k-\bfB_j\omega_j^k\|^2
    +4\|z_t^{k,*}-z_j^{k,*}\|^2+2\|x_t^k\|^2\\
    \leq&8\|\bfB_t\omega_t^k-\bfB_t\omega_j^k\|^2+8\|\bfB_t\omega_j^k-\bfB_j\omega_j^k\|^2
    +4\|z_t^{k,*}-z_j^{k,*}\|^2+2\|x_t^k\|^2\\
    \overset{(a)}{\leq}& 8 \|\omega_{t}^k-\omega_j^k\|^2+ 8 U^2_\omega \|\bfB_{t}-\bfB_j\|^2
    +4L_{*,1}^2\|\theta_t^k-\theta_j^k\|^2 + 2\|x_t^k\|^2\\
    \leq& 8(\tau^{\prime})^2 \frac{U_{\delta}^2}{L^2(1-\gamma)^2}\beta^2
    +8 (\tau^{\prime})^2 U^2_\omega(2\frac{U_{\delta}}{L(1-\gamma)}U_\omega \zeta + 4\frac{U_{\delta}^2}{L^2(1-\gamma)^2}U_\omega^2 \zeta^2)^2 \\
    & \quad +4(\tau^{\prime})^2L_{*,1}^2B^2U_\delta^2\alpha^2  +2\|x_t^k\|^2,
\end{align*}
where inequality (a) follows from Lemma \ref{lmm: L* lip}, and the last inequality uses Lemmas \ref{lmm: crude parameter bound}. %
 
Then, conditioning on $v_{0:t-\tau^{\prime}}$, we get
\begin{align*}
    &\bbE_{t-\tau^{\prime}}\qth{\|x_j^k\|^2}
    \leq 8(\tau^{\prime})^2 \frac{U_{\delta}^2}{L^2(1-\gamma)^2}\beta^2 +  64 (\tau^{\prime})^2 U_\omega^4  \frac{U_{\delta}^2}{L^2(1-\gamma)^2}\zeta^2 \\
     &\qquad +  256 (\tau^{\prime})^2 U_\omega^6  \frac{U_{\delta}^4}{L^4(1-\gamma)^4}\zeta^4 
    +4(\tau^{\prime})^2L_{*,1}^2B^2U_\delta^2\alpha^2  + 2\bbE_{t-\tau^{\prime}}\qth{\|x_t^k\|^2}\\
    \leq& 8(\tau^{\prime})^2 \frac{U_{\delta}^2}{L^2(1-\gamma)^2}\beta^2 +  128 (\tau^{\prime})^2 U_\omega^4  \frac{U_{\delta}^2}{L^2(1-\gamma)^2}\zeta^2 +4(\tau^{\prime})^2L_{*,1}^2B^2U_\delta^2\alpha^2  + 2\bbE_{t-\tau^{\prime}}\qth{\|x_t^k\|^2}.
\end{align*}
The last inequality holds since $256 (\tau^{\prime})^2 U_\omega^6  \frac{U_{\delta}^4}{L^4(1-\gamma)^4}\zeta^4\leq 64 (\tau^{\prime})^2 U_\omega^4  \frac{U_{\delta}^2}{L^2(1-\gamma)^2}\zeta^2$ when $\zeta\le \frac{L(1-\gamma)}{2U_{\delta}U_{\omega}}$. 

For the first-moment bound $\|x_j^k\|$, the same decomposition gives
\begin{align*}
    \|x_j^k\|
    \leq& \|x_t^k\|+\|\bfB_t\omega_t^k-\bfB_j\omega_j^k\|+\|z_t^{k,*}-z_j^{k,*}\|\\
    \leq& \|x_t^k\|+\|\omega_t^k-\omega_j^k\|+U_\omega\|\bfB_t-\bfB_j\|
    +L_{*,1}\|\theta_t^k-\theta_j^k\|\\
    \leq& \|x_t^k\|+\tau^{\prime} \frac{U_{\delta}}{L(1-\gamma)}\beta
    +2\tau^{\prime} \frac{U_{\delta}}{L(1-\gamma)}U_\omega^2 \zeta
    +4\tau^{\prime}\frac{U_{\delta}^2}{L^2(1-\gamma)^2}U_\omega^3 \zeta^2
    +\tau^{\prime} L_{*,1}BU_\delta\alpha.
\end{align*}
Conditioning on $v_{0:t-\tau^{\prime}}$ gives
\begin{align*}
    &\bbE_{t-\tau^{\prime}}\qth{\|x_j^k\|} \\
    \leq&\tau^{\prime} \frac{U_{\delta}}{L(1-\gamma)}\beta
    + 2\tau^{\prime} \frac{U_{\delta}}{L(1-\gamma)}U_\omega^2 \zeta
    + 4\tau^{\prime}\frac{U_{\delta}^2}{L^2(1-\gamma)^2}U_\omega^3 \zeta^2
    +\tau^{\prime} L_{*,1}BU_\delta\alpha +\bbE_{t-\tau^{\prime}}\qth{\|x_t^k\|}\\
    \leq&\tau^{\prime} \frac{U_{\delta}}{L(1-\gamma)}\beta
    + 4\tau^{\prime}\frac{U_{\delta}}{L(1-\gamma)}U_\omega^2 \zeta
    +\tau^{\prime} L_{*,1}BU_\delta\alpha +\bbE_{t-\tau^{\prime}}\qth{\|x_t^k\|}.
\end{align*}
The last inequality holds since $4\tau^{\prime}\frac{U_{\delta}^2}{L^2(1-\gamma)^2}U_\omega^3 \zeta\leq 2\tau^{\prime}\frac{U_{\delta}}{L(1-\gamma)}U_\omega^2$ when $\zeta \le \frac{L(1-\gamma)}{2U_{\delta}U_{\omega}}$.

For $\nabla J^k(\theta_j^k)$, we have
\begin{align*}
    &\bbE_{t-\tau^{\prime}}[\|\nabla J^k(\theta_j^k)\|]
\le
\bbE_{t-\tau^{\prime}}[\|\nabla J^k(\theta_t^k)\|]
+
L_{J'}\bbE_{t-\tau^{\prime}}[\|\theta_j^k-\theta_t^k\|]\\
\leq &\bbE_{t-\tau^{\prime}}[\|\nabla J^k(\theta_t^k)\|] + L_{J'} \sum_{i=j}^{t-1}
\bbE_{t-\tau^{\prime}}[\|\theta_{i+1}^k-\theta_i^k\|]\\
\le &
\bbE_{t-\tau^{\prime}}[\|\nabla J^k(\theta_t^k)\|]
+
\tau^{\prime} \alpha L_{J'} B U_\delta, 
\end{align*}
where the first inequality follows from Lemma \ref{lmm: LJ' lip}, and the last inequality follows from Lemma \ref{lmm: U delta}.  
Similarly, 
\begin{align*}
\bbE_{t-\tau^{\prime}}[\|\nabla J^k(\theta_j^k)\|^2]
\leq 2\bbE_{t-\tau^{\prime}}[\|\nabla J^k(\theta_t^k)\|^2]
+2(\tau^{\prime})^2\alpha^2  L_{J'}^2B^2U_\delta^2.
\end{align*}
\end{proof}

Compared with the coarse bounds implied by Lemma \ref{lmm: U delta}, the following lemma provides a finer-grained characterization of the upper bounds on  
$\|\bbE_{t-\tau}\big[\tilde b_{t,L}^k - \bar b_{L,\theta_t^k}^k \big] \|$, 
$\|\bbE_{t-\tau}\big[\tilde A_{t,L}^k - A_{L,\theta_t^k}^k \big] \|$, 
and 
$\|\bbE_{t-\tau}[\bfb_{t,L}^k]\|$.  

\begin{lemma}
\label{lmm: MC observable mixing} 
Suppose that Assumption \ref{assp: uniform ergodicity}. 
Fix $\tau>0$. For any $t\ge \tau$, $L>0$, and $k\in [K]$, it holds that
\begin{align*}
&\|\bbE_{t-\tau}\big[\tilde b_{t,L}^k - \bar b_{L,\theta_t^k}^k \big] \| 
\le \frac{C_{\mathrm{mix},1}}{(1-\gamma)^2}
     \Bigg(
     (1-\gamma)\tau\alpha \bbE_{t-\tau}\qth{\|\nabla J^k(\theta_t^k)\|}
     +\tau\alpha \bbE_{t-\tau}\|x_t^k\|
     +\frac{\tau^2\alpha \beta}{L(1-\gamma)}\\
&\qquad \qquad \qquad \qquad\qquad 
     +\frac{\tau^2\alpha \zeta}{L(1-\gamma)}
     +\tau^2\alpha^2  
     + \frac{\tau^2\alpha^2 }{(1-\gamma)^2}
     +\frac{\tau\alpha }{\sqrt{L(1-\gamma)}}
     \Bigg)
     +\frac{2U_r}{1-\gamma}m\rho^{\tau L},\\
&\|\bbE_{t-\tau}\big[\tilde A_{t,L}^k - A_{L,\theta_t^k}^k \big] \| 
\leq \frac{C_{\mathrm{mix},2}}{(1-\gamma)^2}
     \Bigg(
     (1-\gamma)\tau\alpha \bbE_{t-\tau}\qth{\|\nabla J^k(\theta_t^k)\|}
     +\tau\alpha \bbE_{t-\tau}\|x_t^k\|
     +\frac{\tau^2\alpha \beta}{L(1-\gamma)}\\
&\qquad \qquad \qquad \qquad\qquad
     +\frac{\tau^2\alpha \zeta}{L(1-\gamma)}
     +\tau^2\alpha^2  
     + \frac{\tau^2\alpha^2 }{(1-\gamma)^2} 
     +\frac{\tau\alpha }{\sqrt{L(1-\gamma)}}
     \Bigg)
     +2m\rho^{\tau L},\\
&\|\bbE_{t-\tau}[\bfb_{t,L}^k]\| 
\le 
\frac{C_{\mathrm{mix},3}}{1-\gamma}
\Bigg(
 (1-\gamma)\tau\alpha \bbE_{t-\tau}\qth{\|\nabla J^k(\theta_t^k)\|}
+\tau\alpha \bbE_{t-\tau}\|x_t^k\|
+\frac{\tau^2\alpha \beta}{L(1-\gamma)}
+\frac{\tau^2\alpha \zeta}{L(1-\gamma)}\\
&\qquad \qquad \qquad\qquad 
+\tau^2\alpha^2  
+ \frac{\tau^2\alpha^2}{(1-\gamma)^2} 
+\frac{\tau\alpha }{\sqrt{L(1-\gamma)}}
\Bigg)
+\frac{2U_\delta}{1-\gamma}m\rho^{\tau L}.
\end{align*}
where 
\begin{align*}
   C_{\mathrm{mix},1}&= \frac{2U_r C_{\mu} + 4 U_r L_{\pi}|\calA|}{(1-\rho)^2}  \max\{1, C_{1}^{\mathrm{act}}, C_{2}^{\mathrm{act}}, C_{3}^{\mathrm{act}}, C_{4}^{\mathrm{act}}, C_{5}^{\mathrm{act}}, C_{6}^{\mathrm{act}}\},\\
    C_{\mathrm{mix},2}&=\pth{4C_\mu +\frac{8}{(1-\rho)^2}L_\pi|\calA|}
    \max\{1,C_{1}^{\mathrm{act}},C_{2}^{\mathrm{act}},C_{3}^{\mathrm{act}},C_{4}^{\mathrm{act}},C_{5}^{\mathrm{act}}, C_{6}^{\mathrm{act}}\},\\
C_{\mathrm{mix},3} & = \frac{2L_{*,1} + 3U_\delta L_\pi|\calA|}{(1-\rho)^2} \max\{1, C_{1}^{\mathrm{act}}, C_{2}^{\mathrm{act}}, C_{3}^{\mathrm{act}}, C_{4}^{\mathrm{act}}, C_{5}^{\mathrm{act}}, C_{6}^{\mathrm{act}}\}.
\end{align*}
and  $C_{1}^{\mathrm{act}},\ldots,C_{6}^{\mathrm{act}}$ are defined in Lemma \ref{lmm: theta onestep distance}, and $C_\mu$ is defined in Lemma \ref{lmm: st-st perturb}. 
\end{lemma}
\begin{proof}

Recall from Eqs.\,(\ref{def: Markov drift}) and \eqref{def: Markov noise} that
\begin{align*}
\bfb_{t,L}^k = \sum_{\ell=0}^{L-1} \gamma^\ell r_{t,\ell}^k \phi(s_{t,0}^k)  + (\gamma^L \phi(s_{t,L}^k)  - \phi(s_{t,0}^k))^\top \bfB^*\omega_t^{k,*}\phi(s_{t,0}^k).       
\end{align*}
Recall from Section \ref{sec: Notations} that, for ease of exposition, for any given $\theta$, we write $\mathbb{E}_{s^{(0)}\sim\mu_{\theta}^k, (a^{(\ell)}, s^{(\ell+1)}) \sim \pi_{\theta} \otimes P^k ~ \forall \ell\in [0, L-1]} $ as  $\mathbb{E}_{\mu_{\theta}^k,  \pi_{\theta}, P^k }$ for short. 
The Bellman equation of TD(L) gives  
\[
\mathbb{E}_{\mu_{\theta_{t-\tau}^k}^k,  \pi_{\theta_{t-\tau}^k}, P^k } \qth{\sum_{\ell=0}^{L-1} \pth{\gamma^\ell r(a^{(\ell)}, s^{(\ell)}) + (\gamma^L\phi(s^{(L)}) - \phi(s^{(0)}))^{\top} \bfB^*\omega_{t-\tau}^{k,*}}\phi(s^{(0)})} =0.
\]
In addition, recall from Section \ref{sec: Notations}, $\tilde a_{t,\ell}^k$, $\tilde s_{t,\ell}^k$, and $\tilde r_{t,\ell}^k$ denote the state-action pairs and the corresponding rewards generated under the auxiliary Markov chain, in which the policy remains fixed at $\theta_{t-\tau}^k$.
We have  
\begin{equation}
\label{eq: stepback: theta}
\begin{aligned}
&\|\bbE_{t-\tau}[\bfb_{t,L}^k]\|  \\
=&\left \| \bbE_{t-\tau}\qth{\sum_{\ell=0}^{L-1}\gamma^\ell\big( r_{t,\ell}^k + (\gamma\phi(s_{t,\ell+1}^k) - \phi(s_{t,\ell}^k))^\top \mathbf{B}^*\omega_t^{k,*}\big)\phi(s_{t,0}^k)} \right.\\
&\quad\left. - \mathbb{E}_{\mu_{\theta_{t-\tau}^k}^k,   \pi_{\theta_{t-\tau}^k}, P^k } \qth{\sum_{\ell=0}^{L-1} \gamma^\ell\pth{r(a^{(\ell)}, s^{(\ell)}) + (\gamma\phi(s^{(\ell+1)}) - \phi(s^{(\ell)}))^{\top} \bfB^*\omega_{t-\tau}^{k,*}}\phi(s^{(0)})} \right \| \\
=&\left \| \bbE_{t-\tau}\qth{\sum_{\ell=0}^{L-1}\gamma^\ell\big(r_{t,\ell}^k + (\gamma\phi(s_{t,\ell+1}^k) - \phi(s_{t,\ell}^k))^\top \mathbf{B}^*\omega_t^{k,*}\big)\phi(s_{t,0}^k)} \right.\\
&\quad -\bbE_{t-\tau}\qth{\sum_{\ell=0}^{L-1}\gamma^\ell\big(r_{t,\ell}^k + (\gamma\phi(s_{t,\ell+1}^k) - \phi(s_{t,\ell}^k))^\top \mathbf{B}^*\omega_{t-\tau}^{k,*}\big)\phi(s_{t,0}^k)}\\
&\quad+\bbE_{t-\tau}\qth{\sum_{\ell=0}^{L-1}\gamma^\ell\big(r_{t,\ell}^k + (\gamma\phi(s_{t,\ell+1}^k) - \phi(s_{t,\ell}^k))^\top \mathbf{B}^*\omega_{t-\tau}^{k,*}\big)\phi(s_{t,0}^k)}\\
&\quad-\bbE_{t-\tau}\qth{\sum_{\ell=0}^{L-1}\gamma^\ell\big(\tilde r_{t,\ell}^k + (\gamma\phi(\tilde s_{t,\ell+1}^k) - \phi(\tilde s_{t,\ell}^k))^\top \mathbf{B}^*\omega_{t-\tau}^{k,*}\big)\phi(\tilde s_{t,0}^k)}\\
&\quad+\bbE_{t-\tau}\qth{\sum_{\ell=0}^{L-1}\gamma^\ell\big(\tilde r_{t,\ell}^k + (\gamma\phi(\tilde s_{t,\ell+1}^k) - \phi(\tilde s_{t,\ell}^k))^\top \mathbf{B}^*\omega_{t-\tau}^{k,*}\big)\phi(\tilde s_{t,0}^k)}\\
&\quad\left. - \mathbb{E}_{\mu_{\theta_{t-\tau}^k}^k, \pi_{\theta_{t-\tau}^k}, P^k} \qth{\sum_{\ell=0}^{L-1} \gamma^\ell\pth{r(a^{(\ell)}, s^{(\ell)})+ (\gamma\phi(s^{(\ell+1)}) - \phi(s^{(\ell)}))^{\top} \bfB^*\omega_{t-\tau}^{k,*}}\phi(s^{(0)})} \right \|.       
\end{aligned}
\end{equation}
The first difference can be bounded through the Lipschitz continuity of $\omega^{k,*}(\theta)$ (i.e., Lemma \ref{lmm: L* lip}), and the last 
difference can be bounded via the geometric mixing of the auxiliary Markov chain to the steady state distribution under Assumption \ref{assp: uniform ergodicity}. As for the second difference, we bound the distance between the original Markov chain and the auxiliary Markov chain via Lemma \ref{lmm: tv dist between oc and ac}. Specifically: 

\underline{The first difference in Eq.\eqref{eq: stepback: theta}:} By Lemma \ref{lmm: L* lip}, we have 
\begin{align*}
&\left \| \bbE_{t-\tau}\qth{\sum_{\ell=0}^{L-1}\gamma^\ell\big( r_{t,\ell}^k + (\gamma\phi(s_{t,\ell+1}^k) - \phi(s_{t,\ell}^k))^\top \mathbf{B}^*\omega_t^{k,*}\big)\phi(s_{t,0}^k)} \right.\\
& \quad \left. -\bbE_{t-\tau}\qth{\sum_{\ell=0}^{L-1}\gamma^\ell\big(r_{t,\ell}^k + (\gamma\phi(s_{t,\ell+1}^k) - \phi(s_{t,\ell}^k))^\top \mathbf{B}^*\omega_{t-\tau}^{k,*}\big)\phi(s_{t,0}^k)}\right\|\\
\le &2L_{*,1}\bbE_{t-\tau}\|\theta_{t}^k-\theta_{t-\tau}^k\| 
\leq 2L_{*,1}\sum_{i=t-\tau}^{t-1}\bbE_{t-\tau}[\|\theta_{i+1}^k-\theta_{i}^k\|]
\end{align*}

\underline{The second difference in Eq.\eqref{eq: stepback: theta}:} By Lemmas \ref{lmm: tv dist between oc and ac} and \ref{lmm: AC, OC dist on f: critic}, we have 
\begin{align*}
&\left \| \bbE_{t-\tau}\qth{\sum_{\ell=0}^{L-1}\gamma^\ell r_{t,\ell}^k\phi(s_{t,0}^k) + (\gamma^L\phi(s_{t,L}^k) - \phi(s_{t,0}^k))^\top \mathbf{B}^*\omega_{t-\tau}^{k,*}\phi(s_{t,0}^k)}\right.\\
&\quad \left.-\bbE_{t-\tau} \qth{\sum_{\ell=0}^{L-1}\gamma^\ell r_{t,\ell}^k\phi(\tilde s_{t,0}^k) + (\gamma^L\phi(\tilde s_{t,L}^k) - \phi(\tilde s_{t,0}^k))^\top \mathbf{B}^*\omega_{t-\tau}^{k,*}\phi(\tilde s_{t,0}^k)}\right \|\\
\leq & \frac{U_\delta}{1-\gamma}  \pth{ \frac{1}{(1-\rho^L)(1-\rho)} +\frac{2-\rho}{1-\rho} } L_\pi |\calA| \sum_{i=t-\tau}^{t-1}\bbE_{t-\tau}[\|\theta_{i+1}^k-\theta_{i}^k\|]\\
\leq& \frac{3U_\delta}{(1-\gamma)(1-\rho)^2}L_\pi|\calA| \sum_{i=t-\tau}^{t-1}\bbE_{t-\tau}[\|\theta_{i+1}^k-\theta_{i}^k\|].  
\end{align*}

\underline{The last difference in Eq.\eqref{eq: stepback: theta}:} By Lemma \ref{lmm: tv dist between oc and ac}, we have 
\begin{align*}
&\left \| \bbE_{t-\tau}\qth{\sum_{\ell=0}^{L-1}\gamma^\ell\big( \tilde r_{t,\ell}^k + (\gamma\phi(\tilde s_{t,\ell+1}^k) - \phi(\tilde s_{t,\ell}^k))^\top \mathbf{B}^*\omega_{t-\tau}^{k,*}\big)\phi(\tilde s_{t,0}^k)} \right.\\
&\quad \left. - \mathbb{E}_{\mu_{\theta_{t-\tau}^k}^k, \pi_{\theta_{t-\tau}^k}, P^k} \qth{\sum_{\ell=0}^{L-1} \gamma^\ell\pth{r(a^{(\ell)}, s^{(\ell)}) + (\gamma\phi(s^{(\ell+1)}) - \phi(s^{(\ell)}))^{\top} \bfB^*\omega_{t-\tau}^{k,*}}\phi(s^{(0)})} \right \| \\
\leq&2\frac{U_{\delta}}{1-\gamma} \cdot d_{TV}(\mathbb{U}_{t-\tau: (t,L)}^k(s_{t,0}^k, a_{t,0}^k, s_{t,1}^k, \cdots, a_{t, L-1}^k, s_{t, L}^k ), \mu_{\theta_{t-\tau}^k}^k \otimes (\pi_{\theta_{t-\tau}^k}\otimes P^k)^{L-1})\\
\overset{(a)}{\le} &2\frac{U_{\delta}}{1-\gamma} \cdot d_{TV}(\mathbb{U}_{t-\tau: (t,L)}^k(s_{t,0}^k), \mu_{\theta_{t-\tau}^k}^k (s_{t,0}^k))\\
\le & 2\frac{U_{\delta}}{1-\gamma} m \rho^{\tau L},
\end{align*}
where (a) holds by coupling argument and marginalization, and the last inequality follows from Assumption \ref{assp: uniform ergodicity} and the fact that each round consists of $L$ local state-action pairs. 

Combining the above three upper bounds, we get 
\begin{align*}
&\|\bbE_{t-\tau}[\bfb_{t,L}^k]\|\\
\leq &\pth{2L_{*,1}+\frac{3U_\delta}{(1-\gamma)(1-\rho)^2}L_\pi|\calA|} \sum_{i=t-\tau}^{t-1}\bbE_{t-\tau}[\|\theta_{i+1}^k-\theta_{i}^k\|] +  \frac{2U_\delta}{1-\gamma} m \rho^{\tau L}\\
\le & \pth{2L_{*,1}+\frac{3U_\delta}{(1-\gamma)(1-\rho)^2}L_\pi|\calA|} 
\tau\left ((1-\gamma)\alpha \bbE_{t-\tau}\qth{\|\nabla J^k(\theta_t^k)\|}
+C_{1}^{\mathrm{act}}\alpha \bbE_{t-\tau}\qth{\|x_t^k\|}  \right.\\
&\qquad \left. +\frac{C_{2}^{\mathrm{act}}\tau\alpha \beta}{L(1-\gamma)} +\frac{C_{3}^{\mathrm{act}}\tau\alpha \zeta}{L(1-\gamma)} +C_{4}^{\mathrm{act}}\tau\alpha^2 + \frac{C_{5}^{\mathrm{act}}\alpha }{\sqrt{L(1-\gamma)}} + \frac{C_{6}^{\mathrm{act}}\tau\alpha^2}{(1-\gamma)^2}\right ) + \frac{2U_\delta}{1-\gamma} m \rho^{\tau L},  
\end{align*}
where the last inequality follows from Lemma \ref{lmm: theta onestep distance}. 
Note that 
\begin{align*}
2L_{*,1}+\frac{3U_\delta}{(1-\gamma)(1-\rho)^2}L_\pi|\calA| 
& \le \frac{2L_{*,1} + 3U_\delta L_\pi|\calA}{(1-\gamma)(1-\rho)^2}. 
\end{align*}
By setting $C_{\mathrm{mix},3}$ as 
\begin{align*}
C_{\mathrm{mix},3} = \frac{2L_{*,1} + 3U_\delta L_\pi|\calA|}{(1-\rho)^2} \max\{1, C_{1}^{\mathrm{act}}, C_{2}^{\mathrm{act}}, C_{3}^{\mathrm{act}}, C_{4}^{\mathrm{act}}, C_{5}^{\mathrm{act}}, C_{6}^{\mathrm{act}}\}.     
\end{align*}

the definition of $C_{\mathrm{mix},3}$, the last display can be further simplified as 
\begin{align*}
\|\bbE_{t-\tau}[\bfb_{t,L}^k]\| 
\le& 
\frac{C_{\mathrm{mix},3}}{1-\gamma}
\Bigg(
(1-\gamma)\tau\alpha \bbE_{t-\tau}\qth{\|\nabla J^k(\theta_t^k)\|}
+\tau\alpha \bbE_{t-\tau}\|x_t^k\|
+\frac{\tau^2\alpha \beta}{L(1-\gamma)}\\
&\qquad
+\frac{\tau^2\alpha \zeta}{L(1-\gamma)}
+\tau^2\alpha^2  
+ \frac{\tau^2\alpha^2}{(1-\gamma)^2} 
+\frac{\tau\alpha }{\sqrt{L(1-\gamma)}}
\Bigg)
+\frac{2U_\delta}{1-\gamma}m\rho^{\tau L}.
\end{align*}

\textbf{Recall} from Eq.\,(\ref{eq: local reward Markovian noise}) that 
\begin{align*}
    \tilde{b}_{t,L}^k = \sum_{\ell=0}^{L-1}\gamma^\ell r_{t,\ell}^k\phi(s_{t,0}^k);\quad \bar{b}_{L,\theta_t^k}^k=  \bbE_{s^{(0)}\sim\mu_{\theta_t^k}^k, (a^{(\ell)}, s^{(\ell)}) \sim \pi_{\theta_t^k} \otimes P^k ~ \forall \ell\in [0, L-1]}\qth{\sum_{\ell=0}^{L-1}\gamma^\ell r(s^{(\ell)},a^{(\ell)})\phi(s^{(0)})}.
\end{align*}

Then, we can similarly decompose the difference into three differences.
\begin{equation}
\label{eqn: b difference: critic}
\begin{aligned}
    &\lnorm{\bbE_{t-\tau}[\tilde{b}_{t,L}^k - \bar{b}_{L,\theta_t^k}^k]}{}\\
    =&\lnorm{\bbE_{t-\tau}\qth{\sum_{\ell=0}^{L-1}\gamma^\ell r_{t,\ell}^k\phi(s_{t,0}^k)}-\bbE_{s^{(0)}\sim\mu_{\theta_t^k}^k, (a^{(\ell)}, s^{(\ell)}) \sim \pi_{\theta_t^k} \otimes P^k ~ \forall \ell\in [0, L-1]}\qth{\sum_{\ell=0}^{L-1}\gamma^\ell r(s^{(\ell)},a^{(\ell)})\phi(s^{(0)})}}{}\\
    \leq& \lnorm{\bbE_{t-\tau}\qth{\sum_{\ell=0}^{L-1}\gamma^\ell r_{t,\ell}^k\phi(s_{t,0}^k)}-\bbE_{t-\tau}\qth{\sum_{\ell=0}^{L-1}\gamma^\ell \tilde r_{t,\ell}^k\phi(\tilde s_{t,0}^k)}}{}\\
    &+\lnorm{\bbE_{t-\tau}\qth{\sum_{\ell=0}^{L-1}\gamma^\ell \tilde r_{t,\ell}^k\phi(\tilde s_{t,0}^k)}- \bbE_{s^{(0)}\sim\mu_{\theta_{t-\tau}^k}^k, (a^{(\ell)}, s^{(\ell)}) \sim \pi_{\theta_{t-\tau}^k} \otimes P^k ~ \forall \ell\in [0, L-1]}\qth{\sum_{\ell=0}^{L-1}\gamma^\ell r(s^{(\ell)},a^{(\ell)})\phi(s^{(0)})}}{}\\
    &+\Bigg\| \bbE_{s^{(0)}\sim\mu_{\theta_{t-\tau}^k}^k, (a^{(\ell)}, s^{(\ell)}) \sim \pi_{\theta_{t-\tau}^k} \otimes P^k ~ \forall \ell\in [0, L-1]}\qth{\sum_{\ell=0}^{L-1}\gamma^\ell r(s^{(\ell)},a^{(\ell)})\phi(s^{(0)})}\\
    &\qquad\qquad-\bbE_{s^{(0)}\sim\mu_{\theta_t^k}^k, (a^{(\ell)}, s^{(\ell)}) \sim \pi_{\theta_t^k} \otimes P^k ~ \forall \ell\in [0, L-1]}\qth{\sum_{\ell=0}^{L-1}\gamma^\ell r(s^{(\ell)},a^{(\ell)})\phi(s^{(0)})}\Bigg\|.    
\end{aligned}
\end{equation}

\underline{The first difference in Eq.\eqref{eqn: b difference: critic}:}
By Lemma \ref{lmm: AC, OC dist on f: critic}, we have 
\begin{align*}
&\lnorm{\bbE_{t-\tau}\qth{\sum_{\ell=0}^{L-1}\gamma^\ell r_{t,\ell}^k\phi(s_{t,0}^k)}-\bbE_{t-\tau}\qth{\sum_{\ell=0}^{L-1}\gamma^\ell \tilde r_{t,\ell}^k\phi(\tilde s_{t,0}^k)}}{}\\
\leq&
    \frac{U_r}{1-\gamma}  \frac{3}{(1-\rho)^2} L_\pi |\calA| \sum_{i=t-\tau}^{t-1}\bbE_{t-\tau}\qth{\|\theta_{i+1}^k-\theta_{i}^k\|}.
\end{align*}

\underline{The second difference in Eq.\eqref{eqn: b difference: critic}:}

\begin{align*}
    &\lnorm{\bbE_{t-\tau}\qth{\sum_{\ell=0}^{L-1}\gamma^\ell \tilde r_{t,\ell}^k\phi(\tilde s_{t,0}^k)}- \bbE_{s^{(0)}\sim\mu_{\theta_{t-\tau}^k}^k, (a^{(\ell)}, s^{(\ell)}) \sim \pi_{\theta_{t-\tau}^k} \otimes P^k ~ \forall \ell\in [0, L-1]}\qth{\sum_{\ell=0}^{L-1}\gamma^\ell r(s^{(\ell)},a^{(\ell)})\phi(s^{(0)})}}{}\\
    \leq& \frac{2U_r}{1-\gamma}d_{TV}(\mathbb{U}_{t-\tau: (t,L)}^k(s_{t,0}^k, a_{t,0}^k, s_{t,1}^k, \cdots, a_{t, L-1}^k, s_{t, L}^k ), \mu_{\theta_{t-\tau}^k}^k \otimes (\pi_{\theta_{t-\tau}^k}\otimes P^k)^{L-1})\\
    \leq& \frac{2U_r}{1-\gamma}d_{TV}(\mathbb{U}_{t-\tau: (t,L)}^k(s_{t,0}^k), \mu_{\theta_{t-\tau}^k}^k) \\ 
    \leq& \frac{2U_r}{1-\gamma}m\rho^{\tau L}.
\end{align*}

\underline{The third difference in Eq.\eqref{eqn: b difference: critic}:} 
\begin{align*}
    &\Bigg\| \bbE_{s^{(0)}\sim\mu_{\theta_{t-\tau}^k}^k, (a^{(\ell)}, s^{(\ell+1)}) \sim \pi_{\theta_{t-\tau}^k} \otimes P^k ~ \forall \ell\in [0, L-1]}\qth{\sum_{\ell=0}^{L-1}\gamma^\ell r(s^{(\ell)},a^{(\ell)})\phi(s^{(0)})}\\
    &\qquad\qquad-\bbE_{s^{(0)}\sim\mu_{\theta_t^k}^k, (a^{(\ell)}, s^{(\ell+1)}) \sim \pi_{\theta_t^k} \otimes P^k ~ \forall \ell\in [0, L-1]}\qth{\sum_{\ell=0}^{L-1}\gamma^\ell r(s^{(\ell)},a^{(\ell)})\phi(s^{(0)})}\Bigg\|\\
    =    &\Bigg\|\sum_{\ell=0}^{L-1}\gamma^\ell\Bigg(\bbE_{s^{(0)}\sim\mu_{\theta_{t-\tau}^k}^k, (a^{(\ell^{\prime})}, s^{(\ell^{\prime}+1)}) \sim \pi_{\theta_{t-\tau}^k} \otimes P^k ~ \forall \ell^{\prime}\le \ell}\qth{ r(s^{(\ell)},a^{(\ell)})\phi(s^{(0)})}\\
    &\qquad\qquad-\bbE_{s^{(0)}\sim\mu_{\theta_t^k}^k, (a^{(\ell^{\prime})}, s^{(\ell^{\prime}+1)}) \sim \pi_{\theta_t^k} \otimes P^k ~ \forall \ell^{\prime}\le \ell}\qth{r(s^{(\ell)},a^{(\ell)})\phi(s^{(0)})}\Bigg)\Bigg\|\\
    \leq& U_r\sum_{\ell=0}^{L-1}\gamma^\ell  \bbE_{t-\tau}\qth{2d_{TV}(\mu_{\theta_{t-\tau}^k}^k\otimes (\pi_{\theta_{t-\tau}^k}\otimes P^k)^\ell,\mu_{\theta_{t}^k}^k\otimes (\pi_{\theta_{t}^k}\otimes P^k)^\ell)}\\
    \leq& U_r\sum_{\ell=0}^{L-1}\gamma^\ell \bbE_{t-\tau}\qth{2 d_{TV}(\mu_{\theta_{t-\tau}^k}^k,\mu_{\theta_{t}^k}^k)+2\ell \sup_s d_{TV}(\pi_{\theta_{t-\tau}^k}(\cdot|s),\pi_{\theta_{t}^k}(\cdot|s))}\\
    \overset{(a)}{\leq}& U_r\sum_{\ell=0}^{L-1}\gamma^\ell (2 \bbE_{t-\tau}\qth{d_{TV}(\mu_{\theta_{t-\tau}^k}^k,\mu_{\theta_{t}^k}^k)}+\ell L_\pi |\calA| \bbE_{t-\tau}\qth{\|\theta_{t-\tau}^k - \theta_t^k\|})\\
    \leq& \frac{2U_rC_\mu}{1-\gamma} \bbE_{t-\tau}\qth{\|\theta_{t-\tau}^k - \theta_t^k\|} +\frac{U_r\gamma}{(1-\gamma)^2} L_\pi |\calA|  \bbE_{t-\tau}\qth{\|\theta_{t-\tau}^k - \theta_t^k\|}, 
\end{align*}
where inequality (a) follows from Assumption \ref{assmp: Lipschitz of policy}, and the last inequality follows from Lemma \ref{lmm: st-st perturb} and the fact that $\sum_{\ell=0}^{L-1} \gamma^{\ell}\ell \le \frac{\gamma}{(1-\gamma)^2}$.  

Therefore,
\begin{align*}
    &\lnorm{\bbE_{t-\tau}[\tilde{b}_{t,L}^k - \bar{b}_{L,\theta_t^k}^k]}{}\\
    \leq& \frac{2U_rC_\mu}{1-\gamma} \bbE_{t-\tau}\qth{\|\theta_{t-\tau}^k - \theta_t^k\|} +\frac{U_r\gamma}{(1-\gamma)^2} L_\pi |\calA|  \bbE_{t-\tau}\qth{\|\theta_{t-\tau}^k - \theta_t^k\|}+\frac{2U_r}{1-\gamma}m\rho^{\tau L}\\
    &+\frac{U_r}{1-\gamma}  \frac{3}{(1-\rho)^2} L_\pi |\calA| \sum_{i=t-\tau}^{t-1}\bbE_{t-\tau}\qth{\|\theta_{i+1}^k-\theta_{i}^k\|}\\
    \leq& \pth{\frac{2U_rC_\mu}{1-\gamma}+\frac{U_r\gamma}{(1-\gamma)^2} L_\pi |\calA| + \frac{U_r}{1-\gamma}  \frac{3}{(1-\rho)^2} L_\pi |\calA| }\sum_{i=t-\tau}^{t-1}\bbE_{t-\tau}\qth{\|\theta_{i+1}^k-\theta_{i}^k\|}+\frac{2U_r}{1-\gamma}m\rho^{\tau L}\\
    \leq& \frac{2U_r C_{\mu} + 4 U_r L_{\pi}|\calA|}{(1-\gamma)^2 (1-\rho)^2}\tau  
    \Bigg((1-\gamma)\alpha \bbE_{t-\tau}\qth{\|\nabla J^k(\theta_t^k)\|}
+C_{1}^{\mathrm{act}}\alpha \bbE_{t-\tau}\qth{\|x_t^k\|}
+C_{2}^{\mathrm{act}}\frac{\tau\alpha \beta}{L(1-\gamma)} \\
& \qquad \qquad +C_{3}^{\mathrm{act}}\frac{\tau\alpha \zeta}{L(1-\gamma)} 
+C_{4}^{\mathrm{act}}\tau\alpha^2  
+C_{5}^{\mathrm{act}}\frac{\alpha }{\sqrt{L(1-\gamma)}} 
+ C_{6}^{\mathrm{act}}\frac{\tau\alpha^2}{(1-\gamma)^2} \Bigg) +\frac{2U_r}{1-\gamma}m\rho^{\tau L},  
\end{align*}
where the last inequality follows from Lemma \ref{lmm: theta onestep distance}.  
Setting $C_{\mathrm{mix},1}$ as 
\begin{align*}
C_{\mathrm{mix},1} =  \frac{2U_r C_{\mu} + 4 U_r L_{\pi}|\calA|}{(1-\rho)^2}  \max\{1, C_{1}^{\mathrm{act}}, C_{2}^{\mathrm{act}}, C_{3}^{\mathrm{act}}, C_{4}^{\mathrm{act}}, C_{5}^{\mathrm{act}}, C_{6}^{\mathrm{act}}\},
\end{align*}
we obtain the upper bound on $\|\bbE_{t-\tau}\big[\tilde b_{t,L}^k - \bar b_{L,\theta_t^k}^k \big] \| $ as stated in Lemma \ref{lmm: MC observable mixing}.

\textbf{Recall} from Eq.\,(\ref{def: Markov drift}) that $$\tilde A_{t,L}^k= \phi(s_{t,0}^k)(\gamma^L\phi(s_{t,L}^k)-\phi(s_{t,0}^k))^\top,$$ and
\[A_{L,\theta_t^k}^k=\bbE_{s^{(0)}\sim\mu_{\theta_t^k}^k, (a^{(\ell)}, s^{(\ell+1)}) \sim \pi_{\theta_t^k} \otimes P^k ~ \forall \ell\in [0, L-1]}\qth{\phi(s^{(0)})(\gamma^{L}\phi(s^{(L)}) - \phi(s^{(0)}))^\top}.\]

Similar to the above derivation, we consider the following decomposition 
\begin{equation}
\label{eqn: A difference: critic} 
\begin{aligned}
&\Big\|\bbE_{t-\tau}\big[\tilde A_{t,L}^k - A_{L,\theta_t^k}^k \big] \Big\| =\Big\|\bbE_{t-\tau}\big[\phi(s_{t,0}^k)(\gamma^L\phi(s_{t,L}^k)-\phi(s_{t,0}^k))^\top \\
&\qquad - \bbE_{s^{(0)}\sim\mu_{\theta_t^k}^k, (a^{(\ell)}, s^{(\ell+1)}) \sim \pi_{\theta_t^k} \otimes P^k ~ \forall \ell\in [0, L-1]}\qth{\phi(s^{(0)})(\gamma^{L}\phi(s^{(L)}) - \phi(s^{(0)}))^\top} \big] \Big\|  \\
=& \Big\|\bbE_{t-\tau}\big[\phi(s_{t,0}^k)(\gamma^L\phi(s_{t,L}^k)-\phi(s_{t,0}^k))^\top\big]-  \bbE_{t-\tau}\big[\phi(\tilde s_{t,0}^k)(\gamma^L\phi(\tilde s_{t,L}^k)-\phi(\tilde s_{t,0}^k))^\top\big]\\
&\qquad+\bbE_{t-\tau}\big[\phi(\tilde s_{t,0}^k)(\gamma^L\phi(\tilde s_{t,L}^k)-\phi(\tilde s_{t,0}^k))^\top\big]\\&\qquad\qquad - \bbE_{t-\tau}\big[\bbE_{s^{(0)}\sim\mu_{\theta_{t-\tau}^k}^k, (a^{(\ell)}, s^{(\ell+1)}) \sim \pi_{\theta_{t-\tau}^k} \otimes P^k ~ \forall \ell\in [0, L-1]}\qth{\phi(s^{(0)})(\gamma^{L}\phi(s^{(L)}) - \phi(s^{(0)}))^\top} \big]\\
&+\bbE_{t-\tau}\big[\bbE_{s^{(0)}\sim\mu_{\theta_{t-\tau}^k}^k, (a^{(\ell)}, s^{(\ell+1)}) \sim \pi_{\theta_{t-\tau}^k} \otimes P^k ~ \forall \ell\in [0, L-1]}\qth{\phi(s^{(0)})(\gamma^{L}\phi(s^{(L)}) - \phi(s^{(0)}))^\top} \big]\\
&-\bbE_{t-\tau}\big[\bbE_{s^{(0)}\sim\mu_{\theta_t^k}^k, (a^{(\ell)}, s^{(\ell+1)}) \sim \pi_{\theta_t^k} \otimes P^k ~ \forall \ell\in [0, L-1]}\qth{\phi(s^{(0)})(\gamma^{L}\phi(s^{(L)}) - \phi(s^{(0)}))^\top} \big] \Big\|\\
\leq& \Big\|\bbE_{t-\tau}\big[\phi(s_{t,0}^k)(\gamma^L\phi(s_{t,L}^k)-\phi(s_{t,0}^k))^\top\big]-  \bbE_{t-\tau}\big[\phi(\tilde s_{t,0}^k)(\gamma^L\phi(\tilde s_{t,L}^k)-\phi(\tilde s_{t,0}^k))^\top\big]\Big\|\\
&\qquad+\Big\|\bbE_{t-\tau}\big[\phi(\tilde s_{t,0}^k)(\gamma^L\phi(\tilde s_{t,L}^k)-\phi(\tilde s_{t,0}^k))^\top\big]\\&\qquad\qquad - \bbE_{s^{(0)}\sim\mu_{\theta_{t-\tau}^k}^k, (a^{(\ell)}, s^{(\ell+1)}) \sim \pi_{\theta_{t-\tau}^k} \otimes P^k ~ \forall \ell\in [0, L-1]}\qth{\phi(s^{(0)})(\gamma^{L}\phi(s^{(L)}) - \phi(s^{(0)}))^\top} \Big\|\\
&+\Big\|\bbE_{t-\tau}\big[\bbE_{s^{(0)}\sim\mu_{\theta_{t-\tau}^k}^k, (a^{(\ell)}, s^{(\ell+1)}) \sim \pi_{\theta_{t-\tau}^k} \otimes P^k ~ \forall \ell\in [0, L-1]}\qth{\phi(s^{(0)})(\gamma^{L}\phi(s^{(L)}) - \phi(s^{(0)}))^\top} \big]\\
&\qquad\qquad-\bbE_{t-\tau}\big[\bbE_{s^{(0)}\sim\mu_{\theta_t^k}^k, (a^{(\ell)}, s^{(\ell+1)}) \sim \pi_{\theta_t^k} \otimes P^k ~ \forall \ell\in [0, L-1]}\qth{\phi(s^{(0)})(\gamma^{L}\phi(s^{(L)}) - \phi(s^{(0)}))^\top} \big] \Big\|.
\end{aligned}
\end{equation}

\underline{The first difference in Eq.\eqref{eqn: A difference: critic}:}
By Lemma \ref{lmm: AC, OC dist on f: critic}, we have 
\begin{align*}
    &\Big\|\bbE_{t-\tau}\big[\phi(s_{t,0}^k)(\gamma^L\phi(s_{t,L}^k)-\phi(s_{t,0}^k))^\top\big]-  \bbE_{t-\tau}\big[\phi(\tilde s_{t,0}^k)(\gamma^L\phi(\tilde s_{t,L}^k)-\phi(\tilde s_{t,0}^k))^\top\big]\Big\|\\
    \leq& \frac{6}{(1-\gamma)(1-\rho)^2} L_\pi |\calA| \bbE_{t-\tau}[\sum_{i=t-\tau}^{t-1}\|\theta_{i+1}^k-\theta_{i}^k\|].
\end{align*}

\underline{The second difference in Eq.\eqref{eqn: A difference: critic}:} By Asssumption \ref{assp: uniform ergodicity}, we have 
\begin{align*}
    &\Big\|\bbE_{t-\tau}\big[\phi(\tilde s_{t,0}^k)(\gamma^L\phi(\tilde s_{t,L}^k)-\phi(\tilde s_{t,0}^k))^\top\big]\\&\qquad\qquad - \bbE_{s^{(0)}\sim\mu_{\theta_{t-\tau}^k}^k, (a^{(\ell)}, s^{(\ell+1)}) \sim \pi_{\theta_{t-\tau}^k} \otimes P^k ~ \forall \ell\in [0, L-1]}\qth{\phi(s^{(0)})(\gamma^{L}\phi(s^{(L)}) - \phi(s^{(0)}))^\top} \Big\|\\
    \leq& 2 m\rho^{\tau L}.
\end{align*}

\underline{The third difference in Eq.\eqref{eqn: A difference: critic}:}
\begin{align*}
&\Big\|\bbE_{s^{(0)}\sim\mu_{\theta_{t-\tau}^k}^k, (a^{(\ell)}, s^{(\ell+1)}) \sim \pi_{\theta_{t-\tau}^k} \otimes P^k ~ \forall \ell\in [0, L-1]}\qth{\phi(s^{(0)})(\gamma^{L}\phi(s^{(L)}) - \phi(s^{(0)}))^\top} \\
&\qquad\qquad-\bbE_{s^{(0)}\sim\mu_{\theta_t^k}^k, (a^{(\ell)}, s^{(\ell+1)}) \sim \pi_{\theta_t^k} \otimes P^k ~ \forall \ell\in [0, L-1]}\qth{\phi(s^{(0)})(\gamma^{L}\phi(s^{(L)}) - \phi(s^{(0)}))^\top} \Big\|\\
=& \Big\|\bbE_{s^{(0)}\sim\mu_{\theta_{t-\tau}^k}^k, (a^{(\ell)}, s^{(\ell+1)}) \sim \pi_{\theta_{t-\tau}^k} \otimes P^k ~ \forall \ell\in [0, L-1]}\qth{\sum_{\ell=0}^{L-1}\gamma^\ell\phi(s^{(0)})(\gamma\phi(s^{(\ell+1)}) - \phi(s^{(\ell)}))^\top} \\
&\qquad\qquad-\bbE_{s^{(0)}\sim\mu_{\theta_t^k}^k, (a^{(\ell)}, s^{(\ell+1)}) \sim \pi_{\theta_t^k} \otimes P^k ~ \forall \ell\in [0, L-1]}\qth{\sum_{\ell=0}^{L-1}\gamma^\ell\phi(s^{(0)})(\gamma\phi(s^{(\ell+1)}) - \phi(s^{(\ell)}))^\top} \Big\|\\
=& \Big\|\sum_{\ell=0}^{L-1}\gamma^\ell\bbE_{s^{(0)}\sim\mu_{\theta_{t-\tau}^k}^k, (a^{(\ell^{\prime})}, s^{(\ell^{\prime}+1)}) \sim \pi_{\theta_{t-\tau}^k} \otimes P^k ~ \forall \ell^{\prime}\le \ell}\qth{\phi(s^{(0)})(\gamma\phi(s^{(\ell+1)}) - \phi(s^{(\ell)}))^\top} \\
&\qquad\qquad-\bbE_{s^{(0)}\sim\mu_{\theta_t^k}^k, (a^{(\ell^{\prime})}, s^{(\ell^{\prime}+1)}) \sim \pi_{\theta_t^k} \otimes P^k ~ \forall \ell^{\prime}\le \ell}\qth{\phi(s^{(0)})(\gamma\phi(s^{(\ell+1)}) - \phi(s^{(\ell)}))^\top} \Big\|\\
\leq& 4\sum_{\ell=0}^{L-1}\gamma^\ell d_{TV}(\mu_{\theta_{t-\tau}^k}^k\otimes (\pi_{\theta_{t-\tau}^k}\otimes P^k)^\ell,\mu_{\theta_{t}^k}^k\otimes (\pi_{\theta_{t}^k}\otimes P^k)^\ell)\\
\leq& 4\sum_{\ell=0}^{L-1}\gamma^\ell\pth{ d_{TV}(\mu_{\theta_{t-\tau}^k}^k,\mu_{\theta_{t}^k}^k)+ \ell \sup_s d_{TV}(\pi_{\theta_{t-\tau}^k}(\cdot|s),\pi_{\theta_{t}^k}(\cdot|s))}\\
\overset{(a)}{\leq}& 4\sum_{\ell=0}^{L-1}\gamma^\ell\pth{ C_\mu \bbE_{t-\tau}\qth{\|\theta_{t-\tau}^k - \theta_t^k\|}  + 
\frac{\ell}{2} L_\pi |\calA |\bbE_{t-\tau}\qth{\|\theta_{t-\tau}^k - \theta_t^k\|}}\\
\leq& \pth{\frac{4C_\mu}{1-\gamma}+ \frac{2\gamma}{(1-\gamma)^2}L_\pi |\calA |}\bbE_{t-\tau}\qth{\|\theta_{t-\tau}^k - \theta_t^k\|}.
\end{align*}
where inequality (a) follows from Lemma \ref{lmm: st-st perturb} and Assumption \ref{assmp: Lipschitz of policy}. 
Therefore, 
\begin{align*}
   & \Big\|\bbE_{t-\tau}\big[\tilde A_{t,L}^k - A_{L,\theta_t^k}^k \big] \Big\| \\
    \leq&\pth{\frac{4C_\mu}{1-\gamma}+ \frac{2\gamma}{(1-\gamma)^2}L_\pi |\calA |+\frac{6}{(1-\gamma)(1-\rho)^2} L_\pi |\calA| } \sum_{i=t-\tau}^{t-1}\bbE_{t-\tau}[\|\theta_{i+1}^k-\theta_{i}^k\|]+2 m\rho^{\tau L}\\
    \leq& \pth{\frac{4C_\mu}{1-\gamma}
+\frac{2\gamma}{(1-\gamma)^2}L_\pi|\calA|
+\frac{6}{(1-\gamma)(1-\rho)^2}L_\pi|\calA|}
\Bigg(
(1-\gamma)\tau\alpha \bbE_{t-\tau}\qth{\|\nabla J^k(\theta_t^k)\|}
+C_{1}^{\mathrm{act}}\tau\alpha \bbE_{t-\tau}\|x_t^k\|\\
&\quad
+C_{2}^{\mathrm{act}}\frac{\tau^2\alpha \beta}{L(1-\gamma)}
+C_{3}^{\mathrm{act}}\frac{\tau^2\alpha \zeta}{L(1-\gamma)}
+C_{4}^{\mathrm{act}}\tau^2\alpha^2  
+C_{5}^{\mathrm{act}}\frac{\tau\alpha }{\sqrt{L(1-\gamma)}}
+ C_{6}^{\mathrm{act}}\frac{\tau^2\alpha^2}{(1-\gamma)^2}
\Bigg)
+2m\rho^{\tau L}, 
\end{align*}
where the last inequaltiy follows from Lemma \ref{lmm: theta onestep distance}. 
By the definition of $C_{\mathrm{mix},2}$, the last display can be further simplified as
\begin{align*}
& \|\bbE_{t-\tau}\big[\tilde A_{t,L}^k - A_{L,\theta_t^k}^k \big] \|\\
\leq& \frac{C_{\mathrm{mix},2}}{(1-\gamma)^2}
\Bigg(
(1-\gamma)\tau\alpha \bbE\qth{\|\nabla J^k(\theta_t^k)\||v_{0:t-\tau}}
+\tau\alpha \bbE_{t-\tau}\|x_t^k\|
+\frac{\tau^2\alpha \beta}{L(1-\gamma)}\\
&\qquad
+\frac{\tau^2\alpha \zeta}{L(1-\gamma)}
+\tau^2\alpha^2  + \frac{\tau^2\alpha^2}{(1-\gamma)^2}
+\frac{\tau\alpha }{\sqrt{L(1-\gamma)}} 
\Bigg)
+2m\rho^{\tau L}.
\end{align*}
\end{proof}

\subsection{One Round and Multi-Round Perturbation}

\begin{lemma}%
\label{lmm: crude parameter bound}
Fix $\tau\in \naturals$. 
When $\frac{\zeta U_{\delta} U_{\omega}}{L} \le 1/2$,  
for any $t\geq \tau> 0$ and $k \in [K]$, it holds that  
\begin{align*}
&\|\omega_t^k - \omega_{t-\tau}^k\|
\leq \tau \frac{U_{\delta}}{L(1-\gamma)}\beta, \\
&\|\theta_t^k-\theta_{t-\tau}^k\|\leq \tau\alpha B U_\delta,\\
&\|\mathbf{B}_t - \mathbf{B}_{t-\tau}\|\leq 2\tau \zeta \frac{U_{\delta}}{L(1-\gamma)}U_{\omega} + 4\tau \zeta^2 \frac{U_{\delta}^2}{L^2(1-\gamma)^2} U_{\omega}^2.
\end{align*}   
\end{lemma}
\begin{proof}
For any $t\geq \tau> 0$ and $k \in [K]$, we have 
    \begin{align*}
        &\|\omega_t^k-\omega_{t-\tau}^k\|
        = \|\omega_t^k - \omega_{t-1}^k + \omega_{t-1}^k - \omega_{t-2}^k + \omega_{t-2}^k + ... -  \omega_{t-\tau}^k\| \\      
        &\leq \sum_{i=t-\tau}^{t-1} \|\omega_{i+1}^k-\omega_{i}^k\|\\
        \overset{(a)}{=}&\sum_{i=t-\tau}^{t-1} \|\Pi_{U_{\omega}} \left( \omega^k_i + \beta\frac{1}{L} \delta_{i, L}^k \mathbf{B}_i^\top \phi(s_{i,0}^k) \right)-\Pi_{U_\omega}(\omega_{i}^k)\|\\
        \overset{(b)}{\leq}& \sum_{i=t-\tau}^{t-1} \|\beta\frac{1}{L} \delta_{i, L}^k \mathbf{B}_i^\top \phi(s_{i,0}^k)\| \\
        \overset{(c)}{\leq}& \sum_{i=t-\tau}^{t-1}\frac{\beta}{L}|\delta_{i,L}^k|
        \leq \tau \frac{U_\delta}{L(1-\gamma)}\beta,
    \end{align*}
    where equalty (a) holds because $\Pi_{U_\omega}(\omega_{i}^k) = \omega_{i}^k$, inequality (b) follows from the non-expansive property of projection, 
    inequality (c) follows from Assumption~\ref{ass: bounded features} together with the orthonormality of $\mathbf{B}_i$,  
    and the last inequality follows from Lemma \ref{lmm: U delta}.
    In addition, by the actor update,
    \begin{align*}
        \|\theta_t^k-\theta_{t-\tau}^k\|
        &\leq \sum_{i=t-\tau}^{t-1}\|\theta_{i+1}^k-\theta_i^k\|
        \leq \alpha \sum_{i=t-\tau}^{t-1}
        \left\|
        \frac{1}{L}\sum_{\ell=0}^{L-1}\delta_{i,\ell}^{k,\mathrm{act}}
        \nabla_\theta\log\pi_{\theta_i^k}(\hat{a}_{i,\ell}^k| \hat{s}_{i,\ell}^k)
        \right\|\\
        & \le \alpha \sum_{i=t-\tau}^{t-1}
        \frac{1}{L}\sum_{\ell=0}^{L-1}|\delta_{i,\ell}^{k,\mathrm{act}}|
        \|\nabla_\theta\log\pi_{\theta_i^k}(\hat{a}_{i,\ell}^k| \hat{s}_{i,\ell}^k)\| \\
        & \leq \tau\alpha B U_\delta, 
    \end{align*}
    where the last inequality follows from Lemma \ref{lmm: U delta} and Assumption \ref{assmp: Lipschitz of policy}. 
    
    For $\bfB$, we have
    \begin{align*}
    &\|\bfB_{t}-\bfB_{t - \tau}\|
    \leq  \sum_{i=t-\tau}^{t-1}\|\bfB_{i+1} - \bfB_{i}\| 
    = \sum_{i=t-\tau}^{t-1}\|\bar{\bfB}_{i+1}\bfR_{i+1}^{-1} - \bfB_{i}\| \\
    =& \sum_{i=t-\tau}^{t-1}\| (\mathbf{B}_{i}
    +\frac{\zeta}{KL}\sum_{k=1}^K \bfB_{i,\perp}\bfB_{i,\perp}^\top\delta_{i,L}^k\phi(s_{i,0}^k)(\omega_{i}^k)^\top)\bfR_{i+1}^{-1}  - \bfB_{i}\|\\
    =& \sum_{i=t-\tau}^{t-1}\| (\mathbf{B}_{i}
    +\frac{\zeta}{KL}\sum_{k=1}^K \bfB_{i,\perp}\bfB_{i,\perp}^\top\delta_{i,L}^k\phi(s_{i,0}^k)(\omega_{i}^k)^\top)\bfR_{i+1}^{-1} -\bfB_i\bfR_{i+1}^{-1}+ \bfB_i\bfR_{i+1}^{-1}- \bfB_{i}\|\\
    \leq&  \sum_{i=t-\tau}^{t-1}\|\frac{\zeta}{KL}\sum_{k=1}^K \bfB_{i,\perp}\bfB_{i,\perp}^\top\delta_{i,L}^k\phi(s_{i,0}^k)(\omega_{i}^k)^\top\|\|\bfR_{i+1}^{-1}\| + \sum_{i=t-\tau}^{t-1}\|\bfR_{i+1}^{-1}-\bfI\| \|\bfB_i\|\\
    \leq & \zeta \frac{U_{\delta}}{L(1-\gamma)}U_{\omega} \sum_{i=t-\tau}^{t-1}\|\bfR_{i+1}^{-1}\| + \sum_{i=t-\tau}^{t-1}\|\bfR_{i+1}^{-1}-\bfI\|\\
    \leq &  \zeta \frac{U_{\delta}}{L(1-\gamma)}U_{\omega} \sum_{i=t-\tau}^{t-1} \frac{1}{1-2\|\bfQ_i\|_F^2}  +  \sum_{i=t-\tau}^{t-1} 4\|\bfQ_i\|_F^2, 
\end{align*}
where the last inequality follows from Lemma \ref{lm: perturbation QR}, and $\bfQ_i$ is defined in Eq.\,(\ref{eq: def Q}) as
\begin{align*}
\bfQ_i = \frac{\zeta}{KL}\sum_{k=1}^K \bfB_{i,\perp}\bfB_{i,\perp}^\top\delta_{i,L}^k\phi(s_{i,0}^k)(\omega_{i}^k)^\top.    
\end{align*}
We have 
\begin{align}
\label{eqn: Q crude bound}
\|\bfQ_i\|_F
& \le \frac{\zeta}{KL}\sum_{k=1}^K \|\bfB_{i,\perp}\bfB_{i,\perp}^\top \delta_{i,L}^k\phi(s_i^k)(\omega_{i}^k)^\top  \|_F \nonumber\\
& \le \frac{\zeta}{KL}\sum_{k=1}^K |\delta_{i,L}^k| \|\bfB_{i,\perp}\bfB_{i,\perp}^\top\| \|\phi(s_i^k)\| \|\omega_{i}^k\| \nonumber\\
& \le \zeta \frac{U_{\delta}}{L(1-\gamma)} U_{\omega}. 
\end{align}
When $\zeta \frac{U_{\delta}}{L(1-\gamma)} U_{\omega} \le 1/2$, $1-2\|\bfQ_i\|_F^2 \ge 1/2$.
Therefore, we have   
\begin{align*}
\|\bfB_{t}-\bfB_{t - \tau}\| 
\leq &  2\tau \zeta \frac{U_{\delta}}{L(1-\gamma)}U_{\omega} + 4\tau \zeta^2 \frac{U_{\delta}^2}{L^2(1-\gamma)^2} U_{\omega}^2.  
\end{align*}

\end{proof}

\begin{lemma}
\label{lmm: one step temporal difference bound of omega}
Fix $\tau \in \naturals$. 
Choose $\zeta \le \frac{L(1-\gamma)}{2U_{\delta} U_{\omega}}$. 
For any $k \in [K]$ and any $t-\tau\leq j\leq t-1$, where $t\ge \tau$, it holds that
   \begin{align*}
   \bbE_{t-\tau}\|\omega_{j+1}^k-\omega_{j}^k\| 
   &\leq\tau \frac{U_{\delta}}{1-\gamma}\frac{2\beta^2}{L^2}
    + 8\tau\frac{U_{\delta}U_\omega^2}{1-\gamma} \frac{\beta\zeta}{L^2}
    +2\tau L_{*,1}BU_\delta\frac{\beta\alpha }{L}
    +\frac{2\beta}{L}\bbE_{t-\tau}\|x_t^k\| +\frac{U_\delta}{1-\gamma}\frac{\beta}{L}.
   \\
   \bbE_{t-\tau}\|\omega_{j+1}^k-\omega_{j}^k\|^2 
   &\leq   64\tau^2 \frac{U_{\delta}^2}{L^4(1-\gamma)^2}\beta^4 +  1024 \tau^2 U_\omega^4  \frac{U_{\delta}^2}{L^4(1-\gamma)^2}\zeta^2\beta^2\\
   &\quad +32\tau^2L_{*,1}^2B^2U_\delta^2\frac{\beta^2\alpha^2  }{L^2}
   + {\frac{16\beta^2}{L^2}}\bbE_{t-\tau}\|x_t^k\|^2+2\frac{\beta^2}{L^2} \frac{U_\delta^2}{(1-\gamma)^2}, 
   \end{align*}
   where $L_{*,1}$ is defined in Lemma \ref{lmm: L* lip}. 
\end{lemma}

\begin{proof}

For any $j$, one has 
\begin{align}
\label{eqn: one step omega decomp}
     &\|\omega_{j+1}^k - \omega_{j}^k\|= \|\Pi_{U_\omega}(\omega_{j}^k+\frac{\beta}{L}\delta_{j,L}^k\bfB_{j}^\top\phi(s_{j,0}^k) ) - \Pi_{U_\omega}(\omega_{j}^k)\|\nonumber \\
    \leq & \|\frac{\beta}{L}\bfB_{j}^\top \delta_{j,L}^k \phi(s_{j,0}^k)\| \nonumber \qquad \qquad \text{(by non-expansive property of projection)}\\
    =&\frac{\beta}{L}\|\bfB_j^\top \pth{\tilde{A}_{j,L}^kx_j^k+\bfb_{j,L}^k} \|\nonumber ~~~~~~ \text{(by Proposition \ref{prop: key intermediate})} \\
    \leq&\frac{\beta}{L}\|\bfB_j^\top { \tilde{A}_{j,L}^k}x_j^k\|+\frac{\beta}{L}\|\bfB_j^\top\bfb_{j,L}^k\|\nonumber \\
    \leq&{ \frac{2\beta}{L}}\|x_j^k\|+\frac{\beta}{L}\|\bfb_{j,L}^k\|.
\end{align}

By Lemma~\ref{lmm: U delta}, we know that $\|\bfb_{j,L}^k\| \le \frac{U_\delta}{1-\gamma}\frac{\beta}{L}$. 
By Lemma \ref{lmm: jth to tth error bound}, for any $t-\tau\leq j\leq t-1$, where $t\ge \tau$, we have   
\begin{align*}
    &\bbE_{t-\tau}\|\omega_{j+1}^k-\omega_{j}^k\|\leq  \frac{2\beta}{L}\bbE_{t-\tau}\|x_j^k\| +\frac{U_\delta}{1-\gamma}\frac{\beta}{L}\\
    \leq&  \frac{2\beta}{L}\pth{\tau \frac{U_{\delta}}{L(1-\gamma)}\beta + 4\tau\frac{U_{\delta}}{L(1-\gamma)}U_\omega^2 \zeta+\tau L_{*,1}BU_\delta\alpha +\bbE_{t-\tau}\|x_t^k\|} +\frac{U_\delta}{1-\gamma}\frac{\beta}{L}\\
    =& \tau \frac{U_{\delta}}{1-\gamma}\frac{2\beta^2}{L^2}
    + 8\tau\frac{U_{\delta}U_\omega^2}{1-\gamma} \frac{\beta\zeta}{L^2}
    +2\tau L_{*,1}BU_\delta\frac{\beta\alpha }{L}
    +\frac{2\beta}{L}\bbE_{t-\tau}\|x_t^k\| +\frac{U_\delta}{1-\gamma}\frac{\beta}{L}.
\end{align*}

Similarly, for any $j$ such that $t-\tau\leq j\leq t-1$, we have 
\begin{align}
\label{eqn: one step omega ^2 decomp}
     &\bbE_{t-\tau}\|\omega_{j+1}^k - \omega_{j}^k\|^2= \bbE_{t-\tau}\|\Pi_{U_\omega}(\omega_{j}^k+\frac{\beta}{L}\delta_{j,L}^k\bfB_{j}^\top\phi(s_{j,0}^k) ) - \Pi_{U_\omega}(\omega_{j}^k)\|^2\nonumber \\
    \leq & \bbE_{t-\tau}\|\frac{\beta}{L}\bfB_{j}^\top \delta_{j,L}^k \phi(s_{j,0}^k)\|^2 \nonumber \\
    =&\frac{\beta^2}{L^2}\bbE_{t-\tau}\|\bfB_j^\top \pth{\tilde{A}_{j,L}^kx_j^k+\bfb_{j,L}^k} \|^2\nonumber \quad \quad ~~~~~~ \text{(by Proposition \ref{prop: key intermediate})} \\
    \leq&2\frac{\beta^2}{L^2}\bbE_{t-\tau}\|\bfB_j^\top { \tilde{A}_{j,L}^k}x_j^k\|^2 + 2\frac{\beta^2}{L^2}\bbE_{t-\tau}\|\bfB_j^\top\bfb_{j,L}^k\|^2 \nonumber \\
    \leq&{\frac{8\beta^2}{L^2}}\bbE_{t-\tau}\|x_j^k\|^2+2\frac{\beta^2}{L^2}\bbE_{t-\tau}\|\bfb_{j,L}^k\|^2 ~~~~ \text{since $\|\tilde{A}_{j,L}^k\|\le 2$}.
\end{align}
By Lemma~\ref{lmm: U delta},  we have $\|\bfb_{j,L}^k\|^2\leq \frac{U_\delta^2}{(1-\gamma)^2}$. 
In addition, by Lemma \ref{lmm: jth to tth error bound}, 
\eqref{eqn: one step omega ^2 decomp} can be written as, we further bound $\bbE_{t-\tau}\|\omega_{j+1}^k - \omega_{j}^k\|^2$ as 
\begin{align*}
    &\bbE_{t-\tau}\|\omega_{j+1}^k - \omega_{j}^k\|^2
    \leq {\frac{8\beta^2}{L^2}}\bbE_{t-\tau}\|x_j^k\|^2+\frac{2\beta^2}{L^2} \frac{U_\delta^2}{(1-\gamma)^2}\\
    \leq& 64\tau^2 \frac{U_{\delta}^2}{L^4(1-\gamma)^2}\beta^4 +  1024 \tau^2 U_\omega^4  \frac{U_{\delta}^2}{L^4(1-\gamma)^2}\zeta^2\beta^2
    +32\tau^2L_{*,1}^2B^2U_\delta^2\frac{\beta^2\alpha^2  }{L^2} \\
    &\quad 
    + {\frac{16\beta^2}{L^2}}\bbE_{t-\tau}\|x_t^k\|^2+2\frac{\beta^2}{L^2} \frac{U_\delta^2}{(1-\gamma)^2}.
\end{align*}
\end{proof}

\begin{lemma}
\label{lmm: theta onestep distance}
Choose $\zeta \le \frac{L(1-\gamma)}{2U_{\delta}U_{\omega}}$. 
Fix $\tau$. 
For $t-\tau\leq j\leq t$, $k \in [K]$,
\begin{align*}
    &\bbE_{t-\tau}\qth{\|\theta_{j+1}^k-\theta_{j}^k\|}
\leq (1-\gamma)\alpha \bbE_{t-\tau}\qth{\|\nabla J^k(\theta_t^k)\|}
+C_{1}^{\mathrm{act}}\alpha \bbE_{t-\tau}\qth{\|x_t^k\|}
+C_{2}^{\mathrm{act}}\frac{\tau\alpha \beta}{L(1-\gamma)} \\
& \qquad \qquad +C_{3}^{\mathrm{act}}\frac{\tau\alpha \zeta}{L(1-\gamma)} 
+C_{4}^{\mathrm{act}}\tau\alpha^2  
+C_{5}^{\mathrm{act}}\frac{\alpha }{\sqrt{L(1-\gamma)}} 
+ C_{6}^{\mathrm{act}}\frac{\tau\alpha^2}{(1-\gamma)^2}, \\
&\bbE_{t-\tau}\qth{\|\theta_{j+1}^k-\theta_j^k\|^2} 
\leq 8(1-\gamma)^2\alpha^2  \bbE_{t-\tau}\qth{\|\nabla J^k(\theta_t^k)\|^2}
+C_{7}^{\mathrm{act}}\alpha^2  \bbE_{t-\tau}\qth{\|x_t^k\|^2} +C_{8}^{\mathrm{act}}\frac{\tau^2\alpha^2  \beta^2}{L^2(1-\gamma)^2}\\
& \qquad \qquad
+C_{9}^{\mathrm{act}}\frac{\tau^2\alpha^2  \zeta^2}{L^2(1-\gamma)^2}
+C_{10}^{\mathrm{act}}\tau^2\alpha^4
+C_{11}^{\mathrm{act}}\frac{\tau\alpha^3}{(1-\gamma)^2}
+C_{12}^{\mathrm{act}}\frac{\alpha^2  }{L(1-\gamma)}.
\end{align*}
where
\begin{align*}
    C_{1}^{\mathrm{act}}&=2B,
    C_{2}^{\mathrm{act}}=2BU_\delta,
    C_{3}^{\mathrm{act}}=8BU_\delta U_\omega^2,\\
    C_{4}^{\mathrm{act}}&=(1-\gamma)L_{J'}BU_\delta+2B^2L_{*,1}U_\delta
    +BU_\delta\pth{2BL_{*,1}+U_\delta L_g+(1-\gamma)L_{J'}},\\
    C_{5}^{\mathrm{act}}&=3BU_\delta,    
    C_{6}^{\mathrm{act}} = 4B^2U_\delta^2L_\pi|\calA|, \\
    C_{7}^{\mathrm{act}}&=32B^2,  
    C_{8}^{\mathrm{act}} =128B^2U_\delta^2,
    C_{9}^{\mathrm{act}} =2048B^2U_\omega^4U_\delta^2,\\
    C_{10}^{\mathrm{act}}&=8(1-\gamma)^2L_{J'}^2B^2U_\delta^2+64B^4L_{*,1}^2U_\delta^2
    +4B^2U_\delta^2\pth{2BL_{*,1}+U_\delta L_g+(1-\gamma)L_{J'}}^2,\\
    C_{11}^{\mathrm{act}} &=128B^3U_\delta^3L_\pi|\calA|,
    C_{12}^{\mathrm{act}} =48B^2 U_\delta^2.
\end{align*}
\end{lemma} 

\begin{proof} 
Recall from Eq.\eqref{eqn: def of g} that, for each $j$ such that $t-\tau \le j\le t$, $g_{j,L}^k$ is defined in as  
\begin{align*}
g_{j,L}^k
=
\frac{1}{L}\sum_{\ell=0}^{L-1}
\delta_{j,\ell}^{k,\mathrm{act}}
\nabla_\theta\log\pi_{\theta_j^k}(\hat{a}_{j,\ell}^k|\hat{s}_{j,\ell}^k)
= \frac{1}{L}\sum_{\ell=0}^{L-1}g(\hat{O}_{j,\ell}^k, \bfB_j, \omega_j^k, \theta_j^k), 
\end{align*}
and that $\theta^k$ is updated as 
\begin{align*}
\theta_{j+1}^k = \theta_{j}^k + \alpha g_{j,L}^k   
\end{align*} 
Adding and subtracting the one-step policy-gradient target gives
\begin{align}
\label{eqn: theta one-step diff decomp to 3 terms}
    \bbE_{t-\tau}\qth{\|\theta_{j+1}^k-\theta_{j}^k\|} 
    \leq (1-\gamma)\alpha \bbE_{t-\tau}\qth{\|\nabla J^k(\theta_j^k)\|}
    +\alpha \bbE_{t-\tau}\qth{\|g_{j,L}^k-(1-\gamma)\nabla J^k(\theta_j^k)\|}.
\end{align}
For the second term in \eqref{eqn: theta one-step diff decomp to 3 terms},
we insert and subtract the true TD fixed point, the frozen-parameter original-chain copy at time $t-\tau$, and the frozen-parameter auxiliary-chain copy at time $t-\tau$, yielding 
\begin{align}
\label{eqn: decomp of theta onestep diff last term}
\bbE_{t-\tau}\qth{\|g_{j,L}^k-(1-\gamma)\nabla J^k(\theta_j^k)\|} %
\leq D_1+D_2+D_3+D_4,
\end{align}
where
\begin{align*}
    D_1
    =&\bbE_{t-\tau}\qth{\left\|\frac{1}{L}\sum_{\ell=0}^{L-1}
    \pth{g(\hat{O}_{j,\ell}^k,\bfB_j,\omega_j^k,\theta_j^k)
    -g( \hat{O}_{j,\ell}^k,\bfB^*,\omega_j^{k,*},\theta_j^k)}\right\|},\\
    D_2
    =&\bbE_{t-\tau}\qth{\left\|\frac{1}{L}\sum_{\ell=0}^{L-1}
    \pth{g(\hat{O}_{j,\ell}^k,\bfB^*,\omega_j^{k,*},\theta_j^k)
    -g(\hat{O}_{j,\ell}^k,\bfB^*,\omega_{t-\tau}^{k,*},\theta_{t-\tau}^k)}
    -(1-\gamma)\pth{\nabla J^k(\theta_j^k)-\nabla J^k(\theta_{t-\tau}^k)}\right\|},\\
    D_3
    =& \bbE_{t-\tau}\qth{\left\|\frac{1}{L}\sum_{\ell=0}^{L-1}
    g(\hat{O}_{j,\ell}^k,\bfB^*,\omega_{t-\tau}^{k,*},\theta_{t-\tau}^k) 
    - \frac{1}{L}\sum_{\ell=0}^{L-1}
    g(\bar O_{j,\ell}^k,\bfB^*,\omega_{t-\tau}^{k,*},\theta_{t-\tau}^k) \right\|} \\ 
    D_4
    =&\bbE_{t-\tau}\qth{\left\|\frac{1}{L}\sum_{\ell=0}^{L-1}
    g(\bar O_{j,\ell}^k,\bfB^*,\omega_{t-\tau}^{k,*},\theta_{t-\tau}^k)
    -(1-\gamma)\nabla J^k(\theta_{t-\tau}^k)\right\|}.
\end{align*} 
Recall from Eq.\eqref{eqn: sample O notation: actor} that 
\[
g(\hat{O}^k_{j,\ell}, \bfB, \omega, \theta)
=\pth{\hat{r}_{j,\ell}+(\gamma\phi(\hat{s}^k_{j,\ell+1})-\phi(\hat{s}^k_{j,\ell}))^\top\bfB\omega}
\nabla_\theta\log\pi_\theta(\hat{a}^k_{j,\ell}|\hat{s}^k_{j,\ell}).
\]
The first difference is the critic approximation error. Since
$x_j^k=\bfB_j\omega_j^k-\bfB^*\omega_j^{k,*}$, we have
\begin{align*}
    D_1
    =&\bbE_{t-\tau}\qth{\left\|\frac{1}{L}\sum_{\ell=0}^{L-1}
    \pth{(\gamma\phi(\hat{s}_{j,\ell+1}^k)-\phi(\hat{s}_{j,\ell}^k))^\top x_j^k}
    \nabla_\theta\log\pi_{\theta_j^k}(\hat{a}_{j,\ell}^k|\hat{s}_{j,\ell}^k)
    \right\|}\\
    \leq& \frac{1}{L}\sum_{\ell=0}^{L-1}
    \bbE_{t-\tau}\qth{
    \left|(\gamma\phi(\hat{s}_{j,\ell+1}^k)-\phi(\hat{s}_{j,\ell}^k))^\top x_j^k\right|
    \left\|\nabla_\theta\log\pi_{\theta_j^k}(\hat{a}_{j,\ell}^k|\hat{s}_{j,\ell}^k)\right\|}\\
    \leq& \frac{1}{L}\sum_{\ell=0}^{L-1}
    2B\bbE_{t-\tau}\qth{\|x_j^k\|}
    =
    2B\bbE_{t-\tau}\qth{\|x_j^k\|},
\end{align*}
where we used Assumption \ref{assmp: Lipschitz of policy} and the fact that $\|\gamma\phi(\hat{s}')-\phi(\hat{s})\|\leq 2$. %

Intuitively, $D_2$ captures the stepback difference. By the triangle inequality, we have 
\begin{align*}
    D_2
    \leq& \frac{1}{L}\sum_{\ell=0}^{L-1}
    \bbE_{t-\tau}\qth{\left\|
    g(\hat{O}_{j,\ell}^k,\bfB^*,\omega_j^{k,*},\theta_j^k)
    -g(\hat{O}_{j,\ell}^k,\bfB^*,\omega_{t-\tau}^{k,*},\theta_{t-\tau}^k)
    \right\|}\\
    &+(1-\gamma)\bbE_{t-\tau}\qth{\|\nabla J^k(\theta_j^k)-\nabla J^k(\theta_{t-\tau}^k)\|}.
\end{align*}
For each $\ell$, inserting and subtracting
$g(\hat{O}_{j,\ell}^k,\bfB^*,\omega_{t-\tau}^{k,*},\theta_j^k)$ gives
\begin{align*}
    &\left\|
    g(\hat{O}_{j,\ell}^k,\bfB^*,\omega_j^{k,*},\theta_j^k)
    -g(\hat{O}_{j,\ell}^k,\bfB^*,\omega_{t-\tau}^{k,*},\theta_{t-\tau}^k)
    \right\|\\
    \leq& \left|
    (\gamma\phi(\hat{s}_{j,\ell+1}^k)-\phi(\hat{s}_{j,\ell}^k))^\top
    \bfB^*(\omega_j^{k,*}-\omega_{t-\tau}^{k,*})
    \right|
    \left\|\nabla_\theta\log\pi_{\theta_j^k}(\hat{a}_{j,\ell}^k|\hat{s}_{j,\ell}^k)\right\|\\
    &+\left|
    \hat{r}_{j,\ell}^k+(\gamma\phi(\hat{s}_{j,\ell+1}^k)-\phi(\hat{s}_{j,\ell}^k))^\top
    \bfB^*\omega_{t-\tau}^{k,*}
    \right|
    \left\|\nabla_\theta\log\pi_{\theta_j^k}(\hat{a}_{j,\ell}^k|\hat{s}_{j,\ell}^k)
    -\nabla_\theta\log\pi_{\theta_{t-\tau}^k}(\hat{a}_{j,\ell}^k|\hat{s}_{j,\ell}^k)\right\|\\
    \leq& 2BL_{*,1}\|\theta_j^k-\theta_{t-\tau}^k\|
    +U_\delta L_g\|\theta_j^k-\theta_{t-\tau}^k\|, 
\end{align*}
where the last inequality follows from Assumption \ref{assmp: Lipschitz of policy} and Lemma \ref{lmm: L* lip}.  
Moreover, by Lemma \ref{lmm: LJ' lip}, we have 
\begin{align*}
    \|\nabla J^k(\theta_j^k)-\nabla J^k(\theta_{t-\tau}^k)\|
    \leq L_{J'}\|\theta_j^k-\theta_{t-\tau}^k\|.
\end{align*}
Therefore,
\begin{align*}
    D_2
    \leq& (2BL_{*,1}+U_\delta L_g+(1-\gamma)L_{J'})
    \bbE_{t-\tau}\qth{\|\theta_j^k-\theta_{t-\tau}^k\|}\\
    \leq& (2BL_{*,1}+U_\delta L_g+(1-\gamma)L_{J'})
    \sum_{i=t-\tau}^{j-1}\bbE_{t-\tau}\qth{\|\theta_{i+1}^k-\theta_i^k\|}\\
    \leq& (2BL_{*,1}+U_\delta L_g+(1-\gamma)L_{J'})
    \tau\alpha BU_\delta,
\end{align*}
where the last inequality uses $\theta_{i+1}^k-\theta_i^k=\alpha g_{i,L}^k$ and $\|g_{i,L}^k\|\leq BU_\delta$.

$D_3$ captures the difference between the original chain and auxiliary chain under actor simulator $\hat{P}^k$. 
Applying the TV variational bound and then Corollary \ref{cor: tv dist between oc and ac: actor} to each per-$\ell$ tuple, we have
\begin{align*}
    D_3
    \leq& \frac{1}{L}\sum_{\ell=0}^{L-1}
    2BU_\delta\,
    d_{TV}\Big(
    \hat{\mathbb{P}}_{t-\tau:(j,\ell)}^k(\hat{s}_{j,0}^k,\hat{a}_{j,0}^k,\hat{s}_{j,\ell}^k,\hat{a}_{j,\ell}^k,\hat{s}_{j,\ell+1}^k),
    \bar{\mathbb{U}}_{t-\tau:(j,\ell)}^k(\bar s_{j,0}^k, \bar a_{j,0}^k,\bar s_{j,\ell}^k,\bar a_{j,\ell}^k,\bar s_{j,\ell+1}^k)
    \Big)\\
    \leq& \frac{B U_\delta}{L}\sum_{\ell=0}^{L-1}
    \Bigg(
    2d_{TV}\Big(\hat{\mathbb{P}}_{t-\tau:j}^k(\cdot),
    \bar{\mathbb{U}}_{t-\tau:j}^k(\cdot)\Big)
    +\frac{3-2\gamma}{1-\gamma}L_\pi|\calA|
    \bbE_{t-\tau}\qth{\|\theta_j^k-\theta_{t-\tau}^k\|}
    \Bigg)\\
    \overset{(a)}{\leq}& B U_\delta\Bigg(
    \sum_{m=1}^{j-(t-\tau)}(\gamma^L)^{m-1}
    \frac{1}{1-\gamma}L_\pi|\calA|
    \bbE\qth{\|\theta_{j-m}^k-\theta_{t-\tau}^k\|}\\
    &\qquad
    +\frac{3-2\gamma}{1-\gamma}L_\pi|\calA|
    \bbE_{t-\tau}\qth{\|\theta_j^k-\theta_{t-\tau}^k\|}
    \Bigg)\\
    \leq& B U_\delta\Bigg(
    \sum_{m=1}^{j-(t-\tau)}(\gamma^L)^{m-1}
    \frac{1}{1-\gamma}L_\pi|\calA|
    \sum_{i=t-\tau}^{j-m-1}
    \bbE_{t-\tau}\qth{\|\theta_{i+1}^k-\theta_i^k\|}\\
    &\qquad
    +\frac{3-2\gamma}{1-\gamma}L_\pi|\calA|
    \sum_{i=t-\tau}^{j-1}
    \bbE_{t-\tau}\qth{\|\theta_{i+1}^k-\theta_i^k\|}
    \Bigg)\\
    \leq& B U_\delta\Bigg(
    \frac{1}{(1-\gamma^L)(1-\gamma)}L_\pi|\calA|
    \sum_{i=t-\tau}^{j-1}
    \bbE_{t-\tau}\qth{\|\theta_{i+1}^k-\theta_i^k\|}\\
    &\qquad
    +\frac{3-2\gamma}{1-\gamma}L_\pi|\calA|
    \sum_{i=t-\tau}^{j-1}
    \bbE_{t-\tau}\qth{\|\theta_{i+1}^k-\theta_i^k\|}
    \Bigg)\\
    =& B U_\delta
    \pth{\frac{1}{(1-\gamma^L)(1-\gamma)}+\frac{3-2\gamma}{1-\gamma}}
    L_\pi|\calA|
    \sum_{i=t-\tau}^{j-1}\bbE_{t-\tau}\qth{\|\theta_{i+1}^k-\theta_i^k\|}\\
    \leq& \frac{4B U_\delta}{(1-\gamma)^2}L_\pi|\calA|
    \sum_{i=t-\tau}^{j-1}\bbE_{t-\tau}\qth{\|\theta_{i+1}^k-\theta_i^k\|}, 
\end{align*}
where inequality (a) is obtained by unrolling the block-start TV distance, i.e., $d_{TV}\Big(\hat{\mathbb{P}}_{t-\tau:j}^k(\cdot),
    \bar{\mathbb{U}}_{t-\tau:j}^k(\cdot)\Big)$, backwards using the last two bounds in Corollary \ref{cor: tv dist between oc and ac: actor}. This backward unrolling terminates at the base case
\begin{align*}
    d_{TV}\Big(\hat{\mathbb{P}}_{t-\tau:t-\tau}^k(\cdot),
    \bar{\mathbb{U}}_{t-\tau:t-\tau}^k(\cdot)\Big)=0,
\end{align*}
since $v_{0:t-\tau}=\tilde v_{0:t-\tau}$.
Applying the same crude bound of $\sum_{i=t-\tau}^{j-1}\bbE_{t-\tau}\qth{\|\theta_{i+1}^k-\theta_i^k\|}$ as in $D_2$, we have
\begin{align*}
    D_3
    \leq& \frac{4\alpha B^2 U^2_\delta}{(1-\gamma)^2}L_\pi|\calA|
    \tau.
\end{align*}
Finally, for the auxiliary-chain term, the one-step policy-gradient identity gives the stationary expectation as $(1-\gamma)\nabla J^k(\theta_{t-\tau}^k)$. Thus,
\begin{align*}
    D_4
    \leq& \bbE_{t-\tau}\qth{\left\|\frac{1}{L}\sum_{\ell=0}^{L-1}
    \pth{g(\bar O_{j,\ell}^k,\bfB^*,\omega_{t-\tau}^{k,*},\theta_{t-\tau}^k)
    -(1-\gamma)\nabla J^k(\theta_{t-\tau}^k)}
    \right\|}\\
    \leq& \sqrt{
    \bbE_{t-\tau}\qth{\left\|\frac{1}{L}\sum_{\ell=0}^{L-1}
    \pth{g(\bar O_{j,\ell}^k,\bfB^*,\omega_{t-\tau}^{k,*},\theta_{t-\tau}^k)
    -(1-\gamma)\nabla J^k(\theta_{t-\tau}^k)}
    \right\|^2}}.
\end{align*}
Expanding the square gives
\begin{align*}
    &\bbE_{t-\tau}\qth{\left\|\frac{1}{L}\sum_{\ell=0}^{L-1}
    \pth{g(\bar O_{j,\ell}^k,\bfB^*,\omega_{t-\tau}^{k,*},\theta_{t-\tau}^k)
    -(1-\gamma)\nabla J^k(\theta_{t-\tau}^k)}
    \right\|^2 }\\
    =&\frac{1}{L^2}\sum_{\ell=0}^{L-1}
    \bbE_{t-\tau}\qth{\left\|
    g(\bar O_{j,\ell}^k,\bfB^*,\omega_{t-\tau}^{k,*},\theta_{t-\tau}^k)
    -(1-\gamma)\nabla J^k(\theta_{t-\tau}^k)
    \right\|^2}\\
    &+\frac{2}{L^2}\sum_{\ell=0}^{L-1}\sum_{\ell'=\ell+1}^{L-1}
    \bbE_{t-\tau}\Big[
    \Big\langle
    g(\bar O_{j,\ell}^k,\bfB^*,\omega_{t-\tau}^{k,*},\theta_{t-\tau}^k)-(1-\gamma)\nabla J^k(\theta_{t-\tau}^k),\\
    &\qquad\qquad\qquad\qquad\qquad\qquad
    g(\bar O_{j,\ell'}^k,\bfB^*,\omega_{t-\tau}^{k,*},\theta_{t-\tau}^k)-(1-\gamma)\nabla J^k(\theta_{t-\tau}^k)
    \Big\rangle\Big]\\
    \overset{(a)}{\leq} & \frac{4B^2U_\delta^2}{L}
    +\frac{2B^2U_\delta^2}{L^2}
    \sum_{\ell=0}^{L-1}\sum_{\ell'=\ell+1}^{L-1}\gamma^{\ell'-\ell}\\
    \leq& \frac{B^2U_\delta^2}{L}\frac{6}{1-\gamma},
\end{align*}
where inequality (a) holds because that 
\begin{align*}
&\left\|
    g(\bar O_{j,\ell}^k,\bfB^*,\omega_{t-\tau}^{k,*},\theta_{t-\tau}^k)
    -(1-\gamma)\nabla J^k(\theta_{t-\tau}^k)
    \right\|^2 \\
    &\le 2 \pth{\|g(\bar O_{j,\ell}^k,\bfB^*,\omega_{t-\tau}^{k,*},\theta_{t-\tau}^k)\|^2 + (1-\gamma)^2\|\nabla J^k(\theta_{t-\tau}^k)\|^2}\\
    &\le 4 B^2 U_{\delta}^2,
\end{align*}
and because of the towering argument of conditional expectation. 
\begin{align*}
    D_4
    \leq
    3BU_\delta\sqrt{\frac{1}{L(1-\gamma)}}.
\end{align*}
Putting the four bounds back into \eqref{eqn: decomp of theta onestep diff last term}, and subsequently into \eqref{eqn: theta one-step diff decomp to 3 terms}, together with Lemma \ref{lmm: jth to tth error bound}, yields 
\begin{align*}
    &\bbE_{t-\tau}\qth{\|\theta_{j+1}^k-\theta_{j}^k\|}\\
\leq& (1-\gamma)\alpha \bbE_{t-\tau}\qth{\|\nabla J^k(\theta_j^k)\| }
+2\alpha B\bbE_{t-\tau}\qth{\|x_j^k\|}
+\alpha^2  \tau BU_\delta\pth{2BL_{*,1}+U_\delta L_g+(1-\gamma)L_{J'}}\\
&+\frac{4\alpha^2  \tau B^2U_\delta^2}{(1-\gamma)^2}L_\pi|\calA|
+3\alpha BU_\delta\sqrt{\frac{1}{L(1-\gamma)}}\\
\leq& (1-\gamma)\alpha \bbE_{t-\tau}\qth{\|\nabla J^k(\theta_j^k)\|}
+2\alpha B\pth{\tau \frac{U_{\delta}}{L(1-\gamma)}\beta
+4\tau\frac{U_{\delta}}{L(1-\gamma)}U_\omega^2 \zeta
+\tau L_{*,1}B U_\delta\alpha 
+\bbE_{t-\tau}\qth{\|x_t^k\|}}\\
&+\alpha^2  \tau BU_\delta\pth{2BL_{*,1}+U_\delta L_g+(1-\gamma)L_{J'}}
+\frac{4\alpha^2  \tau B^2U_\delta^2}{(1-\gamma)^2}L_\pi|\calA|
+3\alpha BU_\delta\sqrt{\frac{1}{L(1-\gamma)}},
\end{align*}
Applying Lemma \ref{lmm: jth to tth error bound} to $\bbE_{t-\tau}\qth{\|\nabla J^k(\theta_j^k)\|}$ to relate to $\bbE_{t-\tau}\qth{\|\nabla J^k(\theta_t^k)\|}$, we obtain  
\begin{align*}
    &\bbE_{t-\tau}\qth{\|\theta_{j+1}^k-\theta_{j}^k\|}\\
\leq& (1-\gamma)\alpha \bbE_{t-\tau}\qth{\|\nabla J^k(\theta_t^k)\|}
+(1-\gamma)\tau\alpha^2  L_{J'}BU_\delta\\
&+2\alpha B\pth{\tau \frac{U_{\delta}}{L(1-\gamma)}\beta
+4\tau\frac{U_{\delta}}{L(1-\gamma)}U_\omega^2 \zeta
+\tau L_{*,1}B U_\delta\alpha 
+\bbE_{t-\tau}\qth{\|x_t^k\|}}\\
&+\alpha^2  \tau BU_\delta\pth{2BL_{*,1}+U_\delta L_g+(1-\gamma)L_{J'}}
+\frac{4\alpha^2  \tau B^2U_\delta^2}{(1-\gamma)^2}L_\pi|\calA|
+3\alpha BU_\delta\sqrt{\frac{1}{L(1-\gamma)}},
\end{align*}
By the definitions of $C_{1}^{\mathrm{act}},\ldots,C_{6}^{\mathrm{act}}$, the above bound can be simplified as
\begin{align*}
\bbE_{t-\tau}\qth{\|\theta_{j+1}^k-\theta_{j}^k\|}
\leq& (1-\gamma)\alpha \bbE_{t-\tau}\qth{\|\nabla J^k(\theta_t^k)\|}
+C_{1}^{\mathrm{act}}\alpha \bbE_{t-\tau}\qth{\|x_t^k\|}
+C_{2}^{\mathrm{act}}\frac{\tau\alpha \beta}{L(1-\gamma)}
\\
&+C_{3}^{\mathrm{act}}\frac{\tau\alpha \zeta}{L(1-\gamma)} +C_{4}^{\mathrm{act}}\tau\alpha^2  
+C_{5}^{\mathrm{act}}\frac{\alpha }{\sqrt{L(1-\gamma)}} 
+ C_{6}^{\mathrm{act}}\frac{\tau\alpha^2}{(1-\gamma)^2}. 
\end{align*}

Simialry, we bound $\bbE_{t-\tau}\qth{\|\theta_{j+1}^k-\theta_{j}^k\|^2}$, where $t-\tau\le j\le t$ as follows:  
\begin{equation}
\label{eqn: decomp of theta onestep diff last term: square}
\begin{aligned}
    &\bbE_{t-\tau}\qth{\|\theta_{j+1}^k-\theta_{j}^k\|^2}\\
    =&\alpha^2  \bbE_{t-\tau}\qth{\left\|
    \frac{1}{L}\sum_{\ell=0}^{L-1}
    g(\hat{O}_{j,\ell}^k,\bfB_j,\omega_j^k,\theta_j^k)
    \right\|^2}\\
    =&\alpha^2  \bbE_{t-\tau}\Bigg[\Bigg\|
    (1-\gamma)\nabla J^k(\theta_j^k)
    +\frac{1}{L}\sum_{\ell=0}^{L-1}
    \pth{g(\hat{O}_{j,\ell}^k,\bfB_j,\omega_j^k,\theta_j^k)
    -g(\hat{O}_{j,\ell}^k,\bfB^*,\omega_j^{k,*},\theta_j^k)}\\
    &
    +\frac{1}{L}\sum_{\ell=0}^{L-1}
    \pth{g(\hat{O}_{j,\ell}^k,\bfB^*,\omega_j^{k,*},\theta_j^k)
    -g(\hat{O}_{j,\ell}^k,\bfB^*,\omega_{t-\tau}^{k,*},\theta_{t-\tau}^k)}
    -(1-\gamma)\pth{\nabla J^k(\theta_j^k)-\nabla J^k(\theta_{t-\tau}^k)}\\
    &
    +\frac{1}{L}\sum_{\ell=0}^{L-1}
    g(\hat{O}_{j,\ell}^k,\bfB^*,\omega_{t-\tau}^{k,*},\theta_{t-\tau}^k)
    -(1-\gamma)\nabla J^k(\theta_{t-\tau}^k)
    \Bigg\|^2\Bigg]\\
    \leq&4(1-\gamma)^2\alpha^2  \bbE_{t-\tau}\qth{\|\nabla J^k(\theta_j^k)\|^2}\\
    &+4\alpha^2  \bbE_{t-\tau}\qth{\left\|
    \frac{1}{L}\sum_{\ell=0}^{L-1}
    \pth{g(\hat{O}_{j,\ell}^k,\bfB_j,\omega_j^k,\theta_j^k)
    -g(\hat{O}_{j,\ell}^k,\bfB^*,\omega_j^{k,*},\theta_j^k)}
    \right\|^2}\\
    &+4\alpha^2  \bbE_{t-\tau}\Bigg[\Bigg\|
    \frac{1}{L}\sum_{\ell=0}^{L-1}
    \pth{g(\hat{O}_{j,\ell}^k,\bfB^*,\omega_j^{k,*},\theta_j^k)
    -g(\hat{O}_{j,\ell}^k,\bfB^*,\omega_{t-\tau}^{k,*},\theta_{t-\tau}^k)}
    \nonumber\\&\qquad\qquad\qquad\qquad -(1-\gamma)\pth{\nabla J^k(\theta_j^k)-\nabla J^k(\theta_{t-\tau}^k)}
    \Bigg\|^2\Bigg]\\
    &+4\alpha^2  \bbE_{t-\tau}\qth{\left\|
    \frac{1}{L}\sum_{\ell=0}^{L-1}
    g(\hat{O}_{j,\ell}^k,\bfB^*,\omega_{t-\tau}^{k,*},\theta_{t-\tau}^k)
    -(1-\gamma)\nabla J^k(\theta_{t-\tau}^k)
    \right\|^2}\\
    \leq&4(1-\gamma)^2\alpha^2  \bbE_{t-\tau}\qth{\|\nabla J^k(\theta_j^k)\|^2}\\
    &+4\alpha^2  \bbE_{t-\tau}\qth{\left\|
    \frac{1}{L}\sum_{\ell=0}^{L-1}
    \pth{g(\hat{O}_{j,\ell}^k,\bfB_j,\omega_j^k,\theta_j^k)
    -g(\hat{O}_{j,\ell}^k,\bfB^*,\omega_j^{k,*},\theta_j^k)}
    \right\|^2}\\
    &+4\alpha^2  \bbE_{t-\tau}\Bigg[\Bigg\|
    \frac{1}{L}\sum_{\ell=0}^{L-1}
    \pth{g(\hat{O}_{j,\ell}^k,\bfB^*,\omega_j^{k,*},\theta_j^k)
    -g(\hat{O}_{j,\ell}^k,\bfB^*,\omega_{t-\tau}^{k,*},\theta_{t-\tau}^k)}\\
    &\qquad\qquad\qquad\qquad
    -(1-\gamma)\pth{\nabla J^k(\theta_j^k)-\nabla J^k(\theta_{t-\tau}^k)}
    \Bigg\|^2\Bigg]\\
    &+8\alpha^2 
    \bbE_{t-\tau}\qth{\left\|
    \frac{1}{L}\sum_{\ell=0}^{L-1}
    g(\hat{O}_{j,\ell}^k,\bfB^*,\omega_{t-\tau}^{k,*},\theta_{t-\tau}^k) - 
    g(\bar O_{j,\ell}^k,\bfB^*,\omega_{t-\tau}^{k,*},\theta_{t-\tau}^k)
    \right\|^2}\\
    &+8\alpha^2  \bbE_{t-\tau}\qth{\left\|
    \frac{1}{L}\sum_{\ell=0}^{L-1}
    g( \bar O_{j,\ell}^k,\bfB^*,\omega_{t-\tau}^{k,*},\theta_{t-\tau}^k)
    -(1-\gamma)\nabla J^k(\theta_{t-\tau}^k)
    \right\|^2}.
\end{aligned}
\end{equation} 
The first difference in Eq.\eqref{eqn: decomp of theta onestep diff last term: square}, which captures the squared critic-approximation difference, can be bounded as 
\begin{align*}
\bbE_{t-\tau}\qth{\left\|
    \frac{1}{L}\sum_{\ell=0}^{L-1}
    \pth{g(\hat{O}_{j,\ell}^k,\bfB_j,\omega_j^k,\theta_j^k)
    -g(\hat{O}_{j,\ell}^k,\bfB^*,\omega_j^{k,*},\theta_j^k)}
    \right\|^2}
\leq4B^2\bbE_{t-\tau}\qth{\|x_j^k\|^2}.
\end{align*}
The second difference in Eq.\eqref{eqn: decomp of theta onestep diff last term: square} can be bounded by 
\begin{align*}
    &\bbE_{t-\tau}\Bigg[\Bigg\|
    \frac{1}{L}\sum_{\ell=0}^{L-1}
    \pth{g(\hat{O}_{j,\ell}^k,\bfB^*,\omega_j^{k,*},\theta_j^k)
    -g(\hat{O}_{j,\ell}^k,\bfB^*,\omega_{t-\tau}^{k,*},\theta_{t-\tau}^k)}\\&\qquad\qquad\qquad\qquad
    -(1-\gamma)\pth{\nabla J^k(\theta_j^k)-\nabla J^k(\theta_{t-\tau}^k)}
    \Bigg\|^2\Bigg]\\
    \leq&\pth{2BL_{*,1}+U_\delta L_g+(1-\gamma)L_{J'}}^2
    \bbE_{t-\tau}\qth{\|\theta_j^k-\theta_{t-\tau}^k\|^2}\\
    \leq&\tau^2\alpha^2  B^2U_\delta^2
    \pth{2BL_{*,1}+U_\delta L_g+(1-\gamma)L_{J'}}^2.
\end{align*}
The third difference in  Eq.\eqref{eqn: decomp of theta onestep diff last term: square}, which quantifies the original-chain and auxiliary-chain deviation for the actor sampling, can be bounded by first applying the total variation variational bound to the squared norm and then invoking the same five-tuple TV bound as used for $D_3$ in Eq.\eqref{eqn: decomp of theta onestep diff last term}, yielding 
\begin{align*}
    &\bbE_{t-\tau}\qth{\left\|
    \frac{1}{L}\sum_{\ell=0}^{L-1}
    g(\hat{O}_{j,\ell}^k,\bfB^*,\omega_{t-\tau}^{k,*},\theta_{t-\tau}^k) 
    -\frac{1}{L}\sum_{\ell=0}^{L-1}
    g(\bar O_{j,\ell}^k,\bfB^*,\omega_{t-\tau}^{k,*},\theta_{t-\tau}^k)
    \right\|^2}\\
    \leq& \frac{4B^2U_\delta^2}{L}\sum_{\ell=0}^{L-1}
    \Bigg(
    2d_{TV}\Big(\hat{\mathbb{P}}_{t-\tau:j}^k(\cdot),
    \bar{\mathbb{U}}_{t-\tau:j}^k(\cdot)\Big)
    +\frac{3-2\gamma}{1-\gamma}L_\pi|\calA|
    \bbE_{t-\tau}\qth{\|\theta_j^k-\theta_{t-\tau}^k\|}
    \Bigg)\\
    \leq& \frac{16B^2U_\delta^2}{(1-\gamma)^2}L_\pi|\calA|
    \sum_{i=t-\tau}^{j-1}
    \bbE_{t-\tau}\qth{\|\theta_{i+1}^k-\theta_i^k\|}\\
    \leq& \frac{16\tau\alpha B^3U_\delta^3}{(1-\gamma)^2}L_\pi|\calA|.
\end{align*}
Finally, we bound the last difference in Eq.\eqref{eqn: decomp of theta onestep diff last term: square} as follows: 
\begin{align*}
    &\bbE_{t-\tau}\qth{\left\|
    \frac{1}{L}\sum_{\ell=0}^{L-1}
    g(\bar O_{j,\ell}^k,\bfB^*,\omega_{t-\tau}^{k,*},\theta_{t-\tau}^k)
    -(1-\gamma)\nabla J^k(\theta_{t-\tau}^k)
    \right\|^2}
    \leq \frac{6B^2U_\delta^2}{L(1-\gamma)}.
\end{align*}
Putting the four bounds together, we have 
\begin{align*}
    \bbE_{t-\tau}\qth{\|\theta_{j+1}^k-\theta_j^k\|^2} 
    \leq&4(1-\gamma)^2\alpha^2  \bbE_{t-\tau}\qth{\|\nabla J^k(\theta_j^k)\|^2 }
    +16\alpha^2  B^2\bbE_{t-\tau}\qth{\|x_j^k\|^2 }\\
    &+4\alpha^4\tau^2B^2U_\delta^2
    \pth{2BL_{*,1}+U_\delta L_g+(1-\gamma)L_{J'}}^2
    +\frac{128\alpha^3\tau B^3U_\delta^3}{(1-\gamma)^2}L_\pi|\calA|\\
    &+48\alpha^2  B^2U_\delta^2\frac{1}{L(1-\gamma)}.
\end{align*}
Applying Lemma \ref{lmm: jth to tth error bound} to relate $\bbE_{t-\tau}\qth{\|x_j^k\|^2}$ with $\bbE_{t-\tau}\qth{\|x_t^k\|^2}$, we have 
\begin{align*}
    &\bbE_{t-\tau}\qth{\|\theta_{j+1}^k-\theta_j^k\|^2}\\
    \leq&8(1-\gamma)^2\alpha^2  \bbE_{t-\tau}\qth{\|\nabla J^k(\theta_t^k)\|^2}
    +8(1-\gamma)^2\tau^2\alpha^4L_{J'}^2B^2U_\delta^2\\
    &+16\alpha^2  B^2
    \Bigg(
    8\tau^2 \frac{U_{\delta}^2}{L^2(1-\gamma)^2}\beta^2
    +128\tau^2U_\omega^4\frac{U_{\delta}^2}{L^2(1-\gamma)^2}\zeta^2
    +4\tau^2L_{*,1}^2B^2U_\delta^2\alpha^2  
    +2\bbE_{t-\tau}\qth{\|x_t^k\|^2}
    \Bigg)\\
    &+4\alpha^4\tau^2B^2U_\delta^2
    \pth{2BL_{*,1}+U_\delta L_g+(1-\gamma)L_{J'}}^2
    +\frac{128\alpha^3\tau B^3U_\delta^3}{(1-\gamma)^2}L_\pi|\calA|\\
    &+48\alpha^2  B^2U_\delta^2\frac{1}{L(1-\gamma)}.
\end{align*}
By the definitions of $C_{7}^{\mathrm{act}},\ldots,C_{12}^{\mathrm{act}}$, the above bound can be simplified as
\begin{align*}
\bbE_{t-\tau}\qth{\|\theta_{j+1}^k-\theta_j^k\|^2} 
\leq& 8(1-\gamma)^2\alpha^2  \bbE_{t-\tau}\qth{\|\nabla J^k(\theta_t^k)\|^2}
+C_{7}^{\mathrm{act}}\alpha^2  \bbE_{t-\tau}\qth{\|x_t^k\|^2} +C_{8}^{\mathrm{act}}\frac{\tau^2\alpha^2  \beta^2}{L^2(1-\gamma)^2}\\
& 
+C_{9}^{\mathrm{act}}\frac{\tau^2\alpha^2  \zeta^2}{L^2(1-\gamma)^2}
+C_{10}^{\mathrm{act}}\tau^2\alpha^4
+C_{11}^{\mathrm{act}}\frac{\tau\alpha^3}{(1-\gamma)^2}
+C_{12}^{\mathrm{act}}\frac{\alpha^2  }{L(1-\gamma)}.
\end{align*}

\end{proof}

\begin{lemma}
\label{lmm: one-step perturbation of B}
Fix $\tau>0$.
Choosing $\zeta$ so that $\zeta \frac{U_{\delta}}{L(1-\gamma)} U_{\omega} \le 1/2$ and $ U_\omega \zeta \leq 1/2$,
for any $t-\tau\leq j\leq t-1$, where $t\ge \tau$, we have
\begin{align*}
\bbE_{t-\tau}[\|\mathbf{B}_{j+1} - \mathbf{B}_j\| ] 
\leq&4U_\omega\frac{\zeta}{L} \frac{1}{K}\sum_{k=1}^K \bbE_{t-\tau}\|x_t^k\|
+4\tau\frac{U_\omega U_\delta}{L^2(1-\gamma)}\beta\zeta
+16\tau\frac{U_\delta U_\omega^3}{L^2(1-\gamma)}\zeta^2\\
&
+4\tau U_\omega L_{*,1}BU_\delta\frac{\alpha \zeta}{L} +2U_\omega\frac{U_\delta}{L(1-\gamma)}\zeta
+4\frac{U_{\delta}^2U_{\omega}^2}{L^2(1-\gamma)^2}\zeta^2.\\
\bbE_{t-\tau}\|\bfB_{j+1}-\bfB_j\|^2 
\leq& 768\tau^2 \frac{U_{\delta}^2 U_\omega^2}{L^4(1-\gamma)^2}\beta^2\zeta^2
+12288\frac{U_{\delta}^2\tau^2 U_\omega^5}{L^4(1-\gamma)^2}\zeta^4
+384\tau^2U_\omega^2L_{*,1}^2B^2U_\delta^2\frac{\alpha^2  \zeta^2}{L^2}\\
&+192U_\omega^2\frac{\zeta^2}{KL^2} \sum_{k=1}^K\bbE_{t-\tau}\|x_t^k\|^2
+24U_\omega^2\frac{U_{\delta}^2}{L^2(1-\gamma)^2}\zeta^2
+32 \frac{U_{\delta}^4 U_{\omega}^4\zeta^4}{L^4(1-\gamma)^4}.
\end{align*}
\end{lemma}
\begin{proof}
Fix $\tau>0$. Recall that $\bfP_{t,\perp} = \bfB_{t,\perp}\bfB_{t,\perp}^{\top}$.  
For any $t-\tau\leq j\leq t-1$, the one-step perturbation of $\bfB_j$ can be bounded as
\begin{align}
\label{eq: auxiliary lemma: principle angle perturbation}
    &\|\bfB_{j+1}-\bfB_j\| = \|\bar{\bfB}_{j+1}\bfR_{j+1}^{-1}-\bfB_j\| \nonumber \\ 
    = & \| (\mathbf{B}_{j}
    +\frac{\zeta}{KL}\sum_{k=1}^K \bfP_{j,\perp}\delta_{j,L}^k\phi(s_{j,0}^k)(\omega_{j}^k)^\top)\bfR_{j+1}^{-1}  - \bfB_{j}\|\nonumber \\
    =&\| (\mathbf{B}_{j}
    +\frac{\zeta}{KL}\sum_{k=1}^K \bfP_{j,\perp}\delta_{j,L}^k\phi(s_{j,0}^k)(\omega_{j}^k)^\top)\bfR_{j+1}^{-1} -\bfB_j\bfR_{j+1}^{-1}+ \bfB_j\bfR_{j+1}^{-1}- \bfB_{j}\|\nonumber \\
    \leq& \left\|\frac{\zeta}{KL}\sum_{k=1}^K \bfP_{j,\perp}\delta_{j,L}^k\phi(s_{j,0}^k)(\omega_{j}^k)^\top\right\|\left\|\bfR_{j+1}^{-1}\right\| + \|\bfR_{j+1}^{-1}-\bfI\| \|\bfB_j\|\nonumber \\
    \leq& \left\|\frac{\zeta}{KL}\sum_{k=1}^K \bfP_{j,\perp}\delta_{j,L}^k\phi(s_{j,0}^k)(\omega_{j}^k)^\top\right\| \frac{1}{1-2\|\bfQ_j\|_F^2} + 4\|\bfQ_j\|_F^2, 
\end{align}
where the last inequality follows from Lemma \ref{lm: perturbation QR}. %
From \eqref{eq: def Q}, we know that 
$\|\bfQ_j\|_F \le \frac{U_{\delta} }{L(1-\gamma)}U_{\omega}\zeta.$ When $ \frac{U_{\delta}}{L(1-\gamma)} U_{\omega}\zeta \le 1/2$, it holds that $\frac{1}{1-2\|\bfQ_j\|_F^2} \le 2$. 
Thus, to bound $\|\bfB_{j+1}-\bfB_j\|$, it remains to bound $\left\|\frac{\zeta}{KL}\sum_{k=1}^K \bfP_{j,\perp}\delta_{j,L}^k\phi(s_{j,0}^k)(\omega_{j}^k)^\top\right\|$, for which we have  
\begin{align}
\label{eq: auxiliary lemma: principle angle perturbation partial}
&\left\|\frac{\zeta}{KL}\sum_{k=1}^K \bfP_{j,\perp}\delta_{j,L}^k\phi(s_{j,0}^k)(\omega_{j}^k)^\top\right\| \nonumber \\
=&\frac{\zeta}{KL}\left\|\sum_{k=1}^K \bfP_{j,\perp}\delta_{j,L}^k\phi(s_{j,0}^k)(\omega_{j}^k)^\top\right\| \nonumber \\
\leq&\frac{\zeta}{KL}\left\|\sum_{k=1}^K \delta_{j,L}^k\phi(s_{j,0}^k)(\omega_{j}^k)^\top\right\| \nonumber \\
\leq&\frac{\zeta}{KL}\left\|\sum_{k=1}^K \pth{{ \tilde{A}_{j,L}^k} x_j^k +\bfb_{j,L}^k}(\omega_{j}^k)^\top\right\|\nonumber ~~~~ \text{by Proposition \ref{prop: key intermediate}}\\
\leq&\frac{\zeta}{KL}\left\|\sum_{k=1}^K { \tilde{A}_{j,L}^k} x_j^k(\omega_{j}^k)^\top\right\|
+\frac{\zeta}{KL}\left\|\sum_{k=1}^K \bfb_{j,L}^k(\omega_{j}^k)^\top\right\| 
\nonumber\\
\leq&2U_\omega\frac{\zeta}{KL}\sum_{k=1}^K   \| x_j^k\|
+U_\omega\frac{U_\delta}{L(1-\gamma)}\zeta, %
\end{align}
where the last inequality follows from Lemma \ref{lmm: U delta}. Plugging \eqref{eq: auxiliary lemma: principle angle perturbation partial} into \eqref{eq: auxiliary lemma: principle angle perturbation}, we get
\begin{align}
\label{eqn: one step B before substituting jth}
\|\mathbf{B}_{j+1} - \mathbf{B}_j\|
\leq 4U_\omega\frac{\zeta}{KL}\sum_{k=1}^K\|x_j^k\|
+2U_\omega\frac{U_\delta}{L(1-\gamma)}\zeta
+4\frac{U_{\delta}^2 }{L^2(1-\gamma)^2}U_{\omega}^2\zeta^2.
\end{align}
Finally, taking the conditional expectation and applying Lemma \ref{lmm: jth to tth error bound} to relate $x_j^k$ to $x_t^k$, we get
\begin{align*}
    &\bbE_{t-\tau}[\|\mathbf{B}_{j+1} - \mathbf{B}_j\| ]  \nonumber \\
    \leq &4U_\omega\frac{\zeta}{KL}\sum_{k=1}^K\bbE_{t-\tau}\|x_j^k\|
+2U_\omega\frac{U_\delta}{L(1-\gamma)}\zeta
+4\frac{U_{\delta}^2 }{L^2(1-\gamma)^2}U_{\omega}^2\zeta^2 \\
    \leq& 4U_\omega\frac{\zeta}{KL}\sum_{k=1}^K\pth{\tau \frac{U_{\delta}}{L(1-\gamma)}\beta + 4\tau\frac{U_{\delta}}{L(1-\gamma)}U_\omega^2 \zeta+\tau L_{*,1}BU_\delta\alpha +\bbE_{t-\tau}\|x_t^k\|}\\
    & +2U_\omega\frac{U_\delta}{L(1-\gamma)}\zeta %
    +4\frac{U_{\delta}^2 }{L^2(1-\gamma)^2}U_{\omega}^2\zeta^2\quad \quad ~~~~~ \text{(applying Lemma \ref{lmm: jth to tth error bound})}\\
    =&4U_\omega\frac{\zeta}{L} \frac{1}{K}\sum_{k=1}^K \bbE_{t-\tau}\|x_t^k\|
    +4\tau\frac{U_\omega U_\delta}{L^2(1-\gamma)}\beta\zeta
    +16\tau\frac{U_\delta U_\omega^3}{L^2(1-\gamma)}\zeta^2
    +4\tau U_\omega L_{*,1}BU_\delta\frac{\alpha \zeta}{L}\\
    &
    +2U_\omega\frac{U_\delta}{L(1-\gamma)}\zeta + 4\frac{U_{\delta}^2U_{\omega}^2}{L^2(1-\gamma)^2}\zeta^2. 
\end{align*}

Analogously, 
\begin{align}
\label{eqn: B perturb ^2}
   &\|\bfB_{j+1}-\bfB_j\|^2\nonumber \\ 
    \leq& 2\left\|\frac{\zeta}{KL}\sum_{k=1}^K \bfP_{j,\perp}\delta_{j,L}^k\phi(s_{j,0}^k)(\omega_{j}^k)^\top\right\|^2 \pth{\frac{1}{1-2\|\bfQ_j\|_F^2}}^2 + 32\|\bfQ_j\|_F^4 \nonumber\\
    \leq& 8\left\|\frac{\zeta}{KL}\sum_{k=1}^K \bfP_{j,\perp}\delta_{j,L}^k\phi(s_{j,0}^k)(\omega_{j}^k)^\top\right\|^2 +  32\|\bfQ_j\|_F^4.
\end{align}
Similar to \eqref{eq: auxiliary lemma: principle angle perturbation partial}, the first term can be decomposed and bounded as,
\begin{align}
\label{eqn: B perturb ^2, 1}
    &8\left\|\frac{\zeta}{KL}\sum_{k=1}^K \bfP_{j,\perp}\delta_{j,L}^k\phi(s_{j,0}^k)(\omega_{j}^k)^\top\right\|^2\nonumber\\
\leq&\frac{24\zeta^2}{K^2L^2}\left\|\sum_{k=1}^K { \tilde{A}_{j,L}^k} x_j^k(\omega_{j}^k)^\top\right\|^2
+\frac{24\zeta^2}{K^2L^2}\left\|\sum_{k=1}^K \bfb_{j,L}^k(\omega_{j}^k)^\top\right\|^2 
\nonumber\\
\leq&96U_\omega^2\frac{\zeta^2}{KL^2} \sum_{k=1}^K  \| x_j^k\|^2
+24U_\omega^2\frac{U_{\delta}^2}{L^2(1-\gamma)^2}\zeta^2.
\end{align}
Plugging \eqref{eqn: B perturb ^2, 1} back to \eqref{eqn: B perturb ^2}, conditioning on $v_{0:t-\tau}$, and applying Lemma \ref{lmm: jth to tth error bound}, we get
\begin{align*}
&\bbE_{t-\tau}\|\bfB_{j+1}-\bfB_j\|^2\\
\leq& 96U_\omega^2\frac{\zeta^2}{KL^2}  \sum_{k=1}^K \bbE_{t-\tau}\| x_j^k\|^2
+24U_\omega^2\frac{U_{\delta}^2}{L^2(1-\gamma)^2}\zeta^2
+ 32 (\frac{U_{\delta} }{L(1-\gamma)}U_{\omega}\zeta)^4\\
\leq& 96U_\omega^2\frac{\zeta^2}{KL^2}  \sum_{k=1}^K \pth{8\tau^2 \frac{U_{\delta}^2}{L^2(1-\gamma)^2}\beta^2 +  128 \tau^2 U_\omega^3  \frac{U_{\delta}^2}{L^2(1-\gamma)^2}\zeta^2+4\tau^2L_{*,1}^2B^2U_\delta^2\alpha^2   + 2\bbE_{t-\tau}\|x_t^k\|^2}
\\
&+24U_\omega^2\frac{U_{\delta}^2}{L^2(1-\gamma)^2}\zeta^2
+ 32 (\frac{U_{\delta} }{L(1-\gamma)}U_{\omega}\zeta)^4~~~~~ \text{(applying Lemma \ref{lmm: jth to tth error bound})}\\
\leq& 768\tau^2 \frac{U_{\delta}^2 U_\omega^2}{L^4(1-\gamma)^2}\beta^2\zeta^2 +  12288   \frac{U_{\delta}^2\tau^2 U_\omega^5}{L^4(1-\gamma)^2}\zeta^4 +384\tau^2U_\omega^2L_{*,1}^2B^2U_\delta^2\frac{\alpha^2  \zeta^2}{L^2}
\\
&+192U_\omega^2\frac{\zeta^2}{KL^2} \sum_{k=1}^K\bbE_{t-\tau}\|x_t^k\|^2
+24U_\omega^2\frac{U_{\delta}^2}{L^2(1-\gamma)^2}\zeta^2
+32 \frac{U_{\delta}^4 U_{\omega}^4\zeta^4}{L^4(1-\gamma)^4}.
\end{align*}
This gives the second claimed bound.

\end{proof}

\subsection{Perturbation between Original Chains and Auxiliary Chains}

\begin{lemma}
\label{lmm: AC, OC dist on f: critic}
Let $f: \calS\times (\calA\times \calS)^L \to \reals$. For any $t$ and any agent $k$,  let $s_{t, 0}^k, a_{t, 0}^k, s_{t, 1}^k, \cdots, s_{t, L}^k$ denote the sampled trajectory for critic, and 
$\tilde s_{t,0}^k, \tilde a_{t, 0}^k, \tilde s_{t, 1}^k, \cdots, \tilde s_{t,L}^k$ denote the hypothetical trajectory sampled from the auxiliary chain associated with the critic. 
Suppose that $f$ takes the form %
\[
f(s_{t, 0}^k, a_{t, 0}^k, s_{t, 1}^k, \cdots, s_{t, L}^k) 
= \sum_{\ell=0}^{L-1} \gamma^{\ell}f_\ell(s_{t,\ell}^k, a_{t,\ell}^k, s_{t,\ell+1}^k). 
\]
Then, for any fixed $\tau$, the following results hold:
    \begin{align*}
    &\bbE_{t-\tau}\qth{f(s_{t,0}^k,..., s_{t,L}^k) \phi(s_{t,0}^k)} - \bbE_{t-\tau}\qth{f(\tilde s_{t,0}^k,..., \tilde s_{t,L}^k)\phi(\tilde s_{t,0}^k)} \\
    \leq &\frac{\max_{\ell\in \{0, \cdots, L-1\}}\sup \|f_\ell\|}{1-\gamma}  \pth{ \frac{2}{(1-\rho^L)(1-\rho)} +\frac{2-\rho}{1-\rho} } L_\pi |\calA| \sum_{i=t-\tau}^{t-1}\bbE_{t-\tau}[\|\theta_{i+1}^k-\theta_{i}^k\|], 
    \end{align*}
    where $\sup \|f_\ell\|$ denotes the sup of function $f_{\ell}$.  
\end{lemma}

\begin{proof}
    \begin{align}
    \label{eqn: expect diff of two chains to tv dist}
    &\bbE_{t-\tau}\qth{f(s_{t,0}^k,..., s_{t,L}^k)\phi( s_{t,0}^k)} - \bbE_{t-\tau}\qth{f(\tilde s_{t,0}^k,..., \tilde s_{t,L}^k)\phi(\tilde s_{t,0}^k)}\nonumber \\
    =    &\bbE_{t-\tau}\qth{\sum_{\ell=0}^{L-1}\gamma^\ell f_\ell(s_{t,\ell}^k, a_{t,\ell}^k, s_{t,\ell+1}^k)\phi( s_{t,0}^k)} - \bbE_{t-\tau}\qth{\sum_{\ell=0}^{L-1}\gamma^\ell f_\ell(\tilde s_{t,\ell}^k, \tilde a_{t,\ell}^k, \tilde s_{t,\ell+1}^k)\phi(\tilde s_{t,0}^k)}\nonumber \\
   =& \sum_{\ell=0}^{L-1}\gamma^\ell\bbE_{t-\tau}\qth{f_\ell(s_{t,\ell}^k, a_{t,\ell}^k, s_{t,\ell+1}^k)\phi(s_{t,0}^k)} - \gamma^\ell\bbE_{t-\tau}\qth{f_\ell(\tilde s_{t,\ell}^k, \tilde a_{t,\ell}^k, \tilde s_{t,\ell+1}^k)\phi(\tilde s_{t,0}^k)}\nonumber \\
    \leq&\sum_{\ell=0}^{L-1}\gamma^\ell \sup \|f_\ell\| 2 d_{TV}(\mathbb{P}_{t-\tau: (t,\ell)}^k(s_{t,0}^k, s_{t,\ell}^k, a_{t,\ell}^k, s_{t,\ell+1}^k), \mathbb{U}_{t-\tau: (t,\ell)}^k(\tilde s_{t,0}^k, \tilde s_{t,\ell}^k, \tilde a_{t,\ell}^k, \tilde s_{t,\ell+1}^k)). 
    \end{align}
By Lemma \ref{lmm: tv dist between oc and ac},
\begin{align*}
    &2d_{TV}(\mathbb{P}_{t-\tau: (t,\ell)}^k(s_{t,0}^k, s_{t,\ell}^k, a_{t,\ell}^k, s_{t,\ell+1}^k), \mathbb{U}_{t-\tau: (t,\ell)}^k(\tilde s_{t,0}^k, \tilde s_{t,\ell}^k, \tilde a_{t,\ell}^k, \tilde s_{t,\ell+1}^k))\\
    \leq& 2 d_{TV}(\mathbb{P}_{t-\tau: t}^k(\cdot), \mathbb{U}_{t-\tau: t}^k(\cdot)) + \frac{2-\rho}{1-\rho} L_\pi|\calA|\bbE_{t-\tau}[\|\theta_t^k-\theta_{t-\tau}^k\|].
\end{align*}

Plugging this back to \eqref{eqn: expect diff of two chains to tv dist},
\begin{align}
\label{eqn: l-summation to single recursion}
    &\sum_{\ell=0}^{L-1}\gamma^\ell \sup \|f_\ell\|2d_{TV}(\mathbb{P}_{t-\tau: (t,\ell)}^k(s_{t,0}^k,s_{t,\ell}^k, a_{t,\ell}^k, s_{t,\ell+1}^k), \mathbb{U}_{t-\tau: (t,\ell)}^k(\tilde s_{t,0}^k, \tilde s_{t,\ell}^k, \tilde a_{t,\ell}^k, \tilde s_{t,\ell+1}^k))\nonumber\\
    \leq& \sum_{\ell=0}^{L-1}\gamma^\ell \sup \|f_\ell\|\pth{2 d_{TV}(\mathbb{P}_{t-\tau: t}^k(\cdot), \mathbb{U}_{t-\tau: t}^k(\cdot)) + \frac{2-\rho}{1-\rho} L_\pi|\calA|\bbE_{t-\tau}[\|\theta_t^k-\theta_{t-\tau}^k\|]}\nonumber\\
    \leq&  \frac{\max_{\ell\in\{0, \cdots, L-1\}}\sup \|f_\ell\|}{1-\gamma} \pth{2 d_{TV}(\bbP_{t-\tau:(t,0)}^k(\cdot),\bbU_{t-\tau:(t,0)}^k(\cdot)) + \frac{2-\rho}{1-\rho}    L_\pi |\calA| \bbE_{t-\tau}[\|\theta_{t}^k-\theta_{t-\tau}^k\|]}\nonumber\\
    =&  \frac{ \max_{\ell\in\{0, \cdots, L-1\}}\sup \|f_\ell\|}{1-\gamma} \pth{2 d_{TV}(\bbP_{t-\tau:(t-1,L)}^k(\cdot),\bbU_{t-\tau:(t-1,L)}^k(\cdot)) + \frac{2-\rho}{1-\rho} L_\pi |\calA| \bbE_{t-\tau}[\|\theta_{t}^k-\theta_{t-\tau}^k\|]}.
\end{align}
Again, by Lemma \ref{lmm: tv dist between oc and ac}, we have 
\begin{align*}
     &d_{TV}(\bbP_{t-\tau:(t-1,L)}^k(\cdot),\bbU_{t-\tau:(t-1,L)}^k(\cdot)) \leq \rho^L d_{TV}(\bbP_{t-\tau:(t-2,L)}^k(\cdot),\bbU_{t-\tau:(t-2,L)}^k(\cdot)) \\
     &+  \frac{1}{2}\frac{1}{1-\rho}L_\pi |\calA| \bbE_{t-\tau}[\|\theta_{t-1}^k-\theta_{t-\tau}^k\|]~ ~~~ \text{(by choosing $\tau^{\prime} = \tau-1$)}\\
     \leq& \rho^L \Bigg(\rho^L d_{TV}(\bbP_{t-\tau:(t-3,L)}^k(\cdot),\bbU_{t-\tau:(t-3,L)}^k(\cdot)) 
     + \frac{1}{2}\frac{1}{1-\rho}L_\pi |\calA| \bbE_{t-\tau}[\|\theta_{t-2}^k-\theta_{t-\tau}^k\|] \Bigg) \\ 
     &+  \frac{1}{2}\frac{1}{1-\rho}L_\pi |\calA| \bbE_{t-\tau}[\|\theta_{t-1}^k-\theta_{t-\tau}^k\|] ~~~ \text{(by choosing $\tau^{\prime} = \tau-2$)}\\
     \leq& \rho^{2L} d_{TV}(\bbP_{t-\tau:(t-3,L)}^k(\cdot),\bbU_{t-\tau:(t-3,L)}^k(\cdot)) \\
     &+ \rho^L\frac{1}{2}\frac{1}{1-\rho}L_\pi |\calA| \bbE[\|\theta_{t-2}^k-\theta_{t-\tau}^k\||v_{0:t-\tau}] +  \frac{1}{2}\frac{1}{1-\rho}L_\pi |\calA| \bbE[\|\theta_{t-1}^k-\theta_{t-\tau}^k\||v_{0:t-\tau}]\\
     \leq &\rho^{(t-(t-\tau)-1)L} d_{TV}(\bbP_{t-\tau:(t-\tau,L)}^k(\cdot),\bbU_{t-\tau:(t-\tau,L)}^k(\cdot))\\
     & +  \sum_{m=1}^{t-(t-\tau)-1} (\rho^L)^{m-1} \frac{1}{2}\frac{1}{1-\rho}L_\pi |\calA| \bbE_{t-\tau}[\|\theta_{t-m}^k-\theta_{t-\tau}^k\|]\\
     =&\sum_{m=1}^{\tau-1} (\rho^L)^{m-1} \frac{1}{2}\frac{1}{1-\rho}L_\pi |\calA| \bbE_{t-\tau}[\|\theta_{t-m}^k-\theta_{t-\tau}^k\|].
\end{align*}
In addition, for any $m=1, \cdots, \tau-1$, we have  
\begin{align*}
\bbE_{t-\tau}[\|\theta_{t-m}^k-\theta_{t-\tau}^k\|]
\le \bbE_{t-\tau}[\sum_{i=t-\tau}^{t-m-1}\|\theta_{i+1}^k-\theta_{i}^k\|]
\le \bbE_{t-\tau}[\sum_{i=t-\tau}^{t-1}\|\theta_{i+1}^k-\theta_{i}^k\|]. 
\end{align*}
Thus, Eq.\eqref{eqn: l-summation to single recursion} can be further bounded as 
\begin{align*}
    \eqref{eqn: l-summation to single recursion} \leq& \frac{\max_{\ell\in \{0, \cdots, L-1\}}\sup \|f_\ell\|}{1-\gamma}  \pth{ \frac{2}{(1-\rho^L)(1-\rho)} +\frac{2-\rho}{1-\rho} } L_\pi |\calA| \bbE_{t-\tau}[\sum_{i=t-\tau}^{t-1}\|\theta_{i+1}^k-\theta_{i}^k\|].
\end{align*}
    
\end{proof}

By following the same line of analysis, analogous results can be established for the sampled trajectory used in the actor update and its corresponding hypothetical trajectory.
\begin{corollary}
\label{cor: AC, OC dist on f: actor}
Let $f: \calS\times (\calA\times \calS)^L \to \reals$. For any $t$ and any agent $k$,  let $\hat{s}_{t, 0}^k, \hat{a}_{t, 0}^k, \hat{s}_{t, 1}^k, \cdots, \hat{s}_{t, L}^k$ denote the sampled trajectory for actor, and 
$\bar s_{t,0}^k, \bar a_{t, 0}^k, \bar s_{t, 1}^k, \cdots, \bar s_{t,L}^k$ denote the hypothetical trajectory sampled from the auxiliary chain associated with the actor. 
Suppose that $f$ takes the form %
\[
f(\hat{s}_{t, 0}^k, \hat{a}_{t, 0}^k, \hat{s}_{t, 1}^k, \cdots, \hat{s}_{t, L}^k) 
= \sum_{\ell=0}^{L-1} \gamma^{\ell}f_\ell(\hat{s}_{t,\ell}^k, \hat{a}_{t,\ell}^k, \hat{s}_{t,\ell+1}^k). 
\]
Then, for any fixed $\tau$, the following results hold:
    \begin{align*}
    &\bbE_{t-\tau}\qth{f(\hat{s}_{t,0}^k,..., \hat{s}_{t,L}^k) \phi(\hat{s}_{t,0}^k)} - \bbE_{t-\tau}\qth{f(\bar s_{t,0}^k,..., \bar s_{t,L}^k)\phi(\bar s_{t,0}^k)} \\
    \leq &\frac{\max_{\ell\in \{0, \cdots, L-1\}}\sup \|f_\ell\|}{1-\gamma}  \pth{ \frac{2}{(1-\gamma^L)(1-\gamma)} +\frac{2-\gamma}{1-\gamma} } L_\pi |\calA| \sum_{i=t-\tau}^{t-1}\bbE_{t-\tau}[\|\theta_{i+1}^k-\theta_{i}^k\|], 
    \end{align*} 
\end{corollary}

\begin{lemma}
\label{lmm: AC, OC dist on f with a0: critic}  
Under the same setup as in Lemma \ref{lmm: AC, OC dist on f: critic}, but replacing $\phi(s_{t,0}^k)$ with a general function $h(s_{t,0}^k, a_{t,0}^k)$ where $\sup\|h\|< \infty$, we have
\begin{align*}
    &\bbE_{t-\tau}\qth{f(s_{t,0}^k,..., s_{t,L}^k) h(s_{t,0}^k, a_{t,0}^k)} - \bbE_{t-\tau}\qth{f(\tilde s_{t,0}^k,..., \tilde s_{t,L}^k)h(\tilde s_{t,0}^k, \tilde a_{t,0}^k)} \\
    \leq &\frac{\sup\|h\| \max_{\ell\in \{0, \cdots, L-1\}}\sup \|f_\ell\|}{1-\gamma}  \pth{ \frac{2}{(1-\rho^L)(1-\rho)} +\frac{3-2\rho}{1-\rho} } L_\pi |\calA| \sum_{i=t-\tau}^{t-1}\bbE_{t-\tau}[\|\theta_{i+1}^k-\theta_{i}^k\|].
\end{align*}
\end{lemma} 
\begin{proof}
The proof follows identically to Lemma \ref{lmm: AC, OC dist on f: critic}, with the only difference that the TV distance now involves the 5-tuple $(s_{t,0}^k, a_{t,0}^k, s_{t,\ell}^k, a_{t,\ell}^k, s_{t,\ell+1}^k)$ instead of the 4-tuple. Specifically, in \eqref{eqn: expect diff of two chains to tv dist}, we bound
\begin{align*}
    &\bbE_{t-\tau}\qth{f_\ell(s_{t,\ell}^k, a_{t,\ell}^k, s_{t,\ell+1}^k)h(s_{t,0}^k,a_{t,0}^k)} - \bbE_{t-\tau}\qth{f_\ell(\tilde s_{t,\ell}^k, \tilde a_{t,\ell}^k, \tilde s_{t,\ell+1}^k)h(\tilde s_{t,0}^k,\tilde a_{t,0}^k)}\\
    \leq& \sup\|h\| \max_{\ell\in \{0, \cdots, L-1\}} \sup\|f_\ell\| \\
   & \quad   \cdot 2d_{TV}(\mathbb{P}_{t-\tau:(t,\ell)}^k(s_{t,0}^k, a_{t,0}^k, s_{t,\ell}^k, a_{t,\ell}^k, s_{t,\ell+1}^k), \mathbb{U}_{t-\tau:(t,\ell)}^k(\tilde s_{t,0}^k, \tilde a_{t,0}^k, \tilde s_{t,\ell}^k, \tilde a_{t,\ell}^k, \tilde s_{t,\ell+1}^k)).
\end{align*}
By Lemma \ref{lmm: tv dist between oc and ac}, we have  
\begin{align*}
    &2d_{TV}(\mathbb{P}_{t-\tau:(t,\ell)}^k(s_{t,0}^k, a_{t,0}^k, s_{t,\ell}^k, a_{t,\ell}^k, s_{t,\ell+1}^k), \mathbb{U}_{t-\tau:(t,\ell)}^k(\tilde s_{t,0}^k, \tilde a_{t,0}^k, \tilde s_{t,\ell}^k, \tilde a_{t,\ell}^k, \tilde s_{t,\ell+1}^k))\\
    \leq& 2d_{TV}(\mathbb{P}_{t-\tau:t}^k(\cdot), \mathbb{U}_{t-\tau:t}^k(\cdot)) + \frac{3-2\rho}{1-\rho}L_\pi|\calA|\bbE_{t-\tau}[\|\theta_t^k-\theta_{t-\tau}^k\|].
\end{align*}
The rest of the proof follows identically to Lemma \ref{lmm: AC, OC dist on f: critic}, with $\frac{3-2\rho}{1-\rho}$ replacing $\frac{2-\rho}{1-\rho}$, giving the final bound with $\frac{2}{(1-\rho^L)(1-\rho)}+\frac{3-2\rho}{1-\rho}$.
\end{proof}

\begin{corollary} 
\label{cor: AC, OC dist on f with a0: actor}  
Under the same setup as in Corollary \ref{cor: AC, OC dist on f: actor} but replacing $\phi(\hat{s}_{t,0}^k)$ with a general function $h(\hat{s}_{t,0}^k, \hat{a}_{t,0}^k)$ where $\sup\|h\|< \infty$, we have
\begin{align*}
    &\bbE_{t-\tau}\qth{f(\hat{s}_{t,0}^k,..., \hat{s}_{t,L}^k) h(\hat{s}_{t,0}^k, \hat{a}_{t,0}^k)} - \bbE_{t-\tau}\qth{f(\bar s_{t,0}^k,..., \bar s_{t,L}^k)h(\bar s_{t,0}^k, \bar a_{t,0}^k)} \\
    \leq &\frac{\sup\|h\| \max_{\ell\in \{0, \cdots, L-1\}}\sup \|f_\ell\|}{1-\gamma}  \pth{ \frac{2}{(1-\gamma^L)(1-\gamma)} +\frac{3-2\gamma}{1-\gamma} } L_\pi |\calA| \sum_{i=t-\tau}^{t-1}\bbE_{t-\tau}[\|\theta_{i+1}^k-\theta_{i}^k\|].
\end{align*}
\end{corollary}

\begin{lemma}
\label{lmm: tv dist between oc and ac}
Fix $\tau^{\prime}$.  
For any $t$ and $\ell$ such that $t\ge \tau^{\prime}$ and $\ell \in \{0, \dots, L-1\}$, we have 
\begin{align*}
    &2d_{TV}(\mathbb{P}_{t-\tau^{\prime}: (t,\ell)}^k(s_{t,0}^k, s_{t,\ell}^k, a_{t,\ell}^k, s_{t,\ell+1}^k), \mathbb{U}_{t-\tau^{\prime}: (t,\ell)}^k(\tilde s_{t,0}^k, \tilde s_{t,\ell}^k, \tilde a_{t,\ell}^k, \tilde s_{t,\ell+1}^k))\\
    &\qquad \qquad \qquad \qquad \leq 2 d_{TV}(\mathbb{P}_{t-\tau^{\prime}: t}^k(\cdot), \mathbb{U}_{t-\tau^{\prime}: t}^k(\cdot)) + \frac{2-\rho}{1-\rho} L_\pi|\calA|\bbE_{t-\tau^{\prime}}[\|\theta_t^k-\theta_{t-\tau^{\prime}}^k\|].\\
    &2d_{TV}(\mathbb{P}_{t-\tau^{\prime}:(t,\ell)}^k(s_{t,0}^k, a_{t,0}^k, s_{t,\ell}^k, a_{t,\ell}^k, s_{t,\ell+1}^k), \mathbb{U}_{t-\tau^{\prime}:(t,\ell)}^k(\tilde s_{t,0}^k, \tilde a_{t,0}^k, \tilde s_{t,\ell}^k, \tilde a_{t,\ell}^k, \tilde s_{t,\ell+1}^k))\nonumber\\
    &\qquad  \qquad \qquad \qquad\leq 2d_{TV}(\mathbb{P}_{t-\tau^{\prime}:t}^k(\cdot), \mathbb{U}_{t-\tau^{\prime}:t}^k(\cdot)) + \frac{3-2\rho}{1-\rho}L_\pi|\calA|\bbE_{t-\tau^{\prime}}[\|\theta_t^k-\theta_{t-\tau^{\prime}}^k\|].\\ 
    &d_{TV}(\bbP_{t-\tau^{\prime}:(t,\ell+1)}^k(\cdot),\bbU_{t-\tau^{\prime}:(t,\ell+1)}^k(\cdot))
    \leq \rho d_{TV}(\bbP_{t-\tau^{\prime}:(t,\ell)}^k(\cdot),\bbU_{t-\tau^{\prime}:(t,\ell)}^k(\cdot)) \\
    &  \qquad  \qquad \qquad \qquad \qquad \qquad \qquad \qquad \qquad \qquad \qquad \qquad  + \frac{1}{2}L_\pi |\calA| \bbE_{t-\tau^{\prime}}[\|\theta_{t}^k-\theta_{t-\tau^{\prime}}^k\|]. \\
    &d_{TV}(\bbP_{t-\tau^{\prime}:(t,\ell+1)}^k(\cdot),\bbU_{t-\tau^{\prime}:(t,\ell+1)}^k(\cdot))\leq \rho^{\ell+1} d_{TV}(\bbP_{t-\tau^{\prime}:(t,0)}^k(\cdot),\bbU_{t-\tau^{\prime}:(t,0)}^k(\cdot))\\
    & \qquad  \qquad \qquad \qquad \qquad \qquad \qquad \qquad \qquad \qquad + \frac{1}{2}\frac{1}{1-\rho}L_\pi |\calA| \bbE_{t-\tau'}[\|\theta_{t}^k-\theta_{t-\tau^{\prime}}^k\|]. 
    \end{align*}

\end{lemma}

\begin{proof}

With a little abuse of notation, we use $\bbE_{\theta_t^k, t-\tau^{\prime}}$ to denote the expectation taken over the randomness of $\theta_t^k$ conditioning on $v_{0: t-\tau^{\prime}}$. We have 
\begin{align}
\label{eqn: decomp of tv dist of s0, sl, al, sl+1}
    &2d_{TV}(\mathbb{P}_{t-\tau^{\prime}: (t,\ell)}^k(s_{t,0}^k, s_{t,\ell}^k, a_{t,\ell}^k, s_{t,\ell+1}^k), \mathbb{U}_{t-\tau^{\prime}: (t,\ell)}^k(\tilde s_{t,0}^k, \tilde s_{t,\ell}^k, \tilde a_{t,\ell}^k, \tilde s_{t,\ell+1}^k))\nonumber\\
    =&\bbE_{\theta_t^k, t-\tau^{\prime}}\int_{s_0}\int_{s_\ell}\int_{s_{\ell+1}}  \sum_{a_\ell} \abth{\mathbb{P}_{t-\tau^{\prime}: t}^k(ds_0, ds_\ell, a_\ell, ds_{\ell+1}) -\mathbb{U}_{t-\tau^{\prime}: t}^k(ds_0, ds_\ell, a_\ell, ds_{\ell+1}) } \nonumber\\
    =& \bbE_{\theta_t^k, t-\tau^{\prime}}\int_{s_0}\int_{s_\ell}\int_{s_{\ell+1}}  \sum_{a_\ell} \Bigg|\mathbb{P}_{t-\tau^{\prime}: t}^k(ds_0) \bbP^k_{t-\tau^{\prime}: t}(d s_{\ell} |ds_0)\pi_{\theta_t^k}(a_\ell|ds_\ell) P^k( ds_{\ell+1}| ds_\ell,a_\ell)\nonumber\\
    &\qquad\qquad\qquad\qquad -\mathbb{U}_{t-\tau^{\prime}: t}^k(ds_0) \bbU_{t-\tau^{\prime}: t}^k (ds_\ell|ds_0)\pi_{\theta_{t-\tau^{\prime}}^k}(a_\ell|ds_\ell) P^k( ds_{\ell+1}| ds_\ell,a_\ell) \Bigg|\nonumber\\
    =& \bbE_{\theta_t^k, t-\tau^{\prime}}\int_{s_0}\int_{s_\ell} \sum_{a_\ell} \Bigg|
    \mathbb{P}_{t-\tau^{\prime}: t}^k(ds_0) \bbP^k_{t-\tau^{\prime}: t}(d s_{\ell} |ds_0)\pi_{\theta_t^k}(a_\ell|ds_\ell)  -\mathbb{U}_{t-\tau^{\prime}: t}^k(ds_0) \bbU_{t-\tau^{\prime}: t}^k(ds_\ell|ds_0)\pi_{\theta_{t-\tau^{\prime}}^k}(a_\ell|ds_\ell) \Bigg| \nonumber\\
    \leq& \bbE_{\theta_t^k, t-\tau^{\prime}}\int_{s_0}\int_{s_\ell} \sum_{a_\ell} \Bigg|\mathbb{P}_{t-\tau^{\prime}: t}^k(ds_0)-\mathbb{U}_{t-\tau^{\prime}: t}^k(ds_0)\Bigg| \bbP^k_{t-\tau^{\prime}: t}(d s_{\ell} |ds_0)\pi_{\theta_t^k}(a_\ell|ds_\ell) \nonumber\\
    & +\bbE_{\theta_t^k, t-\tau^{\prime}}\int_{s_0}\int_{s_\ell} \sum_{a_\ell} \mathbb{U}_{t-\tau^{\prime}: t}^k(ds_0)\bigg|\bbP^k_{t-\tau^{\prime}: t}(d s_{\ell} |ds_0)\pi_{\theta_t^k}(a_\ell|ds_\ell)| -\bbU_{t-\tau^{\prime}: t}^k(ds_\ell|ds_0)\pi_{\theta_{t-\tau^{\prime}}^k}(a_\ell|ds_\ell) \bigg|.
\end{align}
For the first term in \eqref{eqn: decomp of tv dist of s0, sl, al, sl+1}, we have,
\begin{align}
\label{eqn: first term of the 4tuple decomp}
    &\bbE_{\theta_t^k, t-\tau^{\prime}}\int_{s_0}\int_{s_\ell} \sum_{a_\ell} \Bigg|\mathbb{P}_{t-\tau^{\prime}: t}^k(ds_0)-\mathbb{U}_{t-\tau^{\prime}: t}^k(ds_0)\Bigg| \bbP^k_{t-\tau^{\prime}: t}(d s_{\ell} |ds_0)\pi_{\theta_t^k}(a_\ell|ds_\ell)\nonumber \\
    =&\bbE_{\theta_t^k, t-\tau^{\prime}} \int_{s_0}\int_{s_\ell} \Bigg|\mathbb{P}_{t-\tau^{\prime}: t}^k(ds_0)-\mathbb{U}_{t-\tau^{\prime}: t}^k(ds_0)\Bigg| \bbP^k_{t-\tau^{\prime}: t}(d s_{\ell} |ds_0)\nonumber\\
    =& \bbE_{\theta_t^k, t-\tau^{\prime}}\int_{s_0} \Bigg|\mathbb{P}_{t-\tau^{\prime}: t}^k(ds_0)-\mathbb{U}_{t-\tau^{\prime}: t}^k(ds_0)\Bigg| \nonumber\\
    =& 2 d_{TV}(\mathbb{P}_{t-\tau^{\prime}: t}^k(\cdot), \mathbb{U}_{t-\tau^{\prime}: t}^k(\cdot)).
\end{align}

For the second term in \eqref{eqn: decomp of tv dist of s0, sl, al, sl+1}, we have 
\begin{align}
\label{eqn: second term of the 4tuple decomp}
& \bbE_{\theta_t^k, t-\tau^{\prime}}\int_{s_0}\int_{s_\ell} \sum_{a_\ell} \mathbb{U}_{t-\tau^{\prime}: t}^k(ds_0)\bigg|\bbP^k_{t-\tau^{\prime}: t}(d s_{\ell} |ds_0)\pi_{\theta_t^k}(a_\ell|ds_\ell)| -\bbU_{t-\tau^{\prime}: t}^k(ds_\ell|ds_0)\pi_{\theta_{t-\tau^{\prime}}^k}(a_\ell|ds_\ell) \bigg| \nonumber \\
\leq& \bbE_{\theta_t^k, t-\tau^{\prime}}\int_{s_0}\int_{s_\ell} \sum_{a_\ell} \mathbb{U}_{t-\tau^{\prime}: t}^k(ds_0)
\abth{\bbP^k_{t-\tau^{\prime}: t}(d s_{\ell} |ds_0) -  \bbU_{t-\tau^{\prime}: t}^k(ds_\ell|ds_0)}\pi_{\theta_{t}^k}(a_\ell|ds_\ell)\nonumber\\
&+ \bbE_{\theta_t^k, t-\tau^{\prime}}\int_{s_0}\int_{s_\ell} \sum_{a_\ell} \mathbb{U}_{t-\tau^{\prime}: t}^k(ds_0) \bbU_{t-\tau^{\prime}: t}^k(ds_\ell|ds_0) \bigg|\pi_{\theta_t^k}(a_\ell|ds_\ell)  -\pi_{\theta_{t-\tau^{\prime}}^k}(a_\ell|ds_\ell)\bigg|\nonumber\\
=& \bbE_{\theta_t^k, t-\tau^{\prime}}\int_{s_0}\int_{s_\ell} \mathbb{U}_{t-\tau^{\prime}: t}^k(ds_0)
\abth{\bbP^k_{t-\tau^{\prime}: t}(d s_{\ell} |ds_0) -  \bbU_{t-\tau^{\prime}: t}^k(ds_\ell|ds_0)} \nonumber\\
&+ \bbE_{\theta_t^k, t-\tau^{\prime}}\int_{s_0}\int_{s_\ell} \sum_{a_\ell} \mathbb{U}_{t-\tau^{\prime}: t}^k(ds_0) \bbU_{t-\tau^{\prime}: t}^k(ds_\ell|ds_0) \bigg|\pi_{\theta_t^k}(a_\ell|ds_\ell)  -\pi_{\theta_{t-\tau^{\prime}}^k}(a_\ell|ds_\ell)\bigg| \nonumber\\
\leq& \bbE_{\theta_t^k, t-\tau^{\prime}}\int_{s_0} \mathbb{U}_{t-\tau^{\prime}: t}^k(ds_0) 2d_{TV}(\bbP^k_{t-\tau^{\prime}: t}(s_{\ell} |ds_0), \bbU_{t-\tau^{\prime}: t}^k(s_\ell|ds_0)) \nonumber\\
&\qquad + L_\pi |\calA| \bbE_{t-\tau^{\prime}}[ \|{\theta_t^k}-{\theta_{t-\tau^{\prime}}^k}\|], 
\end{align}
where the last inequality follows from Assumption \ref{assmp: Lipschitz of policy}. 

It remains to bound $2d_{TV}(\bbP^k_{t-\tau^{\prime}: t}(s_{\ell} |ds_0), \bbU_{t-\tau^{\prime}: t}^k(s_\ell|ds_0))$. 
\begin{align*}
&2d_{TV}(\bbP^k_{t-\tau^{\prime}: t}(d s_{\ell} |ds_0), \bbU_{t-\tau^{\prime}: t}^k(ds_\ell|ds_0))\\
=& \int_{s_\ell}\abth{\int_{s_{\ell-1}}\sum_{a_\ell} \bbP^k_{t-\tau^{\prime}: t}(ds_{\ell-1}, a_{\ell-1}, ds_\ell|ds_0)  -\bbU_{t-\tau^{\prime}: t}^k(ds_{\ell-1}, a_{\ell-1}, ds_\ell|ds_0)}\\
=& \int_{s_\ell}\Bigg|\int_{s_{\ell-1}}\sum_{a_\ell}\bbP_{t-\tau^{\prime}}^k(ds_{\ell-1}|ds_0)\pi_{\theta_t^k}(a_{\ell-1}| ds_{\ell-1}) P^k(ds_\ell|ds_{\ell-1},a_{\ell-1})\\
&\qquad \qquad \qquad   -\bbU_{\theta_{t-\tau^{\prime}}^k}^k(ds_{\ell-1}|ds_0) \pi_{\theta_{t-\tau^{\prime}}^k}( a_{\ell-1}| ds_{\ell-1})P^k(ds_\ell|ds_{\ell-1},a_{\ell-1})\Bigg|\\
\leq& \int_{s_\ell}\int_{s_{\ell-1}}\sum_{a_\ell} \Bigg|(\bbP_{t-\tau^{\prime}}^k(ds_{\ell-1}|ds_0)-\bbU_{t-\tau^{\prime}}^k(ds_{\ell-1}|ds_0))\pi_{\theta_{t}^k}(a_{\ell-1}| ds_{\ell-1}) P^k(ds_\ell|ds_{\ell-1},a_{\ell-1})\Bigg|\\
&+   \int_{s_\ell}\Bigg|\int_{s_{\ell-1}}\sum_{a_\ell}\bbU_{t-\tau^{\prime}}^k(ds_{\ell-1}|ds_0) (\pi_{\theta_{t}^k}( a_{\ell-1}| ds_{\ell-1})-\pi_{\theta_{t-\tau^{\prime}}^k}( a_{\ell-1}| ds_{\ell-1}))P^k(ds_\ell|ds_{\ell-1},a_{\ell-1})\Bigg|\\
\leq& 2d_{TV}(\bbP_{t-\tau^{\prime}}^k(s_{\ell-1}|ds_0)\otimes \pi_{\theta_{t}^k}\otimes P^k, \bbU_{t-\tau^{\prime}}^k(s_{\ell-1}|ds_0)\otimes \pi_{\theta_{t}^k}\otimes P^k) + L_\pi|\calA|\|\theta_t^k-\theta_{t-\tau^{\prime}}^k\|\\
\leq& 2\rho d_{TV}(\bbP_{t-\tau^{\prime}}^k(s_{\ell-1}|ds_0), \bbU_{t-\tau^{\prime}}^k(s_{\ell-1}|ds_0)) + L_\pi|\calA|\|\theta_t^k-\theta_{t-\tau^{\prime}}^k\|. ~~~~ (\text{by Dobrushin condition})
\end{align*}

Recognizing the recursive structure of the above inequality, we recursively unroll it to obtain 
\begin{align}
\label{eqn: unroll recursion of tv dist from l to 0, given 0}
    &d_{TV}(\bbP^k_{t-\tau^{\prime}: t}(s_{\ell} |ds_0), \bbU_{t-\tau^{\prime}: t}^k(s_\ell|ds_0))\nonumber\\
    \leq &\rho d_{TV}(\bbP^k_{t-\tau^{\prime}: t}(s_{\ell-1} |ds_0), \bbU_{t-\tau^{\prime}: t}^k(s_{\ell-1}|ds_0))+ \frac{1}{2} L_\pi|\calA|\|\theta_t^k-\theta_{t-\tau^{\prime}}^k\|\nonumber \\
    \leq&\rho^\ell d_{TV}(\bbP^k_{t-\tau^{\prime}: t}(s_{0} |ds_0), \bbU_{t-\tau^{\prime}: t}^k(s_0|ds_0)) + \sum_{m=0}^{\ell-1} \rho^m \frac{1}{2} L_\pi|\calA|\|\theta_t^k-\theta_{t-\tau^{\prime}}^k\|\nonumber \\
    \leq& \frac{1}{1-\rho} \frac{1}{2} L_\pi|\calA|\|\theta_t^k-\theta_{t-\tau^{\prime}}^k\|,
\end{align} 
where the last inequality follows from conditioning on the initial state and the fact that $s_0$ is the shared initial state at round $t$ by both the original and auxiliary critic chains. Consequently, the induced distributions at time 0 are identical point masses at $s_0$, yielding 
\begin{align*}
    d_{TV}(\bbP_{t-\tau^{\prime}: t}^k(s_0|s_0), \bbU_{t-\tau^{\prime}: t}^k(s_0|s_0)) = 0.
\end{align*}

Putting \eqref{eqn: unroll recursion of tv dist from l to 0, given 0} back to \eqref{eqn: second term of the 4tuple decomp}, 
\begin{align*}
    \text{\eqref{eqn: second term of the 4tuple decomp}}
    \leq & \frac{1}{1-\rho} L_\pi|\calA|\bbE_{t-\tau^{\prime}}[\|\theta_t^k-\theta_{t-\tau^{\prime}}^k\|] + L_\pi |\calA| \bbE_{t-\tau^{\prime}}[ \|{\theta_t^k}-{\theta_{t-\tau^{\prime}}^k}\|]\\
    =& \frac{2-\rho}{1-\rho} L_\pi|\calA|\bbE_{t-\tau^{\prime}}[\|\theta_t^k-\theta_{t-\tau^{\prime}}^k\|].
\end{align*}
Finally, putting the above inequality and \eqref{eqn: first term of the 4tuple decomp} back to \eqref{eqn: decomp of tv dist of s0, sl, al, sl+1}, we have,
\begin{align*}
    &2d_{TV}(\mathbb{P}_{t-\tau^{\prime}: (t,\ell)}^k(s_{t,0}^k, s_{t,\ell}^k, a_{t,\ell}^k, s_{t,\ell+1}^k), \mathbb{U}_{t-\tau^{\prime}: (t,\ell)}^k(\tilde s_{t,0}^k, \tilde s_{t,\ell}^k, \tilde a_{t,\ell}^k, \tilde s_{t,\ell+1}^k))\\
    \leq &2 d_{TV}(\mathbb{P}_{t-\tau^{\prime}: t}^k(\cdot), \mathbb{U}_{t-\tau^{\prime}: t}^k(\cdot)) + \frac{2-\rho}{1-\rho} L_\pi|\calA|\bbE_{t-\tau^{\prime}}[\|\theta_t^k-\theta_{t-\tau^{\prime}}^k\|].
\end{align*}

\begin{align}
\label{eqn: decomp of tv dist of s0, a0, sl, al, sl+1}
    &2d_{TV}(\mathbb{P}_{t-\tau^{\prime}:(t,\ell)}^k(s_{t,0}^k, a_{t,0}^k, s_{t,\ell}^k, a_{t,\ell}^k, s_{t,\ell+1}^k), \mathbb{U}_{t-\tau^{\prime}:(t,\ell)}^k(\tilde s_{t,0}^k, \tilde a_{t,0}^k, \tilde s_{t,\ell}^k, \tilde a_{t,\ell}^k, \tilde s_{t,\ell+1}^k))\nonumber\\
    =&\bbE_{\theta_t^k, t-\tau^{\prime}} \int_{s_0}\sum_{a_0}\int_{s_\ell}\int_{s_{\ell+1}}\sum_{a_\ell} \abth{\mathbb{P}_{t-\tau^{\prime}: t}^k(ds_0, a_0, ds_\ell, a_\ell, ds_{\ell+1}) -\mathbb{U}_{t-\tau^{\prime}: t}^k(ds_0, a_0, ds_\ell, a_\ell, ds_{\ell+1}) } \nonumber\\
    =& \bbE_{\theta_t^k, t-\tau^{\prime}}\int_{s_0}\sum_{a_0}\int_{s_\ell}\int_{s_{\ell+1}}\sum_{a_\ell} \Bigg|\mathbb{P}_{t-\tau^{\prime}: t}^k(ds_0)\pi_{\theta_t^k}(a_0|ds_0)\bbP^k_{t-\tau^{\prime}: t}(ds_\ell|ds_0, a_0)\pi_{\theta_t^k}(a_\ell|ds_\ell) |P^k( ds_{\ell+1}| ds_\ell,a_\ell)\nonumber\\
    &\qquad \qquad \qquad -\mathbb{U}_{t-\tau^{\prime}: t}^k(ds_0)\pi_{\theta_{t-\tau^{\prime}}^k}(a_0|ds_0) \bbU_{t-\tau^{\prime}: t }^k(ds_\ell|ds_0, a_0)\pi_{\theta_{t-\tau^{\prime}}^k}(a_\ell|ds_\ell) P^k( ds_{\ell+1}| ds_\ell,a_\ell) \Bigg|\nonumber\\
    =& \bbE_{\theta_t^k, t-\tau^{\prime}}\int_{s_0}\sum_{a_0}\int_{s_\ell}\sum_{a_\ell} \Bigg|\mathbb{P}_{t-\tau^{\prime}: t}^k(ds_0)\pi_{\theta_t^k}(a_0|ds_0)\bbP^k_{t-\tau^{\prime}: t}(ds_\ell|ds_0, a_0)\pi_{\theta_t^k}(a_\ell|ds_\ell) |\nonumber\\
    &\qquad \qquad \qquad \qquad \qquad -\mathbb{U}_{t-\tau^{\prime}: t}^k(ds_0)\pi_{\theta_{t-\tau^{\prime}}^k}(a_0|ds_0) \bbU_{t-\tau^{\prime}: t }^k(ds_\ell|ds_0, a_0)\pi_{\theta_{t-\tau^{\prime}}^k}(a_\ell|ds_\ell) \Bigg|\nonumber\\ 
    \leq& \bbE_{\theta_t^k, t-\tau^{\prime}}\int_{s_0}\sum_{a_0}\int_{s_\ell}\sum_{a_\ell} \Bigg|\mathbb{P}_{t-\tau^{\prime}: t}^k(ds_0)-\mathbb{U}_{t-\tau^{\prime}: t}^k(ds_0)\Bigg| \pi_{\theta_t^k}(a_0|ds_0)\bbP_{t-\tau^{\prime}: t}^k(ds_\ell|ds_0,a_0)\pi_{\theta_t^k}(a_\ell|ds_\ell)\nonumber\\
    & +\int_{s_0}\sum_{a_0}\int_{s_\ell}\sum_{a_\ell} \mathbb{U}_{t-\tau^{\prime}: t}^k(ds_0)\Bigg|\pi_{\theta_t^k}(a_0|ds_0)\bbP_{t-\tau^{\prime}: t}^k(ds_\ell|ds_0,a_0)\pi_{\theta_t^k}(a_\ell|ds_\ell)| \nonumber\\
    &\qquad\qquad\qquad\qquad\qquad\qquad \qquad  -\pi_{\theta_{t-\tau^{\prime}}^k}(a_0|ds_0)\bbU_{t-\tau^{\prime}: t}^k(ds_\ell|ds_0,a_0)\pi_{\theta_{t-\tau^{\prime}}^k}(a_\ell|ds_\ell) \Bigg|\nonumber\\
    =:& \bbE_{\theta_t^k, t-\tau^{\prime}}\qth{\text{ Term 1} + \text{ Term 2}}.
\end{align}

For Term 1 in \eqref{eqn: decomp of tv dist of s0, a0, sl, al, sl+1}, integrating out $a_0, s_\ell, a_\ell$ gives 1, so
\begin{align}
\label{eqn: first term of the 5tuple decomp}
    \text{Term 1} = \int_{s_0} \Bigg|\mathbb{P}_{t-\tau^{\prime}: t}^k(ds_0)-\mathbb{U}_{t-\tau^{\prime}: t}^k(ds_0)\Bigg| = 2 d_{TV}(\mathbb{P}_{t-\tau^{\prime}: t}^k(\cdot), \mathbb{U}_{t-\tau^{\prime}: t}^k(\cdot)).
\end{align}

For Term 2 in \eqref{eqn: decomp of tv dist of s0, a0, sl, al, sl+1}, we add and subtract $\pi_{\theta_{t-\tau^{\prime}}^k}(a_0|ds_0)\bbP_{t-\tau^{\prime}: t}^k(ds_\ell|ds_0,a_0)\pi_{\theta_t^k}(a_\ell|ds_\ell)|$ to separate the policy drift on $a_0$ from the transition kernel drift:
\begin{align}
\label{eqn: second term of the 5tuple decomp}
    \text{Term 2}
    \leq& \int_{s_0}\sum_{a_0}\int_{s_\ell}\sum_{a_\ell} \mathbb{U}_{t-\tau^{\prime}: t}^k(ds_0)\abth{\pi_{\theta_t^k}(a_0|ds_0)-\pi_{\theta_{t-\tau^{\prime}}^k}(a_0|ds_0)}\bbP_{t-\tau^{\prime}:t}^k(ds_\ell|ds_0,a_0)\pi_{\theta_t^k}(a_\ell|ds_\ell)|\nonumber\\
    &+ \int_{s_0}\sum_{a_0}\int_{s_\ell}\sum_{a_\ell} \mathbb{U}_{t-\tau^{\prime}: t}^k(ds_0)\pi_{\theta_{t-\tau^{\prime}}^k}(a_0|ds_0)\Bigg|\bbP_{t-\tau^{\prime}: t}^k(ds_\ell|ds_0,a_0)\pi_{\theta_t^k}(a_\ell|ds_\ell)|\nonumber\\
    &\qquad\qquad\qquad\qquad\qquad\qquad\qquad\qquad -\bbU_{t-\tau^{\prime}: t}^k(ds_\ell|ds_0,a_0)\pi_{\theta_{t-\tau^{\prime}}^k}(a_\ell|ds_\ell) \Bigg|\nonumber\\
    =:& \text{ Piece A} + \text{ Piece B}.
\end{align}

For Piece A, integrating out $s_\ell, a_\ell$ gives 1, and applying Assumption \ref{assmp: Lipschitz of policy} (c),
\begin{align}
\label{eqn: piece A of 5tuple}
    \text{Piece A} 
    \leq \int_{s_0}\mathbb{U}_{t-\tau^{\prime}:t}^k(ds_0)\sum_{a_0}\abth{\pi_{\theta_t^k}(a_0|ds_0)-\pi_{\theta_{t-\tau^{\prime}}^k}(a_0|ds_0)} 
    \leq L_\pi|\calA|\bbE[\|\theta_t^k-\theta_{t-\tau^{\prime}}^k\||v_{0:t-\tau^{\prime}}].
\end{align}

For Piece B, for each fixed $(s_0, a_0)$, the inner absolute value has the same structure as the second term in the 4-tuple decomposition \eqref{eqn: second term of the 4tuple decomp}. Specifically, since $P^k(\cdot|s_0, a_0)$ is the shared transition kernel and $(s_0, a_0)$ is the shared starting point, the distribution of $s_1$ is identical for both chains:
\begin{align*}
    P^k(s_1|s_0, a_0) = \mathbb{P}_{t-\tau^{\prime}: t}^k(s_1|s_0,a_0) = \mathbb{U}_{t-\tau^{\prime}: t}^k(s_1|s_0,a_0),
\end{align*}
giving the base case:
\begin{align*}
    d_{TV}(\mathbb{P}_{t-\tau^{\prime}: t}^k(s_1|s_0,a_0), \mathbb{U}_{t-\tau^{\prime}: t}^k(s_1|s_0,a_0)) = 0.
\end{align*}
Applying the same Dobrushin recursion as in \eqref{eqn: unroll recursion of tv dist from l to 0, given 0} from $\ell$ down to 1, and using $\sum_{m=0}^{\ell-2}\rho^m \leq \frac{1}{1-\rho}$, we get for each fixed $(s_0, a_0)$,
\begin{align}
\label{eqn: piece B of 5tuple}
    &\int_{s_\ell}\sum_{a_\ell}\Bigg|\mathbb{P}_{t-\tau^{\prime}: t}^k(ds_\ell|ds_0,a_0)\pi_{\theta_t^k}(a_\ell|ds_\ell) -\mathbb{U}_{t-\tau^{\prime}: t}^k(ds_\ell|ds_0,a_0)\pi_{\theta_{t-\tau^{\prime}}^k}(a_\ell|ds_\ell) \Bigg|\nonumber\\
    & \le \int_{s_\ell}\Bigg|\mathbb{P}_{t-\tau^{\prime}: t}^k(ds_\ell|ds_0,a_0) 
    -\mathbb{U}_{t-\tau^{\prime}: t}^k(ds_\ell|ds_0,a_0)\Bigg|\nonumber\\
    & \qquad +  \int_{s_\ell}\sum_{a_\ell}\mathbb{U}_{t-\tau^{\prime}: t}^k(ds_\ell|ds_0,a_0)\abth{\pi_{\theta_{t}^k}(a_\ell|ds_\ell) - \pi_{\theta_{t-\tau^{\prime}}^k}(a_\ell|ds_\ell)}\nonumber \\
    & \leq \frac{2-\rho}{1-\rho}L_\pi|\calA|\|\theta_t^k-\theta_{t-\tau^{\prime}}^k\|.
\end{align}
Integrating over $s_0$ under $\mathbb{U}_{t-\tau^{\prime}:t}^k(ds_0)$ and summing over $a_0$ under $\pi_{\theta_{t-\tau^{\prime}}^k}(a_0|ds_0)$, which sums to 1, we get
\begin{align*}
    \text{Piece B} \leq \frac{2-\rho}{1-\rho}L_\pi|\calA|\bbE_{t-\tau^{\prime}}[\|\theta_t^k-\theta_{t-\tau^{\prime}}^k\|].
\end{align*}

Putting \eqref{eqn: first term of the 5tuple decomp}, \eqref{eqn: piece A of 5tuple}, and \eqref{eqn: piece B of 5tuple} back to \eqref{eqn: decomp of tv dist of s0, a0, sl, al, sl+1}, we get
\begin{align}
\label{eqn: 5tuple TV dist final bound}
    &2d_{TV}(\mathbb{P}_{t-\tau^{\prime}:(t,\ell)}^k(s_{t,0}^k, a_{t,0}^k, s_{t,\ell}^k, a_{t,\ell}^k, s_{t,\ell+1}^k), \mathbb{U}_{t-\tau^{\prime}:(t,\ell)}^k(\tilde s_{t,0}^k, \tilde a_{t,0}^k, \tilde s_{t,\ell}^k, \tilde a_{t,\ell}^k, \tilde s_{t,\ell+1}^k))\nonumber\\
    \leq& 2d_{TV}(\mathbb{P}_{t-\tau^{\prime}:t}^k(\cdot), \mathbb{U}_{t-\tau^{\prime}:t}^k(\cdot)) + \frac{3-2\rho}{1-\rho}L_\pi|\calA|\bbE_{t-\tau^{\prime}}[\|\theta_t^k-\theta_{t-\tau^{\prime}}^k\|].
\end{align}

Next, we bound $2d_{TV}(\bbP_{t-\tau^{\prime}:(t,\ell+1)}^k(\cdot),\bbU_{t-\tau^{\prime}:(t,\ell+1)}^k(\cdot))$. 
\begin{align*}
&2d_{TV}(\bbP_{t-\tau^{\prime}:(t,\ell+1)}^k(\cdot),\bbU_{t-\tau^{\prime}:(t,\ell+1)}^k(\cdot)) 
= \int_s \abth{\bbP_{t-\tau^{\prime}:(t,\ell+1)}^k(ds) - \bbU_{t-\tau^{\prime}:(t,\ell+1)}^k(ds)} \\
 =& \bbE_{\theta_t^k, t-\tau^{\prime}}\int_s \abth{\int_{s'}\sum_{a}\bbP_{t-\tau^{\prime}:(t,\ell)}^k(ds,a,ds') - \bbU_{t-\tau^{\prime}:(t,\ell)}^k(ds,a,ds')}\\
 =& \bbE_{\theta_t^k, t-\tau^{\prime}}\int_s \left| \int_{s'}\sum_{a}\bbP_{t-\tau^{\prime}:(t,\ell)}^k(ds)\pi_{\theta_{t}^k}(a|ds) P^k(ds'|ds,a)\right. \\ 
 &\left.\qquad\qquad\qquad\qquad\qquad- \bbU_{t-\tau^{\prime}:(t,\ell)}^k(ds) \pi_{\theta_{t-\tau^{\prime}}^k}(a|ds) P^k(ds'|ds,a)\right|\\
\leq& \bbE_{\theta_t^k, t-\tau^{\prime}}\int_s \Bigg|\int_{s'}\sum_{a}\bbP_{t-\tau^{\prime}:(t,\ell)}^k(ds)\pi_{\theta_{t}^k}(a|ds) P^k(ds'|ds,a) \\
&\qquad\qquad\qquad\qquad\qquad- \bbP_{t-\tau^{\prime}:(t,\ell)}^k(ds) \pi_{\theta_{t-\tau^{\prime}}^k}(a|ds) P^k(ds'|ds,a)\Bigg| \\
&+ \bbE_{\theta_t^k, t-\tau^{\prime}} \int_s \Bigg|\int_{s'}\sum_{a}\bbP_{t-\tau^{\prime}:(t,\ell)}^k(ds) \pi_{\theta_{t-\tau^{\prime}}^k}(a|ds) P^k(ds'|ds,a)\\
&\qquad\qquad\qquad\qquad\qquad- \bbU_{t-\tau^{\prime}:(t,\ell)}^k(ds) \pi_{\theta_{t-\tau^{\prime}}^k}(a|ds) P^k(ds'|ds,a)\Bigg|\\
\leq&\bbE_{\theta_t^k, t-\tau^{\prime}} L_\pi |\calA| \bbE_{t-\tau^{\prime}}[\|\theta_{t}^k-\theta_{t-\tau^{\prime}}^k\|] \\
    &+ \bbE_{\theta_t^k, t-\tau^{\prime}}\int_s\int_{s'}\sum_{a} \Bigg|\bbP_{t-\tau^{\prime}:(t,\ell)}^k(ds) \pi_{\theta_{t-\tau^{\prime}}^k}(a|ds) P^k(ds'|ds,a)\\
    &\qquad\qquad\qquad\qquad\qquad- \bbU_{t-\tau^{\prime}:(t,\ell)}^k(ds) \pi_{\theta_{t-\tau^{\prime}}^k}(a|ds) P^k(ds'|ds,a)\Bigg|\\
    =& \bbE_{\theta_t^k, t-\tau^{\prime}}L_\pi |\calA| \bbE_{t-\tau^{\prime}}[\|\theta_{t}^k-\theta_{t-\tau^{\prime}}^k\|] +2d_{TV}(\bbP_{t-\tau^{\prime}:(t,\ell)}^k\otimes \pi_{\theta_{t-\tau^{\prime}}^k}\otimes P^k, \bbU_{t-\tau^{\prime}:(t,\ell)}^k\otimes \pi_{\theta_{t-\tau^{\prime}}^k}\otimes P^k ) \\
    \leq& \bbE_{\theta_t^k, t-\tau^{\prime}} L_\pi |\calA| \bbE_{t-\tau^{\prime}}[\|\theta_{t}^k-\theta_{t-\tau^{\prime}}^k\|] +2\rho d_{TV}(\bbP_{t-\tau^{\prime}:(t,\ell)}^k(\cdot), \bbU_{t-\tau^{\prime}:(t,\ell)}^k(\cdot) ) \\
    &~~~~ (\text{by Dobrushin condition})
\end{align*}

Unrolling the above to $(t,0)$, we get,
\begin{align*}
     &d_{TV}(\bbP_{t-\tau^{\prime}:(t,\ell+1)}^k(\cdot),\bbU_{t-\tau^{\prime}:(t,\ell+1)}^k(\cdot))\\
     \leq& \rho d_{TV}(\bbP_{t-\tau^{\prime}:(t,\ell)}^k(\cdot), \bbU_{t-\tau^{\prime}:(t,\ell)}^k(\cdot) )+\frac{1}{2}L_\pi |\calA| \bbE_{t-\tau^{\prime}}[\|\theta_{t}^k-\theta_{t-\tau^{\prime}}^k\|]\\
     \leq& \rho^{\ell+1} d_{TV}(\bbP_{t-\tau^{\prime}:(t,0)}^k(\cdot), \bbU_{t-\tau^{\prime}:(t,0)}^k(\cdot)) + \frac{1}{1-\rho}\frac{1}{2}L_\pi |\calA| \bbE_{t-\tau^{\prime}}[\|\theta_{t}^k-\theta_{t-\tau^{\prime}}^k\|].
\end{align*}

\end{proof}

\begin{corollary}
\label{cor: tv dist between oc and ac: actor}
Fix $\tau^{\prime}$.  
For any $t$ and $\ell$ such that $t\ge \tau^{\prime}$ and $\ell \in \{0, \dots, L-1\}$, we have 
\begin{align*}
    &2d_{TV}(\hat{\mathbb{P}}_{t-\tau^{\prime}: (t,\ell)}^k(\hat{s}_{t,0}^k, \hat{s}_{t,\ell}^k, \hat{a}_{t,\ell}^k, \hat{s}_{t,\ell+1}^k), \bar{\mathbb{U}}_{t-\tau^{\prime}: (t,\ell)}^k(\bar s_{t,0}^k, \bar s_{t,\ell}^k, \bar a_{t,\ell}^k, \bar s_{t,\ell+1}^k))\\
    &\qquad \qquad \qquad \qquad \leq 2 d_{TV}(\hat{\mathbb{P}}_{t-\tau^{\prime}: t}^k(\cdot), \bar{\mathbb{U}}_{t-\tau^{\prime}: t}^k(\cdot)) + \frac{2-\rho}{1-\rho} L_\pi|\calA|\bbE_{t-\tau^{\prime}}[\|\theta_t^k-\theta_{t-\tau^{\prime}}^k\|].\\
    &2d_{TV}(\hat{\mathbb{P}}_{t-\tau^{\prime}:(t,\ell)}^k(\hat{s}_{t,0}^k, \hat{a}_{t,0}^k, \hat{s}_{t,\ell}^k, \hat{a}_{t,\ell}^k, \hat{s}_{t,\ell+1}^k), \bar{\mathbb{U}}_{t-\tau^{\prime}:(t,\ell)}^k(\bar s_{t,0}^k, \bar a_{t,0}^k, \bar s_{t,\ell}^k, \bar a_{t,\ell}^k, \bar s_{t,\ell+1}^k))\nonumber\\
    &\qquad  \qquad \qquad \qquad\leq 2d_{TV}(\hat{\mathbb{P}}_{t-\tau^{\prime}:t}^k(\cdot), \bar{\mathbb{U}}_{t-\tau^{\prime}:t}^k(\cdot)) + \frac{3-2\rho}{1-\rho}L_\pi|\calA|\bbE_{t-\tau^{\prime}}[\|\theta_t^k-\theta_{t-\tau^{\prime}}^k\|].\\ 
    &d_{TV}(\hat{\bbP}_{t-\tau^{\prime}:(t,\ell+1)}^k(\cdot),\bar{\bbU}_{t-\tau^{\prime}:(t,\ell+1)}^k(\cdot))
    \leq \rho d_{TV}(\hat{\bbP}_{t-\tau^{\prime}:(t,\ell)}^k(\cdot), \bar{\bbU}_{t-\tau^{\prime}:(t,\ell)}^k(\cdot)) \\
    &  \qquad  \qquad \qquad \qquad \qquad \qquad \qquad \qquad \qquad \qquad \qquad \qquad  + \frac{1}{2}L_\pi |\calA| \bbE_{t-\tau^{\prime}}[\|\theta_{t}^k-\theta_{t-\tau^{\prime}}^k\|]. \\
    &d_{TV}(\hat{\bbP}_{t-\tau^{\prime}:(t,\ell+1)}^k(\cdot),\bar{\bbU}_{t-\tau^{\prime}:(t,\ell+1)}^k(\cdot))\leq \rho^{\ell+1} d_{TV}(\hat{\bbP}_{t-\tau^{\prime}:(t,0)}^k(\cdot),\bar{\bbU}_{t-\tau^{\prime}:(t,0)}^k(\cdot))\\
    & \qquad  \qquad \qquad \qquad \qquad \qquad \qquad \qquad \qquad \qquad + \frac{1}{2}\frac{1}{1-\rho}L_\pi |\calA| \bbE_{t-\tau^{\prime}}[\|\theta_{t}^k-\theta_{t-\tau^{\prime}}^k\|]. 
    \end{align*}  
\end{corollary}

\subsection{Proofs of Preliminary Lemmas}
\label{app: proofs of preliminary lemmas}

\textbf{Proof of Lemma \ref{lmm: L* lip}}
\begin{proof}

The TD fixed point $z_\theta^{k,*}$ satisfies

\[
A_{L,\theta}^k z_\theta^{k,*}+ \bar{b}_{L,\theta}^k=0,
\]
where
\begin{align*}
    &A_{L,\theta}^k = \mathbb{E}_{\theta}[\phi(s^{(0)})(\gamma^L\phi(s^{(L)})-\phi(s^{(0)}))^\top]\\
    &\bar{b}_{L,\theta}^k=  \bbE_{\theta}\qth{\sum_{\ell=0}^{L-1}\gamma^\ell R(s^{(\ell)},a^{(\ell)})\phi(s^{(0)})}.
\end{align*}
The expectations are taken over ${s^{(0)}\sim\mu_\theta^k, (a^{(\ell)}, s^{(\ell+1)}) \sim \pi_\theta \otimes P^k ~ \forall \ell\in [0, L-1]}$. Since $\mathbf{B}^*$ is a linear transformation on the feature representation, the fixed point in the projected coordinates satisfies
\[
A^k_{L, \mathbf{B}^*,i}{\omega}^{k,*}({\theta_i},\mathbf{B}^*) + b^k_{L, \mathbf{B}^*,i} =0,
\]
where $A^k_{L,\mathbf{B}^*,i} = \bbE_{\theta_i}[(\mathbf{B}^*)^\top\phi(s^{(0)})(\gamma^L(\mathbf{B}^*)^\top\phi(s^{(L)})-(\mathbf{B}^*)^\top\phi(s^{(0)}))^\top] = (\mathbf{B}^*)^\top A_{L,i}^k  \mathbf{B}^*$, and $b^k_{L, \mathbf{B}^*,i} = \bbE_{\theta_i}[\sum_{\ell=0}^{L-1}\gamma^\ell R(s^{(\ell)},a^{(\ell)})(\mathbf{B}^*)^\top\phi(s^{(0)})] = (\mathbf{B}^*)^\top \bar b_{L,i}^k$.

For any $\theta_1, \theta_2$, we have:
\begin{align*}
    &{\omega}^{k,*}({\theta_1},\mathbf{B}^*) =-(A_{L,\mathbf{B}^*,1}^k)^{-1}b_{L,\mathbf{B}^*,1}^k = -((\mathbf{B}^*)^\top A_{L,1}^k\mathbf{B}^*)^{-1} (\mathbf{B}^*)^\top{\bar b_{L,1}^k};\\
    &{\omega}^{k,*}({\theta_2},\mathbf{B}^*)=-(A_{L,\mathbf{B}^*,2}^k)^{-1}b_{L,\mathbf{B}^*,2}^k = -((\mathbf{B}^*)^\top A_{L,2}^k\mathbf{B}^*)^{-1}(\mathbf{B}^*)^\top{\bar b_{L,2}^k}.
\end{align*}

\begin{align}
    &\|{\omega}^{k,*}({\theta_1},\mathbf{B}^*) - {\omega}^{k,*}({\theta_2},\mathbf{B}^*)\| = \|-(A_{L,\mathbf{B}^*,1}^k)^{-1}b_{L,\mathbf{B}^*,1}^k - (-(A_{L,\mathbf{B}^*,2}^k)^{-1}b_{L,\mathbf{B}^*,2}^k) \|\nonumber\\
    =&\|(A_{L,\mathbf{B}^*,2}^k)^{-1}b_{L,\mathbf{B}^*,2}^k -(A_{L,\mathbf{B}^*,2}^k)^{-1}b_{L,\mathbf{B}^*,1}^k+(A_{L,\mathbf{B}^*,2}^k)^{-1}b_{L,\mathbf{B}^*,1}^k- (A_{L,\mathbf{B}^*,1}^k)^{-1}b_{L,\mathbf{B}^*,1}^k\|\nonumber\\
    \leq& \|(A_{L,\mathbf{B}^*,2}^k)^{-1}b_{L,\mathbf{B}^*,2}^k -(A_{L,\mathbf{B}^*,2}^k)^{-1}b_{L,\mathbf{B}^*,1}^k\| + \|(A_{L,\mathbf{B}^*,2}^k)^{-1}b_{L,\mathbf{B}^*,1}^k- (A_{L,\mathbf{B}^*,1}^k)^{-1}b_{L,\mathbf{B}^*,1}^k\|\nonumber\\
    \leq& \|(A_{L,\mathbf{B}^*,2}^k)^{-1}\|\|b_{L,\mathbf{B}^*,2}^k -b_{L,\mathbf{B}^*,1}^k\| + \|(A_{L,\mathbf{B}^*,2}^k)^{-1}- (A_{L,\mathbf{B}^*,1}^k)^{-1}\|\|b_{L,\mathbf{B}^*,1}^k\|\nonumber\\
    =& \|((\mathbf{B}^*)^\top A_{L,2}^k\mathbf{B}^*)^{-1}\|\|(\mathbf{B}^*)^\top \bar b_{L,2}^k -(\mathbf{B}^*)^\top \bar b_{L,1}^k\|\\& + \|((\mathbf{B}^*)^\top A_{L,2}^k\mathbf{B}^*)^{-1}- ((\mathbf{B}^*)^\top A_{L,1}^k\mathbf{B}^*)^{-1}\|\|(\mathbf{B}^*)^\top \bar b_{L,1}^k\|\nonumber\\
    \leq& \|((\mathbf{B}^*)^\top A_{L,2}^k\mathbf{B}^*)^{-1}\| \lnorm{(\mathbf{B}^*)^\top }{}\|\bar  b_{L,2}^k -\bar b_{L,1}^k\|\nonumber\\& + \|((\mathbf{B}^*)^\top A_{L,2}^k\mathbf{B}^*)^{-1}\| \| ((\mathbf{B}^*)^\top A_{L,1}^k\mathbf{B}^*)-((\mathbf{B}^*)^\top A_{L,2}^k\mathbf{B}^*)\| \|((\mathbf{B}^*)^\top A_{L,1}^k\mathbf{B}^*)^{-1}\| \|(\mathbf{B}^*)^\top \bar b_{L,1}^k\|\nonumber\\
    \leq& \|((\mathbf{B}^*)^\top A_{L,2}^k\mathbf{B}^*)^{-1}\| \lnorm{(\mathbf{B}^*)^\top }{}\|\bar b_{L,2}^k -\bar b_{L,1}^k\| \nonumber\\&+ \|((\mathbf{B}^*)^\top A_{L,2}^k\mathbf{B}^*)^{-1}\| \| ((\mathbf{B}^*)^\top (A_{L,1}^k-A_{L,2}^k)\mathbf{B}^*)\| \|((\mathbf{B}^*)^\top A_{L,1}^k\mathbf{B}^*)^{-1}\| \|(\mathbf{B}^*)^\top \bar b_{L,1}^k\|\label{eq: lmm b.3 temp}.
\end{align}
By Assumption \ref{ass: per-agent full exploration}, the $A_{L,i}^k$'s are negative definite and the largest eigenvalue is upper bounded by $-\lambda L$. Moreover, by Poincar\'e Separation Theorem, since $\mathbf{B}^*$ has orthonormal columns, then the eigenvalues of $(\mathbf{B}^*)^\top A_{L,i}^k\mathbf{B}^*$ are upper bounded by the eigenvalues of $A_{L,i}^k$, meaning that the largest eigenvalue of $(\mathbf{B}^*)^\top A_{L,i}^k\mathbf{B}^*$ is upper bounded by $-\lambda L$. Therefore, \begin{align}
\|((\mathbf{B}^*)^\top A_{L,i}^k\mathbf{B}^*)^{-1}\| = \max_i |\lambda_i(((\mathbf{B}^*)^\top A_{L,i}^k\mathbf{B}^*)^{-1})|=\max_i |\frac{1}{\lambda_i((\mathbf{B}^*)^\top A_{L,i}^k\mathbf{B}^*)}|\leq (\lambda L)^{-1}. \label{eqn: boundedness of BAB}\end{align}
It can be easily shown that $\|\bar b_{L,i}^k\|\leq \frac{U_r}{1-\gamma}$.
Then, \eqref{eq: lmm b.3 temp} can be bounded by:\begin{align*}
    \frac{1}{\lambda L}\|\bar b_{L,2}^k -\bar b_{L,1}^k\|+\frac{U_r}{1-\gamma}(\lambda L)^{-2}\|A_{L,1}^k-A_{L,2}^k\|.
\end{align*}
We first bound $\|A_{L,1}^k-A_{L,2}^k\|$. 
For $i\in\{1,2\}$ and $0\le \ell\le L-1$, let $\mathsf{P}_{i,\ell}^k$ denote the trajectory 
$(s^{(0)},a^{(0)},s^{(1)},\ldots,a^{(\ell)},s^{(\ell+1)})$ generated from 
$s^{(0)}\sim \mu_{\theta_i}^k$ and $(a^{(j)},s^{(j+1)})\sim \pi_{\theta_i}\otimes P^k$ for $0\le j\le \ell$.

We have 
\begin{align*}
&\|A_{L,\theta_1}^k-A_{L,\theta_2}^k\|\\
=&
\Bigg\|
\sum_{\ell=0}^{L-1}\gamma^\ell
\Big(
\bbE_{\mathsf{P}_{1,\ell}^k}
\qth{
\phi(s^{(0)})
\big(\gamma\phi(s^{(\ell+1)})-\phi(s^{(\ell)})\big)^\top
}\\
&\qquad\qquad\qquad\qquad
-
\bbE_{\mathsf{P}_{2,\ell}^k}
\qth{
\phi(s^{(0)})
\big(\gamma\phi(s^{(\ell+1)})-\phi(s^{(\ell)})\big)^\top
}
\Big)
\Bigg\|\\
\leq&
\sum_{\ell=0}^{L-1}\gamma^\ell
\Bigg\|
\bbE_{\mathsf{P}_{1,\ell}^k}
\qth{
\phi(s^{(0)})
\big(\gamma\phi(s^{(\ell+1)})-\phi(s^{(\ell)})\big)^\top
}
-
\bbE_{\mathsf{P}_{2,\ell}^k}
\qth{
\phi(s^{(0)})
\big(\gamma\phi(s^{(\ell+1)})-\phi(s^{(\ell)})\big)^\top
}
\Bigg\|\\
\leq&
4\sum_{\ell=0}^{L-1}\gamma^\ell
d_{TV}\big(\mathsf{P}_{1,\ell}^k,\mathsf{P}_{2,\ell}^k\big),
\end{align*}
where the last inequality follows from the definition of total variation distance and
\[
\sup_{s^{(0)},s^{(\ell)},s^{(\ell+1)}}
\left\|
\phi(s^{(0)})
\big(\gamma\phi(s^{(\ell+1)})-\phi(s^{(\ell)})\big)^\top
\right\|
\leq 1+\gamma \leq 2.
\]
Moreover, since the transition kernel $P^k$ is shared under $\theta_1$ and $\theta_2$, the path-measure difference only comes from the initial stationary distributions and the policies along the trajectory. Thus,
\begin{align*}
d_{TV}\big(\mathsf{P}_{1,\ell}^k,\mathsf{P}_{2,\ell}^k\big)
&\leq
d_{TV}\big(\mu_{\theta_1}^k,\mu_{\theta_2}^k\big)
+
(\ell+1)
\sup_{s\in\calS}
d_{TV}\big(\pi_{\theta_1}(\cdot|s),\pi_{\theta_2}(\cdot|s)\big)\\
&\leq
C_\mu\|\theta_1-\theta_2\|
+
\frac{\ell+1}{2}L_\pi|\calA|\|\theta_1-\theta_2\|,
\end{align*}
where the last inequality follows from Lemma \ref{lmm: st-st perturb} and Assumption \ref{assmp: Lipschitz of policy}. Plugging this into the previous display gives
\begin{align*}
\|A_{L,\theta_1}^k-A_{L,\theta_2}^k\|
&\leq
4\sum_{\ell=0}^{L-1}\gamma^\ell
\left(
C_\mu+\frac{\ell+1}{2}L_\pi|\calA|
\right)
\|\theta_1-\theta_2\|\\
&\leq
\left(
\frac{4C_\mu}{1-\gamma}
+
\frac{2L_\pi|\calA|}{(1-\gamma)^2}
\right)
\|\theta_1-\theta_2\|.
\end{align*}

We next bound $\|\bar b_{L,1}^k-\bar b_{L,2}^k\|$. Using the same partial trajectory laws $\mathsf{P}_{i,\ell}^k$ as above, we have
\begin{align*}
\|\bar b_{L,1}^k-\bar b_{L,2}^k\|
&=
\left\|
\sum_{\ell=0}^{L-1}\gamma^\ell
\left(
\bbE_{\mathsf{P}_{1,\ell}^k}
\left[R(s^{(\ell)},a^{(\ell)})\phi(s^{(0)})\right]
-
\bbE_{\mathsf{P}_{2,\ell}^k}
\left[R(s^{(\ell)},a^{(\ell)})\phi(s^{(0)})\right]
\right)
\right\|\\
&\le
2U_r\sum_{\ell=0}^{L-1}\gamma^\ell
d_{TV}\big(\mathsf{P}_{1,\ell}^k,\mathsf{P}_{2,\ell}^k\big)\\
&\le
2U_r\sum_{\ell=0}^{L-1}\gamma^\ell
\left(
C_\mu+\frac{\ell+1}{2}L_\pi|\calA|
\right)
\|\theta_1-\theta_2\|\\
&\le
\left(
\frac{2U_r C_\mu}{1-\gamma}
+
\frac{U_r L_\pi|\calA|}{(1-\gamma)^2}
\right)
\|\theta_1-\theta_2\|.
\end{align*}

Finally, 

\begin{align*}
    &\|{\omega}^{k,*}({\theta_1},\mathbf{B}^*) - {\omega}^{k,*}({\theta_2},\mathbf{B}^*)\|\\
    \leq& \frac{1}{\lambda L}\|\bar b_{L,2}^k -\bar b_{L,1}^k\|+\frac{U_r}{1-\gamma}(\lambda L)^{-2}\|A_{L,1}^k-A_{L,2}^k\|\\
    \leq& \frac{1}{\lambda L}
\left(
\frac{2U_r C_\mu}{1-\gamma}
+
\frac{U_r L_\pi|\calA|}{(1-\gamma)^2}
\right)\| {\theta_1} - {\theta_2} \|
+
\frac{U_r}{1-\gamma}(\lambda L)^{-2}
\left(
\frac{4C_\mu}{1-\gamma}
+
\frac{2L_\pi|\calA|}{(1-\gamma)^2}
\right)
 \| {\theta_1} - {\theta_2} \|    \\
    =& L_{*,1}\|{\theta_1}-{\theta_2}\|,
\end{align*}
where
\[
L_{*,1}
=
\frac{1}{\lambda L}
\left(
\frac{2U_r C_\mu}{1-\gamma}
+
\frac{U_r L_\pi|\calA|}{(1-\gamma)^2}
\right)
+
\frac{U_r}{1-\gamma}(\lambda L)^{-2}
\left(
\frac{4C_\mu}{1-\gamma}
+
\frac{2L_\pi|\calA|}{(1-\gamma)^2}
\right).
\]

\end{proof}
\textbf{Proof of Lemma \ref{lmm: Ls lip}}
\begin{proof}
Recall
\begin{align*}
    &A_{L,\theta}^k = \mathbb{E}_{\theta}[\phi(s^{(0)})(\gamma^L\phi(s^{(L)})-\phi(s^{(0)}))^\top]\\
    &\bar{b}_{L,\theta}^k=  \bbE_{\theta}\qth{\sum_{\ell=0}^{L-1}\gamma^\ell R(s^{(\ell)},a^{(\ell)})\phi(s^{(0)})}.
\end{align*}
For ease of notation, write $A_{L,i}^k:=A_{L,\theta_i}^k$ and $\bar b_{L,i}^k:=\bar b_{L,\theta_i}^k$ for $i\in\{1,2\}$. For any orthonormal $\mathbf{B}$,
\begin{align*}
    &\nabla_\theta {\omega} ^{k,*}({\theta}_i,\mathbf{B})
    = -\nabla_\theta\qth{(\mathbf{B}^\top A_{L,i}^k \mathbf{B})^{-1}\mathbf{B}^\top \bar b_{L,i}^k}\\
    =&-(\nabla_\theta(\mathbf{B}^\top A_{L,i}^k \mathbf{B})^{-1})\mathbf{B}^\top \bar b_{L,i}^k
    -(\mathbf{B}^\top A_{L,i}^k \mathbf{B})^{-1} \mathbf{B}^\top\nabla_\theta \bar b_{L,i}^k\\
    =&(\mathbf{B}^\top A_{L,i}^k \mathbf{B})^{-1}
    \nabla_\theta(\mathbf{B}^\top A_{L,i}^k \mathbf{B})
    (\mathbf{B}^\top A_{L,i}^k \mathbf{B})^{-1}\mathbf{B}^\top \bar b_{L,i}^k
    -(\mathbf{B}^\top A_{L,i}^k \mathbf{B})^{-1}\mathbf{B}^\top\nabla_\theta \bar b_{L,i}^k\\
    =&(\mathbf{B}^\top A_{L,i}^k \mathbf{B})^{-1}
    \mathbf{B}^\top(\nabla_\theta A_{L,i}^k)\mathbf{B}
    (\mathbf{B}^\top A_{L,i}^k \mathbf{B})^{-1}\mathbf{B}^\top \bar b_{L,i}^k
    -(\mathbf{B}^\top A_{L,i}^k \mathbf{B})^{-1}\mathbf{B}^\top\nabla_\theta \bar b_{L,i}^k.
\end{align*}
The constant terms $\mathbf{B}$ and $\mathbf{B}^\top$ are bounded. From \eqref{eqn: boundedness of BAB}, we know
\[
\|(\mathbf{B}^\top A_{L,i}^k \mathbf{B})^{-1}\| \leq (\lambda L)^{-1}.
\]
Moreover, the inverse map is Lipschitz:
\begin{align*}
    &\|(\mathbf{B}^\top A_{L,1}^k \mathbf{B})^{-1}
    -(\mathbf{B}^\top A_{L,2}^k \mathbf{B})^{-1}\| \\
    =&\|(\mathbf{B}^\top A_{L,1}^k \mathbf{B})^{-1}
    \mathbf{B}^\top(A_{L,2}^k-A_{L,1}^k)\mathbf{B}
    (\mathbf{B}^\top A_{L,2}^k \mathbf{B})^{-1}\|\\
    \leq&(\lambda L)^{-2}\|A_{L,2}^k-A_{L,1}^k\|\\
    \leq&(\lambda L)^{-2}
    \left(
    \frac{4C_\mu}{1-\gamma}
    +
    \frac{2L_\pi|\calA|}{(1-\gamma)^2}
    \right)
    \|\theta_1-\theta_2\|,
\end{align*}
where the last inequality follows from the proof of Lemma \ref{lmm: L* lip}. The same proof also gives
\begin{align*}
\|\bar b_{L,1}^k-\bar b_{L,2}^k\|
&\leq
\left(
\frac{2U_r C_\mu}{1-\gamma}
+
\frac{U_r L_\pi|\calA|}{(1-\gamma)^2}
\right)
\|\theta_1-\theta_2\|,
\end{align*}
and $\|\bar b_{L,i}^k\|\leq U_r/(1-\gamma)$.

It remains to control the derivative terms. Since
\begin{align*}
    \nabla_\theta A_{L,i}^k
    &=
    \nabla_\theta
    \bbE_{\theta_i}
    \qth{
    \phi(s^{(0)})(\gamma^L\phi(s^{(L)})-\phi(s^{(0)}))^\top
    },\\
    \nabla_\theta \bar b_{L,i}^k
    &=
    \nabla_\theta
    \bbE_{\theta_i}
    \qth{
    \sum_{\ell=0}^{L-1}\gamma^\ell
    R(s^{(\ell)},a^{(\ell)})\phi(s^{(0)})
    },
\end{align*}
where the expectations are taken over
$s^{(0)}\sim\mu_{\theta_i}^k$ and
$(a^{(\ell)},s^{(\ell+1)})\sim\pi_{\theta_i}\otimes P^k$ for all $\ell\in[0,L-1]$.
Assumption \ref{assmp: Lipschitz of policy} and the additional smoothness condition in this lemma imply that
$\nabla_\theta A_{L,i}^k$ and $\nabla_\theta \bar b_{L,i}^k$ are both bounded and Lipschitz in $\theta$.
Combining these bounds with the displayed formula for $\nabla_\theta\omega^{k,*}$ and the Lipschitz bound on
$(\mathbf{B}^\top A_{L,i}^k\mathbf{B})^{-1}$ shows that
$\nabla_\theta\omega^{k,*}(\theta,\mathbf{B})$ is Lipschitz in $\theta$.
Specializing to $\mathbf{B}=\mathbf{B}^*$ gives the claimed constant $L_{s,1}$.

\end{proof}
\textbf{Proof of Lemma \ref{lmm: st-st perturb}}
\begin{proof}
We prove each bullet point.
\begin{itemize}
    \item 
Since here $\mu^k_{{\theta_1}},\mu^k_{{\theta_2}}$ are stationary distributions induced by the Markov chain, $n=\infty$.
By Theorem 3.1 of \cite{Mitrophanov2005}, we have \[d_{TV}(\mu^k_{{\theta_1}},\mu^k_{{\theta_2}})\leq (\hat n+\frac{1}{1-\rho})\|\calK_1-\calK_2\|,\] where $\calK_i(s,ds') = \sum_{a} P^k(ds'|s,a)\pi_{{\theta_i}}(a|s), \hat n = \lceil{\log_\rho m^{-1}}\rceil.$

\begin{align*}
    &\|\calK_1-\calK_2\| = \sup_{s\in\calS} \|\calK_1(s,\cdot)-\calK_2(s,\cdot)\|\\
    =&\sup_{s\in\calS} \int_{s'\in\calS}\abth{\sum_{a} P^k(ds'|s,a)\pi_{{\theta_1}}(a|s)-\sum_{a} P^k(ds'|s,a)\pi_{{\theta_2}}(a|s)}\\
    =&\sup_{s\in\calS} \int_{s'\in\calS}\abth{\sum_{a} P^k(ds'|s,a)\pth{\pi_{{\theta_1}}(a|s)-\pi_{{\theta_2}}(a|s)}}\\
    \leq&\sup_{s\in\calS} \int_{s'\in\calS}\sum_{a} P^k(ds'|s,a)\abth{\pi_{{\theta_1}}(a|s)-\pi_{{\theta_2}}(a|s)}\\
    \leq& L_\pi |\calA| \|\theta_1-\theta_2\|.
\end{align*}

\item 
\begin{align*}
 &d_{TV}(\mu_{{\theta_1}}^k\otimes\pi_{{\theta_1}},\mu_{{\theta_2}}^k\otimes\pi_{{\theta_2}}) \leq d_{TV}(\mu_{{\theta_1}}^k,\mu_{{\theta_2}}^k)+d_{TV}(\pi_{{\theta_1}},\pi_{{\theta_2}})\\
& =|\calA|L_\pi(\lceil{\log_\rho m^{-1}}\rceil+\frac{1}{1-\rho})\|{\theta_2}-{\theta_1}\| + \frac{1}{2}\sum_{a} |\pi_{{\theta_1}}(a|s)-\pi_{{\theta_2}}(a|s)|\\
& =|\calA|L_\pi(\lceil{\log_\rho m^{-1}}\rceil+\frac{1}{1-\rho})\|{\theta_2}-{\theta_1}\| + \frac{1}{2}|\calA| L_\pi \|{\theta_2}-{\theta_1}\|.
\end{align*}
\item
\begin{align*}
    &d_{TV}(\mu_{{\theta_1}}^k\otimes\pi_{{\theta_1}}\otimes P^k,\mu_{{\theta_2}}^k\otimes\pi_{{\theta_2}}\otimes P^k)\leq d_{TV}(\mu_{{\theta_1}}^k\otimes\pi_{{\theta_1}},\mu_{{\theta_2}}^k\otimes\pi_{{\theta_2}}) + d_{TV}(P^k,P^k)\\
    &= d_{TV}(\mu_{\theta_1}^k\otimes\pi_{{\theta_1}},\mu_{{\theta_2}}^k\otimes\pi_{{\theta_2}}).\\
\end{align*}
\end{itemize}
\end{proof}
\section{Detailed Experimental Setup}
\label{app:setup}

\subsection{Environment and Heterogeneity Construction}
\label{app:env}

We use Gymnasium's \texttt{Hopper-v5} with the MuJoCo backend as the underlying environment. The only environment-side modification is reducing the standard healthy-reward bonus from $1.0$ to $0.2$, so that the reward signal is dominated by forward progress rather than the alive bonus. We use \texttt{max\_episode\_steps}$=1000$ and the default Hopper termination criteria.

For each client, we wrap the environment with an \emph{action map}: before an action is passed to the MuJoCo simulator, its three joint-torque coordinates are permuted and then multiplied by a client-specific positive scale. Thus, each client's policy emits actions in its own permuted and scaled coordinate system, while the underlying physics receives the mapped action. The observation space and MuJoCo dynamics are shared across clients, but the same policy action can produce different physical torques and hence different next-state distributions. This matches the heterogeneous-transition setting of our theory, where each agent has its own transition kernel $P^k$.

The grouped setup uses two action maps: permutation $[1,0,2]$ with scale $0.7$, and permutation $[2,1,0]$ with scale $1.5$, with three clients assigned to each map. The \textsc{6-unique} setup uses all six permutations of $[0,1,2]$ paired with distinct scales $\{1.0,0.8,1.2,0.7,1.5,0.9\}$, one client per permutation-scale pair. The OOD transfer experiment in Section~\ref{sec:exp-foundation} uses the identity permutation with an out-of-range action scale of $0.5$, which is below the pretraining range.

\subsection{Algorithms}
\label{app:algos}

All algorithms share the same network architecture: a $2$-layer MLP shared \emph{trunk} of width $256$ with ReLU activations, followed by client-side output components, including a linear actor head producing the action mean, a linear critic head, and a diagonal $\log\sigma$ parameter, the log-standard-deviation of the Gaussian policy's exploration noise. The algorithms differ only in which parameters are synchronized across clients:
\begin{itemize}
  \item \textbf{Single PPO}: each client trains its full network independently, with no cross-client communication.
  \item \textbf{FedAvg PPO}: every $10$ episode rounds, all parameters are averaged across clients.
  \item \textbf{FedPer PPO} (ours): every $10$ episode rounds, only the trunk parameters are averaged. The actor heads, critic heads, and $\log\sigma$ parameters remain personalized.
\end{itemize}

\subsection{PPO Hyperparameters}
\label{app:ppo}

We use a Hopper-tuned PPO recipe taken from a public reference implementation, identical across all three algorithms and across all experiments (Table~\ref{tab:ppo-hparams}). We deliberately do not re-tune per algorithm; any reported difference is therefore attributable to the federated synchronization scheme rather than hyperparameter optimization.

\begin{table}[h]
\centering
\small
\begin{tabular}{lr}
\toprule
\textbf{Hyperparameter} & \textbf{Value} \\
\midrule
discount $\gamma$               & $0.999$ \\
return mode                     & discounted \\
return normalization            & enabled (running mean/std) \\
advantage normalization         & disabled \\
full-batch PPO update           & enabled (no minibatching) \\
update epochs                   & $10$ \\
clip coefficient                & $0.2$ \\
entropy coefficient             & $2.30\times10^{-3}$ \\
value coefficient               & $0.5$ \\
gradient clipping               & disabled \\
learning rate                   & $9.81\times10^{-5}$ \\
Adam $\epsilon$                 & $1\times10^{-8}$ \\
hidden sizes                    & $(256, 256)$ \\
activation                      & ReLU \\
action-mean squash              & $\tanh$ \\
$\log\sigma$ initialization     & $-0.5$ \\
action clipping                 & disabled \\
mapped-action clipping          & disabled \\
observation normalization       & disabled \\
\bottomrule
\end{tabular}
\caption{PPO hyperparameters used in all experiments. Identical across Single, FedAvg, and FedPer.}
\label{tab:ppo-hparams}
\end{table}

\subsection{Hardware and Running Time}
\label{app:hardware}

All experiments are run on a MacBook Air with an Apple M2 chip, using the PyTorch CPU backend without GPU acceleration. Wall-clock totals for the experiments reported in this paper are listed in Table~\ref{tab:walltime}.

\begin{table}[h]
\centering
\small
\begin{tabular}{ll}
\toprule
\textbf{Item} & \textbf{Value} \\
\midrule
Hardware & MacBook Air, Apple M2 \\
Backend & PyTorch CPU \\
Heterogeneity ablation & $\approx 5.3$\,h total \\
Grouped setup & $\approx 2.6$\,h \\
\textsc{6-unique} setup & $\approx 2.6$\,h \\
Foundation transfer & $\approx 1.1$\,h total \\
\bottomrule
\end{tabular}
\caption{Hardware and wall-clock running times for the experiments.}
\label{tab:walltime}
\end{table}

The pretrained trunk used in the foundation-transfer ablation comes from the post-final-sync FedPer checkpoint of the \textsc{6-unique} run. Because the trunk is synchronized at episode round $5000$ before the final checkpoint is written, the trunk is identical across all six clients of that run.

\section{Additional Experimental Results}
\label{app:additional-results}

\subsection{Heterogeneity Training Dynamics}
\label{app:hetero-dynamics}

\begin{figure}[h]
    \centering
    \includegraphics[width=\linewidth]{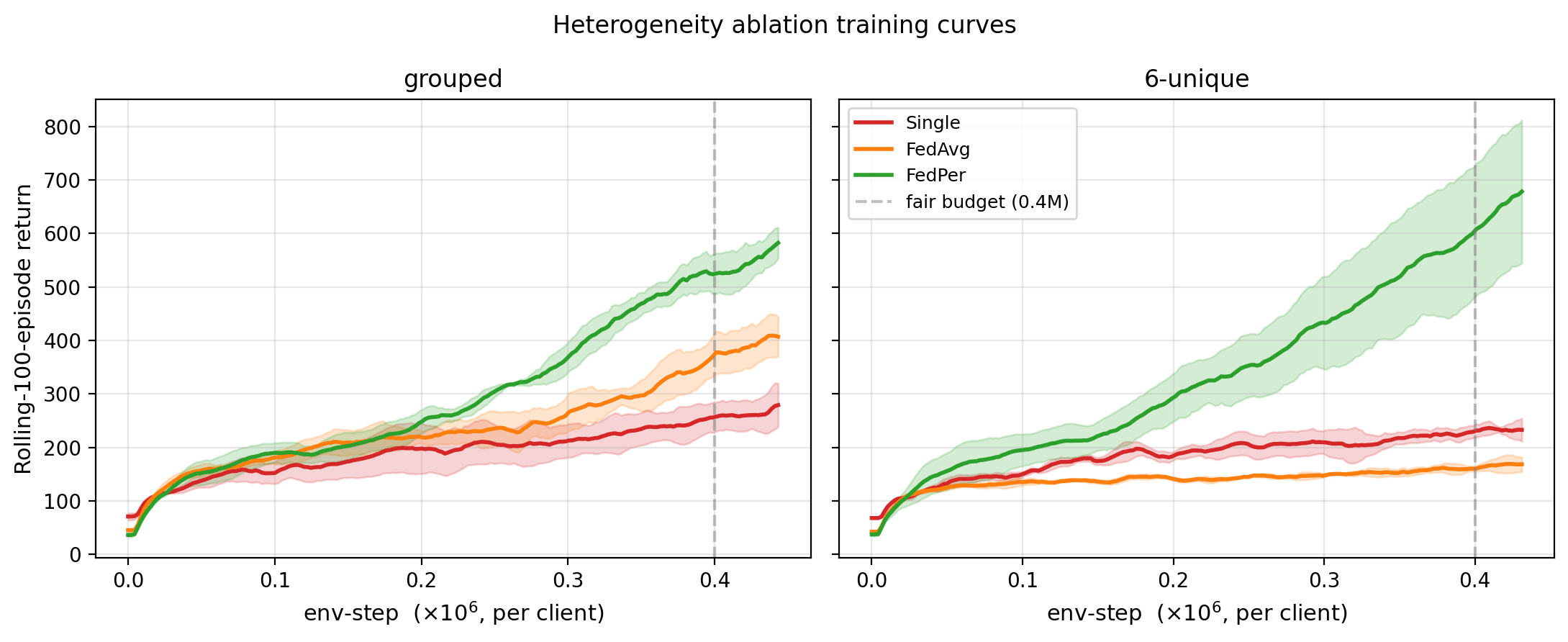}
    \caption{Training curves over per-client environment steps for the heterogeneity ablation in Section~\ref{sec:exp-hetero}. Bands show $\pm1$ standard deviation over $3$ seeds; the dashed vertical line marks the $0.4\,\text{M}$ common-budget reporting point used in Fig.~\ref{fig:hetero-bar}.}
    \label{fig:hetero-dynamics}
\end{figure}

A natural concern about Fig.~\ref{fig:hetero-bar} is that, under a fixed-episode training budget, stronger policies may collect more environment interactions because Hopper episodes terminate early after falls. Fig.~\ref{fig:hetero-dynamics} addresses this by re-plotting the same runs against per-client environment steps and truncating all methods to the smallest shared horizon. The relative ordering is preserved under this env-step-fair view, confirming that FedPer's advantage in Fig.~\ref{fig:hetero-bar} reflects faster learning per environment interaction rather than a longer-episode confound.

\clearpage

\end{document}